\documentclass[11pt]{article}

\usepackage[margin=1in]{geometry}
\usepackage{latexsym}
\usepackage{amsmath}
\usepackage{amssymb}
\usepackage{amsthm}
\usepackage{epsfig}
\usepackage{xcolor}
\usepackage{graphicx}
\usepackage{bm}
\usepackage{enumitem}
\usepackage{mathtools}
\usepackage{graphicx}
\usepackage{hyperref}
\usepackage[T1]{fontenc}
\usepackage[utf8]{inputenc}
\usepackage{verbatim}
\usepackage{float}
\usepackage{tikz}
\usepackage{stmaryrd}
\usepackage{mathrsfs}
\usepackage{cite}

\usepackage{mathtools}
\mathtoolsset{showonlyrefs}

\usepackage{subfigure}
\makeatother

\providecommand{\fref}[1]{Figure~\ref{fig:#1}}

\usepackage{dsfont}
\newcommand{\charfn}{\mathds{1}}

\renewcommand{\P}{\mathbb{P}}
\newcommand{\E}{\mathbb{E}}
\newcommand{\Var}{\text{Var}}

\newcommand{\cN}{\mathcal{N}}

\newcommand{\Z}{\mathbb{Z}}
\newcommand{\R}{\mathbb{R}}
\newcommand{\C}{\mathbb{C}}
\newcommand{\N}{\mathbb{N}}
\renewcommand{\S}{\mathbb{S}}

\newcommand{\eps}{\varepsilon} 

\def\id{{\mathbf I}}

\newcommand{\<}{\langle}
\renewcommand{\>}{\rangle}
\newcommand{\diag}{\text{diag}}

\newcommand{\op}{{\rm op}}

\def\sT{{\mathsf T}}
\def\bzero{{\boldsymbol 0}}

\DeclareMathOperator*{\argmin}{arg\,min}

\def\simiid{{\stackrel{i.i.d.}{\sim}}}

\newtheorem{theorem}{Theorem}
\newtheorem*{theorem*}{Theorem}
\newtheorem{lemma}{Lemma}

\newtheorem{assumption}{Assumption}
\newtheorem{definition}{Definition}
\newtheorem{proposition}{Proposition}

\theoremstyle{definition}

\newtheorem{remark}{Remark}[section]

\usepackage[toc,page]{appendix}
\usepackage{graphicx}
\usepackage{bbm}

\DeclareSymbolFont{rsfs}{U}{rsfs}{m}{n}
\DeclareSymbolFontAlphabet{\mathscrsfs}{rsfs}

\def\bA{{\boldsymbol A}}
\def\bB{{\boldsymbol B}}
\def\bC{{\boldsymbol C}}
\def\bD{{\boldsymbol D}}
\def\bE{{\boldsymbol E}}
\def\bF{{\boldsymbol F}}
\def\bG{{\boldsymbol G}}

\def\bH{{\boldsymbol H}}

\def\bK{{\boldsymbol K}}
\def\bL{{\boldsymbol L}}

\def\bP{{\boldsymbol P}}
\def\bQ{{\boldsymbol Q}}
\def\bR{{\boldsymbol R}}
\def\bS{{\boldsymbol S}}
\def\bT{{\boldsymbol T}}
\def\bU{{\boldsymbol U}}
\def\bV{{\boldsymbol V}}
\def\bW{{\boldsymbol W}}
\def\bX{{\boldsymbol X}}

\def\bZ{{\boldsymbol Z}}

\def\ba{{\boldsymbol a}}

\def\be{{\boldsymbol e}}

\def\bg{{\boldsymbol g}}

\def\bj{{\boldsymbol j}}
\def\bk{{\boldsymbol k}}

\def\bu{{\boldsymbol u}}
\def\bv{{\boldsymbol v}}
\def\bw{{\boldsymbol w}}
\def\bx{{\boldsymbol x}}
\def\by{{\boldsymbol y}}

\def\bf{{\boldsymbol f}}

\def\bmu{{\boldsymbol \mu}}
\def\bbeta{{\boldsymbol \beta}}

\def\beps{{\boldsymbol \eps}}

\def\bpsi{{\boldsymbol \psi}}
\def\bphi{{\boldsymbol \phi}}
\def\btheta{{\boldsymbol \theta}}

\def\bchi{{\boldsymbol \chi}}
\def\bvartheta{{\boldsymbol \vartheta}}

\def\bDelta{{\boldsymbol \Delta}}
\def\bLambda{{\boldsymbol \Lambda}}
\def\bPsi{{\boldsymbol \Psi}}
\def\bPhi{{\boldsymbol \Phi}}
\def\bSigma{{\boldsymbol \Sigma}}
\def\bTheta{{\boldsymbol \Theta}}
\def\bOmega{{\boldsymbol \Omega}}

\def\bPi{{\boldsymbol \Pi}}

\def\hba{{\hat {\boldsymbol a}}}
\def\hf{{\hat f}}

\def\de{{\rm d}}
\def\Trace{{\rm Tr}}

\def\de{{\rm d}}
\def\Unif{{\rm Unif}}

\def\RF{{\rm RF}}

\def\cV{{\mathcal V}}
\def\cG{{\mathcal G}}
\def\cO{{\mathcal O}}

\def\cC{{\mathcal C}}
\def\cQ{{\mathcal Q}}
\def\cL{{\mathcal L}}
\def\cF{{\mathcal F}}
\def\cE{{\mathcal E}}
\def\cS{{\mathcal S}}

\def\cV{{\mathcal V}}
\def\cG{{\mathcal G}}
\def\cO{{\mathcal O}}

\def\cH{{\mathcal H}}
\def\cA{{\mathcal A}}

\def\Unif{{\sf Unif}}
\def\normal{{\sf N}}
\def\proj{{\mathsf P}}

\def\RF{{\sf RF}}

\def\NN{{\sf NN}}
\def\reals{{\mathbb R}}

\def\naturals{{\mathbb N}}

\def\normal{{\sf N}}
\def\proj{{\mathsf P}}

\def\Unif{{\sf Unif}}
\def\normal{{\sf N}}

\def\proj{{\mathsf P}}

\def\RF{{\sf RF}}

\def\NN{{\sf NN}}
\def\reals{{\mathbb R}}

\def\naturals{{\mathbb N}}
\def\proj{{\mathsf P}}

\def\App{{\rm App}}

\def\hba{{\hat {\boldsymbol a}}}
\def\hf{{\hat f}}

\def\cE{{\mathcal E}}

\def\cX{{\mathcal X}}
\def\cF{{\mathcal F}}
\def\cS{{\mathcal S}}

\def\He{{\rm He}}
\def\Trace{{\rm Tr}}

\def\de{{\rm d}}
\def\Unif{{\rm Unif}}
\def\RF{{\rm RF}}

\def\cE{{\mathcal E}}

\def\normal{{\sf N}}

\def\bDelta{{\boldsymbol \Delta}}

\def\cX{{\mathcal X}}

\def\RF{{\rm RF}}

\def\bA{{\boldsymbol A}}
\def\btheta{{\boldsymbol \theta}}
\def\bTheta{{\boldsymbol \Theta}}
\def\bLambda{{\boldsymbol \Lambda}}

\def\Tr{{\rm Tr}}

\def\cV{{\mathcal V}}
\def\bP{{\boldsymbol P}}
\def\diag{{\rm diag}}
\def\bS{{\boldsymbol S}}

\def\bD{{\boldsymbol D}}
\def\bPsi{{\boldsymbol \Psi}}

\def\bL{{\boldsymbol L}}

\def\tbu{\Tilde \bu}
\def\tbZ{\Tilde \bZ}

\def\tbpsi{\Tilde \bpsi}

\def\tbD{\Tilde \bD}

\def\bsigma{{\boldsymbol \sigma}}

\def\cbpsi{\breve \bpsi}
\def\hbpsi{\hat \bpsi}
\def\hbPsi{\hat \bPsi}

\def\hf{\hat f}

\def\bR{{\boldsymbol R}}
\def\bpsi{{\boldsymbol \psi}}

\def\bC{{\boldsymbol C}}

\def\tx{\Tilde x}

\def\bgamma{\boldsymbol{\gamma}}

\def\tbC{\Tilde \bC}

\def\tbeta{\Tilde \beta}
\def\tbbeta{\Tilde \bbeta}

\def\cB{\mathcal{B}}

\def\Im{{\rm Im}}

\def\bq{{\boldsymbol q}}
\def\tbPhi{\Tilde \bPhi}
\def\tbPsi{\Tilde \bPsi}
\def\cbPsi{\breve \bPsi}

\def\opnorm#1{\lVert#1\rVert_{\op}}%
\newcommand{\taumeasure}{\tau_{d,1}}

\def\usp{q}
\def\tusp{\Tilde q}

\def\tbR{\Tilde \bR}
\def\tbx{\Tilde \bx}
\def\tbw{\Tilde \bw}
\def\tbW{\Tilde \bW}
\def\tbX{\Tilde \bX}
\def\tbB{\Tilde \bB}
\def\tbC{\Tilde \bC}
\def\tbD{\Tilde \bD}
\def\tbE{\Tilde \bE}

\def\tbS{\Tilde \bS}
\def\tbH{\Tilde \bH}
\def\tM{\Tilde M}

\def\tbw{\Tilde \bw}
\def\bbmu{\bar \bmu}
\def\hbA{\hat \bA}

\def\diag{{\rm diag}}
\def\tlambda{\Tilde \lambda}
\def\tbU{\Tilde \bU}

\def\RF{{\sf RF}}
\def\KRR{{\sf KRR}}
\def\App{{\sf App}}
\def\barlambda{\bar \lambda}
\def\train{{\sf train}}
\def\test{{\sf test}}
\def\lnorm{\sf norm}
\def\cR{\mathcal{R}}

\def\teq#1{\overset{#1}{=}}%
\def\tleq#1{\overset{#1}{\leq}}%

\def\gratio{\theta}
\def\sfL{{\sf L}}
\def\sfR{{\sf R}}
\def\sfB{{\sf B}}
\def\sfV{{\sf V}}

\usepackage{hyperref}
\hypersetup{
    colorlinks,
    linkcolor={blue!80!black},
    citecolor={green!50!black},
}
\colorlet{linkequation}{blue}

\title{Asymptotics of Random Feature Regression \\Beyond the Linear Scaling Regime}

\author{Hong Hu\thanks{Department of Statistics and Data Science, Wharton School, University of Pennsylvania},\;\; Yue M. Lu\thanks{John A. Paulson School of Engineering and Applied Sciences, Harvard University}, \;\;Theodor Misiakiewicz\thanks{Toyota Technological Institute at Chicago}
}
\date{}

\begin{document}

\maketitle

\begin{abstract}

Recent advances in machine learning have been achieved by using overparametrized models trained until near interpolation of the training data. It was shown, e.g., through the double descent phenomenon, that the number of parameters is a poor proxy for the model complexity and generalization capabilities. This leaves open the question of understanding the impact of parametrization on the performance of these models. How does model complexity and generalization depend on the number of parameters $p$? How should we choose $p$ relative to the sample size $n$ to achieve optimal test error? 

In this paper, we investigate the example of random feature ridge regression (RFRR). This model can be seen either as a finite-rank approximation to kernel ridge regression (KRR), or as a simplified model for neural networks trained in the so-called lazy regime. 
We consider covariates uniformly distributed on the $d$-dimensional sphere and compute sharp asymptotics for the RFRR test error in the high-dimensional polynomial scaling, where $p,n,d \to \infty$ while $p/ d^{\kappa_1}$ and $n / d^{\kappa_2}$ stay constant, for all $\kappa_1 , \kappa_2 \in \R_{>0}$. These asymptotics precisely characterize the impact of the number of random features and regularization parameter on the test performance. In particular, RFRR exhibits an intuitive trade-off between approximation and generalization power. For $n = o(p)$, the sample size $n$ is the bottleneck and RFRR achieves the same performance as KRR (which is equivalent to taking $p = \infty$). On the other hand, if $p = o(n)$, the number of random features $p$ is the limiting factor and RFRR test error matches the approximation error of the random feature model class (akin to taking $n = \infty$). Finally, a double descent appears at $n= p$, a phenomenon that was previously only characterized in the linear scaling $\kappa_1 = \kappa_2 = 1$.  This completes the picture initiated in \cite{ghorbani2021linearized,mei2022generalization,mei2022generalizationRF,xiao2022precise}.
\end{abstract}

\tableofcontents

\section{Introduction}

Consider the supervised learning problem in which we collect $n$ i.i.d.~training data points $\{  y_i , \bx_i \}_{i \leq n}$, from a common probability distribution on $\R \times \cX$. The goal is to learn a model $\hf : \cX \to \R$ which, given a new covariate vector $\bx_{\text{new}}$, predicts the response $y_{\text{new}}$ using $\hf (\bx_{\text{new}})$. To solve this problem, a typical approach 
proceeds as follows. First, select a parametric class of models  $\cF := \{ f ( \cdot ; \btheta ): \btheta \in \bTheta \}$ parameterized by a $p$-dimensional vector $\btheta \in \bTheta \subseteq \R^p$. Second, fit a predictor $\hat f \equiv f (\cdot ; \hat \btheta)$ on the $n$ training samples by minimizing a (possibly regularized) empirical risk over $\btheta \in \bTheta$, often via gradient descent or its variants. When choosing the model class (in particular the number $p$ of parameters) and the training algorithm (e.g., regularization or learning schedule), the statistician  has to keep two goals in mind: 1) the class of models must be expressive enough to approximate the relationship between covariate and response; and 2) the predictor needs to generalize to new data. 

A classical approach to managing these two goals---the \textit{uniform convergence} paradigm developed in the second half of the 20th century \cite{vapnik1999overview}---recommends to control the model complexity (e.g., the number $p$ of parameters or the norm $\| \btheta\|_2$) in order to balance the approximation and generalization errors. More precisely, consider the square loss and denote the test and training errors
\begin{equation*}
    R (f) := \E_{(\bx,y)} \Big[ \big(y - f (\bx) \big)^2\Big]\, , \qquad \hat R_n (f) := \frac{1}{n} \sum_{i \in [n]} \big( y_i - f(\bx_i) \big)^2 \, .
\end{equation*}
We consider a family of nested model classes $\{ \cF (B) \}_{B \in \R_{>0}}$, i.e., $\cF(B) \subseteq \cF (B')$ for $B < B'$. We can think about $\cF (B)$ as containing the set of models with complexity bounded by $B$. Denote $\hf_B$ the model obtained by minimizing the empirical risk $\hat R_n$ over functions in $\cF(B)$.  A classical decomposition \cite{bottou2007tradeoffs} yields the following upper bound on the test error
\begin{equation}\label{eq:approx_gen_decomposition}
    R ( \hf_B ) \leq \inf_{f \in \cF_B} R (f) + 2 \sup_{f \in \cF(B)} \Big\vert \hat R_n ( f) - R (f) \Big\vert \, .
\end{equation}
The first term corresponds to the \textit{approximation error}, which measures how well we can approximate the response with functions in $\cF (B)$. The second term corresponds to the \textit{generalization error}, which measures the uniform deviation between the empirical and population risks over $\cF(B)$. The two terms are respectively decreasing and increasing in $B$, and the classical recommendation selects $B$ such as to balance the two.

On the other hand, recent successes in machine learning have been achieved using highly overparametrized models, namely multi-layer neural networks. Such models operate in a regime that is very different than the classical uniform convergence paradigm. These models are very expressive, with a number of parameters much larger than the number of training samples, and can perfectly fit the training labels, even when they are replaced by pure noise \cite{zhang2021understanding}.  And yet, they show excellent performance on test data
despite being trained with no apparent model complexity control, e.g., until they interpolate the training data $\hat R_n (\hf ) =0$. This is strikingly illustrated by the \textit{double descent} phenomenon, which was pointed out in a number of models including neural networks, random feature models, and random forests \cite{belkin2018understand,belkin2019reconciling,belkin2019does,advani2020high,hastie2022surprises,mei2022generalizationRF}. As the number of parameters increases, the test error first follows the classical U-shaped bias-variance curve, with an initial decrease due to a reduction in model misspecification, followed by an increase due to variance explosion as it approaches the interpolation threshold $p = n$, i.e., the threshold above which the training error vanishes. However, after a peak at the interpolation threshold, the test error decreases again and often becomes much smaller than the minimum test error achieved in the underparametrized regime. This phenomenon illustrates how the number of parameters $p$ is a poor proxy for model complexity and generalization capabilities in overparametrized models. 
This leaves open two fundamental questions:
\begin{itemize}
    \item[(1)] How does the parametrization impact the performance of these models? In particular, how does the model complexity and the generalization error depend on finite $p$?
    \item[(2)] How should we choose $p$ relative to the sample size $n$ to achieve optimal test error?
\end{itemize}

In this paper, we consider the class of random feature models \cite{neal1995bayesian,rahimi2008random,huang2006extreme} and provide precise answers to both questions in the high-dimensional regime. Random feature models are given by
\begin{equation}\label{eq:RF_model_class}
\cF_{\RF} ( \bW) := \Big\{ h_{\RF} ( \bx ; \ba  ) = \frac{1}{\sqrt{p}} \sum_{j \in [p]} a_j \sigma ( \< \bx , \bw_j \> ) : \,\,\, a_j \in \R, \forall j \in [p] \Big\} \, ,
\end{equation}
where $\bW = [ \bw_1 , \ldots , \bw_p ] \in \R^{p \times d}$ is a weight matrix whose $j^{\text{th}}$ row $\bw_j$ is chosen randomly and independently of the data. By analogy with neural networks, we will call $\sigma$ the activation function. To learn the coefficients $\ba = (a_j )_{j \in [p]}$, we perform ridge regression with respect to the random feature model class:
\begin{equation}\label{eq:def_RFRR}
\hat \ba_{ \lambda} = \argmin_{\ba \in \R^p} \Big\{ \sum_{i \in [n]} \big( y_i - h_{\RF} ( \bx_i ; \ba  ) \big)^2 + \lambda \| \ba \|_2^2 \Big\}\, .
\end{equation}
We will refer to this scheme as \textit{random feature ridge regression} (RFRR).

The random feature model \eqref{eq:RF_model_class} can be viewed either as 1) a finite-dimensional approximation of kernel methods; or as 2) a stylized model for two-layer neural networks trained in the linear (lazy) regime. It is useful to detail these connections below.

\vspace{0.3cm}

\noindent\emph{Approximation to kernel methods.} In the case of kernel methods, models belong to $\cH$ a reproducing kernel Hilbert space (RKHS). We denote $\< \cdot , \cdot\>_{\cH}$ and $\| \cdot \|_{\cH}$ the associated scalar product and norm. The RKHS $\cH$ is often defined implicitly via either a positive definite kernel $K : \cX \times \cX \to \R$, or a feature map $\psi : \cX \to ( \cF , \< \cdot , \cdot \>_{\cF} )$ that embeds data in a Hilbert space, with the correspondence $K (\bx , \bx ' ) = \< \psi (\bx) , \psi (\bx') \>_{\cF}$. 
The kernel ridge regression (KRR) estimator is given by
\begin{equation}
    \hf_{\lambda} := \argmin_{f \in \cH} \Big\{ \sum_{i \in [n]} \big( y_i - f(\bx_i) \big)^2 + \lambda \| f \|_{\cH}^2 \Big\} \, .
\end{equation}
We can view RFRR as a specific KRR with mapping $\psi (\bx) = \frac{1}{\sqrt{p}} [ \sigma ( \< \bw_1 , \bx \>) , \ldots , \sigma (\< \bw_p , \bx \>) ] \in \R^p$ and associated kernel
\begin{equation}\label{eq:finite_width_kernel}
K_p ( \bx , \bx ' ) = \< \psi (\bx_1 ) , \psi (\bx_2) \> = \frac{1}{p} \sum_{j \in [p]} \sigma (\< \bx , \bw_j \>) \sigma ( \< \bw_j ,  \bx ' \>) \, .
\end{equation}
The kernel $K_p$ is random, because of the randomness in the weights $\bw_1 , \ldots , \bw_p$, and finite rank (at most) $p$. By law of large number, however, this kernel function concentrates for large $p$ on its deterministic expectation 
\begin{equation}\label{eq:infinite_width_kernel}
K ( \bx , \bx') := \E_{\bw} [K_p ( \bx , \bx ' ) ] = \E_{\bw} [ \sigma (\< \bx , \bw \>) \sigma ( \< \bw ,  \bx ' \>) ]\, .
\end{equation}
Hence the solution of RFRR converges to the deterministic (conditional on the data) KRR solution with kernel \eqref{eq:infinite_width_kernel} as $p \to \infty$. Nonetheless, if $p \ll n$, then RFRR has much lower computational complexity as it deals with matrices $n \times p$, which was the original motivation for RFRR \cite{rahimi2008random}.

\vspace{0.3cm}

\noindent\emph{Neural networks in the linear regime.} A line of research has shown that neural networks trained in a certain optimization regime can be well approximated by their first order Taylor expansion around their random initialization \cite{jacot2018neural,li2018learning,du2018gradient,lee2019wide,du2019gradient,allen2019convergence2,arora2019fine}. Specifically, consider $\bx \mapsto f_{\NN} (\bx ; \btheta)$ a neural network with weights $\btheta \in \R^p$ trained using gradient descent from a random initialization $\btheta^0$. It was shown that for a certain scaling of the parameters at initialization and sufficiently wide neural networks, the weights stay close to their initialization throughout the dynamics. Subsequently, the neural network can be effectively replaced by its linearization around $\btheta^0$
\[
f_{\NN} (\bx ; \btheta) \approx f_{\NN} (\bx ; \btheta^0) + \< \btheta - \btheta^0 , \nabla_{\btheta}  f_{\NN} (\bx ; \btheta^0)  \> \, .
\]
For simplicity, we can set $f_{\NN} (\bx ; \btheta^0) =0$ (this term is not trained and only play the role of an offset) and the linearized model can be written as $\bx \mapsto \<\ba , \nabla_{\btheta}  f_{\NN} (\bx ; \btheta^0) \>$, which is known as the \textit{neural tangent} (NT) model. Hence neural networks trained in the linear regime converge to the KRR solution associated to the feature map $\psi (\bx) =  \nabla_{\btheta}  f_{\NN} (\bx ; \btheta^0)$ and with regularization parameter $\lambda \to 0^+$, the minimum RKHS-norm interpolating solution.  The RF model can be seen as a simplified NT model, associated to a two-layer neural network where the gradient is only taken with respect to the second layer weights $\ba$. 

While RF models are much simpler models than neural networks, where both $\ba, \bW$ are trained jointly, they share some of the key surprising behavior: double descent \cite{belkin2019reconciling,mei2022generalizationRF}, benign overfitting \cite{mei2022generalization} and multi-phase learning curves \cite{mei2022generalization,xiao2022precise}. Furthermore, it was shown in \cite{ghorbani2021linearized,montanari2022interpolation} that some properties of the RF model generalizes to NT models provided that we match the number of parameters in the two models.

\vspace{0.3cm}

In this paper, we consider data uniformly distributed on the sphere $\S^{d-1} (\sqrt{d}) := \{ \bx \in \R^d : \| \bx \|_2 = \sqrt{d} \} $ of radius $\sqrt{d}$ in $\R^d$. Our goal is to learn a target function $f_* \in L^2 (\S^{d-1}(\sqrt{d}))$ given i.i.d.~data $\{ \bx_i, y_i \}_{i \in [n]}$, with $\bx_i \sim_{iid} \Unif(\S^{d-1} (\sqrt{d}))$ and $y_i = f_*(\bx_i) + \eps_i$, where $\eps_i$ are independent noise with $\E [ \eps_i ] = 0$, $\E[\eps_i^2] = \rho_\eps^2$, and $\E [ \eps_i^4] < \infty$. We fit this data using RFRR \eqref{eq:def_RFRR}, where the weights are fixed i.i.d.~uniformly at random on the unit sphere, i.e., $\bw_j \sim_{iid} \Unif ( \S^{d-1} (1))$. We will be interested in the excess test error
\begin{equation}\label{eq:test_error_RF}
    R_{\test} ( f_* ; \bX , \bW , \beps , \lambda ) =\E_{\bx} \Big[ \big( f_* (\bx) -  h_{\RF} (\bx ; \hat \ba_{\lambda} ) \big)^2 \Big] \, ,
\end{equation}
where we made explicit the dependency of the test error on the training data $\bX := [\bx_1 , \ldots , \bx_n]^\sT \in \R^{n \times d}$, the feature weights $\bW := [\bw_1 , \ldots , \bw_p]^\sT \in \R^{p \times d}$, the label noise $\beps := (\eps_1 , \ldots , \eps_n) \in \R^n $ and the regularization parameter $\lambda > 0 $. Following previous works, we will consider for simplicity target functions where we randomize the high-frequency coefficients. Our statements will hold with high probability over the randomness in the target function (equivalently, our statements will hold for `typical' functions in this function class). We conjecture that our results hold for any fixed target function $f_* \in L^2 (\S^{d-1}(\sqrt{d}))$, and leave it to future work.

Our main result is an asymptotic characterization of the RFRR test error \eqref{eq:test_error_RF} in the high-dimensional \textit{polynomial scaling}, where $p,n,d \to \infty$ with
\[
p/d^{\kappa_1} \to \theta_1\, \;\;\;\; \text{ and } \;\;\;\; n / d^{\kappa_2}\to \theta_2 \, ,
\]
for all $\kappa_1 , \kappa_2, \theta_1, \theta_2 \in \R_{>0}$. The convergence holds in probability over the randomness in $\bX,\bW,\beps, f_*$.
The interpolating solution, and the connection to neural networks in the linear regime, corresponds to taking $\lambda \to 0^+$. In our formal results, we will require $\lambda >0$ and can only consider $\lambda \to 0^+$ after $p,n,d \to \infty$. We leave to future work to show that  the limits $\lambda \to 0^+$ and $p,n,d \to \infty$ commute. Note that in the overparametrized regime, the training error is $O(\lambda^2)$: we can set $\lambda$ small such that the training error is much smaller than the test error, i.e., outside the uniform convergence paradigm \eqref{eq:approx_gen_decomposition}. In particular, our results hold for fixed $\lambda$ and do not require $\lambda$ carefully tuned as in previous work \cite{caponnetto2007optimal,rahimi2008weighted,rudi2017generalization,wainwright2019high}.

This setting was previously studied in a string of papers \cite{ghorbani2021linearized,mei2022generalization,mei2022generalizationRF,xiao2022precise}. \cite{ghorbani2021linearized} computed the test error of RFRR in the case of either $n = \infty$ and $\kappa_1 \not\in \naturals$ (approximation error) or $p = \infty$ and $\kappa_2 \not\in \naturals$ (KRR test error), and showed that the risk curves follow a staircase decay where polynomial approximations to the target function of increasing degree are incrementally learned as $n$ or $p$ increases. These results were extended in \cite{mei2022generalization} to $p,n$ both finite with $\kappa_1,\kappa_2 \not\in \naturals $ and $\kappa_1 \neq \kappa_2$. The degree of the polynomial fit to the target function is then given by $\min ( \lfloor \kappa_1 \rfloor , \lfloor \kappa_2 \rfloor )$. The case $\kappa_1 = \kappa_2 = 1$ was investigated separately in \cite{mei2022generalizationRF}, where it was shown that RFRR presents a double descent at the interpolation threshold $n=p$. Finally, the recent work \cite{xiao2022precise} computed the asymptotic risk of KRR ($p = \infty $) at the transition regions $\kappa_2 \in \naturals$, and showed that peaks in the risk curves can appear at these scalings. 

The present paper completes this line of work by computing the RFRR test error when both $\kappa_1,\kappa_2 \in \naturals$ or $\kappa_1 = \kappa_2 = \kappa \in \R_{>0}$. Thus we extend the complete high-dimensional asymptotics from the linear scaling $p,n \asymp d$ from \cite{mei2022generalizationRF} to the polynomial scaling $\log (p),\log(n) \asymp \log (d)$. Note that this is a necessary step in order to study the approximation-statistical trade-off in a more realistic learning setting:  in the linear scaling, RFRR can only fit linear polynomials as $p/d, n/d \to \infty$, while it can fit any target functions $f_* \in L^2$ in the polynomial scaling as $\log(p)/\log(d), \log(n)/\log (d) \to \infty$. In particular, our asymptotics results fully characterize the multi-phase learning as $\kappa_1,\kappa_2$ increase, the non-monotonic behavior at $\kappa_1 > \kappa_2 \in \naturals$ and $\kappa_1 = \kappa_2$, and the optimal overparametrization and regularization parameters.

From a mathematical viewpoint, we present novel characterizations for random matrices with entries polynomial in the weights and covariates. Notably, our results for $\kappa_1 = \kappa_2$ are obtained by deriving the Stieljes transform of a block kernel matrix, analogous to the one studied by \cite{mei2022generalizationRF} in the linear scaling.  However the entries are now given as spherical harmonics of the weights and training samples, and require a more careful leave-one-out analysis which uses an orthonormal expansion in terms of Gegenbauer polynomials. A similar expansion was previously considered in \cite{lu2022equivalence} in the case of a simpler kernel matrix beyond the linear scaling regime.

\subsection{Summary of the RFRR asymptotics in the polynomial scaling}
\label{sec:summary_main_results}

\definecolor{alizarin}{HTML}{d62728}
\definecolor{bluegraph}{HTML}{1f77b4}

\begin{figure}[t!]
\begin{center}
\begin{tikzpicture}

\node[inner sep=0pt] (russell) at (0,0)
    {\includegraphics[width=.7\textwidth]{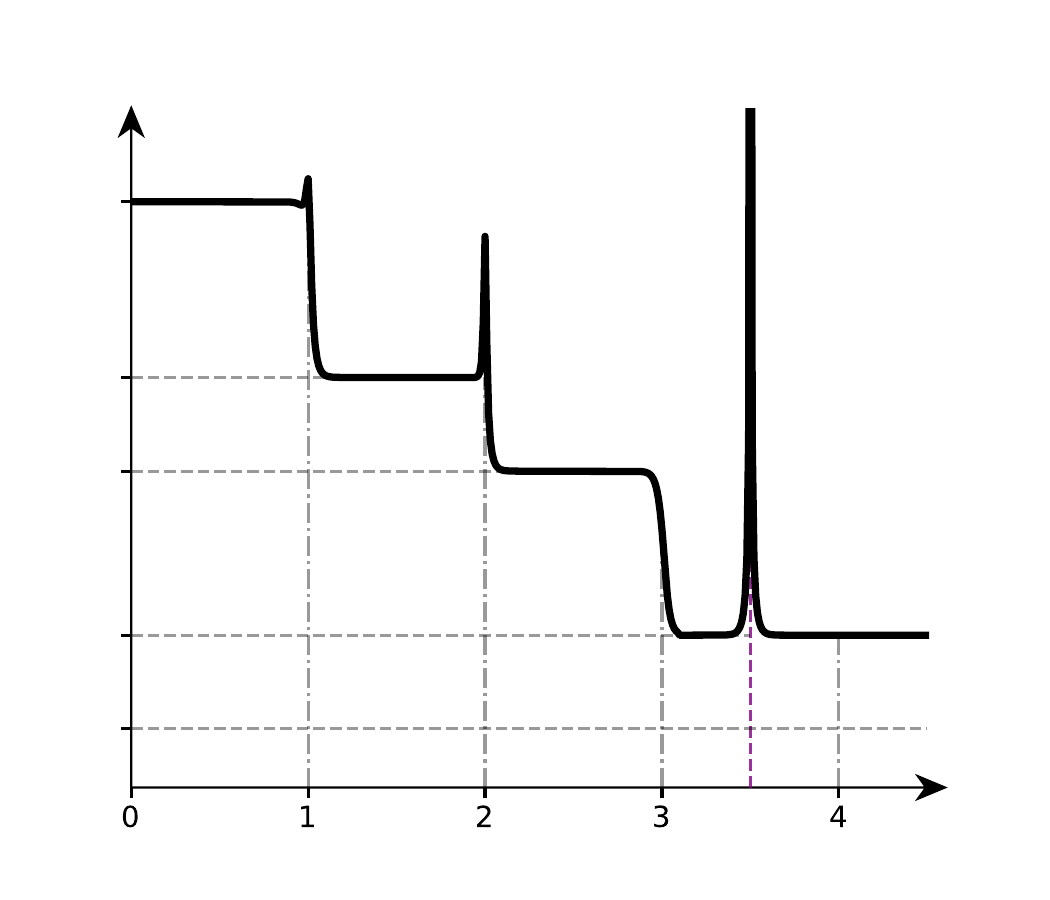}};

\node[inner sep=0pt] (russell) at (-4,-8)
    {\includegraphics[width=.47\textwidth]{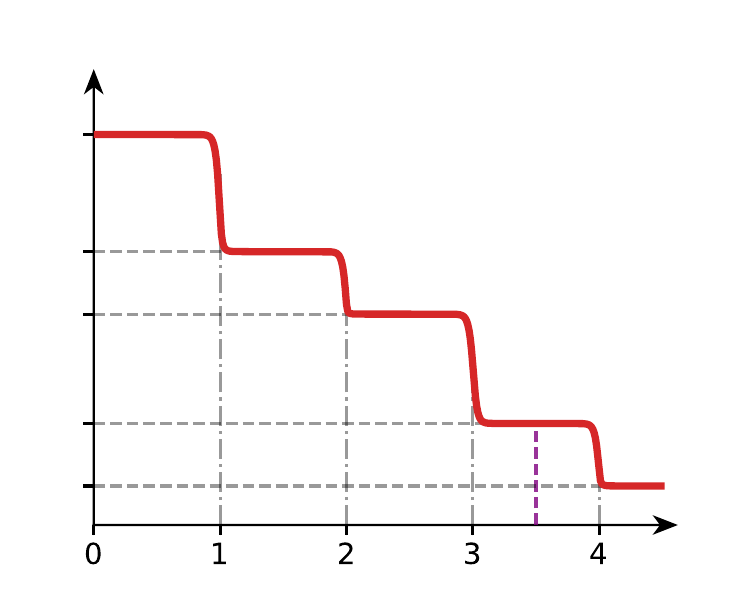}};

\node[inner sep=0pt] (russell) at (4,-8)
    {\includegraphics[width=.47\textwidth]{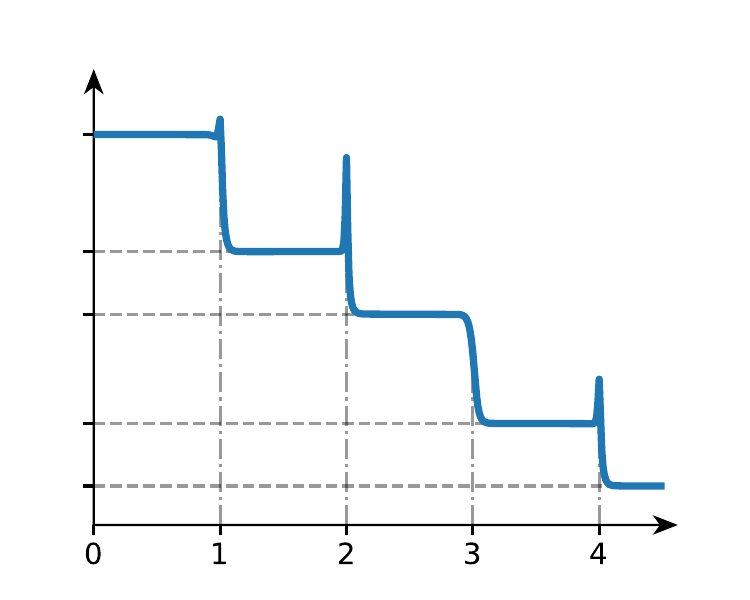}};

\node[rectangle,draw,thick,align=center] (a) at (-1,4) {\large \textbf{Test error of RFRR}}; 
 \node[rectangle,draw=bluegraph,thick,align=center] (a) at (4.6,-5.5) {\large {\color{bluegraph} \textbf{Test error of KRR ($p = \infty$)}}};

\node[rectangle,draw=alizarin,thick,align=center] (a) at (-3.4,-5.5) {\large {\color{alizarin} \textbf{Approximation error ($n = \infty$)}}};

 \node[rectangle] (a) at (-5.4,2.7) {$\| \proj_{>0} f_* \|_{L^2}^2$};
 \node[rectangle] (a) at (-5.4,0.8) {$\| \proj_{>1} f_* \|_{L^2}^2$};
 \node[rectangle] (a) at (-5.4,-0.2) {$\| \proj_{>2} f_* \|_{L^2}^2$};
 \node[rectangle] (a) at (-5.4,-2) {$\| \proj_{>3} f_* \|_{L^2}^2$};
 \node[rectangle] (a) at (-5.4,-3.1) {$\| \proj_{>4} f_* \|_{L^2}^2$};
 
 \node[rectangle] (a) at (5.2,-3.7) {$\frac{\log(n)}{\log(d)}$};

  \node[rectangle] (a) at (7.6,-10.35) {$\frac{\log(n)}{\log(d)}$};

   \node[rectangle] (a) at (-0.4,-10.35) {$\frac{\log(p)}{\log(d)}$};
   
 \node[rectangle] (a) at (0,-4.4) {$\kappa_2$};
  \node[rectangle] (a) at (2.5,-4.1) {$n=p$};

   \node[rectangle] (a) at (-4,-11.05) {$\kappa_1$};

    \node[rectangle] (a) at (4,-11.05) {$\kappa_2$};

\end{tikzpicture}

\end{center}
\caption{Cartoon illustration of the test error of RFRR in the high-dimensional polynomial scaling $p/d^{\kappa_1} \to \theta_1$ and $n/d^{\kappa_2} \to \theta_2 $ as $p,n,d \to \infty$, for $\kappa_1,\kappa_2,\theta_1 , \theta_2 \in \R_{>0}$. \emph{Top:} test error of RFRR versus $\log(n)/\log(d)$ for fixed $p$. \emph{Bottom left:} approximation error ($n= \infty$) of random feature models  versus $\log(p)/\log(d)$. \emph{Bottom right:} test error of KRR ($p = \infty$) versus $\log (n) / \log(d)$. The approximation error (resp.~KRR test error) follows a staircase decay where each time $\log(p)/\log(d)$ (resp.~$\log(n)/\log(d)$) crosses an integer value, the RF model fits one more degree polynomial approximation to the target function. Peaks can appear in the KRR risk curve at $n = d^\ell/\ell!, \ell \in \naturals$, depending on some effective regularization and effective signal-to-noise ratio at that scale. The  RFRR test error first follows the KRR test error for $n \ll p $, then presents a peak at the interpolation threshold $n = p$, before saturating on the approximation error for $n \gg p$.\label{fig:CartoonIllustration}}

\end{figure}
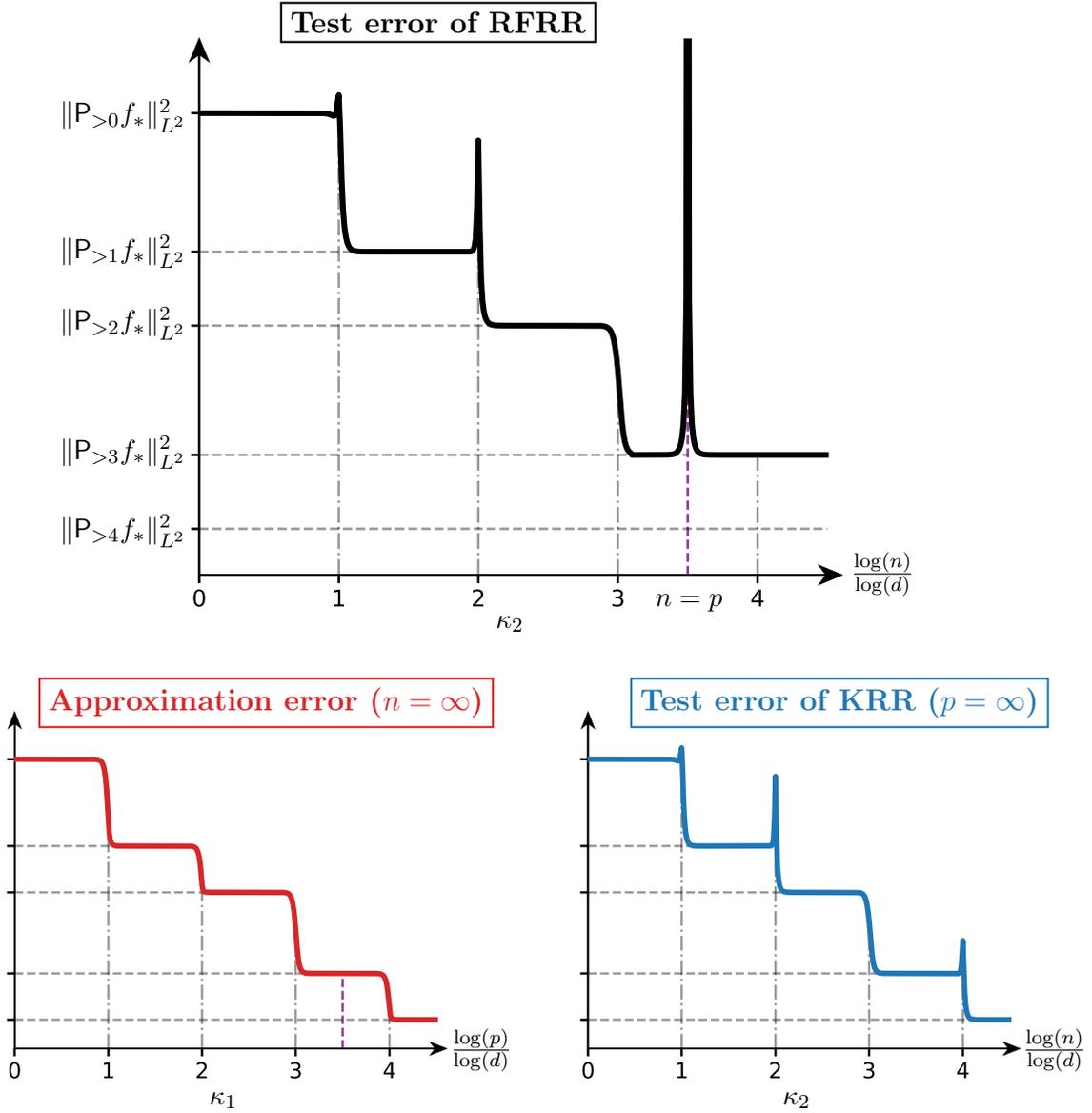

To summarize these asymptotics, it is useful to separate the two limiting factors in the performance of RFRR. First, we have a finite number of random features $p$, which limits the class of functions that RF models can approximate. We define the approximation error
\begin{equation}\label{eq:App_def}
\begin{aligned}
R_{\App} (f_*; \bW) = &~ \inf_{\ba\in \R^p} \E_{\bx} \Big[ \big(f_* (\bx) -  h_{\RF} ( \bx ; \ba  ) \big)^2 \Big] 
\end{aligned}
\end{equation}
as the best fit to the target function using $p$ features with random weights $\bW$. We can think of $R_{\App}$ as being the test error of RFRR if we had access to an infinite number of training samples $n = \infty$. Second, we have a `statistical error' due to the finite number of training data points $n$. We isolate this contribution by considering $p =\infty$ (which can approximate any $f_* \in L^2$) and define
\begin{equation}\label{eq:KRR_def}
\begin{aligned}
R_{\KRR} (f_*; \bX, \beps, \lambda) =&~ \E_{\bx} \Big[ \big(f_* (\bx) -  \hat{h}_{\KRR, \lambda} ( \bx   ) \big)^2 \Big]\, , \\
\hat h_{\KRR , \lambda} =&~ \argmin_h \Big\{ \sum_{i\in [n]} (y_i - h (\bx_i) )^2 + \lambda \| h \|_{\cH}^2 \Big\} \, ,
\end{aligned}
\end{equation}
where $\hat h_{\KRR}$ is the solution of the kernel ridge regression problem with $\| \cdot \|_{\cH}$ the RKHS norm associated to kernel 
\begin{equation}\label{eq:kernel_asymp}
K (\bx , \bx') = \E_{\bw \sim \Unif ( \S^{d-1} (1) ) } [\sigma (\< \bx , \bw \>) \sigma (\< \bx' , \bw \>) ]\, .
\end{equation}
Note that by rotational invariance of $\bw$ and using that the covariates are normalized $\| \bx \|_2 = \sqrt{d}$, $K$ is an inner-product kernel and can be written as $K ( \bx , \bx') = h_d ( \< \bx , \bx ' \> / d)$, where $h_d : [-1,1] \to \R$.

We summarize the asymptotic predictions for the RFRR test error using a cartoon illustration in Figure \ref{fig:CartoonIllustration}. We plot the RFRR test error versus $\log(n)/\log(d)$ for a fixed $p$ (top), the approximation error versus $\log(p)/\log(d)$ (bottom left) and the KRR test error versus $\log(n) /\log (d)$ (bottom right). In the figure, we denoted $\proj_{>\ell} f_*$ the projection of $f_*$ orthogonal to the subspace of polynomials of degree at most $\ell$ with respect to the uniform measure on the sphere. In particular, if the risk is given by $\| \proj_{>\ell} f_* \|_{L^2}^2$, this implies that we fit the best degree-$\ell$ polynomial approximation to $f_*$ and none of its higher frequency components.

The precise formulas for these curves can be found in Section \ref{sec:MainResults}. Below we discuss some of the key features of Figure \ref{fig:CartoonIllustration}:
\begin{description}
\item[Overparametrized regime $\kappa_2 < \kappa_1$:] We have $n \ll p $ and RFRR achieves the same test error as KRR
    \[
    R_{\test} ( f_* ; \bX , \bW , \beps , \lambda ) = R_{\KRR} (f_*; \bX, \beps, \lambda) + o_{d,\P} (1) \, .
    \]
    In this regime, the number of samples is the bottleneck for learning $f_*$, and RFRR behaves as if we had $p =\infty$ random features. The test error of KRR in the polynomial scaling was characterized in \cite{ghorbani2021linearized,xiao2022precise}. If $d^{\ell} \ll n \ll d^{\ell+1}$, then the test error is given by $\| \proj_{>\ell} f_* \|_{L^2}^2$, i.e., KRR fits the best degree-$\ell$ polynomial approximation to $f_*$. At the critical scalings $n \asymp d^\ell$, where KRR transition from fitting all degree-$(\ell-1)$ polynomials to all degree-$\ell$ polynomials, a finite-sized peak can appear in the risk curve at $n = (1 + o_d(1))d^\ell / \ell!$, depending on an effective regularization and effective signal-to-noise ratio at that scale (see \cite{xiao2022precise} for a detailed discussion). This peak is due to the degeneracy of the eigenvalues associated to degree-$\ell$ spherical harmonics in the eigendecomposition of the kernel \eqref{eq:kernel_asymp}.

\item[Underparametrized regime $\kappa_2 > \kappa_1$:] In this case, $n \gg p$ and RFRR test error matches the approximation error
    \[
    R_{\test} ( f_* ; \bX , \bW , \beps , \lambda ) = R_{\App} (f_*; \bW) + o_{d,\P} (1) \, .
    \]
    The number of features is now the bottleneck for learning $f_*$, and RFRR achieves the best approximation error over the random feature model class $\cF_{\RF} (\bW)$ akin to having $n= \infty$. If $d^{\ell} \ll p \ll d^{\ell+1}$, RF models can approximate any degree-$\ell$ polynomials, and the approximation error is given by $\| \proj_{>\ell} f_* \|_{L^2}^2$.

\item[Critical parametrization regime $\kappa_1 = \kappa_2$:] In this regime, RFRR test error interpolates between the KRR test error for $n/p = \theta_2 /\theta_1 \to 0$ and the approximation error for $n/p = \theta_2 /\theta_1 \to \infty$, with a peak at the interpolation threshold $n =p$ (when the RF model has enough parameters to interpolate the $n$ training data points) that diverges as $\lambda \to 0$. This peak is due to the divergence of the conditioning number of the feature matrix $\bZ = (\sigma (\< \bx_i , \bw_j \>) )_{i \in [n], j \in [p]} \in \R^{n \times p}$ as $p$ approaches $n$ and $\bZ$ becomes a square matrix.

This critical regime was characterized in \cite{mei2022generalizationRF} in the linear scaling $\kappa_1 = \kappa_2 = 1$. We show that the test error follows the same form  for all $\kappa_1 = \kappa_2 = \ell \in \naturals$, and only depends on $\theta_1 \ell!, \theta_2\ell!$, noise level $\rho_\eps^2$ and $\| \proj_\ell \sigma (\< \be , \cdot \>) \|_{L^2}$, $\| \proj_{>\ell} \sigma (\< \be , \cdot \>) \|_{L^2}$, $\| \proj_\ell f_* \|_{L^2}$ and $\| \proj_{>\ell} f_* \|_{L^2}$.
\end{description}

Denoting $R(n,p)$ the test error with $n$ training samples and $p$ random features, we can summarize these results with the following heuristic
\[
R(n,p) \approx \max \{ R(n,\infty) , R(\infty , p) \} \, .
\]
Hence the performance of RFRR follows a simple trade-off between approximation and statistical errors. If $p = o(n)$, the approximation error dominates: the performance of RFRR is limited by the number of random features $p$ and matches the best approximation error achieved by the random feature class with $n = \infty$. On the other hand, if $n = o(p)$, the statistical error dominates: the sample size is now the limiting factor and RFRR matches the performance of KRR with $p = \infty$. In particular, for $p\asymp d^{\kappa_1}, n \asymp d^{\kappa_2}$, RFRR fits the best degree-$\min (\lfloor \kappa_1 \rfloor, \lfloor \kappa_2 \rfloor )$ polynomial approximation to the target function. This intuitively matches a dimension-counting heuristic lower-bound: the space of degree-$\ell$ polynomials has dimension $\Theta (d^\ell)$ and we need $n = \Omega (d^\ell)$ samples and $p = \Omega (d^\ell)$ parameters to fit this subspace.

From these asymptotic predictions of the RFRR test error, several insights can be gained on random feature models which we summarize below:

\begin{description}
\item[Effect of parametrization:] 
For linear models with ridge penalty, the norm $\| \hba_\lambda \|_2$ is a better complexity measure than counting the number of parameters \cite{vapnik1999overview,ng2000cs229,hastie2022surprises}. We can indeed verify that when plotting the test error versus $\| \hba_\lambda \|_2$, we recover a U-shaped curve instead of the double-descent curve. 

For completeness, we include Figure \ref{fig:testerr_RKHSnorm} for RFRR (see \cite[Figure 8.12]{ng2000cs229} for a similar plot in the case of standard linear regression). When $p$ increases, $\| \hba_\lambda \|_2$ first increases until the interpolation threshold $n=p$ and then decreases ---instead of diverging--- before concentrating on $\| \hf_\lambda \|_{\cH}$ (after proper normalization) as $p \to \infty$.
 As for the test error, it exhibits the double-descent behavior with a peak at $p=n$ where $\| \hba_\lambda \|_2/n$ is maximized. In particular, when $p\to\infty$, the test error does not detoriate and converges to the KRR test error. At the same time the training error stays much lower than the noise level (the training error is of order $10^{-6}$  for $p \gtrsim n$ in Figure \ref{fig:testerr_RKHSnorm}). This benign overfitting phenomenon can be understood as follows: as $p \to \infty$, the high-frequency part of the activation function behaves as an additive self-induced regularization, and the effective ridge regularization of RFRR is bounded away from $0$ (see \cite{mei2022generalization,mei2022generalizationRF,xiao2022precise} for further discussions). 

 We further note from Figure \ref{fig:testerr_RKHSnorm} that the minimum test error is achieved in the overparametrized regime, when $p/n = \theta_1/\theta_2 \to \infty$. Increasing parametrization in the random feature model allows to approximate a growing class of functions until it approximates the KRR solution and the test error saturates on the KRR test error.

\begin{figure}[t]
\begin{tikzpicture} 
\node[inner sep=0pt] (russell) at (0,0.5)
{\includegraphics[width=.32\textwidth]{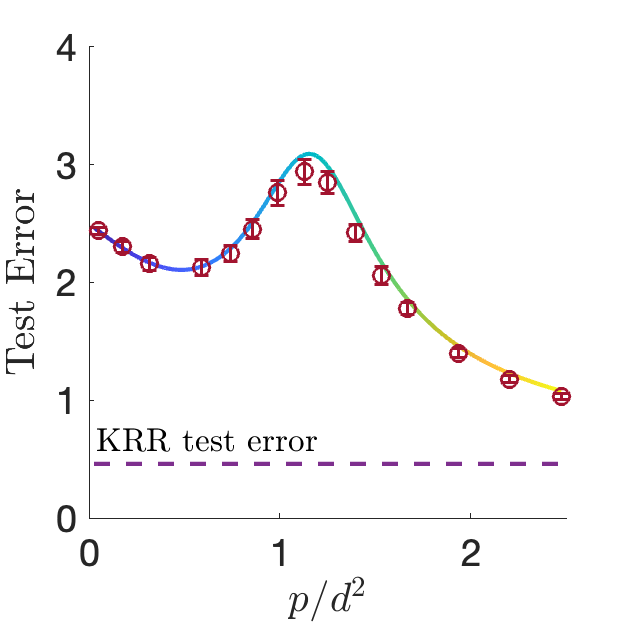}};

\node[inner sep=0pt] (russell) at (5.1,0.48)
{\includegraphics[width=.32\textwidth]{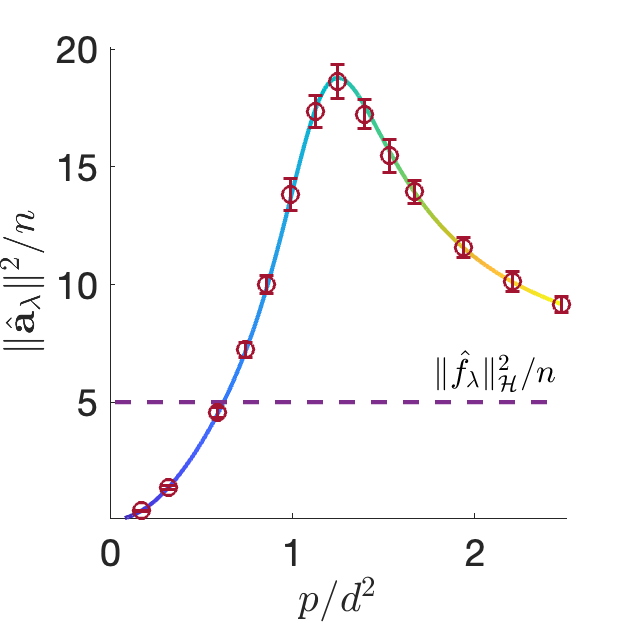}};

\node[inner sep=0pt] (russell) at (10.85,0.48)
{\includegraphics[width=.38\textwidth]{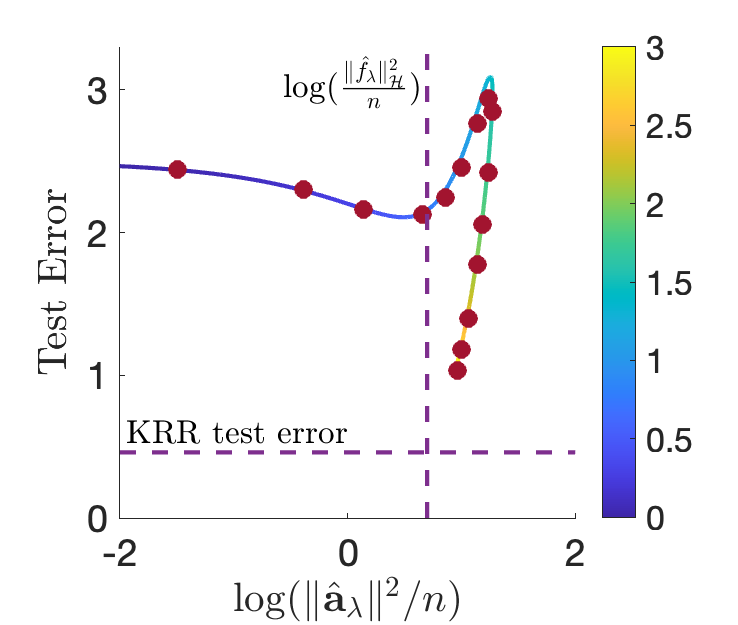}};
\end{tikzpicture}
\caption{Test error and $\|\hba_\lambda\|^2/n$ in the polynomial scaling $\kappa_1 = \kappa_2 = 2$. We fix $d=100$, $n = 1.25 d^2$ and vary the number of random features $p$. Here, $f_{*,d}(\bx) = 1.5 q_2^{(d)}(\bx^\sT \bbeta) + 0.5 q_3^{(d)}(\bx^\sT \bbeta)$ with $\|\bbeta\| = 1$, $\sigma(x) = 0.5 q_2^{(d)}(x) + 0.3 q_3^{(d)}(x)$, $\rho_\varepsilon = 0.2$ and $\lambda = 2.5 \times 10^{-3}$. 
The error bars and dots are the empirical results averaged over 100 independent trials and the solid curves are the theoretical predictions given in the main theorem (Theorem \ref{thm:main_theorem_RF}).
\textit{Left and middle:} the dashed curves are the theoretical predictions for the asymptotic test error and squared RKHS norm of KRR, when $d\to\infty$ and $n/d^2 \to 1.25$ (given in Theorem \ref{thm:KRR_asymptotics}). \textit{Right figure:} the numbers on the colorbar correspond to the values of $p/d^2$.
} 
\label{fig:testerr_RKHSnorm}
\end{figure}


    \item[Optimal number of features:] From a practical point of view, it is interesting to ask the following: how small can we take $p$ to achieve optimal test error? Our results show that taking $p/n = \theta_1/\theta_2 \to \infty$ (after $p,n,d \to \infty$) is enough to achieve the KRR test error and larger overparametrization scalings (taking $\kappa_1 > \kappa_2$) do not improve the test error. At the same time, $p/n = O(1)$ can result in sub-optimal performance. We precisely capture this suboptimality as a function of the target function, activation function and scalings $\kappa_1,\kappa_2,\theta_1,\theta_2$. We further note that, in some under-regularized cases (i.e., when $\lambda$ is chosen too small and there is a large effective variance contribution in the KRR test error), overparametrization can hurt and optimal test error is achieved in the underparametrized regime $p \ll n$.

    \item[Optimal regularization:] For $\kappa_1 \neq \kappa_2$ with $\kappa_1 , \kappa_2 \not\in \naturals$, the interpolating solution $\lambda \to 0^+$ achieves optimal test error (which is given by $\| \proj_{>\min (\lfloor \kappa_1 \rfloor , \lfloor \kappa_2 \rfloor)} f_* \|_{L^2}^2$), and taking $\lambda $ larger can result in sub-optimal performance. In the overparametrized $\kappa_1 > \kappa_2 \in \naturals$ and critical $\kappa_1 = \kappa_2$ regimes, the test error can be non-monotonic with respect to $n$ or $p$ under a given $\lambda$ and  peaks can appear in the risk curve. An illustration of this phenomenon in the $\kappa_1 = \kappa_2$ regime can be found in Figure \ref{fig:optlambda}. We can see near these peaks, the optimal test error is achieved at non-zero regularization parameter, while the risk curve with $\lambda$ chosen optimally at each point is monotonically decreasing in $p, n, \text{and SNR}$. 
    On the other hand, we can also find that $0^+$ regularization tends to be optimal when (1) $p$ is much smaller or larger than $n$ and (2) SNR is high.

    \item[Asymptotic equivalence with a Gaussian model:] The activation function can be diagonalized as $\sigma (\< \bx, \bw \>) = \< \phi ( \bx), \bSigma \phi (\sqrt{d} \bw) \>_{\ell_2}$, where $\phi (\bx) = (\phi_j (\bx) )_{j \geq 1}$ are spherical harmonics that form a complete orthonormal basis of $L^2 (\S^{d-1} (\sqrt{d}))$, and $\bSigma$ is a diagonal matrix that contains the singular values (with signs) of $\sigma$. While $\phi (\bx)$ and $\phi (\sqrt{d} \bw)$ have entries that are not independent or subgaussian, we note that the asymptotics of ridge regression with the random feature model is the same as ridge regression in a simpler Gaussian model where $\phi (\bx_i)$ and $\phi (\sqrt{d}\bw_j)$ are replaced by iid Gaussian vectors $\bg_i$ and $\bf_j$ with matching first two moments with $\phi$. This equivalence was already noticed in \cite{mei2022generalizationRF} in the linear scaling $\kappa_1 = \kappa_2 = 1$ (with a simplified Gaussian model). Here we show that this equivalence holds more generally in the entire polynomial scaling. In the  KRR limit, this Gaussian covariate model simplify and was described in \cite{misiakiewicz2022spectrum,hu2022sharp}.
\end{description}

We provide additional discussion on each of these points in Section \ref{sec:Interpretations}.

The rest of the paper is organized as follows. We discuss related work in Section \ref{sec:RelatedWork} and introduce notations in Section \ref{sec:notations}. Section \ref{sec:MainResults} provides the complete set of asymptotics for RFRR in the polynomial scaling and states our main theorem and assumptions. In Section \ref{sec:GaussianEquivalence}, we present the equivalence of the asymptotic test error between the RF model and a simpler Gaussian covariate model in the polynomial scaling. Finally, we outline the proof of the main results in Section \ref{sec:ProofMainResults}. We defer some of the most technical parts to the appendices.

\begin{figure}[t]
\begin{tikzpicture} 
\node[inner sep=0pt] (russell) at (0,0.5)
{\includegraphics[width=.33\textwidth]{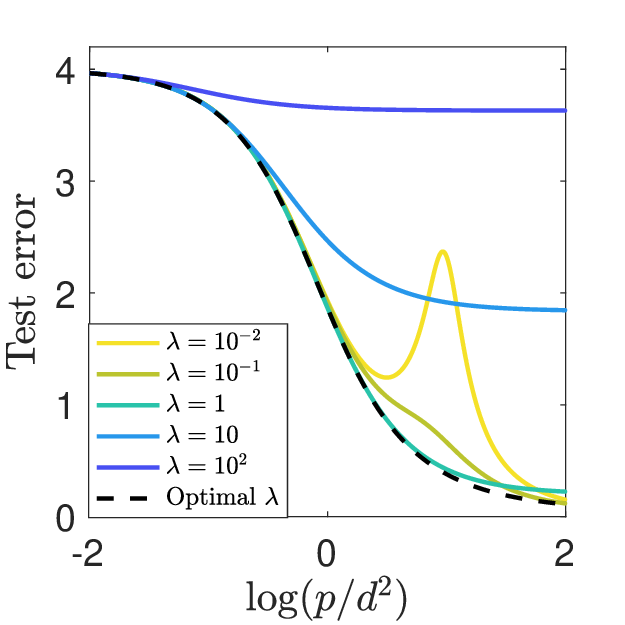}};

\node[inner sep=0pt] (russell) at (5.6,0.48)
{\includegraphics[width=.335\textwidth]{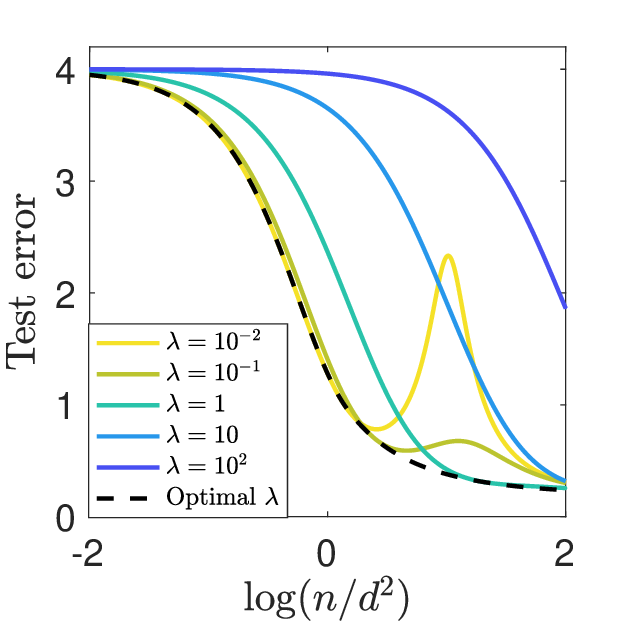}};

\node[inner sep=0pt] (russell) at (11.05,0.45)
{\includegraphics[width=.338\textwidth]{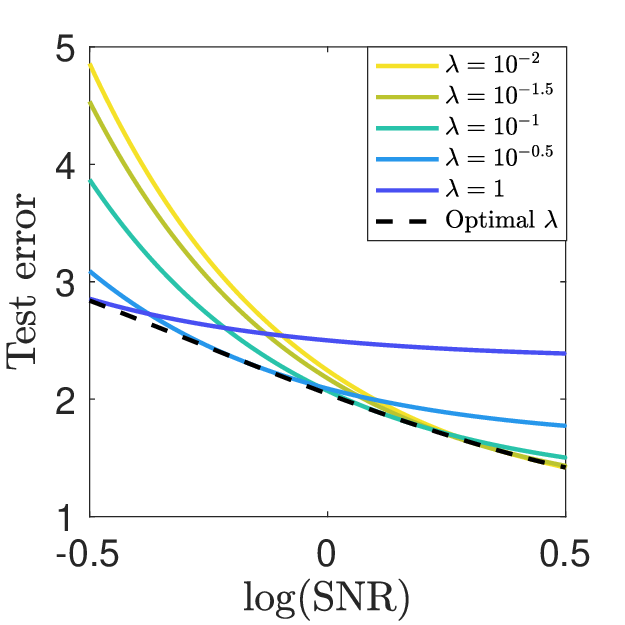}};
\end{tikzpicture}
\caption{Test error for different regularization parameters $\lambda$ in the polynomial scaling $\kappa_1 = \kappa_2 = \ell = 2$. Here,  $f_{*,d}(\bx) = 2 q_2^{(d)}(\bx^\sT \bbeta)$ with $\|\bbeta\|=1$, $\sigma(x) = 0.5 q_2^{(d)}(x) + 0.5 q_3^{(d)}(x)$ and
$\text{SNR}:= \| f_* \|_{L^2}^2 / \rho_\varepsilon^2$. 
\textit{Left figure:} $n/d^2=10$ and $\text{SNR} = 5$. \textit{Middle figure:} $p/d^2=10$ and $\text{SNR} = 5$. \textit{Right figure:} $p/d^2=10$ and $n/d^2=1$.
} 
\label{fig:optlambda}
\end{figure}

\subsection{Related work}
\label{sec:RelatedWork}

Classical statistical theory has sought to study the approximation and generalization properties of neural networks decoupled from computational questions. This approach typically proceeds in two steps. First, it bounds the number of neurons and the norm of the weights needed to approximate the class of target functions \cite{barron1993universal,maiorov1999best,mhaskar1996neural,pinkus1999approximation}. Second, it postulates a neural network that minimizes a regularized empirical risk and bounds the statistical complexity for this estimator via uniform convergence \cite{bach2017breaking,schmidt2020nonparametric}. However, this approach does not provide efficient algorithms to construct these neural networks. On the other hand, theory and practice have shown that overparametrization, and having much more neurons than the minimal width needed for approximation, can make gradient-based optimization much easier \cite{bartlett2021deep}, e.g., by linearizing the landscape in the lazy regime \cite{jacot2018neural}. Thus, while regularized ERM can inform on the optimal number of parameters and samples needed to learn a class of target functions, it provides limited insights on the interplay between approximation, generalization and regularization in practical neural networks trained by gradient descent. In this paper, we focus on a limited class of gradient-trained neural networks, where only the second-layer weights are learned, and provide a complete picture for \emph{efficiently trained networks} in this restricted setting.

The random feature model was introduced by Rahimi and Recht \cite{rahimi2008random} to lower the computational complexity of kernel methods via a randomized finite-rank approximation of kernel functions. \cite{rahimi2008random} showed that the empirical kernel $K_p (\bx , \bx ')$ (Equation \eqref{eq:finite_width_kernel}) converges to the limiting kernel $K (\bx, \bx')$ (Equation \eqref{eq:infinite_width_kernel}) uniformly over compact sets. The approximation and generalization errors of random feature models were later studied in \cite{rahimi2008weighted,rudi2017generalization,bach2017equivalence,bach2017breaking,ma2020towards}. In particular, \cite{rudi2017generalization} proved that roughly $ p  \gtrsim \sqrt{n}$ random features are sufficient to match the performance of kernel ridge regression, in contrast to $p \gtrsim n$ in the present paper. The setting of \cite{rudi2017generalization} is fairly different to the one considered here: they require the target function to be in a fixed RKHS (fixed $d$) and compute error rates that are minimax optimal up to a multiplicative constant, while this paper considers a high-dimensional regime and prove pointwise test errors that hold up to a vanishing additive constant for more general square integrable functions (see \cite{mei2022generalization} for a discussion contrasting these two settings). 

Recently, the random feature model has attracted renewed interest due to its connection to neural networks, either via the neural tangent kernel \cite{jacot2018neural} or Gaussian process \cite{novak2018bayesian,matthews2018gaussian} theories of wide neural networks. In particular, it was argued that the random feature model shares some key surprising behavior with deep learning: double descent phenomenon \cite{belkin2019reconciling,belkin2018understand,hastie2022surprises,belkin2020two} and benign overfitting \cite{bartlett2020benign,liang2020just,muthukumar2020harmless}. To capture these phenomena, several papers considered computing the precise asymptotics of the test error of RFRR in the linear high-dimensional scaling $p/d \to \theta_1$ and $n/d \to \theta_2$ \cite{mei2022generalizationRF,adlam2020neural,liao2020random}. In this regime, RFRR can fit at most a linear approximation to the target function. A second line of work has studied a more general polynomial scaling with $p/d^{\kappa_1} \to \theta_1$ and $n/d^{\kappa_2} \to \theta_2$ \cite{ghorbani2020neural,ghorbani2021linearized,mei2022generalization,xiao2022precise}. These works revealed a staircase decay of the risk curves where polynomials of growing degree are progressively fitted as $\kappa_1,\kappa_2$ increase \cite{mei2022generalization}, and a multiple descent behavior where peaks can appear at each $\kappa_2 \in \naturals$ \cite{xiao2022precise}.

From a technical aspect, our analysis require to control empirical kernel matrices with inner-product kernels. The paper \cite{el2010spectrum} considered matrices of the form $\bK_n := f(\bX\bX^\sT/d)$ where $\bX \in \R^{n \times d}$ is a random matrix with i.i.d.~entries, and showed that $\bK_n$ can be well approximated by its linearization in the linear scaling $n/d \to \theta$, and its spectrum converges to a scaled Marchenko-Pastur law. On the other hand, considering a different scaling $\Tilde{\bK}_n := f(\bX\bX^\sT/\sqrt{d})/\sqrt{n}$, the papers \cite{cheng2013spectrum,fan2019spectral} showed that the spectrum of $\Tilde{\bK}_n$ converges to the free convolution of a Marchenko-Pastur law and a semi-circular law. These results were generalized to the polynomial scaling $n/d^\ell \to \theta$ for $\bK_n$ in \cite{misiakiewicz2022spectrum,xiao2022precise} and for $\Tilde{\bK}_n$ in \cite{lu2022equivalence}, where it was shown that the kernel is well approximated by its degree-$\ell$ polynomial approximation plus an independent noise matrix coming from the higher-degree terms. Finally, note that the asymmetric case $\bZ := f(\bX \bW^\sT / \sqrt{d})/\sqrt{n}$ was studied in the linear scaling in \cite{pennington2017nonlinear,louart2018random}. These work focused on the asymptotic spectrum, while the asymptotic risk of RFRR also depends on the singular vectors of $\bZ$. The paper \cite{mei2022generalizationRF} showed in the linear scaling how the test error can be obtained as derivatives of the log-determinant of a block matrix. To derive their asymptotics, they first compute the Stieltjes transform of the block matrix using a leave-one-out analysis, and then integrate this Stieltjes transform to obtain the formula for the log-determinant. In the present paper, we apply the same strategy to the polynomial regime. Note that while the block matrix is well approximated by matrices with iid entries in the linear scaling, the entries are given by degree-$\ell$ spherical harmonics in the polynomial scaling which requires a more involved leave-on-out analysis.

Finally, our work shows that the asymptotic risk of RFRR in the polynomial scaling is the same as a simpler Gaussian covariate model. This was first noted in \cite{mei2022generalizationRF} for RFRR in the linear scaling. Following work \cite{goldt2022gaussian,hu2022universality,montanari2022universality} proved that in fact, this universality phenomenon holds for random feature models and more general loss function and regularization in the linear scaling. The present paper shows that this Gaussian equivalence remains valid for the square loss and $\ell_2 $ regularization in the polynomial scaling. However, we note that this equivalence will not be true in general for other losses and regularization functions beyond the linear regime.

\subsection{Notations}
\label{sec:notations}

Let $\text{Re} (z)$ and $\text{Im} (z)$ denote the real and imaginary parts of a complex number $z \in \C$. We further denote $\C_+ = \{ z \in \C : \text{Im} (z) >0 \}$ the upper half-plane of complex numbers with positive imaginary part.
For a positive integer $n$, let $[n]$ be the set $\{1, 2 , \ldots , n\}$. For vectors $\bu,\bv \in \R^n$, we denote the standard euclidean scalar product $\< \bu, \bv \> = u_1v_1 + \ldots + u_d v_d$, and $\ell_2$ norm $\| \bu \|_2 = \< \bu , \bu \>^{1/2}$. Let $\S^{d-1} (r) = \{ \bu \in \R^d : \| \bu\|_2 = r\}$ be the sphere of radius $r$ in $d$ dimensions. For the unit ball, we will sometimes simply write $\S^{d-1} := \S^{d-1}(1)$.

For a matrix $\bA \in \R^{n \times d}$, we denote by $\| \bA \|_\op = \max_{\| \bu \|_2 = 1} \| \bA \bu \|_2$ its operator norm and $\| \bA \|_F = \big( \sum_{i,j} A_{ij}^2 \big)^{1/2}$ its Frobenius norm. For a square matrix $\bA \in \R^{n \times n}$, we denote by $\Tr(\bA) = \sum_{i \in[n]} A_{ii}$ its trace. For a measurable function $h : \R \to \R$ and a matrix $\bA \in \R^{n \times d}$, we denote by $h(\bA) = (h(A_{ij})_{i\in[n],j\in[d]}$ the elementwise application of $h$ to the entries of $\bA$.

Throughout the proofs, we use $O_d(\cdot)$ (resp.~$o_d(\cdot)$) for the standard big-O (resp.~little-o) relations, where the subscript $d$ is the asymptotic variables. We will further write $f = \Omega_d (g)$ if $g(d) = O_d(f(d))$, $f =\omega_d(g)$ if $g(d) = o_d(f(d))$, and $f = \Theta_d (g)$ if we have both $f = O_d(g)$ and $g = O_d(f)$. We will denote $O_{d,\P} (\cdot)$ (resp.~$o_{d,\P} (\cdot)$) the big-O (rep.~little-o) in probability relations. Recall that for two sequences of random variables $X_1(d)$ and $X_2(d)$, we have $X_1(d) = O_{d,\P}(X_2(d))$ if for any $\eps >0$, there exists $C_\eps >0$ and $d_\eps \in \N$, such that
\[
\P ( | X_1(d) / X_2(d) | > C_\eps) \leq \eps \, , \qquad \forall d \geq d_\eps \, ,
\]
and $X_1 (d) = o_{d,\P}(X_2(d))$ if $X_1(d)/X_2(d)$ converges to $0$ in probability. Similarly, we denote $X_1(d) = \Omega_{d,\P} (X_2(d))$ if $X_2(d) = O_{d,\P} (X_1(d))$, $X_1 (d) = \omega_{d,\P} (X_2(d))$ if $X_2(d) = o_{d,\P} (X_1(d))$, and $X_2(d) = \Theta_{d,\P} (X_1(d))$ if we have both $X_2 (d) = O_{d,\P}(X_1(d))$ and $X_1(d) = O_{d,\P} (X_2(d))$.

Finally, for two sequences of nonnegative random variables $X_1(d)$ and $X_2(d)$, we say that $X_1(d)$ is \textit{stochatiscally dominated} by $X_2(d)$, if for any $\eps >0$ and $D > 0$, there exists $d_{\eps,D}$ such that 
\[
\P (X_1(d) > d^\eps X_2(d) ) \leq d^{-D}  \, , \qquad \forall d \geq d_{\eps, D}\, .
\]
We denote by $X_1(d) \prec X_2(d)$ if $X_1(d)$ is stochastically dominated by $X_2(d)$. Moreover, if $| X_1(d) | \prec X_2(d)$, we also write $X_1(d) = O_{d,\prec} (X_2(d))$, or simply $X_1(d) = O_{\prec} (X_2(d))$ with $d$ clear from context.

\section{Main Results}
\label{sec:MainResults}

In this section, we present our results on the asymptotics of random feature ridge regression in the polynomial scaling. We begin in Section \ref{sec:FunctionSphere} by introducing some notations and reviewing some basic properties of the functional space over the sphere. We then describe our assumptions in Section \ref{sec:Assumptions_RFRR} and state our main theorem (Theorem \ref{thm:main_theorem_RF}) in Section \ref{sec:MainTheorem}. Finally, we discuss some key features of these asymptotics in Section \ref{sec:Interpretations}.

 \subsection{Functional space over the sphere}\label{sec:FunctionSphere}

 We start by introducing some notations and technical background relevant to our study. In this paper, we focus on the setting of data uniformly distributed on the sphere $\S^{d-1}(\sqrt{d})$ of radius $\sqrt{d}$ in $\R^d$. Let $\tau_d$ represent the uniform probability measure over $\S^{d-1}(\sqrt{d})$. Throughout, we will assume all functions to be elements of $L^2(\S^{d-1}(\sqrt{d})) := L^2(\S^{d-1}(\sqrt{d}), \tau_d)$, the space of square-integrable functions over $\S^{d-1}(\sqrt{d})$ with respect to $\tau_d$. We denote by $\< \cdot , \cdot \>_{L^2}$ the scalar product and $\| \cdot \|_{L^2}$ the norm in $L^2 ( \S^{d-1} (\sqrt{d}))$, where
\[
\< f,g\>_{L^2} := \int_{\S^{d-1} (\sqrt{d})} f(\bx) g(\bx) \tau_{d} (\de \bx) \, .
\]
We will write $\E_\bx$ the expectation over $\bx \sim \tau_d$.

Our results crucially depend on the following orthogonal decomposition of  $L^2(\S^{d-1}(\sqrt{d}))$ \cite{dai2013approximation}:
\[
L^2(\S^{d-1}(\sqrt{d})) = \bigoplus_{k=0}^{\infty} V_{d,k} \, ,
\]
where $V_{d,k}$ is the subspace of all degree-$k$ polynomials that are orthogonal (with respect to $\<\cdot , \cdot \>_{L^2}$) to all polynomials with degree less than $k$. For each $k \in \Z_{\geq 0}$, let $N_k := N_{d,K} = \dim(V_{d,k})$ be the dimension of the subspace $V_{d,k}$, and $\big\{Y_{ks}^{(d)}\big\}_{s \in [N_k]}$ be an orthonormal basis of $V_{d,k}$ of degree-$k$ spherical harmonics on $\S^{d-1} (\sqrt{d})$. We will denote by $\proj_k$ the orthogonal projection onto $V_{d,k}$ in $L^2(\S^{d-1} (\sqrt{d}))$, which can be written explicitly as
\[
\proj_k f(\bx) := \sum_{l=1}^{N_k} \< f , Y^{(d)}_{ks} \>_{L^2}Y^{(d)}_{ks} (\bx) \, .
\]
We further introduce $\proj_{\leq k} := \sum_{l=0}^k \proj_k$ and $\proj_{>k} := \id - \proj_{\leq k}$. We denote $N_{\leq k} = N_0  + \ldots + N_k$ the dimension of $\bigoplus_{l=0}^{k} V_{d,l}$, the subspace spanned by all polynomials of degree at most $k$ in $L^2(\S^{d-1}(\sqrt{d}))$.

We will also work with the one-dimensional functional spaces $L^2([-\sqrt{d},\sqrt{d}], \tau_{d,1})$ and $L^2(\R,\tau_{\text{g}})$. Here, $\tau_{d,1}$ denotes the marginal distribution of $x_1$ when $\bx \sim \tau_{d}$, and $\tau_{\text{g}}$ denotes the standard Gaussian measure, i.e., $\tau_{\text{g}} (\de x) = e^{-x^2/2} \de x / \sqrt{2\pi}$. The set of Gegenbauer polynomials $\big\{q_{k}^{(d)}\big\}_{k=0}^{\infty}$ and Hermite polynomials $\{ \He_{k}\}_{k=0}^{\infty}$ form orthonormal bases for $L^2([-\sqrt{d},\sqrt{d}], \tau_{d,1})$ and $L^2(\R,\tau_g)$ respectively. Note that since $\tau_{d,1}$ converges weakly to $\tau_{\text{g}}$, 
the coefficients of the polynomial $q_k^{(d)}$ converge to the coefficients of $\He_k$ as $d \to \infty$.

A brief review of some key properties of $\{Y_{ks}^{(d)}(\bx)\}_{s=1}^{N_k}$, $\{q_{k}^{(d)}\}_{k=0}^{\infty}$ and $\{ \He_{k}\}_{k=0}^{\infty}$ can be found in Appendix \ref{sec:Spherical-Harmonics}, and we refer to \cite{szeg1939orthogonal,chihara2011introduction,dai2013approximation} for a more complete exposition.
Note that $Y_{ks}^{(d)}$ and $q_{k}^{(d)}$ both depend on the dimension $d$.
For simplicity, we will drop the superscript $d$ when the dimension is clear from the context, and write $Y_{ks} := Y_{ks}^{(d)}$ and $q_{k} := q_{k}^{(d)}$.

 \subsection{Assumptions}\label{sec:Assumptions_RFRR}

Recall that we consider the high-dimensional polynomial scaling, where all $d$, $p:= p(d)$ and $n := n(d)$ diverge as $d \to \infty$ while staying polynomially related to each other. 
Specifically, we assume that there exist $\kappa_1,\kappa_2,\theta_1,\theta_2 >0$ such that
\begin{equation}\label{eq:def_theta_12}
\lim_{d \to \infty} \frac{p}{d^{\kappa_1}} = \theta_1\, , \qquad \lim_{d \to \infty} \frac{n}{d^{\kappa_2}} = \theta_2 \, .
\end{equation}
We introduce $\ell = \lceil \min(\kappa_1 , \kappa_2 ) \rceil$.
 
 We begin by describing our assumption on the activation function $\sigma$. 
 
 \begin{assumption}[Assumptions on $\sigma$ at level $\ell \in \naturals$]\label{ass:sigma} Let $\sigma: \R \to \R$ be an activation function. For all $k \in \Z_{\geq 0}$, define
 \[
 \mu_k := \mu_k (\sigma) = \E [\sigma (G) \He_k (G)]\, , \qquad \mu_{>k}^2 := \mu_{>k} (\sigma)^2 = \E [ \sigma (G)^2] - \sum_{j = 0}^k \mu_j^2\, ,
 \]
 where the expectation is with respect to $G \sim \normal (0,1)$ and $\He_k$ is the degree-$k$ normalized Hermite polynomial (i.e., $\E_G [ \He_k (G) \He_l (G) ] = \delta_{kl}$). The coefficient $\mu_k$ corresponds to the (normalized) $k$-th Hermite coefficient of $\sigma $ and $\mu_{>k}^2$ is the squared norm of the projection of $\sigma$ orthogonal to polynomials of degree at most $k$ in $L^2(\R , \tau_{\text{g}})$. 

 We assume that the following hold.
 \begin{itemize}
     \item[(a)] There exists a constant $C >0$ such that $| \sigma (x)| \leq C (1 + |x| )^C$ for all $x \in \R$.
     
     \item[(b)] For $k = 0, \ldots , \ell$, we have $\mu_k \neq 0$.
     
     \item[(c)]  We have $\mu_{>\ell}^2 > 0$, meaning that $\sigma $ is not a polynomial of degree $\ell$. 
 \end{itemize}
 \end{assumption}
 
 Let us comment on these conditions. Assumption \ref{ass:sigma}.(b) amounts to a universality condition: the non-linearity $\sigma$ can approximate any polynomials of degree at most $\ell$. Assumption \ref{ass:sigma}.(c) requires that the high-degree part of the non-linearity $\sigma$ is non-vanishing, and therefore induces an \textit{implicit regularization} (see discussions in \cite{ghorbani2021linearized,mei2022generalization}). While we state our assumption for a given level $\ell \in \naturals$, we note that $\sigma$ satisfies Assumption \ref{ass:sigma} at all level if $\mu_k \neq 0$ for all $k \in \Z_{\geq 0}$ (i.e., $\sigma$ is universal in $L^2 (\R,\tau_{\text{g}})$). In particular, this assumption will be satisfied by most commonly-used activations, such as sigmoid functions or shifted ReLus $\sigma (x) = (x - c)_+$ for $c \in \R \setminus \{0\}$\footnote{Note that the unshifted ReLu is not universal in this setting, as $\E [ \sigma(G) \He_k(G) ] = 0$ for $k \geq 3$ odd.}. 

 Recall that for each $d$, we consider labels $y_i = f_{*,d} (\bx_i) + \eps_i$ with $\E [ \eps_i | \bx ]=0$ and $f_{*,d} \in L^2 (\S^{d-1} (\sqrt{d}))$. We will assume the following on the sequence of target functions $\{ f_{*,d}  \}_{d \geq 1}$.

\begin{assumption}[Target function with random high-degree coefficients at level $\ell \in \naturals$]\label{ass:random_f_star}
Consider a sequence of target functions $\{ f_{*,d} \in  L^2 (\S^{d-1} (\sqrt{d})) \}_{d \geq 1}$. We assume that there exist constants $C, F_\ell, F_{>\ell}$, such that for each $d \geq 1$, and writing the decomposition of $f_{*,d}$ in the orthonormal basis of spherical harmonics
\[
f_{*,d} (\bx) = \sum_{k = 0}^{\ell-1} \sum_{s \in [N_k]} \beta^*_{d,ks} Y_{ks} (\bx) + \sum_{k \geq \ell} \sum_{s\in [N_k]} \tbeta_{d,ks} Y_{ks} (\bx) \, ,
\]
the following hold.
\begin{itemize}
    \item[(a)] The vector $\bbeta^*_d = (\beta^*_{d,ks})_{k \in [\ell - 1], s\in [N_k]} \in \R^{N_{\leq \ell -1}}$ is deterministic with $\| \proj_{<\ell} f_{*,d} \|_{L^2} = \| \bbeta^*_d \|_2 \leq C$.

    \item[(b)] The higher-degree coefficients $\tbbeta_d  = (\tbeta_{d,ks})_{k \geq \ell, s\in [N_k]}$ are zero-mean independent random variables with
    \[
\E \left[ \tbeta^2_{d,ks} \right] = \frac{F_{k}^2}{N_k}\, , \qquad \E \left[ \tbeta^4_{d,ks}\right] \leq C \frac{F_{k}^4}{N_k^2}\, ,
\]
such that the coefficients $\{ F_k\}_{k> \ell}$ satisfy 
\[
\E \left[ \| \proj_{>\ell} f_{*,d} \|_{L^2}^2 \right] = \sum_{k \geq \ell+1} F^2_k = F_{>\ell}^2 \, .
 \]
\end{itemize}
\end{assumption}

 Assumption \ref{ass:random_f_star} with random high-frequency coefficients was used in \cite{mei2022generalizationRF,misiakiewicz2022spectrum,xiao2022precise} and serves to simplify the derivation: it reduces the computation to controlling the trace of the resolvent, i.e., the Stieljes transform of the covariance of the high-frequency part of the features. We will show that the test and training error converges in $L^2$ to the asymptotic test and training risks over the randomness of the target functions. In particular, this implies that these asymptotic results hold for typical functions in this function class. On the other hand, \cite{mei2022generalization} showed in the case $\kappa_1 \neq \kappa_2$ with $\kappa_1, \kappa_2 \not\in \naturals$, that the convergence hold \textit{pointwise}, i.e., for any deterministic sequence $\{ f_{*,d} \}_{d \geq 1}$. It is an interesting open problem to relax Assumption \ref{ass:random_f_star} and show the convergence pointwise, which would require to compute deterministic equivalents for our random matrix functionals \cite{couillet2022random,cheng2022dimension}. 

We will denote $\E_{f_*}$ the expectation over the random coefficients $\tbbeta_d$ in Assumption \ref{ass:random_f_star}.

 \subsection{Statement of the main theorem}\label{sec:MainTheorem}

We begin by introducing a set of fixed points that will be used to state our asymptotics for RFRR. Recall that $p/d^{\kappa_1} \to \theta_1$ and $n/d^{\kappa_2} \to \theta_2$, and we defined $\ell =  \lceil \min (\kappa_1 , \kappa_2 ) \rceil$. Let $\psi_1, \psi_2 \in \R_{\geq 0} \cup \{ + \infty\}$ be given by
\begin{equation}\label{eq:def_psi_12}
  \psi_1 := \lim_{d\to\infty} \frac{p}{d^\ell/\ell!}\, , \qquad 
  \psi_2 := \lim_{d\to\infty} \frac{n}{d^\ell/\ell!} \, .
\end{equation}
For example, if $\kappa_1 = \kappa_2 = \ell$, then $\psi_1 = \theta_1 / \ell!$ and $\psi_2 = \theta_2 /\ell!$, and if $\kappa_1 < \ell < \kappa_2$, then $\psi_1 = 0$ and $\psi_2 = +\infty$. We further introduce the following normalized Hermite coefficients and regularization parameter:
 \begin{align}\label{eq:def_zeta_ell}
     \zeta := \zeta_\ell = \frac{\mu_{\ell}}{\mu_{>\ell}}\, , \qquad
     \barlambda := \barlambda_\ell = \frac{\lambda}{\mu_{>\ell}^2} \, .
 \end{align}

 We define two functions $\tau_1 , \tau_2 : \C_+ \to \C_+$ implicitly as the solutions of a system of two degree-$3$ polynomial equations that can be easily evaluated numerically.
 
 \begin{definition}[Fixed points at level $\ell \in \naturals$]\label{def:fixedPointsTau}  We define $\tau_1 , \tau_2 : \C^+ \to \C^+$ as follow.

 \begin{itemize}
     \item[\emph{(1)}] When $\kappa_1 = \kappa_2 = \ell$, $\tau_1 , \tau_2$ are the unique functions such that $(i)$ $\tau_1,\tau_2$ are analytic functions in the upper half-plane $\C_+$; $(ii)$ $\tau_1 (z) , \tau (z)$ are solutions of the system of polynomial equations:
 \begin{equation}\label{eq:FixedPoints}
 \begin{aligned}
 \zeta^2 \tau_1 \tau_2 (z \tau_1 - 1) + \frac{\psi_1}{\psi_2} \left[ \zeta^2 \tau_1 \tau_2 + (\tau_2 - \tau_1) \cdot \frac{1}{\psi_2} \right] =&~ 0 \, , \\
  \zeta^2 \tau_1 \tau_2 ( z \tau_1 - 1)  + (\tau_1 - \tau_2 ) ( \tau_1 + \zeta^2 \tau_2 ) \cdot \frac{1}{\psi_2} = &~ 0 \, ,
 \end{aligned}
 \end{equation}
 where $\psi_1,\psi_2$, and $\zeta$ are defined as per Eqs.~\eqref{eq:def_psi_12} and \eqref{eq:def_zeta_ell}. 

\item[\emph{(2)}] When $\kappa_1 = \kappa_2 < \ell$, we have the following explicit analytical formula with $\tau_2 = \tau_1$ and 
\begin{equation}\label{eq:FixedPoints_1}
\begin{aligned}
    \tau_1 (z) = \frac{1}{2 z} \left\{ \left( 1 - \gamma - \frac{\gamma z}{1+\zeta^2} \right) + \sqrt{ \left( 1 - \gamma - \frac{\gamma z}{1+\zeta^2}  \right)^2 + \frac{4\gamma z}{1+\zeta^2} } \right\} \, ,
\end{aligned}    
\end{equation}
where $\gamma = \theta_1/\theta_2$ with $\theta_1,\theta_2$ defined as per Eq.~\eqref{eq:def_theta_12}. 
 \end{itemize}
 \end{definition}

The fixed points in Definition \ref{def:fixedPointsTau}.(1) are the same as the ones already introduced in \cite{adlam2020neural} in the case $\ell = 1$. As explained in \cite{adlam2020neural}, $\tau_1$ and $\tau_2$ correspond to the asymptotic traces of two functionals of the data and weight matrices. Further, note that the fixed points \eqref{eq:FixedPoints_1} can be obtained as the limit of the fixed point \eqref{eq:FixedPoints} when $\psi_1, \psi_2 \to 0$, while $\psi_1/\psi_2 \to \gamma$. We defer to Section \ref{sec:Interpretations} and Appendix \ref{sec:explicit_formula} for a discussion on the interpretation of these fixed points.

Equipped with these definitions, we can now state the formula for the asymptotic test and training errors of RFRR. Recall that we get $n$ i.i.d.~samples $(\bx_i,y_i)_{i\in [n]}$ and $p$ random features $(\bw_j)_{j \in [p]}$, where $\bx_i \sim \Unif (\S^{d-1} (\sqrt{d}))$, $\bw_j \sim \Unif (\S^{d-1} (1))$, and $y_i = f_{*,d} (\bx_i) + \eps_i$, with independent label noise $\E[\eps_i] = 0$, $\E[ \eps_i^2] = \rho_\eps^2$ and $\E[\eps_i^4] < \infty$. Denote the covariate matrix $\bX = [ \bx_1 , \ldots , \bx_n]^\sT \in \R^{n \times d}$, the weight matrix $\bW = [\bw_1 , \ldots , \bw_p]^\sT \in \R^{p \times d}$ and the label noise vector $\beps = (\eps_1 , \ldots , \eps_n) \in \R^n$. The test/training errors and the (normalized) squared $\ell_2$ norm of $\hba_\lambda$ are given respectively by
\begin{align}
R_{\test} ( f_* ; \bX , \bW , \beps , \lambda ) =&~ \E_{\bx} \Big[ \big( f_* (\bx) -  h_{\RF} (\bx ; \hat \ba_{\lambda} ) \big)^2 \Big] \, ,\\
R_{\train} (f_* ; \bX , \bW , \beps, \lambda) =&~ \frac{1}{n} \sum_{i \in [n]} (y_i - h_{\RF} ( \bx_i ; \hat \ba_{\lambda}  ) )^2 \, 
\end{align}
and
\begin{align}
    L_{\lnorm} (f_* ; \bX , \bW , \beps, \lambda) =&~ 
    \begin{cases}
        \frac{ \| \hba_\lambda \|_2^2 }{n}, ~~~\text{when } n = \cO_d(p),  \\
        \frac{ \| \hba_\lambda \|_2^2 }{p}, ~~~\text{when } p = o_d(n).
    \end{cases}
\end{align}

We show that the test/training errors and the (normalized) squared norm of $\hba_\lambda$ converge in $L^1$ (and therefore in probability) over the randomness over $\bX,\bW,\beps, f_*$, to
\begin{align}\label{eq:asymptotic_formula}
      R_{\test} ( f_{*,d} ; \bX, \bW , \beps, \lambda) =&~ \sfR_{\test} + o_{d,\P} (1) \, , \qquad 
      R_{\train} (f_* ; \bX , \bW , \beps, \lambda) =  \sfR_{\train} + o_{d,\P}(1) \, ,
\end{align}
and
\begin{align}
    L_{\lnorm} (f_* ; \bX , \bW , \beps, \lambda) = \sfL_{\lnorm} + o_{d,\P}(1),
\end{align}
where $\sfR_{\test}$, $\sfR_{\train}$ and $\sfL_{\lnorm}$ are defined below.

\begin{definition}[Asymptotic formulas for  RFRR]\label{def:asymptotic_formula_RFRR} Recall that we denote $\ell =  \lceil \min (\kappa_1 , \kappa_2 ) \rceil$. The asymptotic test/training errors and (normalized) squared  
 $\ell_2$ norm of $\hba_\lambda$ are given by
    \begin{align}
        \sfR_{\test} = &~ \left( F_\ell^2 \cdot \cB_{\test} +  F_{>\ell}^2 \right) + \left(F_{>\ell}^2 + \rho_\eps^2 \right) \cdot \cV_{\test}  \, , \\
        \sfR_{\train} =&~  \alpha_c \left\{ \left( F_\ell^2 \cdot \cB_{\test} +  F_{>\ell}^2 \right) + \left(F_{>\ell}^2 + \rho_\eps^2 \right) \cdot \cV_{\test} + \rho_\eps^2 \right\} \, ,
    \end{align}
    and
    \begin{align}
        \sfL_{\lnorm} = F_\ell^2 \cdot \cB_{\lnorm} + (F_{>\ell}^2 + \rho_\varepsilon^2) \cdot \cV_{\lnorm}, 
    \end{align}
    where $(\cB_{\test}, \cV_{\test}, \alpha_c)$ and $(\cB_{\lnorm}, \cV_{\lnorm} )$ are defined as follow.
\begin{itemize}
    \item[\emph{(1)}] \underline{Critical regime $\kappa_1 = \kappa_2$:} let $\tau_1,\tau_2$ be the fixed points defined in Definition \ref{def:fixedPointsTau}. Then
     \begin{equation}
    \label{eq:B_V_test_1}
    \cB_{\test} = -\frac{\tau_2 ' (\barlambda)}{\tau_1^2 (\barlambda)}\, , \qquad 
     \cV_{\test} = -\frac{\tau_1 ' (\barlambda)}{\tau_1^2 (\barlambda)} - 1\, , \qquad 
     \alpha_c = \barlambda^2 \tau_1^2 (\barlambda)\, ,
 \end{equation}
 and
 \begin{equation}
 \label{eq:B_V_norm_1}
     \cB_{\lnorm} = \frac{\tau_2(\barlambda) + \barlambda \tau_2'(\barlambda)}{\mu_{>\ell}^2}\, , \qquad 
     \cV_{\lnorm} = \frac{\tau_1(\barlambda) + \barlambda \tau_1'(\barlambda)}{\mu_{>\ell}^2},
 \end{equation}
 where $\barlambda$ is defined per Eq.~\eqref{eq:def_zeta_ell}. If $\kappa_1 = \kappa_2 = \ell$, then $\cB_{\test}$, $\cV_{\test}$, and $\alpha_c$ only depend on $(\psi_1,\psi_2,\zeta,\barlambda)$. If $\kappa_1 = \kappa_2 < \ell$, then $\cB_{\test}$, $\cV_{\test}$, and $\alpha_c$ only depend on $(\theta_1/\theta_2,\zeta,\barlambda)$.

 \item[\emph{(2)}] \underline{Overparametrized regime $\kappa_1 > \kappa_2$:} let $\eta :=  ( \barlambda + 1) / \zeta^2$ and define
 \begin{equation}
 \label{eq:defchi}
     \vartheta :=~ 
   \begin{cases}
   \frac{[(\psi_2 + \eta +1)^2 - 4\psi_2]^{1/2} + \psi_2 - 1 - \eta}{2 \eta \psi_2} \, , &\text{if }\;\kappa_2 = \ell\, ,\\
   \frac{1}{1+\eta}\, , &\text{if }\;\kappa_2 < \ell \, .
   \end{cases}
 \end{equation}
Then
 \begin{equation}
      \label{eq:B_V_test_2}
     \begin{aligned}
      \cB_{\test} = &~  \frac{\psi_2^2 \eta^2 \vartheta^3 + (\psi_2 \eta^2 +\psi_2 \eta  - \psi_2^2 \eta) \vartheta^2 + 1 - \psi_2}{\eta \psi_2 \vartheta^2 +1}\, , \\ 
      \cV_{\test} = &~ \frac{\psi_2 \vartheta - \psi_2 \eta \vartheta^2}{\eta \psi_2 \vartheta^2 +1}\, , \\
      \alpha_c = &~ \frac{\barlambda^2 \vartheta^2}{\zeta^4}\, ,
     \end{aligned}
     \end{equation}
     and
     \begin{equation}
     \label{eq:B_V_norm_2}
     \begin{aligned}
         \cB_{\lnorm} = \frac{1}{\mu_{>\ell}^2} \Big(\frac{1 - \eta\vartheta}{ \zeta^2 } - \frac{\barlambda \cB_{\test} \vartheta^2}{ \zeta^4 }\Big) \, ,\qquad
         \cV_{\lnorm} = \frac{1}{\mu_{>\ell}^2} \Big( \frac{\vartheta}{\zeta^2} - \frac{\barlambda (\cV_{\test} + 1)\vartheta^2}{\zeta^4} \Big).
     \end{aligned}    
     \end{equation}
     If $\kappa_2 = \ell$, then $\cB_{\test}$, $\cV_{\test}$, and $\alpha_c$ only depend on $(\psi_2, \zeta , \barlambda)$. If $\kappa_2 < \ell$, then  $\cB_{\test} =1 $, $\cV_{\test}= 0$, and $\alpha_c$ only depends on $( \zeta , \barlambda)$.

  \item[\emph{(3)}]  \underline{Underparametrized regime $\kappa_1 < \kappa_2$:} in this regime, 
  \begin{equation}
  \label{eq:B_V_test_3}
 \begin{aligned}
\cB_{\test} = &~ \begin{cases}
\frac{1}{2} \Big[ 1 - \psi_1 - \zeta^{-2} + \sqrt{(1 + \psi_1 + \zeta^{-2})^2 - 4\psi_1} \Big] & \text{if }\;\kappa_1 = \ell \, , \\
1& \text{if }\;\kappa_1 < \ell\, ,
\end{cases} \\
\cV_{\test} = &~ 0 \, , \\
\alpha_c = &~ 1\, ,
\end{aligned}
\end{equation}
and
\begin{equation}
\label{eq:B_V_norm_3}
\begin{aligned}
 \cB_{\lnorm} = \frac{1}{\mu_{>\ell}^2} \cdot \frac{ \vartheta \zeta^2 (1-\psi_1)\psi_1 + \vartheta^2 \zeta^2 \psi_1^2 }{ 1 + \zeta^2(1-\psi_1+2\psi_1\vartheta) } \, ,\qquad
 \cV_{\lnorm} = 0.    
\end{aligned}    
\end{equation}
If $\kappa_1 = \ell$, then $\cB_{\test}$ only depends on $(\psi_1,\zeta)$. Otherwise $\cB_{\test}$, $\cV_{\test}$, and $\alpha_c$ are constants independent of the parameters of the problem.
\end{itemize}
\end{definition}

For convenience, Definition \ref{def:asymptotic_formula_RFRR} includes the explicit formulas in the overparametrized and underparametrized regimes. However, these equations can be unified using only Equation~\eqref{eq:B_V_test_1}, where the functions $\tau_1,\tau_2$ are now obtained as the limiting solutions of the fixed point equations \eqref{eq:FixedPoints} where we replace $\psi_1$ and $\psi_2$ by $p/(d^\ell \ell!)$ and $n/(d^\ell \ell!)$ before taking $p,n,d \to \infty$. Additional details can be found in Appendix \ref{sec:specialregime}.

We are now ready to state our main theorem.
\begin{theorem}[RFRR asymptotics]\label{thm:main_theorem_RF} 
 Assume $(p(d),n(d))_{d \geq 1}$ are two sequences of integers such that $p/d^{\kappa_1} \to \gratio_1$ and $n / d^{\kappa_2} \to \gratio_2$ for some $\kappa_1,\kappa_2,\theta_1,\theta_2 \in \R_{>0}$, and denote $\ell = \lceil \min (\kappa_1 , \kappa_2) \rceil$. Let $\{ f_{*,d} \in L^2 (\S^{d-1} (\sqrt{d}))\}_{d \geq 1}$ be a sequence of functions that satisfy Assumption \ref{ass:random_f_star} at level $\ell$, and $\sigma :\R \to \R$ be an activation function satisfying Assumption \ref{ass:sigma} at level $\ell$.
 
 Let $\bX = [ \bx_1 , \ldots , \bx_n ]^\sT \in \R^{n \times d}$ and $\bW = [\bw_1 , \ldots , \bw_p ]^\sT \in \R^{p \times d}$ with $(\bx_i)_{i \in [n]} \sim_{iid} \Unif (\S^{d-1} (\sqrt{d} ) ) $ and $(\bw_j)_{j \in [p]} \sim_{iid} \Unif (\S^{d-1})$ independently. Let $y_i = f_{*,d} (\bx_i) + \eps_i$ with $\eps_i$ independent noise with $\E [ \eps_i] = 0$, $\E [ \eps_i^2] = \rho_\eps^2$ and $\E [ \eps_i^4 ] < \infty$. 
 
 Then for any regularization $\lambda >0$, the asymptotic test/training errors and the (normalized) squared $\ell_2$ norm of the minimizer of random feature ridge regression (RFRR) satisfy
 \begin{align}
    \label{eq:asymp_testerr}
     \E_{\bX,\bW,\beps,f_*} \Big\vert R_{\test} ( f_{*,d} ; \bX , \bW , \beps, \lambda) - \Big[( F_\ell^2 \cdot \cB_{\test} + F_{>\ell}^2) + (F_{>\ell}^2 + \rho_\eps^2 ) \cdot \cV_{\test} \Big]  \Big\vert =&~ o_d(1) \, ,\\
 \label{eq:asymp_trainerr}
  \E_{\bX,\bW,\beps,f_*} \Big\vert R_{\train} ( f_{*,d} ; \bX , \bW , \beps, \lambda) - \alpha_c \Big[ (F_\ell^2 \cdot \cB_{\test} + F_{>\ell}^2) + (F_{>\ell}^2 + \rho_\eps^2 ) \cdot \cV_{\test} +  \rho_\eps^2 \Big]  \Big\vert =&~ o_d(1) \, ,   \;\;\;\;\;
 \end{align}
 and
 \begin{align}
     \label{eq:asymp_anorm}
     \E_{\bX,\bW,\beps,f_*} \Big\vert  L_{\lnorm} (f_* ; \bX , \bW , \beps, \lambda) - \Big[ F_\ell^2 \cdot \cB_{\lnorm} + (F_{>\ell}^2 + \rho_\varepsilon^2) \cdot \cV_{\lnorm} \Big] \Big\vert = o_d(1),
 \end{align}
 where $(\cB_{\test}, \cV_{\test}, \alpha_c)$ and $(\cB_{\lnorm}, \cV_{\lnorm})$ are defined as in [\eqref{eq:B_V_test_1}, \eqref{eq:B_V_test_2}, \eqref{eq:B_V_test_3}] and [\eqref{eq:B_V_norm_1}, \eqref{eq:B_V_norm_2}, \eqref{eq:B_V_norm_3}], in the case of (1) $\kappa_1 = \kappa_2$, (2) $\kappa_1 > \kappa_2$, and (3) $\kappa_1 < \kappa_2$ respectively.

 \end{theorem}
 \begin{figure}[t]
\begin{tikzpicture} 
\node[inner sep=0pt] (russell) at (0,-1)
    {\includegraphics[width=.36\textwidth]{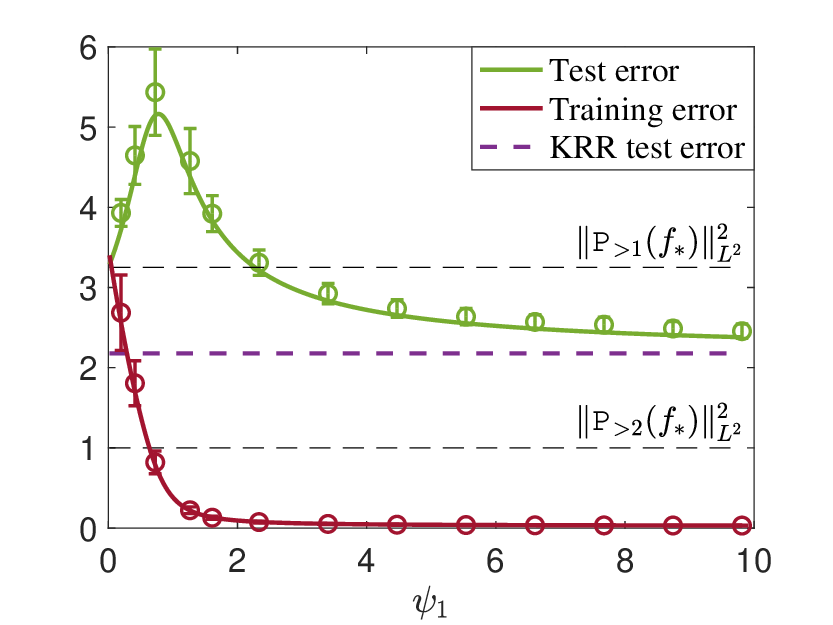}};
\node[inner sep=0pt] (russell) at (5.5,-1)
    {\includegraphics[width=.36\textwidth]{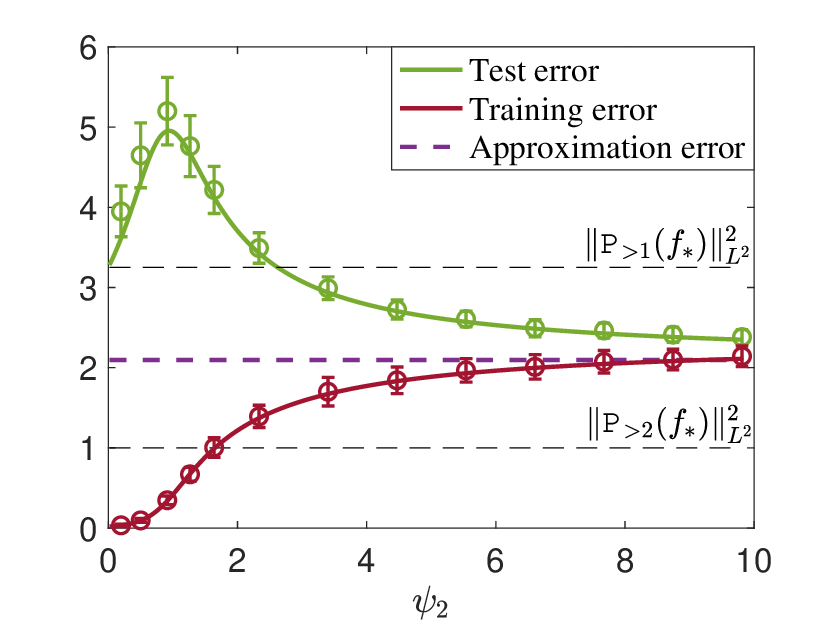}};
\node[inner sep=0pt] (russell) at (11.0,-1)
    {\includegraphics[width=.36\textwidth]{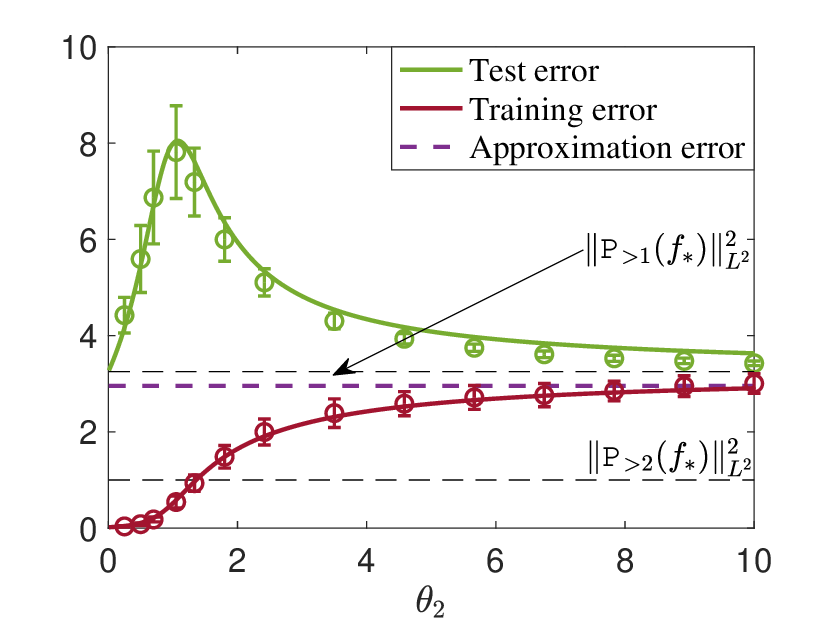}}; 
\end{tikzpicture}
\caption{Test error and training error of RFRR in the critical regime $\kappa_1 = \kappa_2$. We choose the target function to be $f_{*,d}(\bx) = 0.5 \bbeta^\sT\bx + 1.5(\bbeta^\sT\bx)^2 + (\bbeta^\sT\bx)^3$ with $\|\bbeta\|_2=1$, and the activation function $\sigma(x) = 1.5x + 3x^2 + 2x^3$. We set $\lambda = 1.0$ and $\rho_\eps^2 = 0.25$. The solid lines correspond to the analytical predictions for the test and training errors obtained in Theorem \ref{thm:main_theorem_RF}, the purple dashed line to the analytical predictions for the KRR test error, and the grey dashed lines to the values of the projections $\| \proj_{>1} f_* \|_{L^2}^2$ and $\| \proj_{>2} f_* \|_{L^2}^2$.  The dots are the empirical results with $d = 50$, and the mean and error bars are computed over 100 independent runs. \textit{On the left:} we set $\kappa_1 = \kappa_2 = 2$ and $\psi_2= 2n/d^2 =1$ and plot the errors versus $\psi_1 = 2p/d^2$. \textit{In the middle:} we set $\kappa_1 = \kappa_2 = 2$ and $\psi_1 = 2p/d^2=1$, and plot the errors versus $\psi_2 = 2n/d^2$. \textit{On the right:} we set $\kappa_1 = \kappa_2 = 1.5$ and $\theta_1 = p/d^{1.5} =  1$, and plot the errors versus $\theta_2 = n / d^{1.5}$. }

\label{fig:test_train_err}
\end{figure}
\begin{figure}[t]
\begin{tikzpicture} 
\node[inner sep=0pt] (russell) at (0,0)
    {\includegraphics[width=.5\textwidth]{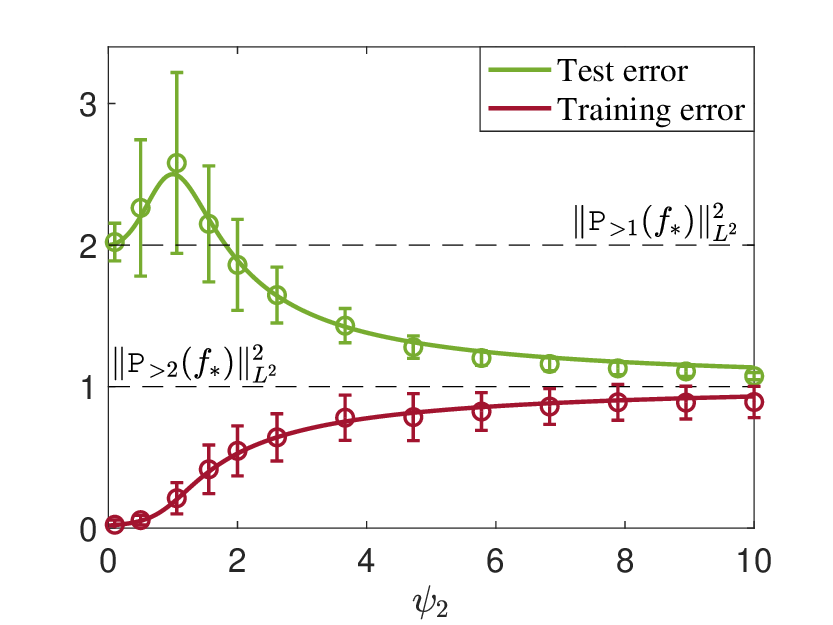}};
\node[inner sep=0pt] (russell) at (8,0)
    {\includegraphics[width=.5\textwidth]{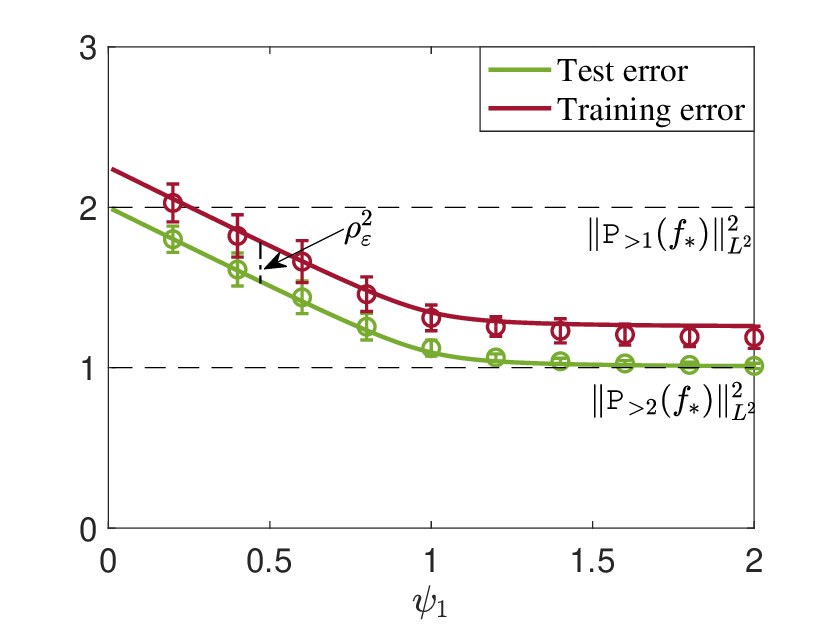}};  
\end{tikzpicture}
\caption{Test error and training error of RFRR in the overparametrized regime $\kappa_1 > \kappa_2$ (left) and underparametrized regime $\kappa_1 < \kappa_2$ (right). We choose the target function to be $f_{*,d}(\bx) = \bbeta^\sT\bx + (\bbeta^\sT\bx)^2$, $\| \bbeta \|_2 = 1$, and the activation function $\sigma(x) = x + 0.1 x^2$. The solid lines correspond to the analytical predictions for the test and training errors obtained in Theorem \ref{thm:main_theorem_RF}, and the grey dashed lines to the values of the projections $\| \proj_{>0} f_* \|_{L^2}^2$ and $\| \proj_{>1} f_* \|_{L^2}^2$. The dots are the empirical results with $d = 50$, and the mean and error bars are computed over 100 independent runs. \textit{On the left:} we set $\kappa_1 = 2$, $\kappa_2 = 1$, and $\psi_1 = 2$, and plot the errors versus $\psi_2 = n/d$. \textit{On the right:} we set $\kappa_1 = 1$, $\kappa_2 = 2$, and $\psi_2 = 2$, and plot the errors versus $\psi_1 = p/d$.  }
\label{fig:test_train_err_2}
\end{figure}

The proof of this theorem can be found in Section \ref{sec:ProofMainResults}, with some of the more technical details deferred to the appendices. In the following, we comment on two key features of Theorem \ref{thm:main_theorem_RF} and defer a longer discussion to Section \ref{sec:Interpretations}. 
\begin{description}
    \item[Staircase decay:] From Eq.~\eqref{eq:asymp_testerr} and Definition \ref{def:asymptotic_formula_RFRR},  we see that for $\ell =  \lceil \min (\kappa_1 , \kappa_2) \rceil$, RFRR fits completely $\proj_{\leq \ell -1} f_*$, the degree-$(\ell-1)$ polynomial approximation to $f_*$, and none of its higher degree part $\proj_{>\ell} f_*$. Furthermore, if $\min ( p /d^\ell, n/d^\ell) \to \infty$, then RFRR fits completely $\proj_\ell f_*$, whereas if $\min ( p /d^\ell, n/d^\ell) \to 0$, RFRR does not fit $\proj_{\ell} f_*$ at all. Thus, as $n,p \to \infty$, RFRR incrementally fits polynomial approximations of $f_*$ of increasing degree, with $\proj_\ell f_*$ learnt as soon as $n/d^\ell$ and $p/d^\ell \to \infty$.

    \item[Invariance of asymptotics:] The form of the asymptotics is the same for all polynomial scaling $\ell \in \naturals$. In particular, they only depend on the activation function through $\mu_\ell^2$ and $\mu_{>\ell}^2$. For example, for $\kappa_1 = \kappa_2 = \ell$, the training and test errors only depend on $\psi_1 , \psi_2$, $\zeta^2 = \mu_\ell^2/\mu_{>\ell}^2$, $\barlambda = \lambda / \mu_{>\ell}^2$, $\| \proj_{\ell} f_*\|_{L^2}$, and $\| \proj_{>\ell} f_* \|_{L^2}$, and the asymptotics take the same form as the ones in \cite{mei2022generalizationRF} which only considered the case $\ell = 1$.
\end{description}

We illustrate Theorem \ref{thm:main_theorem_RF} in  \fref{test_train_err} and \fref{test_train_err_2} where we compare numerically the training and test errors with their asymptotic predictions in 4 different regimes: (1) $\kappa_1 = \kappa_2 = \ell$, (2) $\kappa_1 = \kappa_2 < \ell$, (3) $\kappa_1 > \kappa_2$, (4) $\kappa_1 < \kappa_2$.  
 \fref{test_train_err} considers the critical regimes $\kappa_1 = \kappa_2 \in \{1.5, 2\}$. In this regime, the test error exhibits a non-monotonic behavior with a peak at the interpolation threshold $n/p = \theta_2/\theta_1 = 1$. This corresponds to the double-descent phenomenon, which was previously characterized by \cite{mei2022generalizationRF} in the linear scaling $\kappa_1 = \kappa_2 = 1$. Furthermore, when the sample size is fixed and the number of parameters changes (left plot), the test error goes from $\| \proj_{>\ell-1} f_* \|_{L^2}^2$ as $\psi_1 \to 0$, to the test error of KRR as $\psi_1 \to \infty$. On the other hand, when the number of parameters is fixed and the sample size changes (middle plot), the test error goes from $\| \proj_{>\ell-1} f_* \|_{L^2}^2$ as $\psi_2 \to 0$, to the approximation error as $\psi_2 \to \infty$. Finally when $\kappa_1 = \kappa_2 < \ell$ (right plot), the KRR test error and approximation error are both equal to $\| \proj_{>\ell-1} f_* \|_{L^2}^2$ and thus $R_{\test} \to \| \proj_{>\ell-1} f_* \|_{L^2}^2$ under both limits $\theta_2 \to 0$ and $\theta_2 \to \infty$.
 \fref{test_train_err_2} considers the overparametrized regime $\kappa_1 > \kappa_2$ (left plot) and underparametrized regime $\kappa_1 < \kappa_2$ (right plot). We see that, in the overparametrized regime $\kappa_1 > \kappa_2$, the errors are the same as those of KRR and when $\kappa_1 < \kappa_2$, the errors converge to the approximation error of the RF model class.

 \subsection{Discussion}
\label{sec:Interpretations}

In this section, we discuss Theorem \ref{thm:main_theorem_RF} and provide additional intuitions on the form of the asymptotics in Definition \ref{def:asymptotic_formula_RFRR}. 
 
\paragraph*{Bias-variance decomposition and staircase decay.} 
The asymptotic test error can be decomposed into the classical bias and variance terms with respect to the label noise $(\eps_i)_{i \in [n]}$, as $\sfR_{\test} = \sfB_{\test} + \sfV_{\test}$ where 
\begin{align}\label{eq:Bias_Var_decompo_asymptotics}
    \sfB_{\test} &= F_\ell^2 \cdot \cB_{\test} + F_{>\ell}^2 \cdot (1 + \cV_{\test})\, , \qquad \sfV_{\test} = \rho_\epsilon^2 \cdot \cV_{\test} \, .
\end{align}
The high-frequency part of the target function $\| \proj_{>\ell} f_* \|_{L^2}^2 = F_{>\ell}^2$ effectively plays the role of and additive noise in this high-dimensional setting. In other words, $\proj_{>\ell} f_* (\bx_i) + \eps_i$ behaves as an effective additive noise to the labels $\proj_{\leq \ell} f_* (\bx_i)$. For $\kappa_1 \neq \kappa_2$, the bias term $\cB_{\test}$ is monotonically nonincreasing in $p$ and $n$, while it presents a peak at the interpolation threshold $n = p$ for $\lambda$ small enough due to the conditioning number of the feature matrix $\bZ = (\sigma (\< \bx_i , \bw_j\>)_{i \in [n], j \in [p]} $ diverging. Hence, RFRR can display a double descent at the interpolation threshold even when $F_{<\ell} = \rho_\eps = 0$. The variance term $\cV_{\test}$ present peaks not only at $n = p$, but also at $n = (1 + o_d(1) ) d^\ell/ \ell!$ and $\kappa_1\geq \kappa_2$ for $\barlambda$ and $1/\zeta$ large enough. These additional peaks are due to the degeneracy of the singular values of $\sigma$ associated to degree-$\ell$ spherical harmonics, and will appear in the test error only if $F_{>\ell}$ or $\rho_\eps \neq 0$.

It is instructive to decompose the test error into the contributions for estimating each frequency of the target function
\[
\sfR_{\test} = \sum_{k =0}^\infty \sfR_{\test,k} = \sum_{k =0}^\infty \sfB_{\test,k} + \sfV_{\test,k}\, ,
\]
where $\sfR_{\test,k} $ is the asymptotic of $\| \proj_k (f_* - h_{\RF} (\cdot ; \hat \ba_{\lambda} ) ) \|_{L^2}^2$ (the contribution of subspace $V_{d,k}$ to the test error), and $\sfB_{\test,k}$ and $ \sfV_{\test,k}$ the associated bias and variance terms. We have
\begin{align}
    \sfB_{\test,k} = 
    \begin{cases}
        0\, ,\\
        F_\ell^2 \cdot \cB_{\test} + F_{>\ell}^2  \cdot \cV_{\test}\, ,\\
        F_k^2\, ,\\
    \end{cases} \qquad     \sfV_{\test,k} = 
    \begin{cases}
        0\, ,& k < \ell \, , \\
        \rho_\varepsilon^2 \cdot \cV_{\test}\, ,& k = \ell\, , \\
        0\, ,& k > \ell \, .
    \end{cases}
\end{align}
This decomposition offers a particularly simple explanation for the staircase decay of the test error. For $k < \ell$, the signal on these subspaces are fitted perfectly with $\sfB_{\test,k} = \sfV_{\test,k} = 0$, while for $k >\ell$, it is not fitted at all with $\sfB_{\test,k} = \| \proj_k f_* \|_{L^2}^2$ and $\sfV_{\test,k} =0$. A richer phenomenology happens on the subspace of critical degree $\ell$, where the error can exhibit non-monotone behavior. We illustrate this discussion in \fref{bias-var-psi1change} and \fref{bias-var-ratiofixed}, where we plot the analytical formula for the bias, variance and test errors at different scaling.

 \begin{figure}[t!]
\begin{tikzpicture} 
\node[inner sep=0pt] (russell) at (0,0.5)
    {\includegraphics[width=.35\textwidth]{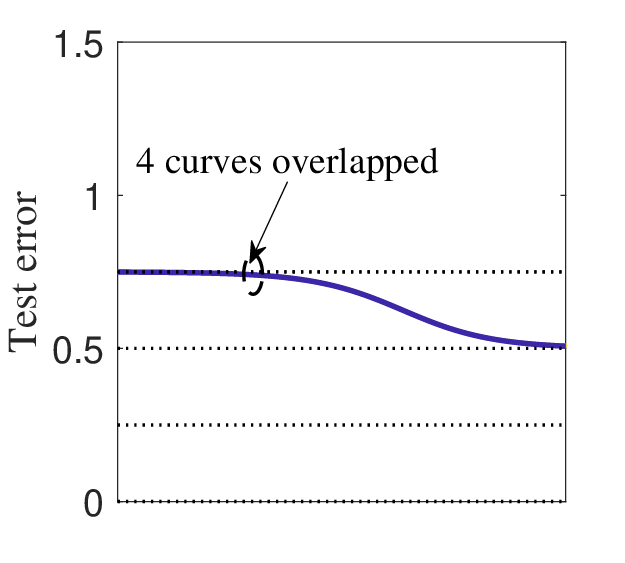}};
\node[inner sep=0pt] (russell) at (5.3,0.5)
    {\includegraphics[width=.35\textwidth]{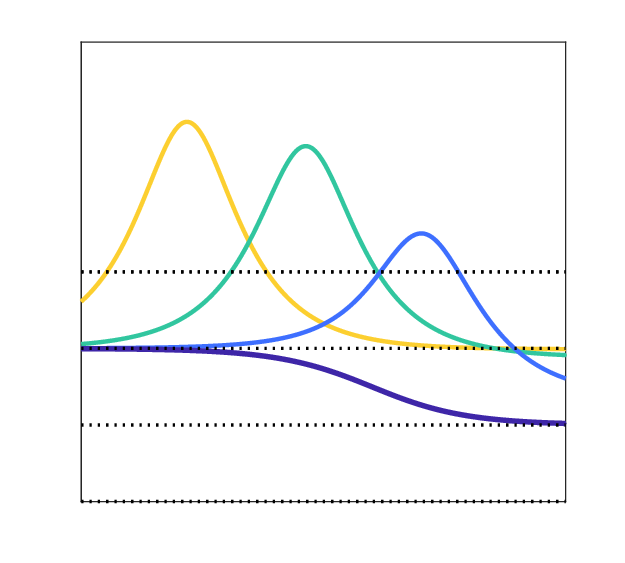}}; 
\node[inner sep=0pt] (russell) at (10.6,0.5)
    {\includegraphics[width=.35\textwidth]{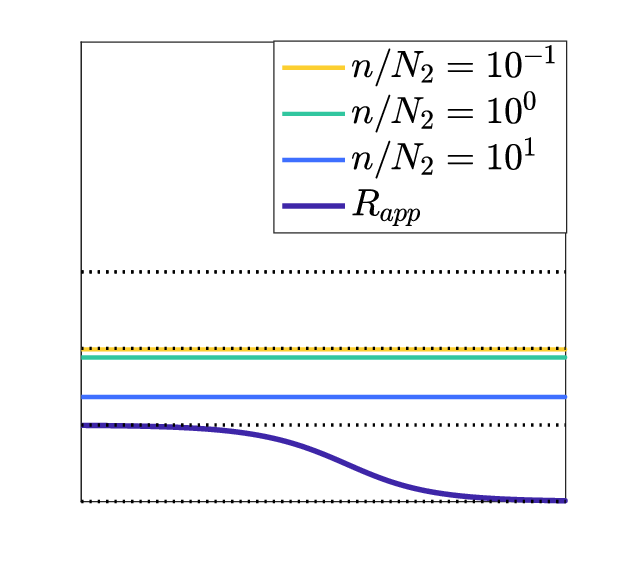}};  
\node[inner sep=0pt] (russell) at (0,-4.6)
    {\includegraphics[width=.35\textwidth]{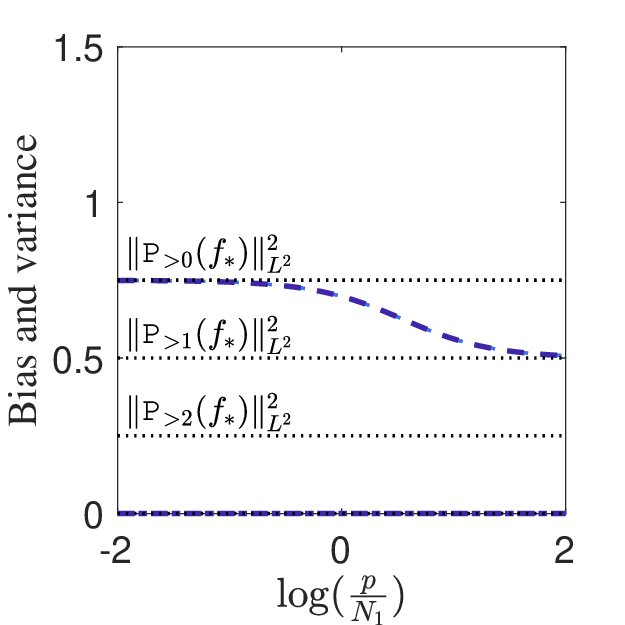}};
\node[inner sep=0pt] (russell) at (5.3,-4.6)
    {\includegraphics[width=.35\textwidth]{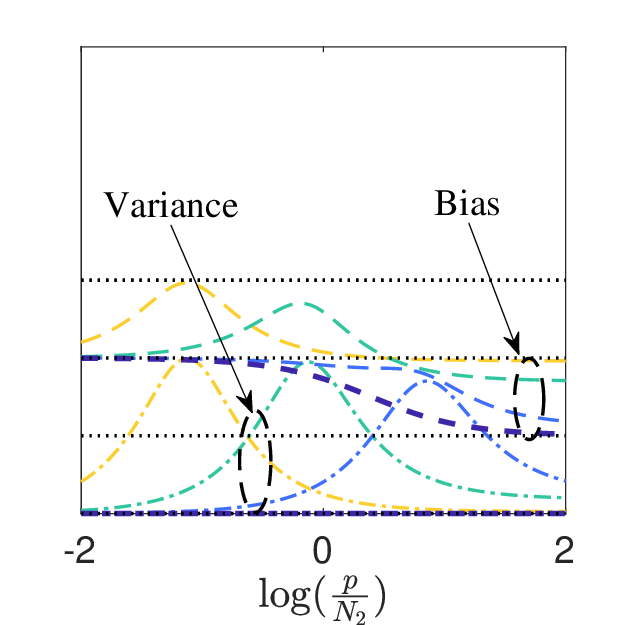}}; 
\node[inner sep=0pt] (russell) at (10.6,-4.6)
    {\includegraphics[width=.35\textwidth]{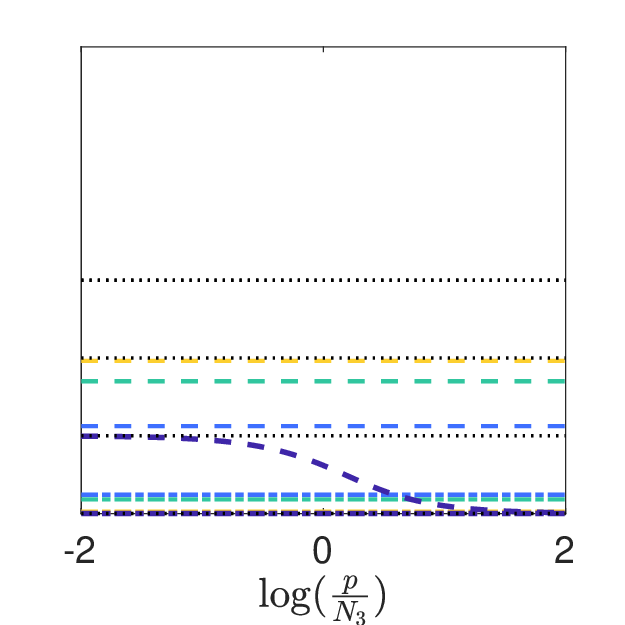}}; 
\end{tikzpicture}
\caption{Illustration of the bias (dash curves) and variance (dash-dot curves) decomposition and the incremental learning process of RFRR. We plot the analytical asymptotic predictions from Definition \ref{def:asymptotic_formula_RFRR} versus $p$, while $n$ is kept fixed with $\kappa_2 = 2$ and three values of $\psi_2= 2n/d^2 \in \{0.1, 1, 10\}$.  We set $\lambda = \rho_\varepsilon^2 =1$, $\sigma(x)=x+x^2+x^3+x^4$ and $f_*(\bx) = 0.5 \bbeta^\sT \bx + 0.5 (\bbeta^\sT \bx)^2 + 0.5 (\bbeta^\sT \bx)^3 $. The dotted lines correspond to  the squared norm of each of the frequencies of $f_*$. Recall that $R_{\App}$ denotes the approximation error. }
\label{fig:bias-var-psi1change}
\end{figure}

 \paragraph*{Optimal parametrization and regularization parameter.} From these asymptotics, we see that for $p/n \to \infty$, RFRR test error concentrates on the KRR test error, while for $n/p \to \infty$, it concentrates on the approximation error. Hence, in general, the optimal test error will be achieved for $p/n$ sufficiently large (matching the test error of KRR), and larger overparametrization scaling (taking $\kappa_1 > \kappa_2$) will not improve the test error. However, we can construct cases where the optimal test error is sometimes achieved in the underparametrized regime $p \ll n$. For example, in under-regularized case where KRR presents a peak at $n = (1 + o_d(1))d^\ell / \ell!$, then the test error of KRR is bigger than $\|\proj_{\geq \ell} f_* \|_{L^2}^2$ while taking  $p/d^{\ell - 1} \to \theta_1$ with $\theta_1$ sufficiently large will achieve the approximation error $\|\proj_{\geq \ell} f_* \|_{L^2}^2$. For convenience, we provide a brief overview on the asymptotics of KRR and approximation errors of RFRR in Appendix \ref{sec:kernel_approx}.

 At the interpolation threshold $n=p$, the peak diverges as $\lambda \to 0^+$, while taking an optimally-tuned regularization eliminates the double descent: with this choice, the test error becomes monotone decreasing in $p$ for fixed $n$ (and in $n$ for fixed $p$). On the other hand, $\lambda \to 0^+$ is often optimal (i.e., achieved lowest test error) away from the interpolation threshold thanks to the additive self-induced regularization $\mu^2_{>\ell}$ coming from the high frequency-part of the activation function. Intuitively, we have with high probability $\bZ \bZ^\sT \succeq \mu_{>\ell}^2 \id /2 $ (overparametrized regime) and $(p/n)\cdot \bZ^\sT \bZ \succeq \mu_{>\ell}^2 \id /2 $ (underparametrized regime). We refer the reader to \cite{mei2022generalization,mei2022generalizationRF,xiao2022precise} for additional discussions on the self-induced regularization and the optimality of interpolation.

 \paragraph*{Generalized cross validation.} Consider the prediction error with label noise 
 \[
 R^P(h_{\RF}) = \E_{\bx} [ (y - h_{\RF} (\bx ; \hat \ba_{\lambda} ) )^2 ] =   R ( f_* ; \bX , \bW , \beps , \lambda ) + \rho_\eps^2 \, ,
 \]
 which converges to $\sfR^P_\test := \sfR_\test + \rho_\eps^2$. From Theorem \ref{thm:main_theorem_RF}, we see that the asymptotic prediction error is proportional to the asymptotic training error, with $ \sfR^P_\test = \sfR_\train / \alpha_c$. From our proofs, we can check that (see Appendix \ref{sec:alternate} for an explanation)
 \[
 \left( \frac{\lambda}{n} \Tr\left[ ( \bZ \bZ^\sT + \lambda \id_n )^{-1} \right] \right)^2 \to \alpha_c  \, ,
 \]
where the convergence holds in probability. We deduce that
 \[
 \frac{R_{\train} (f_* ; \bX , \bW , \beps, \lambda)}{ \left( \frac{\lambda}{n} \Tr\left[ ( \bZ \bZ^\sT + \lambda \id_n )^{-1} \right] \right)^2 } =   R^P (h_{\RF}) + o_{d,\P}(1) \, .
 \]
 Hence the estimator on the left-hand side converges in probability to the prediction error in the high-dimensional polynomial scaling. This estimator is exactly the Generalized Cross-Validation (GCV) estimator that was introduced in \cite{craven1978smoothing,golub1979generalized}, and our results directly imply the (weak) consistency of the GCV estimator in the polynomial scaling. This was already noted in \cite{adlam2020neural} in the linear scaling.

 \begin{figure}[t]
\begin{tikzpicture} 
\node[inner sep=0pt] (russell) at (0,0.5)
    {\includegraphics[width=.35\textwidth]{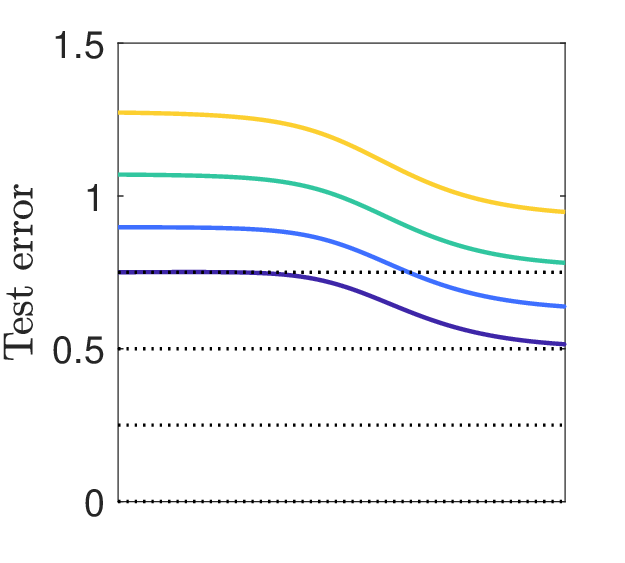}};
\node[inner sep=0pt] (russell) at (5.3,0.5)
    {\includegraphics[width=.35\textwidth]{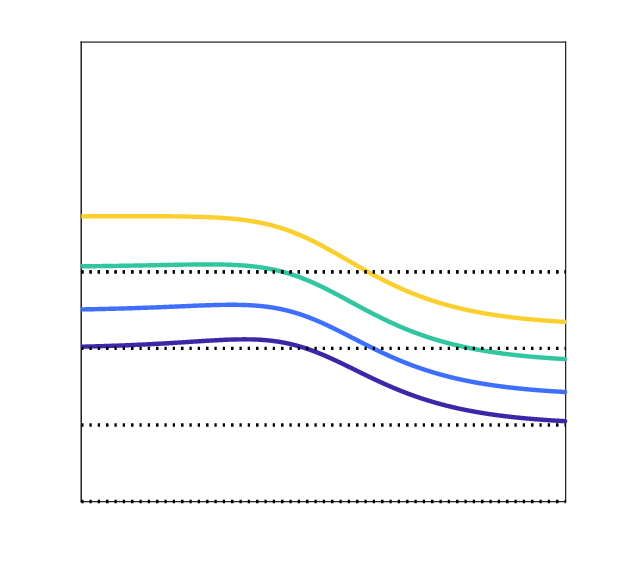}}; 
\node[inner sep=0pt] (russell) at (10.6,0.5)
    {\includegraphics[width=.35\textwidth]{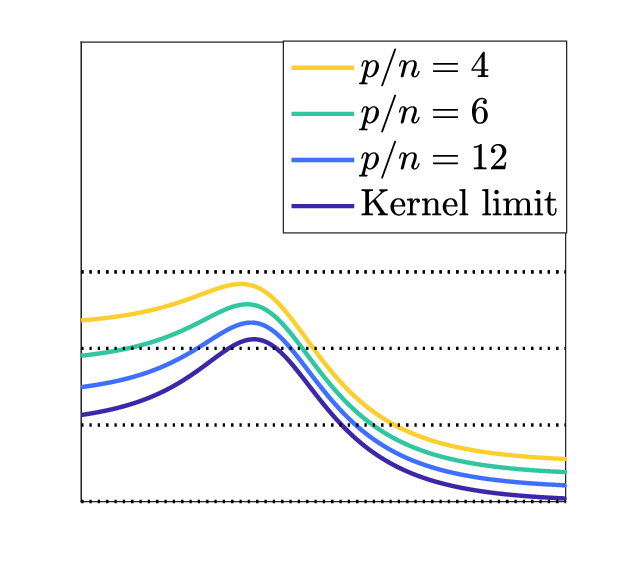}};  
\node[inner sep=0pt] (russell) at (0,-4.6)
    {\includegraphics[width=.35\textwidth]{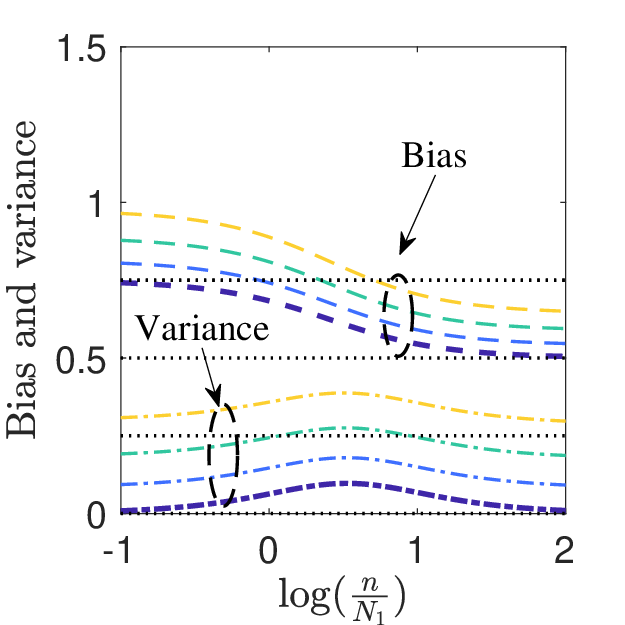}};
\node[inner sep=0pt] (russell) at (5.3,-4.6)
    {\includegraphics[width=.35\textwidth]{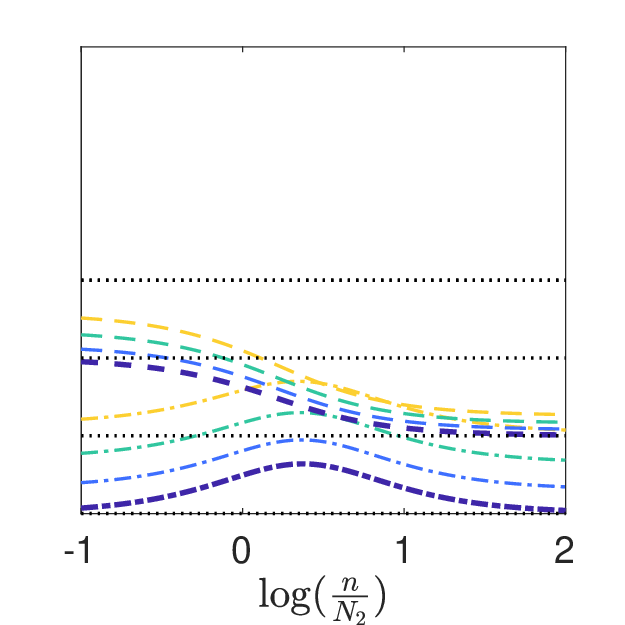}}; 
\node[inner sep=0pt] (russell) at (10.6,-4.6)
    {\includegraphics[width=.35\textwidth]{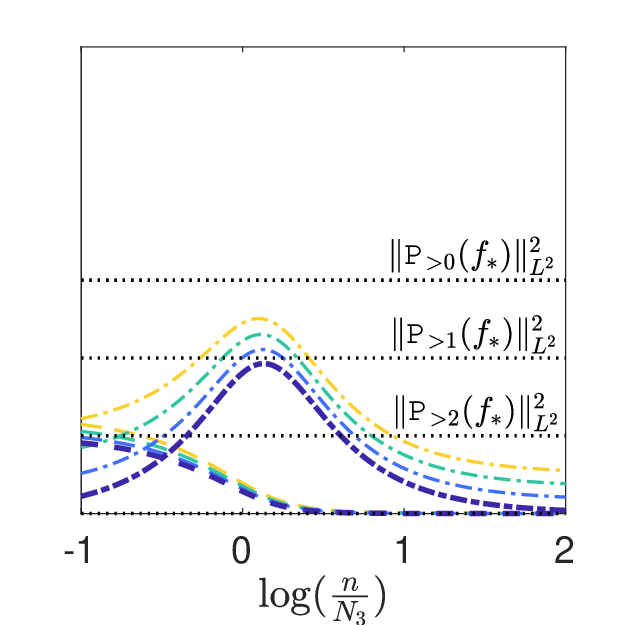}}; 
\end{tikzpicture}
\caption{Illustration of the bias-variance decomposition and the incremental learning process of RFRR, when $n$ and $p$ grows proportionally, i.e., $p/n = \theta_1/\theta_2$ kept constant while $\theta_2$ grows and $\kappa_1 = \kappa_2 \in \{1,2,3\}$. The choices of $\lambda$, $\rho_{\varepsilon}$, $f_{*}(\bx)$ and $\sigma(x)$ is the same as those in Figure \ref{fig:bias-var-psi1change}. }
\label{fig:bias-var-ratiofixed}
\end{figure}

\section{Equivalence with a Gaussian Covariate Model}
\label{sec:GaussianEquivalence}
\begin{figure}[t]
\begin{tikzpicture} 
\node[inner sep=0pt] (russell) at (0,0)
    {\includegraphics[width=.33\textwidth]{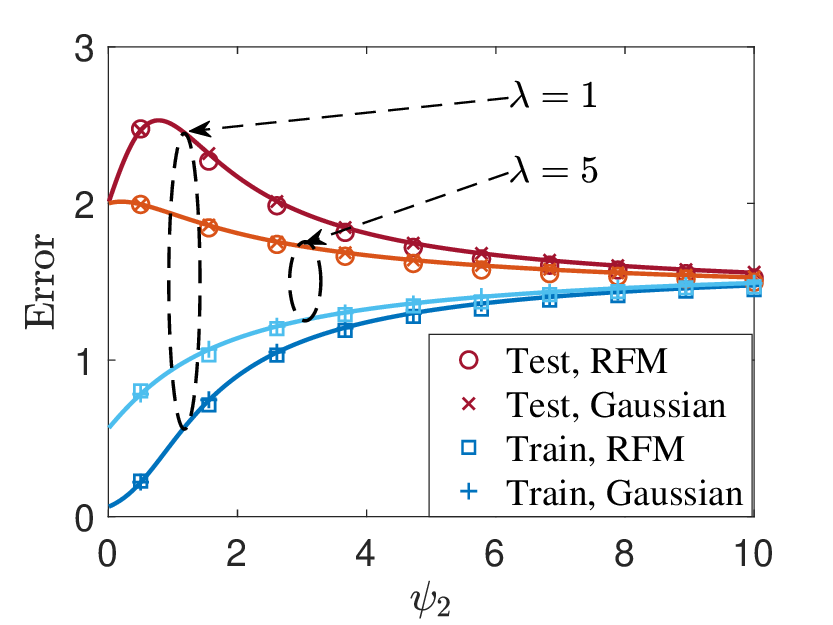}};
\node[inner sep=0pt] (russell) at (5.5,0)
    {\includegraphics[width=.33\textwidth]{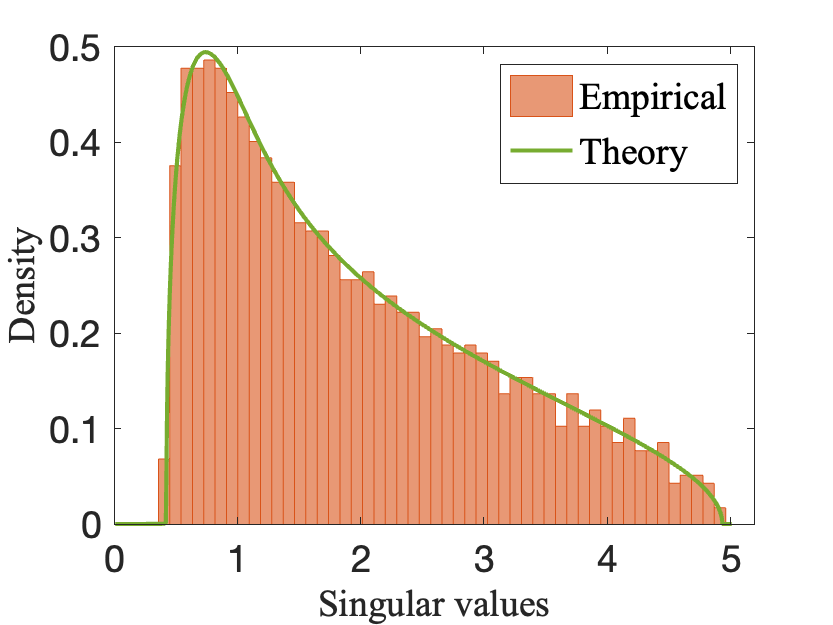}};
\node[inner sep=0pt] (russell) at (11,0)
    {\includegraphics[width=.33\textwidth]{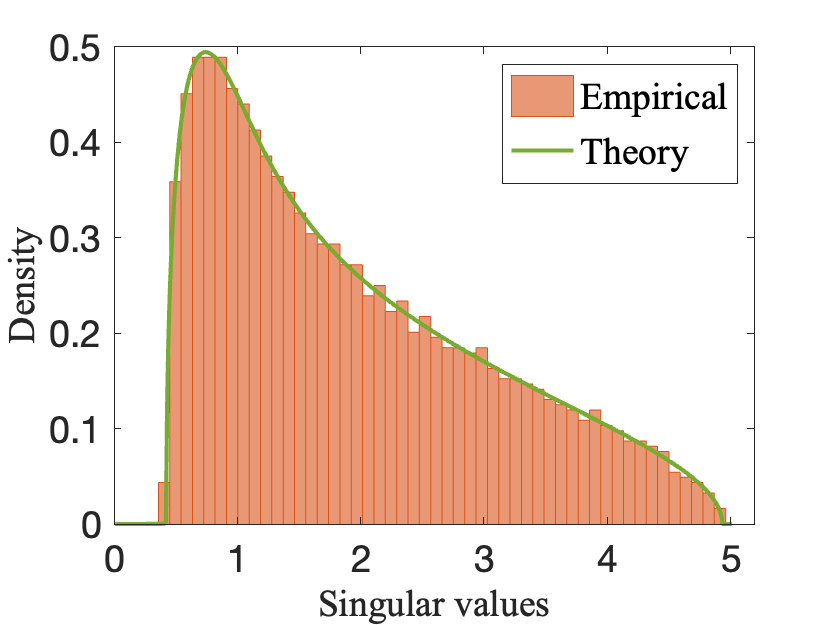}};    
\end{tikzpicture}
\caption{Numerical illustration of the equivalence with the Gaussian covariate model. We set $\kappa_1 = \kappa_2 = 2$ and use $d = 50$ for the numerical experiments. \textit{Left plot:} comparison of the training and test errors between the random feature model and the Gaussian model. We plot the errors versus $\psi_2 = 2n/d^2$, while keeping $\psi_1 =2 p / d^2 =  1$. The continuous lines correspond to the analytical predictions from Theorem \ref{thm:main_theorem_RF}. \textit{Middle and right plots:} singular values distribution of the random feature matrix $\bZ = ( \sigma (\< \bx_i , \bw_j \>) /\sqrt{p} )_{i \in [n], j\in [p]}$ (middle) and the equivalent Gaussian feature matrix $\bZ_G = ( \< \bg_i , \bSigma \bf_j \> / \sqrt{p})_{ i \in [n],j\in[p]}$ (right). We set $\sigma(x) = 2x^2 + x^3$, $\psi_1 = 1$, and $\psi_2 = 2$.}
\label{fig:spectrum_equi}
\end{figure}

The paper \cite{mei2022generalizationRF} noted an intriguing phenomenon in the linear scaling: ridge regression with the random feature model has the same asymptotic risk  as ridge regression with an `equivalent' Gaussian covariate model. The model is linear in the Gaussian covariates, but with a special covariance structure, and much simpler than the non-linear random feature model. Follow-up work \cite{hu2022universality,montanari2022universality} showed that in fact, this universality phenomenon holds for more general loss functions and regularizations in the linear scaling. Our proof reveals that this equivalence to a (more general) Gaussian covariate model continues beyond the linear scaling for ridge regression.

To make the correspondence between the two models more apparent, we first consider the original random feature model with covariates $\bx \sim \Unif (\S^{d-1} (\sqrt{d}))$. We can decompose the target function and the activation function in the orthonormal basis of spherical harmonics (for convenience, we introduce $\bv = \sqrt{d} \cdot \bw$)
\begin{equation}\label{eq:decompo_RF_uni}
\begin{aligned}
    f_* (\bx) =&~ \sum_{k=0}^\infty \sum_{s \in [N_k]} \beta_{*,ks} Y_{ks} (\bx) =: \< \phi (\bx) , \bbeta_* \> \, , \\
    \sigma( \< \bx , \bw \>) =&~ \sum_{k = 0}^\infty \xi_k \sum_{s \in [N_k]} Y_{ks} (\bx) Y_{ks} (\bv) =: \< \phi (\bx) , \bSigma \phi (\bv ) \> \, ,
\end{aligned}
\end{equation}
where we denoted $\phi (\bx) = ( Y_{ks} (\bx) )_{k\geq 0, s \in [N_k]}$, $\bbeta_* = ( \bbeta_{*,k} )_{k\geq 0}$ with $\bbeta_k \in \R^{N_k}$, and\footnote{The vectors and matrices are infinite dimensional here, and to be fully rigorous, we would define (compact) linear operators on $L^2(\S^{d-1})$. However, we keep this presentation for simplicity. Note that $\bSigma^2$ correspond to a trace class self-adjoint operator, where $\Tr (\bSigma^2) = \E_\bx[\sigma (\<\bx,\be_1\>)^2] < \infty$.} 
\[
\bSigma = \diag (\xi_0, \xi_1 \id_{N_1} , \xi_2 \id_{N_2} , \xi_3 \id_{N_3} , \ldots )\, .
\]
In particular, $\| \bbeta_k\|_2^2 = \| \proj_k f_* \|_2^2$. Note that $\E [ \phi (\bx) \phi(\bx)^\sT] = \id$. For example, using these notations, we can write the feature matrix as $\bZ = \phi (\bX) \bSigma \phi (\bW)^\sT / \sqrt{p}$ where
\[
 \phi (\bX) = [ \phi (\bx_1) , \ldots , \phi (\bx_n)]^\sT \in \R^{n \times \infty}\, , \qquad \phi (\bW) = [ \phi (\bv_1) , \ldots , \phi (\bv_p)]^\sT \in \R^{p \times \infty} \, .
\]
By definition, the entries of the feature map $\phi (\bx)$ are uncorrelated, but not independent or even subgaussian (the entries are polynomials of increasing degree). However, our proof indicates that RFRR behaves asymptotically as a model where $\phi (\bx)$ is replaced by a Gaussian vector with matching first two moments. 

Using this intuition, we are now ready to introduce the equivalent Gaussian covariate model:
\begin{itemize}
    \item[(a)] The covariates $\bg = ( g_{ks} )_{k \geq 0, s \in [N_k]}$ are (infinite-dimensional) Gaussian vectors with $\E[\bg] = \E[\phi(\bx)]$ and $\E [ \bg \bg^\sT] = \E[\phi(\bx) \phi (\bx)^\sT]$, i.e., $g_{00} = 1$ and for all $k,k' \geq 1, s\in [N_k], s' \in [N_{k'}]$,
    \[
    \E [ g_{ks} ] = 0 \, , \qquad \E[ g_{ks}g_{k's'}] = \delta_{k =k', s=s'}  \, .
    \]
    \item[(b)] The response $y^G = f^{G}_* ( \bg) + \eps^G$ with a linear target function $f^G_* (\bg) = \< \bbeta_* , \bg \>$ with $\bbeta_* = ( \bbeta_{*,k} )_{k \geq 0}$ defined as per Eq.~\eqref{eq:decompo_RF_uni} and independent noise $\eps^G \sim \normal (0, \rho_\eps^2)$. We denote
    \begin{equation}\label{eq:def_target_Gaussian}
    F_{\ell}^2 =  \| \bbeta_k \|_2^2\, , \qquad F_{>\ell} = \sum_{k = \ell+1}^\infty\| \bbeta_k\|_2^2 \, . 
    \end{equation}

    \item[(c)] The $p$ random features $(\bf_j)_{j \in [p]}$ are iid with same distribution as the covariates $\bg$ and denote $\bF = [ \bf_1 , \ldots , \bf_p ]^\sT \in \R^{p \times \infty}$. The random Gaussian feature model class is defined as 
    \begin{equation}\label{eq:Gaussian_random_feature_class}
    \cF_{\RF}^G (\bF) := \Big\{ h^G_{\RF} (\bg ; \ba ) = \frac{1}{\sqrt{p}} \sum_{j \in [p] } a_j \< \bf_j , \bSigma \bg \> = \frac{1}{\sqrt{p}} \ba^\sT \bF \bSigma \bg : \,\,\, \ba \in \R^p \Big\} \, .
    \end{equation}
\end{itemize}

We get $n$ iid samples $(\bg_i,y^G_i)_{i \in [n]}$ from the linear Gaussian covariate model and we fit this data using ridge regression with respect to the random Gaussian feature model class:
\[
\hat \ba_{\lambda}^G = \argmin_{\ba \in \R^p} \Big\{ \sum_{i \in [n]} \big( y_i^G - h^G_{\RF} (\bg_i ; \ba) \big)^2 + \lambda \| \ba \|_2^2 \Big\} = (\bZ_G^\sT   \bZ_G + \lambda \id_p )^{-1}  \bZ_G^\sT \by^G \, ,
\]
where we defined the feature matrix $\bZ_G = \bG \bSigma \bF^\sT / \sqrt{p}$, with $\bG = [ \bg_1 , \ldots , \bg_n ]^\sT \in \R^{n \times \infty}$. As in the RF model, we are interested in the test/training errors:
\begin{align}
R_{\test}^G ( \bbeta_* ; \bG , \bF , \beps^G , \lambda ) =&~ \E_{\bg} \Big[ \big( f_*^G (\bg) -  h_{\RF}^G (\bg ; \hat \ba_{\lambda}^G ) \big)^2 \Big] = \left\| \bbeta_* - \bSigma \bF^\sT \hat \ba_{\lambda}^G \right\|_2^2 \, ,\\
R_{\train}^G (\bbeta_* ; \bG , \bF , \beps^G, \lambda) =&~ \frac{1}{n} \sum_{i \in [n]} (y_i^G - h_{\RF}^G ( \bg_i ; \hat \ba_{\lambda}^G  ) )^2 \, 
\end{align}
and the normalized squared $\ell_2$ norm of $\hat \ba_{\lambda}^G$:
\begin{align}
    L_{\lnorm}^{G} (\btheta_* ; \bG , \bF , \beps^G, \lambda) =&~ 
    \begin{cases}
        \frac{ \| \hat \ba_{\lambda}^G \|_2^2 }{n}, ~~~\text{when } n = \cO_d(p),  \\
        \frac{ \| \hat \ba_{\lambda}^G \|_2^2 }{p}, ~~~\text{when } p = o_d(n).
    \end{cases}
\end{align}

The following theorem states that above `equivalent' Gaussian model displays the same asymptotics as the original random feature model.
\begin{theorem}[Gaussian equivalent model]
\label{thm:gauss_equivalence}
Consider $(p(d),n(d))_{d \geq 1}$ two sequences of integers such that $p/d^{\kappa_1} \to \gratio_1$ and $n / d^{\kappa_2} \to \gratio_2$ for some $\kappa_1,\kappa_2,\theta_1,\theta_2 \in \R_{>0}$, and denote $\ell = \lceil \min (\kappa_1 , \kappa_2) \rceil$. Assume that $(\xi_k)_{k \geq 0}$ are the singular values of $\sigma$ that satisfy Assumption \ref{ass:sigma} at level $\ell$, and that there exist constants $C,F_\ell^2,F_{>\ell}^2$, such that the target functions satisfy $\| \bbeta_* \|_2 \leq C$ and Equation \eqref{eq:def_target_Gaussian}. Then, for any $\lambda > 0$, we have 
\begin{align}
    \label{eq:asymp_testerr_gauss}
     \E_{\bG,\bF,\beps^G} \Big\vert R_{\test}^G ( \btheta_* ; \bG , \bF , \beps^G , \lambda ) - \Big[ (F_\ell^2 \cdot \cB_{\test} + F_{>\ell}^2) + (F_{>\ell}^2 + \rho_\eps^2 ) \cdot \cV_{\test}  \Big]  \Big\vert =&~ o_d(1) \, ,\\
 \label{eq:asymp_trainerr_gauss}
  \E_{\bG,\bF,\beps^G} \Big\vert R_{\train}^G (\btheta_* ; \bG , \bF , \beps^G, \lambda)  - \alpha_c \Big[ (F_\ell^2 \cdot \cB_{\test} + F_{>\ell}^2)  + (F_{>\ell}^2 + \rho_\eps^2 ) \cdot \cV_{\test} + \rho_\eps^2 \Big]  \Big\vert =&~ o_d(1) \, 
 \end{align}
 and
 \begin{align}
     \label{eq:asymp_anorm_gauss}
     \E_{\bG,\bF,\beps,f_*} \Big\vert  L_{\lnorm}^{G} (\btheta_* ; \bG , \bF , \beps^G, \lambda) - \Big[ F_\ell^2 \cdot \cB_{\lnorm} + (F_{>\ell}^2 + \rho_\varepsilon^2) \cdot \cV_{\lnorm} \Big] \Big\vert = o_d(1).
 \end{align}
\end{theorem}
The proof of Theorem \ref{thm:gauss_equivalence} can be done via the same procedure as Theorem \ref{thm:main_theorem_RF} and is omitted here for brevity. Note that here, we do not randomize the high-degree coefficients and can use a conditioning argument instead to replace the different quantities by traces (with slightly different traces, see for example \cite{hu2022sharp}). 

In Figure \ref{fig:spectrum_equi}, we illustrate this equivalence between the asymptotic behavior of the random feature model and the Gaussian covariate model. Intuitively, this equivalence holds in the polynomial scaling regime because the empirical distribution of singular values of the feature matrices in the original random feature model and the random Gaussian feature model are asymptotically the same.
The right two sub-figures in Figure \ref{fig:spectrum_equi} plot the empirical distributions of singular values of random feature matrix $\bZ = \sigma(\bX \bW^\sT) / \sqrt{p}$ and the equivalent Gaussian feature matrix $\bZ_G = \bG \bSigma \bF/\sqrt{p}$. We can see the empirical spectral densities of these two models both match with the same theoretical density curve. In \cite{lu2022equivalence}, a similar Gaussian equivalence principle is established in the case of a symmetric inner-product kernel random matrix. In particular, it is shown that in the polynomial scaling regime, the empirical eigenvalue distribution of this matrix is asymptotically equivalent to a linear combination of independent Wishart matrices.



\section{Proof of Theorem \ref{thm:main_theorem_RF}}
\label{sec:ProofMainResults}

This section presents the proof for the asymptotic test error in Theorem \ref{thm:main_theorem_RF}. The proofs for the training error and the $\ell_2$ norm are very similar and are deferred to Appendix \ref{app:train_err}.
We start by introducing some background and notations in Section \ref{sec:proof_notations}. The proof strategy is outlined in Section \ref{sec:proof_outline}. We defer the proof of some of the technical claims to the appendices.

\subsection{Some background and notations}\label{sec:proof_notations}

We begin with some notations and simple remarks. Denote $\bZ = \sigma(\bX \bW^\sT) / \sqrt{p} \in \R^{n \times p}$ the random feature matrix. For convenience, we will drop the subscript $d$ and simply write $f_* $ for the target function. Let $\bf = (f_*(\bx_1) , \ldots , f_*(\bx_n))$, $\beps = (\varepsilon_1,\ldots, \varepsilon_n)$, and $\by = (y_1 , \ldots , y_n)$, so that $\by = \bf + \beps$. It will be convenient to introduce for each $k \in \Z_{\geq 0}$, the vectors
\[
\bpsi_k (\bx) := ( Y_{ks} (\bx) )_{s \in [N_k]} \in \R^{N_k}\, , \qquad \bphi_k ( \bw ) := ( Y_{ks} (\sqrt{d} \cdot \bw) )_{s \in [N_k]} \in \R^{N_k}\, .
\]
 By addition theorem (see Appendix \ref{sec:Spherical-Harmonics}), we have the following decomposition of $\sigma( \< \bx , \bw \>)$ in the orthonormal basis of spherical harmonics:
\[
\sigma ( \< \bx , \bw \>) = \sum_{k = 0}^\infty \xi_k \bpsi_k ( \bx)^\sT \bphi_k (\bw) \, .
\]
where $\xi_k = \E_\bx [\sigma (\<  \be, \bx\>) q_k (\<\be , \bx \>)]$ with $\be$ an arbitrary unit vector in $\R^d$. In the following, we assume that the target function $f_*$ satisfy Assumption \ref{ass:random_f_star}. Namely, decomposing the target function in the orthonormal basis
\[
f_* (\bx ) = \sum_{k =0}^\infty \bbeta_k^\sT \bpsi_k (\bx) \, ,
\]
the coefficients $\bbeta_k $ are independent random vectors for $k \geq \ell$ with mean zero and covariance matrix $\E [ \bbeta_k \bbeta_k^\sT ] = F_k^2 \id_{N_k} / N_{k}$.

Define $\bD_k = \xi_k \id_{N_k}$ and recall that we denote $N_{\leq \ell} = N_0 + N_1 + \ldots + N_\ell$. We define the following matrices
\[
\begin{aligned}
\bPsi_k =&~ [ \bpsi_k (\bx_1) , \ldots , \bpsi_k (\bx_n) ]^\sT \in \R^{n \times N_k}\, ,\\
\bPhi_k =&~ [ \bphi_k (\bw_1) , \ldots , \bphi_k ( \bw_p) ]^\sT \in \R^{p \times N_k}\, , \\
\bD_{\leq \ell} = &~ \diag ( \bD_0 , \ldots , \bD_\ell ) \in \R^{N_{\leq \ell} \times N_{\leq \ell} }\, ,\\
\bPsi_{\leq \ell } =&~ [ \bPsi_0 , \bPsi_1 , \ldots , \bPsi_\ell] \in \R^{n \times N_{\leq \ell}}\, ,\\
\bPhi_{\leq \ell} = &~ [ \bPhi_0 , \bPhi_1 , \ldots , \bPhi_\ell ] \in \R^{p \times N_{\leq \ell}} \, , \\
\bbeta_{\leq \ell} =&~ (\bbeta_0^\sT, \bbeta_1^\sT , \ldots , \bbeta_\ell^\sT)^\sT \in \R^{N_{\leq \ell}} \, .
\end{aligned}
\]
We will further write $N_{<\ell} := N_{\leq \ell-1}$, $\bD_{<\ell} := \bD_{\leq \ell - 1}$, and similarly for $\bPsi_{< \ell }$, $\bPhi_{<\ell}$, and $\bbeta_{<\ell}$. Using these notations, we have for example 
\[
\begin{aligned}
\bZ =&~ p^{-1/2} \sum_{k = 0}^\infty \xi_k \bPsi_k \bPhi_k^\sT = p^{-1/2} \bPsi_{<\ell} \bD_{<\ell} \bPhi_{<\ell} +  p^{-1/2} \sum_{k = \ell}^\infty \xi_k \bPsi_k \bPhi_k^\sT \, , \\
\bf =&~ \sum_{k = 0}^\infty \bPsi_k \bbeta_k = \bPsi_{<\ell} \bbeta_{<\ell} + \sum_{k = \ell}^\infty \bPsi_k \bbeta_k \, .
\end{aligned}
\]
We also introduce the following Gegenbauer matrices: 
\[
\begin{aligned}
\bQ_k^{\bX} := &~\frac{1}{\sqrt{N_k}} q_k (\bX \bX^\sT) = \bPsi_k \bPsi_k^\sT / N_k \in \R^{n \times n}\, , \\
\bQ_k^\bW :=&~ \frac{1}{\sqrt{N_k}} q_k (\bW \bW^\sT) = \bPhi_k \bPhi_k^\sT / N_k \in \R^{p \times p}\, ,
\end{aligned}
\]
where we used the decomposition of the Gegenbauer polynomials $q_k$ in terms of spherical harmonics of degree $k$ (see Appendix \ref{sec:Spherical-Harmonics}).

In the statement and proof of some of our results, we will adopt the following stochastic domination notation for high-probability bound, used in the random matrix theory literature \cite{erdHos2017dynamical}. Consider two sequences of nonnegative random variables $X = \{ X (d) \}_{d \geq 1}$ and $Y = \{ Y (d) \}_{d \geq 1}$. We say that $X$ is \textit{stochastically dominated} by $Y$, if for any $\varepsilon>0$ and $D>0$, there exists $d_0 := d_0(\eps,D)$ sufficiently large such that
$$
 \P\big( X^{(d)} \geq d^\varepsilon Y^{(d)} \big) \leq d^{-D}\, , \qquad \forall d \geq d_0 \,.
$$
We will denote $ X \prec Y$ if X is stochastically dominated by $Y$. Moreover, if $|X| \prec Y$, we will write $X = O_{d,\prec}(Y)$, or simply $X = O_{\prec} (Y)$ when $d$ is clear from context.

\subsection{Outline of the proof}\label{sec:proof_outline}

Recall that we assume $p / d^{\kappa_1} \to \theta_1$ and $n/d^{\kappa_2} \to \theta_2$, and denote $\ell = \lceil \min ( \kappa_1 , \kappa_2 ) \rceil$. In particular, we will show that RFRR learns completely the degree-$(\ell-1)$ polynomial component of the target function $\proj_{< \ell} f_*$ and none of its components $\proj_{>\ell} f_*$ of degree at least $\ell +1$.

Solving for the coefficients of the random feature ridge regression problem yields
\begin{equation}
\hat\ba_{\lambda}= \argmin_{\ba \in \R^p} \Big\{ \| \by - \bZ \ba \|_2^2 + \lambda \| \ba \|_2^2 \Big\} = (\bZ^\sT \bZ + \lambda \id_p )^{-1} \bZ^\sT \by \, ,
\end{equation}
so that the prediction function at location $\bx$ is given by
\begin{equation}\label{eq:fct_sol_RFRR}
h_{\RF} (\bx ; \hat \ba_{\lambda} ) = \by^\sT \bZ (\bZ^\sT \bZ + \lambda \id_p )^{-1} \bsigma (\bx ) \, ,
\end{equation}
where $\bsigma (\bx) = ( \sigma ( \< \bx , \bw_1 \>) , \ldots , \sigma ( \< \bx , \bw_p \>))/ \sqrt{p} \in \R^p$. 

It will be useful to introduce the following resolvent matrix
\begin{equation}
    \bR : = ( \bZ^\sT \bZ + \lambda \id_p )^{-1} \in \R^{p \times p}\, ,
\end{equation}
so that $\hat\ba_{\lambda} = \bR \bZ^\sT \by$. Using the explicit solution \eqref{eq:fct_sol_RFRR}, we expand the test error and obtain
\begin{equation}\label{eq:decompo_Rtest}
    R_{\test} ( f_{*} ; \bX , \bW , \beps, \lambda) = \E_{\bx} [ f_{*} (\bx)^2 ] - 2 \by^\sT \bZ \bR \bV + \by^\sT \bZ \bR \bU \bR \bZ^\sT \by \, ,
\end{equation}
where we introduced $\bV = (V_1 , \ldots, V_p)^{\sT} \in \R^p$ and $\bU = (U_{ij})_{ij \in [p]} \in \R^{p \times p}$ defined by
\begin{align}
V_{j} = &~ p^{-1/2} \E_{\bx} [ f_{*} (\bx) \sigma ( \< \bx , \bw_j \>)]  \, ,\\
U_{ij} = &~ p^{-1} \E_{\bx} [ \sigma ( \< \bx , \bw_i\>) \sigma ( \< \bx , \bw_j \>) ] \, . \label{eq:def_U}
\end{align}

We first show in the next proposition that we can replace $R_{\test}$ by its expectation over the label noise $\beps$ and the randomness in the target function $f_{*}$.

\begin{proposition}\label{prop:concentration_on_expectation}
Under the same assumptions as in Theorem \ref{thm:main_theorem_RF}, 
we have
\begin{equation}
    \E_{\bX,\bW,\beps,f_*} \Big[ \Big\vert R_{\test} ( f_{*} ; \bX , \bW , \beps, \lambda) - \E_{\beps, f_*} [ R_{\test} ( f_{*} ; \bX , \bW , \beps, \lambda)] \Big\vert \Big] = o_d (1) \, .
\end{equation}
\end{proposition}
The proof of Proposition \ref{prop:concentration_on_expectation} can be found in Appendix \ref{sec:concentration_on_expectation_proof}. Denote 
\begin{align}
    \overline{R}_{\test} := \E_{\beps, f_*} [ R_{\test} ( f_{*} ; \bX , \bW , \beps, \lambda)] \, .
\end{align}
Thus, it is sufficient to show the convergence to the asymptotic test error directly for $\overline{R}_{\test}$. Our next steps consist in simplifying the expression of $\overline{R}_{\test}$.

\subsubsection{Bias-variance decomposition}

First, we decompose the risk into a bias and a variance term (over the label noise $\beps$)
\begin{equation}
\label{eq:Rtest_bar}
\begin{aligned}
    \overline{R}_{\test} = \E_{f_{*}} \left[\cB ( f_* ; \bX , \bW, \lambda)\right] + \cV ( \bX , \bW, \lambda)\, , 
    \end{aligned}
\end{equation}
where 
\[
\begin{aligned}
\cB ( f_* ; \bX , \bW, \lambda) = &~ \| f_* \|_{L^2}^2 - 2 \bf^\sT \bZ \bR \bV  + \bf^\sT \bZ \bR \bU \bR \bZ^\sT \bf  \, , \\
 \cV ( \bX , \bW, \lambda) = &~ \rho_\eps^2 \cdot \Tr ( \bZ \bR \bU \bR \bZ^\sT ) \, .
\end{aligned}
\]
Let us further decompose the matrix $\bU$ into
\begin{equation}\label{eq:decompo_bU}
\begin{aligned}
    \bU =&~ \frac{1}{p} \sum_{k=0}^\infty \xi_k^2 \bPhi_k \bPhi_k^\sT = : \bU_{<\ell} + \bU_{\geq \ell}\, ,
\end{aligned}
\end{equation}
where $\bU_{<\ell} = p^{-1} \bPhi_{<\ell} \bD_{<\ell}^2 \bPhi_{<\ell}^\sT$ and $\bU_{\geq \ell} = \bU - \bU_{<\ell}$.

The bias and variance terms can be both split into low- and high-degree parts. 

\paragraph*{Bias term.} For the bias term, we can write
\begin{align}
\label{eq:E_bias_fstar}
    E_{f_*} [\cB ( f_* ; \bX , \bW, \lambda) ] : = B_{< \ell} + B_{\geq \ell}\, ,
\end{align}
where, taking the expectation over $\proj_{\geq \ell} f_*$, we defined
\begin{equation}\label{eq:bias_highdegree}
\begin{aligned}
    B_{< \ell} := &~ \| \bbeta_{d,< \ell} \|_2^2 - \frac{2}{\sqrt{p}} \bbeta_{< \ell}^\sT \bPsi_{< \ell}^\sT \bZ \bR \bPhi_{< \ell} \bD_{< \ell} \bbeta_{< \ell} \nonumber \\
    &\hspace{3.5em} + \bbeta_{< \ell}^\sT \bPsi_{< \ell}^\sT \bZ \bR \bU \bR \bZ^\sT \bPsi_{< \ell}\bbeta_{< \ell} + \Tr ( \bH_F \bZ \bR \bU_{< \ell} \bR \bZ^\sT)\, , \\
    B_{\geq \ell} :=&~ F_{\geq \ell}^2 - 2\Tr ( \bZ_F^\sT \bZ \bR) + \Tr ( \bH_F \bZ \bR \bU_{\geq \ell} \bR \bZ^\sT) \, .   
\end{aligned}
\end{equation}
and we introduced the matrices
\[
\begin{aligned}
\bZ_F :=&~ \E_{f_*} [ \bf \bV^\sT ] = \frac{1}{\sqrt{p}}\sum_{k = \ell}^\infty F_{k}^2 \frac{\xi_k}{ N_k} \bPsi_k \bPhi_k^\sT  \, , \\
\bH_F :=&~ \E_{f_*} [ \bf \bf^\sT ] = \sum_{k = \ell}^\infty F_{k}^2 \bQ_k^\bX\, .
\end{aligned}
\]

\paragraph*{Variance term.} Similarly, we can split the variance term into 
\begin{align}
\label{eq:E_variance_epsilon}
    \cV ( \bX , \bW, \lambda) = V_{< \ell} + V_{\geq \ell}\, ,
\end{align}
where
\begin{equation}\label{eq:variance_highdegree}
\begin{aligned}
    V_{< \ell} :=&~ \rho_\varepsilon^2 \cdot \Tr ( \bZ \bR \bU_{< \ell} \bR \bZ^\sT )\, , \\
    V_{\geq \ell} :=&~ \rho_\varepsilon^2  \cdot \Tr ( \bZ \bR \bU_{\geq \ell} \bR \bZ^\sT )\, . 
\end{aligned}
\end{equation}

\subsubsection{Simplifying the bias and variance terms}

We simplify the expressions for the bias and variance terms by using that 1) the terms with spherical harmonics of degree less than $\ell$ involves low-dimensional matrices with diverging eigenvalues, and therefore concentrates to $0$; and 2) the high-degree components of the random matrices  concentrate in operator norm and can be replaced by deterministic matrices.

\paragraph*{Vanishing low-degree part.} Using the decomposition $\bU =p^{-1} \bPhi_{<\ell} \bD_{<\ell}^2 \bPhi_{<\ell}^\sT + \bU_{>\ell}$ in Eq.~\eqref{eq:decompo_bU}, we can rewrite $B_{<\ell}$ as
\begin{equation}\label{eq:decompo_B_leq_ell}
\begin{aligned}
    B_{< \ell} = &~\left\| \left( \id_{B_\ell} - \bD_{< \ell}  \bPhi_{< \ell}^\sT \bR \bZ^\sT \bPsi_{< \ell} / \sqrt{p} \right) \bbeta_{d,< \ell} \right\|_2^2 \\
    &~ + \bbeta_{< \ell}^\sT \bPsi_{< \ell}^\sT \bZ \bR \bU_{\geq \ell} \bR \bZ^\sT \bPsi_{< \ell}\bbeta_{< \ell} + \Tr ( \bH_F \bZ \bR \bU_{< \ell} \bR \bZ^\sT)\, . \\
\end{aligned}
\end{equation}

In Appendix \ref{sec:non-RMT}, we show the following proposition.

\begin{proposition}\label{prop:tech_bounds_ell}
Under the same assumptions and notations as in Theorem \ref{thm:main_theorem_RF}, we have
\begin{align}
    \E \| \id_{N_{<\ell}} - \bPsi_{< \ell}^\sT \bZ \bR \bPhi_{<\ell} \bD_{< \ell} / \sqrt{p} \|_\op = &~ o_{d} (1) \, , \label{eq:low_f_b1} \\
    \E \| \bPsi_{< \ell}^\sT \bZ \bR \bU_{\geq\ell} \bR \bZ^\sT \bPsi_{< \ell} \|_{\op} =&~ o_{d} (1) \, ,\label{eq:low_f_b2} \\
\E |\Tr ( \bH_F \bZ \bR \bU_{< \ell } \bR \bZ^\sT )| =&~ o_{d}(1)\, ,  \label{eq:low_f_b3} \\
\E |\Tr ( \bZ \bR \bU_{< \ell } \bR \bZ^\sT )| =&~ o_{d}(1)\, , \label{eq:low_f_b4}
\end{align}
where the expectations are over $\bX,\bW$.
\end{proposition}

Combining these bounds with the expressions \eqref{eq:variance_highdegree} and \eqref{eq:decompo_B_leq_ell} directly implies that 
\begin{align}
\label{eq:test_variance_low_degree_vanish}
    \E |B_{< \ell}| = o_d(1) \, , \;\;\; \text{ and }\;\;\; \E |V_{<\ell}| = o_d(1)\, .
\end{align}

\paragraph*{Concentration of the high-degree part of random matrices.} 
We are left with controlling
\begin{align}\label{eq:Rtest_trace_chi_j}
    \overline{R}_{\test,\geq \ell} :=&~ B_{\geq\ell} + V_{\geq\ell}
    = F_{\geq\ell}^2 -2 \chi_1 + \chi_2 + \rho_{\varepsilon}^2 \cdot \chi_3 \, ,
\end{align}
where
\begin{align*}
    \chi_1 :=& \Tr ( \bZ_F^\sT \bZ \bR) \, ,\\
    \chi_2 :=& \Tr ( \bH_F \bZ \bR \bU_{\geq \ell} \bR \bZ^\sT ) \, ,\\
    \chi_3 :=&  \Tr ( \bZ \bR \bU_{\geq \ell} \bR \bZ^\sT )\, .
\end{align*}
Therefore, it suffices to calculate the limits of the traces $\chi_1$, $\chi_2$ and $\chi_3$.

It is convenient to consider separately the following three cases: (I) $\kappa_1 > \kappa_2$, (II) $\kappa_1 < \kappa_2$, (III) $\kappa_1 = \kappa_2$.
In the following proposition, we  simplify the expression of the traces by using the concentration of the high-degree part of the random matrices. In particular, we will use that 
\[
\begin{aligned}
\| \bQ_k^{\bW} - \id_p \|_\op =&~ O_{d,\prec} \left(d^{\frac{\kappa_1 - k}{2}}\right) \, ,\\
\| \bQ_k^{\bX} - \id_n \|_\op =&~ O_{d,\prec} \left(d^{\frac{\kappa_2 - k}{2}}\right) \, ,\\
\end{aligned}
\]
which implies for example, that for $\kappa_1 > \kappa_2$ and any $f(x) = \sum_{k\geq \ell} \varsigma_k \usp_k(x) $ with $\| f \|_{L^2} < \infty$, we have
\[
\left\| \frac{f(\bX\bW^\sT )f(\bW\bX^\sT)}{p} - \sum_{k\geq 0}\varsigma_k^2 \bQ_k^\bX \right\|_\op = O_{d,\prec} \left(\sqrt{\frac{n}{p}}\right) \, .
\]
We gather the random matrix concentration and spectral bound proofs in Appendix \ref{sec:matrix_concentration_gather}.

\begin{proposition}
\label{prop:high_degree_concentration}
    Under the assumptions of Theorem \ref{thm:main_theorem_RF}, 
    we have:
\begin{itemize}
    \item[\emph{(I)}] If $\kappa_1 > \kappa_2$, then
    \begin{align}
        \E \Big| \chi_1 - \frac{F_\ell^2 \mu_\ell^2}{N_\ell} \Tr ( \bQ_\ell^{\bX} \bG_{\geq \ell}^{\bX} ) \Big| =&~  o_d(1)\, , \label{eq:over_chi1_concentrate}\\
        \E \Big| \chi_2 - \frac{ \mu_\ell^4}{N_\ell} \Tr \Big( (F_\ell^2 \bQ_\ell^{\bX} + \sum_{k>\ell} F_k^2 ) \bG_{\geq \ell}^{\bX} \bQ_\ell^{\bX} \bG_{\geq \ell}^{\bX} \Big) \Big| = &~ o_d(1)\, , \label{eq:over_chi2_concentrate}\\
        \E \Big| \chi_3 - \frac{\mu_\ell^4}{N_\ell} \Tr \big(  \bG_{\geq \ell}^{\bX} \bQ_\ell^{\bX} \bG_{\geq \ell}^{\bX}  \big) \Big| = &~ o_d(1) \, ,\label{eq:over_chi3_concentrate}
    \end{align}
    where $\bG_{\geq \ell}^{\bX} = \big( \mu_\ell^2 \bQ_{\ell}^{\bX} + ( \mu_{>\ell}^2 + \lambda) \id_n \big)^{-1}$.

    \item[\emph{(II)}] If $\kappa_1 < \kappa_2$, then
    \begin{align}
        \E \Big| \chi_1 - \frac{F_\ell^2}{N_\ell} \Tr (  \mu_\ell^2 \bQ_\ell^{\bW} \cdot \bG_{\geq \ell}^{\bW} ) \Big| =&~  o_d(1) \, , \label{eq:under_chi1_concentrate}\\
        \E \Big| \chi_2 - \frac{F_\ell^2 }{N_\ell} \Tr ( \mu_\ell^2 \bQ_\ell^{\bW} \cdot \bG_{\geq \ell}^{\bW} ) \Big| = &~ o_d(1) \, , \label{eq:under_chi2_concentrate}\\
        \E | \chi_3 | = &~ o_d(1)\, , \label{eq:under_chi3_concentrate}
    \end{align}
    with $\bG_{\geq \ell}^{\bW} = ( \mu_\ell^2 \bQ_{\ell}^{\bW} + \mu_{>\ell}^2 \id_p )^{-1}$.

    \item[\emph{(III)}] If $\kappa_1 = \kappa_2$, then
    \begin{align}
        \E \Big| \chi_1 - \frac{F_{\ell}^2 }{ {N_{\ell} }  } \Tr \Big[\frac{ \mu_{\ell} }{ \sqrt{p} } \usp_{\ell}(\bW \bX^\sT) \bZ \bR \Big] \Big| 
        &= o_{d}(1)\, , \label{eq:critical_chi1_concentrate}\\
        \E \Big|\chi_2 - \frac{1}{p}\Tr[\bH_{\geq \ell}^{\bX} \bZ \bR ( \bG_{\geq \ell}^{\bW} )^{-1} \bR \bZ^\sT]\Big| &= o_{d}(1) \, , \label{eq:critical_chi2_concentrate}\\
        \E \Big|\chi_3 - \frac{1}{p} \Tr[\bZ \bR ( \bG_{\geq \ell}^{\bW} )^{-1} \bR \bZ^\sT]\Big| &= o_{d}(1) \, , \label{eq:critical_chi3_concentrate}
    \end{align}
where $\bH_{\geq \ell}^{\bX} = F_{\ell}^2 \bQ_{\ell}^{\bX} + F_{>\ell}^2 \id_n$.   
\end{itemize}
\end{proposition}
The proof of Proposition \ref{prop:high_degree_concentration} is given in Appendix \ref{sec:high_degree_concentration_proof}.
From Proposition \ref{prop:high_degree_concentration} and Eq.~\eqref{eq:Rtest_trace_chi_j}, it remains to compute the limits of all the normalized traces of matrices  appearing in Eqs.~\eqref{eq:over_chi1_concentrate}-\eqref{eq:critical_chi3_concentrate}.

\subsubsection{Computing the asymptotics of the traces}

As in Proposition \ref{prop:high_degree_concentration}, we discuss the three regimes separately.

\paragraph*{(I) Overparametrized regime $\kappa_1 > \kappa_2$.} We see that the traces in Proposition \ref{prop:high_degree_concentration}.(I) corresponds exactly to the traces computed in the KRR limit, which were already studied in \cite{misiakiewicz2022spectrum,hu2022sharp,xiao2022precise}. For example, combining Lemma 7 in \cite{misiakiewicz2022spectrum} with Eqs.~\eqref{eq:over_chi1_concentrate}, \eqref{eq:over_chi2_concentrate} and \eqref{eq:over_chi3_concentrate}, we can obtain Eq.~\eqref{eq:asymp_testerr}, with $(\cB_{\test}, \cV_{\test}, \alpha_c)$ given in Eq.~\eqref{eq:B_V_test_2}. See Appendix \ref{sec:kernel_approx} for additional details.

\paragraph*{(II) Underparametrized regime $\kappa_1 < \kappa_2$.}
After substituting Eqs.~\eqref{eq:under_chi1_concentrate}-\eqref{eq:under_chi3_concentrate} into Eq.~\eqref{eq:Rtest_trace_chi_j} and using Eq.~\eqref{eq:test_variance_low_degree_vanish} and Proposition \ref{prop:concentration_on_expectation}, we get
\begin{align}
\label{eq:under_Rtrain_concentrate}
    \E \left\vert {R}_{\test} - \left[ F_{\ell}^2\left( 1 - \frac{1}{ N_\ell } \Tr( \mu_{\ell}^2  \bQ_{\ell}^\bW \cdot \bG_{\geq\ell}^{\bW} ) \right) + F_{>\ell}^2 \right] \right\vert = o_d(1) \, .
\end{align}
Using the definition of $\bG_{\geq\ell}^{\bW}$, we can rewrite the trace as
\begin{align*}
    \frac{1}{ N_\ell } \Tr( \mu_{\ell}^2  \bQ_{\ell}^\bW \cdot \bG_{\geq\ell}^{\bW} ) = \frac{1}{ N_\ell } \Tr\big[ \id_{p} - ( \zeta^2 \bQ_{\ell}^{\bW} + \id_{p} )^{-1} \big] \, ,
\end{align*}
so when $\kappa_1 = \ell$, we can use Theorem 1 in \cite{lu2022equivalence} to compute the limit of the Stieltjes transform $ \frac{1}{p} \Tr [ ( \zeta^2 \bQ_{\ell}^{\bW} + \id_{p} )^{-1} ] $. After substituting the limiting formula into Eq.~\eqref{eq:under_Rtrain_concentrate}, we obtain Eq.~\eqref{eq:asymp_testerr}, with $\cB_{\test}$ and $\cV_{\test}$ given in Eq.~\eqref{eq:B_V_test_3}. On the other hand, when $\kappa_1 < \ell$, we have
\begin{align}
\label{eq:Tr_QW_GW_large_p}
    \frac{1}{ N_\ell } \Tr \left( \mu_{\ell}^2  \bQ_{\ell}^\bW \cdot \bG_{\geq\ell}^{\bW} \right) &= \frac{1}{ N_\ell } \sum_{i=1}^{p} \frac{ \zeta^2 \lambda_{i}( \bQ_{\ell}^{\bW} ) }{ \zeta^2 \lambda_{i}( \bQ_{\ell}^{\bW} ) + 1 } \, ,
\end{align}
where $\lambda_i( \bQ_{\ell}^{\bW} )$ denotes the $i$-th largest eigenvalue of $\bQ_{\ell}^{\bW}$. By Lemma \ref{lem:offdiagonal_kernelmtx_opnormbd} in Appendix \ref{sec:matrix_concentration} and taking $0< \eps< \ell - \kappa_1$, we have $|\lambda_i( \bQ_{\ell}^{\bW} ) - 1| = \cO_{\prec}(d^{-\eps})$ for any $1\leq i \leq p$. Then we can show
\begin{align}
\label{eq:Tr_QW_GW_large_p_1}
    \E\left[ \sup_{1\leq i \leq p} \Big| \frac{ \zeta^2 \lambda_{i}( \bQ_{\ell}^{\bW} ) }{ \zeta^2 \lambda_{i}( \bQ_{\ell}^{\bW} ) + 1 } - \frac{ \zeta^2 }{ \zeta^2 + 1 } \Big| \right] = o_d(1)\, .
\end{align}
Combining Eqs.~\eqref{eq:Tr_QW_GW_large_p}, \eqref{eq:Tr_QW_GW_large_p_1} and \eqref{eq:under_Rtrain_concentrate} and $\frac{p}{N_\ell} = O_d(d^{\kappa_1 - \ell}) =s o_d(1)$, we obtain Eq.~\eqref{eq:asymp_testerr} for $\kappa_1 < \ell$.

\paragraph*{(III) Critical regime $\kappa_1 = \kappa_2$.}
Recall that in this case both $\psi_1, \psi_2 < \infty$ (and $\psi_1 = \psi_2 = 0$ whenever $\kappa_1 = \kappa_2 < \ell$). 
From Eqs.~\eqref{eq:critical_chi1_concentrate}, \eqref{eq:critical_chi2_concentrate} and \eqref{eq:critical_chi3_concentrate}, we can substitute $\chi_1$, $\chi_2$ and $\chi_3$ with traces of some non-commutative rational functions of $\bZ$, ${\usp_{\ell}(\bW \bX^\sT)} $, $\bQ_{\ell}^{\bX}$ and $\bQ_{\ell}^{\bW}$. To compute these traces, we will follow the same proof strategy as in \cite{mei2022generalizationRF}, and apply the same linearization construction (block matrix) in order to reuse their analytical characterization for the asymptotic log determinant. The asymptotic traces are obtained as the derivatives of this log determinant. We detail the steps of this derivation below, and include the analytical formulas from \cite{mei2022generalizationRF} for convenience. 

Let $m = n+p$ and $\psi = \psi_1+\psi_2$. First, we rescale $\bZ$ and $\bR$ as: 
\begin{align}
    \tbZ &= \frac{1}{ \sqrt{m} }  \sigma(\bX \bW^\sT) \, ,\\
    \tbR &= \big( \gratio_1 \gratio^{-1} \lambda \id_p + \tbZ^\sT \tbZ \big)^{-1}\, ,
\end{align}
where
\[
\gratio = \gratio_1 + \gratio_2\, .
\]
Note that we have $\bZ =  \sqrt{\frac{ \gratio }{ \gratio_1 }} \tbZ$ and $\bR = \gratio_1 \gratio^{-1} \tbR$.

For $\bq = (s_1, s_2, t_1, t_2, t)$, define the block matrix as
\begin{align}\label{eq:def_block_matrix_A}
    \bA = \bA(\bq) := \begin{bmatrix}
s_1 \id_p + s_2 \bQ_{\ell}^{\bW} &
\tbZ^\sT +  \frac{ t \mu_{\ell} }{ \sqrt{m} } \usp_{\ell}(\bW \bX^\sT) \\
\tbZ + \frac{ t \mu_{\ell} }{ \sqrt{m} } \usp_{\ell}(\bX \bW^\sT)
&  t_1 \id_n +  t_2 \bQ_{\ell}^{\bX}
\end{bmatrix}\, .
\end{align}
The Stieltjes transform of $\bA$ is given by
\[
M_d ( z ; \bq) = \frac{1}{m} \Tr \big[ ( \bA - z \id_m )^{-1} \big]\, , \qquad m_d ( z , \bq) = \E [ M_d ( z ; \bq) ]\, .
\]
where $z\in\C_{+}$.
In the sequel, we will denote by $z = E+i\eta$ our complex number. We also introduce the log-determinant:
\begin{align}
\label{eq:Gd_def}
    G_d ( z ; \bq) = \frac{1}{m} \log \det ( \bA - z \id_m ) \, .
\end{align}
It can be easily checked by direct differentiation (see \cite[Proposition 8.2]{mei2022generalizationRF} for details):
\begin{align}   
    \label{eq:partial_G_1}
    \frac{1}{2}\partial_t G_d( i\sqrt{\gratio_1 \gratio^{-1} \lambda}; \bzero ) &= \frac{1}{ m } \Tr \left[\frac{\mu_{\ell}}{ \sqrt{m} } \usp_{\ell}(\bW \bX^\sT) \tbZ \tbR \right] \, , \\
    \label{eq:partial_G_2}
    \partial_{s_1,t_1} G_d( i\sqrt{\gratio_1 \gratio^{-1} \lambda}; \bzero ) &= -\frac{1}{m} \Tr[ \tbZ \tbR \tbR \tbZ^\sT ] \, ,\\
    \label{eq:partial_G_3}
    \partial_{s_1,t_2} G_d( i\sqrt{\gratio_1 \gratio^{-1} \lambda}; \bzero ) &= -\frac{1}{m} \Tr[ \bQ_{\ell}^{\bX} \tbZ \tbR \tbR \tbZ^\sT]\, ,\\
    \label{eq:partial_G_4}
    \partial_{s_2,t_1} G_d( i\sqrt{\gratio_1 \gratio^{-1} \lambda}; \bzero ) &= -\frac{1}{m} \Tr[ \tbZ \tbR \bQ_{\ell}^{\bW} \tbR \tbZ^\sT ]\, ,\\
    \label{eq:partial_G_5}
    \partial_{s_2,t_2} G_d( i\sqrt{\gratio_1 \gratio^{-1} \lambda}; \bzero ) &= -\frac{1}{m} \Tr[ \bQ_{\ell}^{\bX} \tbZ \tbR \bQ_{\ell}^{\bW} \tbR \tbZ^\sT]\, .
\end{align}
Substituting Eqs.~\eqref{eq:partial_G_1}-\eqref{eq:partial_G_5}
into Eqs.~\eqref{eq:critical_chi1_concentrate}-\eqref{eq:critical_chi3_concentrate} and using $\bZ =  \sqrt{\frac{ \gratio }{ \gratio_1 }} \tbZ$ and $\bR = \gratio_1 \gratio^{-1} \tbR$,
we get
\begin{align}
    \E | \chi_1 - F_{\ell}^2 \cdot \Gamma_1 | 
    =&~ o_{d}(1)\, , \label{eq:critical_chi1_concentrate_1}\\
    \E |\chi_2 - ( F_{\ell}^2 \cdot \Gamma_2 + F_{>\ell}^2\cdot  \Gamma_3) | =&~ o_{d}(1)\, , \label{eq:critical_chi2_concentrate_1}\\
    \E |\chi_3 - \Gamma_3 | =&~ o_{d}(1)\, , \label{eq:critical_chi3_concentrate_1}
\end{align}
where
\begin{align}
    \label{eq:Gamma1_partial}
    \Gamma_1 &= \frac{\psi}{2}\partial_t G_d( i\sqrt{\gratio_1 \gratio^{-1} \lambda}; \bzero ) \, ,\\
    \label{eq:Gamma2_partial}
    \Gamma_2 &= - \mu_{\ell}^2 \partial_{s_2,t_2}G_d( i\sqrt{\gratio_1 \gratio^{-1} \lambda}; \bzero ) - \mu_{>\ell}^2 \partial_{s_1,t_2}G_d( i\sqrt{\gratio_1 \gratio^{-1} \lambda}; \bzero ) \, ,\\
    \label{eq:Gamma3_partial}
    \Gamma_3 &= - \mu_{\ell}^2 \partial_{s_2,t_1}G_d( i\sqrt{\gratio_1 \gratio^{-1} \lambda}; \bzero ) - \mu_{>\ell}^2 \partial_{s_1,t_1}G_d( i\sqrt{\gratio_1 \gratio^{-1} \lambda}; \bzero ) \, .
\end{align}
It remains to compute the limits ($d\to\infty$) of partial derivatives of $G_d(z;\bq)$ with respect to $\bq$ in Eqs.~\eqref{eq:Gamma1_partial}, \eqref{eq:Gamma2_partial} and \eqref{eq:Gamma3_partial}. Note that $G_d(z;\bq)$ can be written as an integral of $M_d(z;\bq)$ with respect to $z$. Hence, the rest of the proof consists in implementing the following three steps:
\begin{enumerate}
    \item 
    Compute the limit $\lim_{d\to\infty} M_d(z;\bq)$.
    \item
    Compute the limit $\lim_{d\to\infty} G_d(z;\bq)$, by integrating $M_d(z;\bq)$ over $z$.
    \item
    Show the limits of partial derivatives of $G_d(z;\bq)$ (with respective to $\bq$) are equal to the partial derivatives (with respective to $\bq$) of $\lim_{d\to\infty} G_d(z;\bq)$.
\end{enumerate}

\noindent\textbf{Step 1: Limit of Stieltjes transform.}

Define the following two functions
\begin{equation}\label{eq:F1_def}
\begin{aligned}
   &~ {\textsf{F}}_1(m_1, m_2 ; z, \bq) \\=&~ \frac{\gratio_1}{\gratio} \Big[  -z + s_1 - \mu_{>\ell}^2 m_2 + \frac{(1+\psi t_2 m_2 ) s_2 - (1+t)^2\mu_{\ell}^2 m_2}{(1+\psi s_2m_1)(1+\psi t_2m_2)-\psi(1+t)^2\mu_{\ell}^2 m_1m_2}\Big]^{-1}\, , \\
    &~{\textsf{F}}_2(m_1, m_2 ; z, \bq)\\ =&~ \frac{\gratio_2}{\gratio} \Big[  -z + t_1 - \mu_{>\ell}^2 m_1 + \frac{(1+\psi s_2 m_1 ) t_2 - (1+t)^2\mu_{\ell}^2 m_1}{(1+\psi t_2m_2)(1+\psi s_2m_1)-\psi(1+t)^2\mu_{\ell}^2 m_1m_2}\Big]^{-1} \, ,
\end{aligned}
\end{equation}
and the following set
\begin{align}
    \cQ := \big\{ (s_1, s_2, t_1, t_2, t) : \psi s_2 t_2 \leq {(1+t)^2 \mu_{\ell}^2}/{2} \big\} \, .
\end{align}
Then the following result shows that $\lim_{d\to\infty} M_d(z;\bq)$ corresponds to the solution of a fixed point equation defined via $\textsf{F}_1$ and $\textsf{F}_2$. The proof is provided in Appendix \ref{sec:Stieltjes_transform}, which follows the same strategy as the proof of Proposition 8.3 in \cite{mei2022generalizationRF}.
\begin{proposition}
\label{prop:resolvent_conv}
Suppose $\kappa_1 = \kappa_2$. For given $z\in\C_{+}$ and $\bq$, let $m_1(z,\bq)$ and $m_2(z,\bq)$ be the unique analytic solutions in $\C_{+}$ to the following equations:
\begin{equation}
\label{eq:fixedpoint_eq}
    \begin{aligned}
        {m}_{1} &= {\textsf{F}}_1({m}_{1}, {m}_{2} ; z, \bq)\, , \\
        {m}_{2} &= {\textsf{F}}_2({m}_{1}, {m}_{2} ; z , \bq) \, .
    \end{aligned}
\end{equation}
Define $m(z;\bq) := m_1(z;\bq) + m_2(z;\bq)$. Then for any compact set $\Omega\subseteq \C_{+}$, we have
\begin{align}
    \E \Big[ \sup_{z\in\Omega}|M_d(z;\bq) - m(z,\bq)| \Big] = o_{d}(1)\, .
\end{align}
\end{proposition}

\noindent\textbf{Step 2: Limit of log-determinant.}

Based on Proposition \ref{prop:resolvent_conv}, we can now compute $\lim_{d\to\infty} G_d(z;\bq)$ and its partial derivatives.
To state the results, we define the following function $\cG(z, m_1, m_2, \bq)$ for $q\in \cQ$:
1) when $\psi>0$
\begin{equation}
\label{eq:calG_z_def}
\begin{aligned}
    \cG(z, m_1, m_2, \bq)
    :=& \frac{1}{\psi}\log\big[(\psi s_2 m_1 + 1)(\psi t_2 m_2 + 1) - \psi(1+t)^2 \mu_{\ell}^2 m_1 m_2\big]  + s_1 m_1 + t_1 m_2\\
    &\hspace{2.4em} - \mu_{>\ell}^2 m_1 m_2 - \frac{\theta_1}{\theta} \log({\theta} m_1/\theta_1) - \frac{\theta_2}{\theta} \log({\theta} m_2/\theta_2) - z (m_1+m_2) - 1\, ,  
\end{aligned}   
\end{equation}
2) when $\psi=0$,
\begin{equation}
\label{eq:calG_z_def_psi0}
\begin{aligned}
    \cG(z, m_1, m_2, \bq)
    :=& (s_1+s_2) m_1 + (t_1+t_2) m_2 - (1+t)^2 \mu_{\ell}^2 m_1 m_2\\
    &\hspace{2.4em} - \mu_{>\ell}^2 m_1 m_2 - \frac{\theta_1}{\theta} \log({\theta} m_1/\theta_1) - \frac{\theta_2}{\theta} \log({\theta} m_2/\theta_2) - z (m_1+m_2) - 1  \, .
\end{aligned}
\end{equation}
We further define
\begin{align}
    g(z;\bq) := \cG(z, m_1(z, \bq), m_2(z, \bq), \bq)\, ,
\end{align}
where $m_1(z, \bq)$ and $m_2(z, \bq)$ are defined as in Proposition \ref{prop:resolvent_conv}.
Then we have the following result.
\begin{proposition}
\label{prop:derivative_converge}
Suppose $\kappa_1 = \kappa_2$. Then for any $z\in\C_{+}$ and $q$, we have
\begin{align}
    \E |G_d(z;\bq) - g(z;\bq)| = o_{d}(1)\, ,
\end{align}
and for any $\eta > 0$,
\begin{align}
    \label{eq:G_1stDeri_conv}
    \E \|\nabla_{\bq}G_d(i\eta;\bzero) - \nabla_{\bq}g(i\eta;\bzero)\|_2 &= o_{d}(1)\, , \\
    \label{eq:G_2ndDeri_conv}
    \E \|\nabla_{\bq}^2 G_d(i\eta;\bzero) - \nabla_{\bq}^2 g(i\eta;\bzero)\|_{\op} &= o_{d}(1)\, .
\end{align}
\end{proposition}

The proof is completely analogous to that of Proposition 8.4 in \cite{mei2022generalizationRF}, with Proposition \ref{prop:resolvent_conv}, Lemma \ref{lem:Gdgrad_prbd} and Lemma \ref{lem:Gdgd_largeK} (see Appendix \ref{sec:log_determinant}) in place of Proposition 8.3, Lemma 11.3 and Lemma 11.2 in \cite{mei2022generalizationRF}, respectively. The details are omitted.

It remains to compute the partial derivatives of $g(z;\bq)$. By direct differentiation, 
we can get (same as Lemma 8.1 \cite{mei2022generalizationRF}, with a slightly different scaling):

\begin{equation}
\label{eq:g_partial_deri}
    \begin{aligned}
        \partial_{t} g(z;\bzero) =& 2 m_0 \mu_{\ell}^2/( \psi m_0 \mu_\ell^2 - 1 ) \, ,\\
        \partial_{s_1,t_1} g(z;\bzero) =& [\psi^3 m_0^5 \mu_{\ell}^6 \mu_{>\ell}^2 - 3 \psi^2 m_0^4 \mu_{\ell}^4 \mu_{>\ell}^2 + \psi m_0^3 \mu_{\ell}^4 + 3 \psi m_0^3 \mu_{\ell}^2 \mu_{>\ell}^2 - m_0^2 \mu_{\geq \ell}^2]/S \, , \\
        \partial_{s_1,t_2} g(z;\bzero) =& [(\psi_2-1) \psi m_0^3 \mu_{\ell}^4 + \psi m_0^3 \mu_{\ell}^2 \mu_{>\ell}^2 - (\psi_2+1) m_0^2\mu_{\ell}^2 - m_0^2 \mu_{>\ell}^2]/S\, , \\
        \partial_{s_2,t_1} g(z;\bzero) =& [(\psi_1-1) \psi m_0^3 \mu_{\ell}^4 + \psi m_0^3 \mu_{\ell}^2 \mu_{>\ell}^2 - (\psi_1+1) m_0^2\mu_{\ell}^2 - m_0^2 \mu_{>\ell}^2]/S \, ,\\
        \partial_{s_2,t_2} g(z;\bzero) =& [-\psi^4 m_0^6 \mu_{\ell}^6 \mu_{>\ell}^4 + 2 \psi^3 m_0^5 \mu_{\ell}^4 \mu_{>\ell}^4 +  (\psi_1 - 1)(\psi_2 - 1) \psi^2 m_0^4 \mu_{\ell}^6 - \psi^2 m_0^4 \mu_{\ell}^2 \mu_{>\ell}^2 \mu_{\geq\ell}^2 \, ,\\
        & + 2 (1- \psi_1\psi_2) \psi m_0^3 \mu_{\ell}^4 + ({\psi_1}+1)({\psi_2}+1) m_0^2 \mu_{\ell}^2 + m_0^2 \mu_{>\ell}^2] / [(\psi m_0 \mu_{\ell}^2 - 1)S]\, ,
    \end{aligned}
\end{equation}
where 
\begin{align*}
    m_0 := m_0(z,\bzero) :=  m_1(z,\bzero) \cdot m_2(z,\bzero)\, ,
\end{align*}
and
\begin{align*}
    S =&~ \psi^3 m_0^5 \mu_{\ell}^6 \mu_{>\ell}^4 - 3 \psi^2 m_0^4 \mu_{\ell}^4 \mu_{>\ell}^4 + (\psi_1 - 1)(\psi_2 - 1) \psi m_0^3 \mu_{\ell}^6 + 2 \psi m_0^3 \mu_{\ell}^4 \mu_{>\ell}^2 + 3 \psi m_0^3 \mu_{\ell}^2 \mu_{>\ell}^4\\
    &~ + ({3 \psi_1\psi_2} - \psi - 1) m_0^2 \mu_{\ell}^4 - 2 m_0^2 \mu_{\ell}^2 \mu_{>\ell}^2 - m_0^2 \mu_{>\ell}^4 - \frac{3\theta_1\theta_2}{\theta^2} \psi m_0 \mu_{\ell}^2 + \frac{\theta_1\theta_2}{\theta^2} \, .
\end{align*}
Substituting Proposition \ref{prop:derivative_converge} and Eq.~\eqref{eq:g_partial_deri} into Eqs.~\eqref{eq:critical_chi1_concentrate_1}-\eqref{eq:critical_chi3_concentrate_1} and using the fact that
\begin{align*}
    \E_{\bX,\bW,\beps,f_*}  \Big\vert R_{\test} ( f_{*,d} ; \bX , \bW , \beps, \lambda) -  (F_{\geq\ell}^2 - 2\chi_1 + \chi_2 + \rho_\varepsilon^2 \chi_3 ) \Big\vert = o_d(1) \, ,
\end{align*}
which follows from Eqs.~\eqref{eq:Rtest_bar}, \eqref{eq:E_bias_fstar} and \eqref{eq:E_variance_epsilon}, and Proposition \ref{prop:concentration_on_expectation}, we can obtain
\begin{align}
    \label{eq:R_test_l1_conv}
    \E_{\bX,\bW,\beps,f_*}  \Big\vert R_{\test} ( f_{*,d} ; \bX , \bW , \beps, \lambda) - \cR_{\test} ( F_* , \zeta , \gratio_1 , \gratio_2, \psi, \barlambda ) \Big\vert = o_d(1) \, ,
\end{align}
where
\begin{equation}
 \begin{aligned}
\label{eq:R_test_form1}
    &\cR_{\test} ( F_* , \zeta , \gratio_1 , \gratio_2, \psi, \barlambda )\\
    := &~ F_\ell^2 \cdot \cB_{\test} (\zeta , \gratio_1 , \gratio_2, \psi, \barlambda )  + (F_{>\ell}^2 + \rho_\eps^2) \cdot \cV_{\test} (\zeta , \gratio_1 , \gratio_2, \psi, \barlambda ) +  F_{>\ell}^2 \, .
\end{aligned}   
\end{equation}
In Eq.~\eqref{eq:R_test_form1} and following the presentation of \cite{mei2022generalizationRF} (see Appendix \ref{sec:explicit_formula} for additional details), we defined
\begin{align}
    \label{eq:B_test_def}
     \cB_{\test} (\zeta , \gratio_1 , \gratio_2, \psi, \barlambda ) := &~ \frac{\cE_1 (\zeta , \gratio_1 , \gratio_2, \psi, \barlambda )}{\cE_0 (\zeta , \gratio_1 , \gratio_2, \psi, \barlambda) }\, , \\
     \label{eq:V_test_def}
     \cV_{\test} (\zeta , \gratio_1 , \gratio_2, \psi, \psi, \barlambda ) := &~ \frac{\cE_2 (\zeta , \gratio_1 , \gratio_2, \psi, \barlambda )}{\cE_0 (\zeta , \gratio_1 , \gratio_2, \psi, \barlambda) }\, ,
\end{align}
where
\begin{equation}
     \begin{aligned}
          \cE_0(\zeta , \gratio_1 , \gratio_2, \psi, \barlambda) := &~ -\psi^3 \chi^5 \zeta^6 + 3 \psi^2 \chi^4 \zeta^4 + ( \psi_1 \psi_2 - \psi_2 - \psi_1 +1) \psi \chi^3 \zeta^6 - 2 \psi \chi^3 \zeta^4 - 3 \psi \chi^3 \zeta^2 \\
          &~ + (\psi_1 + \psi_2 - 3\psi_1 \psi_2 +1) \chi^2\zeta^4 + 2\chi^2\zeta^2 + \chi^2 + 3 \frac{\theta_1 \theta_2}{\theta^2} \psi \chi \zeta^2  - \frac{\theta_1 \theta_2}{\theta^2} \, , \\
          \cE_1(\zeta , \gratio_1 , \gratio_2, \psi, \barlambda)  := &~ \psi_2 \psi \chi^3 \zeta^4 - \psi_2 \chi^2 \zeta^2 + \frac{\theta_1 \theta_2}{\theta^2} \psi \chi \zeta^2 - \frac{\theta_1 \theta_2}{\theta^2} \, ,\\
          \cE_2(\zeta , \gratio_1 , \gratio_2, \psi, \barlambda)  := &~ \psi^3 \chi^5 \zeta^6 - 3 \psi^2 \chi^4 \zeta^4 \\
          &~ + (\psi_1 - 1 ) \psi \chi^3 \zeta^6 + 2 \psi \chi^3 \zeta^4 + 3 \psi \chi^3 \zeta^2  + ( - \psi_1 - 1) \chi^2 \zeta^4 - 2 \chi^2 \zeta^2  - \chi^2 \, ,
    \end{aligned}
\end{equation}
with $\zeta := \frac{ \mu_{\ell} }{ \mu_{>\ell} }$ and $\chi := m_0(i(\theta_1 \theta^{-1} \lambda)^{1/2};\bzero ) \cdot \mu_{>\ell}^2$.
In particular, $\chi$ can also be expressed as
\begin{equation}
\label{eq:chi_def}
    \begin{aligned}
        \chi &:= \nu_1\big( i (\theta_1 \theta^{-1} \bar\lambda)^{1/2} \big) \cdot \nu_2\big( i (\theta_1 \theta^{-1} \bar\lambda)^{1/2} \big) \, ,
    \end{aligned}
\end{equation}
where $\Bar{\lambda} = \frac{\lambda}{\mu_{>\ell}^2}$, and $\nu_1(z)$ and $\nu_2(z)$ are defined as solutions of the following fixed point equations (these are analogous to the fixed points in \cite{mei2022generalizationRF} with slightly different scaling):
 \begin{definition}[Fixed points formula]\label{def:fixedpoint_nu_0} We define $\nu_1,\nu_2 : \C_+ \to \C_+$ to be the unique functions that satisfy the following conditions:
 \begin{itemize}
     \item[(i)] $\nu_1,\nu_2$ are analytic functions on $\C_+$;
     \item[(ii)] For $\Im (z) > 0$, $\nu_1(z) , \nu_2 (z)$ are fixed points of 
     \begin{equation}
    \label{eq:fixed_point_nu_rescaled}
         \begin{aligned}
              \nu_1 =&~ \frac{\theta_1}{\theta} \Big( -z - \nu_2 - \frac{\zeta^2 \nu_2}{1 - \zeta^2 \psi \nu_1 \nu_2} \Big)^{-1} \, ,\\
              \nu_2 =&~ \frac{\theta_2}{\theta} \Big( -z - \nu_1 - \frac{\zeta^2 \nu_1}{1 - \zeta^2 \psi \nu_1 \nu_2} \Big)^{-1} \,.
         \end{aligned}
     \end{equation}
     \item[(iii)] $(\nu_1 (z) , \nu_2 (z) )$ is the unique fixed point of Eq.~\eqref{eq:fixed_point_nu} with $| \nu_1 ( z) | \leq \theta_1 \theta^{-1} / \Im (z)$, $| \nu_2 ( z) | \leq \theta_2 \theta^{-1} / \Im (z)$ for $\Im (z)> C$ and $C$ sufficiently large.
 \end{itemize}
 \end{definition}
From Eqs.~\eqref{eq:fixedpoint_eq} and \eqref{eq:fixed_point_nu_rescaled}, we can also check that
\begin{equation}
\label{eq:v_m_relation}
    \begin{aligned}
        \nu_1(z) :=&~ m_1(\mu_{>\ell}z;\bzero) \cdot \mu_{>\ell} \, ,\\
        \nu_2(z) :=&~ m_2(\mu_{>\ell}z;\bzero) \cdot \mu_{>\ell} \, .
    \end{aligned}
\end{equation}

Finally, using Lemma \ref{lem:equi_fixedpoints} in Appendix \ref{sec:explicit_formula} on the correspondence between fixed points in Definition \ref{def:fixedpoint_nu_0} and fixed points in Definition \ref{def:fixedPointsTau}, we obtain  Eq.~\eqref{eq:asymp_testerr}.

\bibliographystyle{amsalpha}
\bibliography{refs.bbl}

\newcommand{\etalchar}[1]{$^{#1}$}
\providecommand{\bysame}{\leavevmode\hbox to3em{\hrulefill}\thinspace}
\providecommand{\MR}{\relax\ifhmode\unskip\space\fi MR }
\providecommand{\MRhref}[2]{%
  \href{http://www.ams.org/mathscinet-getitem?mr=#1}{#2}
}
\providecommand{\href}[2]{#2}
\begin{thebibliography}{MWW{\etalchar{+}}20}

\bibitem[ADH{\etalchar{+}}19]{arora2019fine}
Sanjeev Arora, Simon Du, Wei Hu, Zhiyuan Li, and Ruosong Wang,
  \emph{Fine-grained analysis of optimization and generalization for
  overparameterized two-layer neural networks}, International Conference on
  Machine Learning, PMLR, 2019, pp.~322--332.

\bibitem[ALP22]{adlam2022random}
Ben Adlam, Jake~A Levinson, and Jeffrey Pennington, \emph{A random matrix
  perspective on mixtures of nonlinearities in high dimensions}, International
  Conference on Artificial Intelligence and Statistics, PMLR, 2022,
  pp.~3434--3457.

\bibitem[AP20]{adlam2020neural}
Ben Adlam and Jeffrey Pennington, \emph{The neural tangent kernel in high
  dimensions: Triple descent and a multi-scale theory of generalization},
  International Conference on Machine Learning, PMLR, 2020, pp.~74--84.

\bibitem[ASS20]{advani2020high}
Madhu~S Advani, Andrew~M Saxe, and Haim Sompolinsky, \emph{High-dimensional
  dynamics of generalization error in neural networks}, Neural Networks
  \textbf{132} (2020), 428--446.

\bibitem[AZLS19]{allen2019convergence2}
Zeyuan Allen-Zhu, Yuanzhi Li, and Zhao Song, \emph{A convergence theory for
  deep learning via over-parameterization}, International Conference on Machine
  Learning, PMLR, 2019, pp.~242--252.

\bibitem[Bac17a]{bach2017breaking}
Francis Bach, \emph{Breaking the curse of dimensionality with convex neural
  networks}, The Journal of Machine Learning Research \textbf{18} (2017),
  no.~1, 629--681.

\bibitem[Bac17b]{bach2017equivalence}
\bysame, \emph{On the equivalence between kernel quadrature rules and random
  feature expansions}, The Journal of Machine Learning Research \textbf{18}
  (2017), no.~1, 714--751.

\bibitem[Bar93]{barron1993universal}
Andrew~R Barron, \emph{Universal approximation bounds for superpositions of a
  sigmoidal function}, IEEE Transactions on Information theory \textbf{39}
  (1993), no.~3, 930--945.

\bibitem[BB07]{bottou2007tradeoffs}
L{\'e}on Bottou and Olivier Bousquet, \emph{The tradeoffs of large scale
  learning}, Advances in neural information processing systems \textbf{20}
  (2007).

\bibitem[Bec92]{beckner1992sobolev}
William Beckner, \emph{Sobolev inequalities, the poisson semigroup, and
  analysis on the sphere sn.}, Proceedings of the National Academy of Sciences
  \textbf{89} (1992), no.~11, 4816--4819.

\bibitem[BHMM19]{belkin2019reconciling}
Mikhail Belkin, Daniel Hsu, Siyuan Ma, and Soumik Mandal, \emph{Reconciling
  modern machine-learning practice and the classical bias--variance trade-off},
  Proceedings of the National Academy of Sciences \textbf{116} (2019), no.~32,
  15849--15854.

\bibitem[BHX20]{belkin2020two}
Mikhail Belkin, Daniel Hsu, and Ji~Xu, \emph{Two models of double descent for
  weak features}, SIAM Journal on Mathematics of Data Science \textbf{2}
  (2020), no.~4, 1167--1180.

\bibitem[BLLT20]{bartlett2020benign}
Peter~L Bartlett, Philip~M Long, G{\'a}bor Lugosi, and Alexander Tsigler,
  \emph{Benign overfitting in linear regression}, Proceedings of the National
  Academy of Sciences \textbf{117} (2020), no.~48, 30063--30070.

\bibitem[BMM18]{belkin2018understand}
Mikhail Belkin, Siyuan Ma, and Soumik Mandal, \emph{To understand deep learning
  we need to understand kernel learning}, arXiv preprint arXiv:1802.01396
  (2018).

\bibitem[BMR21]{bartlett2021deep}
Peter~L Bartlett, Andrea Montanari, and Alexander Rakhlin, \emph{Deep learning:
  a statistical viewpoint}, Acta numerica \textbf{30} (2021), 87--201.

\bibitem[BRT19]{belkin2019does}
Mikhail Belkin, Alexander Rakhlin, and Alexandre~B Tsybakov, \emph{Does data
  interpolation contradict statistical optimality?}, The 22nd International
  Conference on Artificial Intelligence and Statistics, PMLR, 2019,
  pp.~1611--1619.

\bibitem[CDV07]{caponnetto2007optimal}
Andrea Caponnetto and Ernesto De~Vito, \emph{Optimal rates for the regularized
  least-squares algorithm}, Foundations of Computational Mathematics \textbf{7}
  (2007), 331--368.

\bibitem[Chi11]{chihara2011introduction}
Theodore~S Chihara, \emph{An introduction to orthogonal polynomials}, Courier
  Corporation, 2011.

\bibitem[CL22]{couillet2022random}
Romain Couillet and Zhenyu Liao, \emph{Random matrix methods for machine
  learning}, Cambridge University Press, 2022.

\bibitem[CM22]{cheng2022dimension}
Chen Cheng and Andrea Montanari, \emph{Dimension free ridge regression}, arXiv
  preprint arXiv:2210.08571 (2022).

\bibitem[CS13]{cheng2013spectrum}
Xiuyuan Cheng and Amit Singer, \emph{The spectrum of random inner-product
  kernel matrices}, Random Matrices: Theory and Applications \textbf{2} (2013),
  no.~04, 1350010.

\bibitem[CW78]{craven1978smoothing}
Peter Craven and Grace Wahba, \emph{Smoothing noisy data with spline functions:
  estimating the correct degree of smoothing by the method of generalized
  cross-validation}, Numerische mathematik \textbf{31} (1978), no.~4, 377--403.

\bibitem[DLL{\etalchar{+}}19]{du2019gradient}
Simon Du, Jason Lee, Haochuan Li, Liwei Wang, and Xiyu Zhai, \emph{Gradient
  descent finds global minima of deep neural networks}, International
  Conference on Machine Learning, PMLR, 2019, pp.~1675--1685.

\bibitem[DX13]{dai2013approximation}
Feng Dai and Yuan Xu, \emph{Approximation theory and harmonic analysis on
  spheres and balls}, vol.~23, Springer, 2013.

\bibitem[DZPS18]{du2018gradient}
Simon~S Du, Xiyu Zhai, Barnabas Poczos, and Aarti Singh, \emph{Gradient descent
  provably optimizes over-parameterized neural networks}, International
  Conference on Learning Representations, 2018.

\bibitem[EK{\etalchar{+}}10]{el2010spectrum}
Noureddine El~Karoui et~al., \emph{The spectrum of kernel random matrices}, The
  Annals of Statistics \textbf{38} (2010), no.~1, 1--50.

\bibitem[EY17]{erdHos2017dynamical}
L{\'a}szl{\'o} Erd{\H{o}}s and Horng-Tzer Yau, \emph{A dynamical approach to
  random matrix theory}, vol.~28, American Mathematical Soc., 2017.

\bibitem[FM19]{fan2019spectral}
Zhou Fan and Andrea Montanari, \emph{The spectral norm of random inner-product
  kernel matrices}, Probability Theory and Related Fields \textbf{173} (2019),
  no.~1-2, 27--85.

\bibitem[GHW79]{golub1979generalized}
Gene~H Golub, Michael Heath, and Grace Wahba, \emph{Generalized
  cross-validation as a method for choosing a good ridge parameter},
  Technometrics \textbf{21} (1979), no.~2, 215--223.

\bibitem[GLR{\etalchar{+}}22]{goldt2022gaussian}
Sebastian Goldt, Bruno Loureiro, Galen Reeves, Florent Krzakala, Marc
  M{\'e}zard, and Lenka Zdeborov{\'a}, \emph{The gaussian equivalence of
  generative models for learning with shallow neural networks}, Mathematical
  and Scientific Machine Learning, PMLR, 2022, pp.~426--471.

\bibitem[GMMM20]{ghorbani2020neural}
Behrooz Ghorbani, Song Mei, Theodor Misiakiewicz, and Andrea Montanari,
  \emph{When do neural networks outperform kernel methods?}, Advances in Neural
  Information Processing Systems \textbf{33} (2020), 14820--14830.

\bibitem[GMMM21]{ghorbani2021linearized}
\bysame, \emph{Linearized two-layers neural networks in high dimension}, The
  Annals of Statistics \textbf{49} (2021), no.~2, 1029--1054.

\bibitem[HL22a]{hu2022sharp}
Hong Hu and Yue~M Lu, \emph{Sharp asymptotics of kernel ridge regression beyond
  the linear regime}, arXiv preprint arXiv:2205.06798 (2022).

\bibitem[HL22b]{hu2022universality}
\bysame, \emph{Universality laws for high-dimensional learning with random
  features}, IEEE Transactions on Information Theory \textbf{69} (2022), no.~3,
  1932--1964.

\bibitem[HMRT22]{hastie2022surprises}
Trevor Hastie, Andrea Montanari, Saharon Rosset, and Ryan~J Tibshirani,
  \emph{Surprises in high-dimensional ridgeless least squares interpolation},
  The Annals of Statistics \textbf{50} (2022), no.~2, 949--986.

\bibitem[HZS06]{huang2006extreme}
Guang-Bin Huang, Qin-Yu Zhu, and Chee-Kheong Siew, \emph{Extreme learning
  machine: theory and applications}, Neurocomputing \textbf{70} (2006),
  no.~1-3, 489--501.

\bibitem[JGH18]{jacot2018neural}
Arthur Jacot, Franck Gabriel, and Cl{\'e}ment Hongler, \emph{Neural tangent
  kernel: Convergence and generalization in neural networks}, Advances in
  neural information processing systems, 2018, pp.~8571--8580.

\bibitem[LCM20]{liao2020random}
Zhenyu Liao, Romain Couillet, and Michael~W Mahoney, \emph{A random matrix
  analysis of random fourier features: beyond the gaussian kernel, a precise
  phase transition, and the corresponding double descent}, Advances in Neural
  Information Processing Systems \textbf{33} (2020), 13939--13950.

\bibitem[LL18]{li2018learning}
Yuanzhi Li and Yingyu Liang, \emph{Learning overparameterized neural networks
  via stochastic gradient descent on structured data}, Advances in Neural
  Information Processing Systems, 2018, pp.~8157--8166.

\bibitem[LLC18]{louart2018random}
Cosme Louart, Zhenyu Liao, and Romain Couillet, \emph{A random matrix approach
  to neural networks}, The Annals of Applied Probability \textbf{28} (2018),
  no.~2, 1190--1248.

\bibitem[LR{\etalchar{+}}20]{liang2020just}
Tengyuan Liang, Alexander Rakhlin, et~al., \emph{Just interpolate: Kernel
  ``ridgeless'' regression can generalize}, Annals of Statistics \textbf{48}
  (2020), no.~3, 1329--1347.

\bibitem[LXS{\etalchar{+}}19]{lee2019wide}
Jaehoon Lee, Lechao Xiao, Samuel Schoenholz, Yasaman Bahri, Roman Novak, Jascha
  Sohl-Dickstein, and Jeffrey Pennington, \emph{Wide neural networks of any
  depth evolve as linear models under gradient descent}, Advances in neural
  information processing systems \textbf{32} (2019), 8572--8583.

\bibitem[LY22]{lu2022equivalence}
Yue~M Lu and Horng-Tzer Yau, \emph{An equivalence principle for the spectrum of
  random inner-product kernel matrices}, arXiv preprint arXiv:2205.06308
  (2022).

\bibitem[Mai99]{maiorov1999best}
Vitaly~E Maiorov, \emph{On best approximation by ridge functions}, Journal of
  Approximation Theory \textbf{99} (1999), no.~1, 68--94.

\bibitem[Mha96]{mhaskar1996neural}
Hrushikesh~N Mhaskar, \emph{Neural networks for optimal approximation of smooth
  and analytic functions}, Neural computation \textbf{8} (1996), no.~1,
  164--177.

\bibitem[Mis22]{misiakiewicz2022spectrum}
Theodor Misiakiewicz, \emph{Spectrum of inner-product kernel matrices in the
  polynomial regime and multiple descent phenomenon in kernel ridge
  regression}, arXiv preprint arXiv:2204.10425 (2022).

\bibitem[MM22]{mei2022generalizationRF}
Song Mei and Andrea Montanari, \emph{The generalization error of random
  features regression: Precise asymptotics and the double descent curve},
  Communications on Pure and Applied Mathematics \textbf{75} (2022), no.~4,
  667--766.

\bibitem[MMM22]{mei2022generalization}
Song Mei, Theodor Misiakiewicz, and Andrea Montanari, \emph{Generalization
  error of random feature and kernel methods: hypercontractivity and kernel
  matrix concentration}, Applied and Computational Harmonic Analysis
  \textbf{59} (2022), 3--84.

\bibitem[MRH{\etalchar{+}}18]{matthews2018gaussian}
Alexander G de~G Matthews, Mark Rowland, Jiri Hron, Richard~E Turner, and
  Zoubin Ghahramani, \emph{Gaussian process behaviour in wide deep neural
  networks}, arXiv preprint arXiv:1804.11271 (2018).

\bibitem[MS22]{montanari2022universality}
Andrea Montanari and Basil~N Saeed, \emph{Universality of empirical risk
  minimization}, Conference on Learning Theory, PMLR, 2022, pp.~4310--4312.

\bibitem[MVSS20]{muthukumar2020harmless}
Vidya Muthukumar, Kailas Vodrahalli, Vignesh Subramanian, and Anant Sahai,
  \emph{Harmless interpolation of noisy data in regression}, IEEE Journal on
  Selected Areas in Information Theory \textbf{1} (2020), no.~1, 67--83.

\bibitem[MWW{\etalchar{+}}20]{ma2020towards}
Chao Ma, Stephan Wojtowytsch, Lei Wu, et~al., \emph{Towards a mathematical
  understanding of neural network-based machine learning: what we know and what
  we don't}, arXiv preprint arXiv:2009.10713 (2020).

\bibitem[MZ22]{montanari2022interpolation}
Andrea Montanari and Yiqiao Zhong, \emph{The interpolation phase transition in
  neural networks: Memorization and generalization under lazy training}, The
  Annals of Statistics \textbf{50} (2022), no.~5, 2816--2847.

\bibitem[Nea95]{neal1995bayesian}
Radford~M Neal, \emph{Bayesian learning for neural networks}, Ph.D. thesis,
  Citeseer, 1995.

\bibitem[Ng00]{ng2000cs229}
Andrew Ng, \emph{Cs229 lecture notes}, CS229 Lecture notes \textbf{1} (2000),
  no.~1, 1--3.

\bibitem[NXL{\etalchar{+}}18]{novak2018bayesian}
Roman Novak, Lechao Xiao, Jaehoon Lee, Yasaman Bahri, Greg Yang, Jiri Hron,
  Daniel~A Abolafia, Jeffrey Pennington, and Jascha Sohl-Dickstein,
  \emph{Bayesian deep convolutional networks with many channels are gaussian
  processes}, arXiv preprint arXiv:1810.05148 (2018).

\bibitem[Pin99]{pinkus1999approximation}
Allan Pinkus, \emph{Approximation theory of the mlp model in neural networks},
  Acta numerica \textbf{8} (1999), 143--195.

\bibitem[PW17]{pennington2017nonlinear}
Jeffrey Pennington and Pratik Worah, \emph{Nonlinear random matrix theory for
  deep learning}, Advances in Neural Information Processing Systems, 2017,
  pp.~2637--2646.

\bibitem[RR08a]{rahimi2008random}
Ali Rahimi and Benjamin Recht, \emph{Random features for large-scale kernel
  machines}, Advances in neural information processing systems, 2008,
  pp.~1177--1184.

\bibitem[RR08b]{rahimi2008weighted}
\bysame, \emph{Weighted sums of random kitchen sinks: Replacing minimization
  with randomization in learning}, Advances in neural information processing
  systems \textbf{21} (2008).

\bibitem[RR17]{rudi2017generalization}
Alessandro Rudi and Lorenzo Rosasco, \emph{Generalization properties of
  learning with random features}, Advances in neural information processing
  systems \textbf{30} (2017).

\bibitem[SH20]{schmidt2020nonparametric}
Johannes Schmidt-Hieber, \emph{Nonparametric regression using deep neural
  networks with relu activation function}, Annals of statistics \textbf{48}
  (2020), no.~4, 1875--1897.

\bibitem[Sze39]{szeg1939orthogonal}
Gabor Szeg, \emph{Orthogonal polynomials}, vol.~23, American Mathematical Soc.,
  1939.

\bibitem[Tro12]{tropp2012user}
Joel~A Tropp, \emph{User-friendly tail bounds for sums of random matrices},
  Foundations of computational mathematics \textbf{12} (2012), no.~4, 389--434.

\bibitem[Vap99]{vapnik1999overview}
Vladimir~N Vapnik, \emph{An overview of statistical learning theory}, IEEE
  transactions on neural networks \textbf{10} (1999), no.~5, 988--999.

\bibitem[Ver18]{vershynin2018high}
Roman Vershynin, \emph{High-dimensional probability: {An} introduction with
  applications in data science}, Cambridge University Press, 2018.

\bibitem[Wai19]{wainwright2019high}
Martin~J Wainwright, \emph{High-dimensional statistics: A non-asymptotic
  viewpoint}, vol.~48, Cambridge university press, 2019.

\bibitem[XHM{\etalchar{+}}22]{xiao2022precise}
Lechao Xiao, Hong Hu, Theodor Misiakiewicz, Yue Lu, and Jeffrey Pennington,
  \emph{Precise learning curves and higher-order scalings for dot-product
  kernel regression}, Advances in Neural Information Processing Systems
  \textbf{35} (2022), 4558--4570.

\bibitem[ZBH{\etalchar{+}}21]{zhang2021understanding}
Chiyuan Zhang, Samy Bengio, Moritz Hardt, Benjamin Recht, and Oriol Vinyals,
  \emph{Understanding deep learning (still) requires rethinking
  generalization}, Communications of the ACM \textbf{64} (2021), no.~3,
  107--115.

\end{thebibliography}

\clearpage

\appendix

\section{Spherical Harmonics and Gegenbauer
Polynomials}
\label{sec:Spherical-Harmonics}

In this appendix, we give a brief overview of some properties of the
spherical harmonics and Gegenbauer polynomials that are frequently
used in our analysis. We refer the reader to \cite{szeg1939orthogonal,chihara2011introduction,dai2013approximation} for more in-depth expositions.

Recall that we denote $\S^{d-1} (r) = \{ \bx \in \R^d : \| \bx \|_2 = r \}$ the sphere of radius $r$ in $\R^d$. Throughout this paper, we choose the normalization $\bx \in \S^{d-1} (\sqrt{d})$ for the covariates and $\bw \in \S^{d-1} (1)$ for the weights, such that the input $\< \bw , \bx \>$ to the activation function is of order $1$. Without loss of generality, we will define the spherical harmonics and Gegenbauer polynomials on $\S^{d-1} (\sqrt{d})$, and simply rescale the inputs $Y_{ks} (\sqrt{d} \cdot \bw)$ and $q_k (\sqrt{d} \cdot \< \bw, \bw'\>)$ for $\bw,\bw' \in \S^{d-1} (1)$.

\paragraph{Spherical harmonics:}
Spherical harmonics are defined as homogeneous harmonic polynomials restricted
to $\S^{d-1}(\sqrt{d})$. Namely, a polynomial $P(\bx)$ with $\bx\in \S^{d-1}(\sqrt{d})$
is a spherical harmonic if and only if (1) $P(t\bx)=t^{k}P(\bx)$,
for any $t\in\R$, (2) $\Delta P=0$, where $\Delta$ is the Laplace
operator. Let $V_{d,k}$ be the space of all degree-$k$
spherical harmonics in $d$ dimension. We denote $N_k := \dim(V_{d,k})$, which is given by
\begin{equation}
N_k=\begin{cases}
1 & k=0\, ,\\
d & k=1\, ,\\
\frac{d+2k-2}{k}{d+k-3 \choose k-1} & k\geq2 \, .
\end{cases}
\label{eq:usdim_def}
\end{equation}
In particular, we have $|N_k/{d \choose k}-1|\leq\frac{C_{k}}{d}$,
where $C_{k}>0$ is some constant that only depends on $k$, i.e.,
$N_k$ is approximately equal to the binomial coefficients $d \choose k$ up
to an $O(d^{-1})$ correction. For each $V_{d,k}$, $k\geq0$,  we choose an orthonormal basis of spherical harmonics that we denote $\{Y_{ks}(\bx)\}_{s=1}^{N_k}$, so that
\begin{equation}
\int_{\S^{d-1}}Y_{ks}(\bx)Y_{ks'}(\bx)\tau_{d}(\de\bx)=\charfn_{s=s'}\, .\label{eq:orthogonal_Y}
\end{equation}
where $\tau_{d}$ denotes the uniform distribution over $\S^{d-1}(\sqrt{d})$.

\paragraph{Gegenbauer polynomials:} Gegenbauer polynomials $\{q_k\}_{k=0}^{\infty}$ are orthonormal polynomials on $L^{2}([-\sqrt{d},\sqrt{d}],\tau_{d,1})$, i.e., $q_k$ is a degree-$k$ polynomial and
\[
\int_{-\sqrt{d}}^{\sqrt{d}}q_k(x)q_{\ell}(x)\tau_{d,1}(\de x)=\charfn_{k=\ell} \, ,
\]
where $\tau_{d,1}$ is the distribution of $\< \bx , \be_{1}\>$
with $\bx\sim\tau_{d}$ and $\be_1$ an arbitrary unit vector. Note that $\tau_{d,1}$ has the explicit form
$\tau_{d,1}(\de x)=\frac{\omega_{d-2}}{\sqrt{d}\omega_{d-1}}(1-x^{2}/d)^{\frac{d-3}{2}}\de x$,
where $\omega_{d-1}=\frac{2\pi^{d/2}}{\Gamma(d/2)}$ is the surface
area of $\mathcal{S}^{d-1}$. The moments of $\tau_{d,1}$
are given by
\begin{equation}
\E_{\tau_{d,1}}X^{m}=\begin{cases}
0 & m=2k+1\, ,\\
\frac{(2k-1)!!}{\prod_{0\leq i<k}(1+2i/d)} & m=2k\, .
\end{cases}\label{eq:moments_of_usp}
\end{equation}
Using Eq.~\eqref{eq:moments_of_usp}, we can explicitly write out the
first three $q_k(x)$ via the Gram-Schmidt procedure:
\begin{align*}
q_0(x) & =1 \, , & q_1(x) & =x\, , & q_2(x) & =\frac{1}{\sqrt{2}}\sqrt{\frac{d+2}{d-1}}(x^{2}-1)\, .
\end{align*}

\paragraph{Addition theorem:} A crucial property of Gegenbauer polynomials is the following correspondence between   the degree-$k$ Gegenbauer polynomial $q_k(x)$ and the degree-$k$ spherical harmonics $\{Y_{ks}(\bx)\}_{s=1}^{N_k}$. For any $\bx,\bx'\in\S^{d-1}(\sqrt{d})$, we have
\begin{equation}
q_k(\< \bx , \bx' \>/\sqrt{d})=\frac{1}{\sqrt{N_k}}\sum_{i=1}^{N_k}Y_{ks}(\bx)Y_{ks}(\bx')\, ,\label{eq:q_Y_relation}
\end{equation}
which is also known as the addition theorem. 
In particular, for any $\bx\in\S^{d-1}$,
\begin{equation}
\frac{1}{\sqrt{N_k}}\sum_{i=1}^{N_k}Y_{ks}(\bx)^{2}=q_k(\sqrt{d})=\sqrt{N_k}\, ,\label{eq:diagonal_Y_constant}
\end{equation}
which follows from the fact that for any $\bx\in\S^{d-1}$,
\begin{align*}
q_k(\sqrt{d}) & =q_k(\|\bx\|^{2}_2/\sqrt{d})\\
 & =\int_{\S^{d-1}}q_k(\|\bx\|^{2}_2/\sqrt{d})\tau_{d}(\de \bx)\\
 & \teq{\text{(a)}}\frac{1}{\sqrt{N_k}}\sum_{s=1}^{N_k}\int_{\S^{d-1}}Y_{ks}(\bx)^{2}\tau_{d}(\de \bx)\\
 & \teq{\text{(b)}}\sqrt{N_k},
\end{align*}
where (a) follows from Eq.~\eqref{eq:q_Y_relation} and (b) follows from
Eq.~\eqref{eq:orthogonal_Y}.

As a consequence of the addition theorem, consider a function $f : [-\sqrt{d},\sqrt{d}] \to \R$ whose decomposition in the orthonormal Gegenbauer polynomial basis is given by
\[
f ( x ) = \sum_{k=0}^\infty \varsigma_k q_k (x) \, , \qquad \varsigma_k = \int_{-\sqrt{d}}^{\sqrt{d}} f(x) q_k (x) \tau_{d,1} (\de x)\, ,
\]
then the function $\Tilde{f} : \S^{d-1} (1) \times \S^{d-1} (\sqrt{d}) \to \R$, $\Tilde{f}(\bw,\bx) = f(\< \bw, \bx \>)$, admits the following eigendecomposition in the (tensor product) spherical harmonic basis
\[
f ( \< \bw , \bx \> ) = \sum_{k=0}^\infty \varsigma_k \sqrt{N_k} \sum_{s \in [N_k]} Y_{ks} (\sqrt{d} \cdot \bw) Y_{ks} ( \bx ) \, .
\]

\paragraph{Hermite polynomial:} Hermite polynomials $\{\He_k\}_{k=0}^{\infty}$ are orthonormal polynomials on $L^2(\R, \tau_{\text{g}})$, where $\tau_{\text{g}}$ is the standard Gaussian measure. 
We can explicitly write out the
first three $q_k(x)$ via the Gram-Schmidt procedure:
\begin{align*}
{\He}_0(x) & =1\, , & {\He}_1(x) & =x\, , & {\He}_2(x) & =\frac{1}{\sqrt{2}}(x^{2}-1)\, .
\end{align*}

Hermite polynomials can be viewed as the limit of Gegenbauer polynomial when $d\to \infty$. Indeed, from Eq.~\eqref{eq:moments_of_usp}, we can see that for any fixed $k$, $\lim_{d\to\infty}\E_{\tau_{d-1,1}}X^{m}=\E_{\tau_{\text{g}}}X^{m}$. Hence, $\tau_{d-1,1}$ converges weakly to $\tau_{\text{g}}$. Consequently, we can show the following result (see for example \cite[Lemma 12]{lu2022equivalence}): for function $f:\R \mapsto \R$ satisfying $|f(x)|\leq c_1 e^{c_2|x|}$ for all $x \in \R$, for some constants $c_1,c_2,c_3 >0$, then it holds for any $k$ that
\begin{align}
\label{eq:q_He_coeff_conv}
\big| \E_{G\sim\tau_{\text{g}}}[f(G) {\He}_k(G)] - \E_{Z\sim\tau_{d,1}}[f(Z) q_k(Z)] \big| = o_d(1)\, .
\end{align}

\clearpage

\section{Proof of Propositions \ref{prop:concentration_on_expectation}, \ref{prop:tech_bounds_ell} and \ref{prop:high_degree_concentration}}

In this appendix, we prove the technical propositions used to simplify the expressions of the test error in Section \ref{sec:proof_outline}. In Appendix \ref{sec:concentration_on_expectation_proof}, we prove the concentration of the test error over the randomness in the target function and label noise (Proposition \ref{prop:concentration_on_expectation}). Appendix \ref{sec:non-RMT} prove the concentration to zero of the low-degree part (Proposition \ref{prop:tech_bounds_ell}). Finally, Appendix \ref{sec:high_degree_concentration_proof} prove concentration of the high-degree part (Proposition \ref{prop:high_degree_concentration}). Throughout, we will use matrix concentration results whose proofs are gathered in Appendix \ref{sec:matrix_concentration}.

\subsection{Proof of Proposition \ref{prop:concentration_on_expectation}}
\label{sec:concentration_on_expectation_proof}

We begin by decomposing $R_{\test} ( f_{*} ; \bX , \bW , \beps, \lambda)$ in \eqref{eq:decompo_Rtest} as follows
    \begin{align*}
        R_{\test} ( f_{*,d} ; \bX , \bW , \beps, \lambda) = T_1 - 2 (T_2 + T_3) + T_4 + 2 T_5 + T_6\, ,
    \end{align*}
    where 
    \begin{align*}
        T_1 =&~ \E_{\bx} [ f_{*}(\bx)^2 ] \, ,\\
        T_2 =&~ \bf^\sT \bZ \bR \bV \, ,\\
        T_3 =&~ \beps^\sT \bZ \bR \bV \, , \\
        T_4 =&~ \beps^\sT \bZ \bR \bU \bR \bZ^\sT \beps \, ,\\
        T_5 =&~ \beps^\sT \bZ \bR \bU \bR \bZ^\sT \bf \, , \\
        T_6 =&~ \bf^\sT \bZ \bR \bU \bR \bZ^\sT \bf \, .
    \end{align*}
    
    It suffices to show $\E[\Var(T_i \mid \bX,\bW)] = o_d(1)$ for each $1 \leq i \leq 6$. In what follows, we present the proof for $T_2$. The proof for the other $T_i$ follows from similar arguments and we omit them for brevity.

\paragraph*{Proof for the term $T_2$.}
    To start, we can decompose $\bf$ and $\bV$ into low-degree and high-degree parts $\bf = \bf_{<\ell} + \bf_{\geq\ell}$ and $\bV = \bV_{<\ell} + \bV_{\geq\ell}$, where
    \[
    \bf_{<\ell} = \bPsi_{< \ell} \bbeta^*_{d}\, , \qquad \bf_{\geq\ell} = \sum_{k\geq\ell} \bPsi_{k} \tbbeta_{d,k}\, , \qquad \bV_{<\ell} = \frac{1}{\sqrt{p} } \bPhi_{<\ell} \bD_{<\ell} \bbeta^*_{d}\, , \qquad \bV_{\geq\ell} = \frac{1}{\sqrt{p} } \sum_{k\geq\ell} \xi_k \bPhi_{k} \tbbeta_{d,k} \, .
    \]
    By Cauchy-Schwarz inequality, we have
    \begin{equation}
    \label{eq:T2_varbd}
        \begin{aligned}
           \E [\Var(T_2 \mid \bX, \bW)] \leq& 3\big\{~ \E[\Var(\bf_{\geq \ell}^\sT \bZ \bR \bV_{<\ell} \mid \bX, \bW)] + \E[\Var(\bf_{<\ell}^\sT \bZ \bR \bV_{\geq\ell} \mid \bX, \bW)] \\
           &\hspace{1em}+ \E[\Var(\bf_{\geq\ell}^\sT \bZ \bR \bV_{\geq\ell} \mid \bX, \bW)] ~\big\} \, ,
        \end{aligned}
    \end{equation}
    where we used that $\bbeta^*_{d}$ is deterministic. Thus, it suffices to bound the three terms on the right-hand side separately. 
    
     For the first term in Eq.~\eqref{eq:T2_varbd}, we can write
    \begin{align}
    \label{eq:T2_firstterm_varbd}
        \Var(\bf_{\geq \ell}^\sT \bZ \bR \bV_{<\ell} \mid \bX, \bW) &= \frac{1}{p} (\bbeta^*_{d})^{\sT} \bD_{<\ell} \bPhi_{<\ell}^{\sT} \bR \bZ^\sT \bH_F \bZ \bR \bPhi_{<\ell} \bD_{<\ell} \bbeta^*_{d}\, ,
    \end{align}
    where we recall that we denote $\bH_F = \E_{f_*} [ \bf \bf^\sT ] = \sum_{k = \ell}^\infty F_{k}^2 \bQ_k^\bX$.
    By Eq.~\eqref{eq:low_f_b2} in Proposition \ref{prop:tech_bounds_ell} and switching the role of $\bX$ and $\bW$, we have
    \begin{align}
    \label{eq:T2_firstterm_varbd_1}
        \frac{1}{p} \E \opnorm{ \bPhi_{<\ell}^{\sT} \bR \bZ^\sT \bH_F \bZ \bR \bPhi_{<\ell} } = o_{d}(1) \, .
    \end{align}
    Besides, $\|\bbeta_{d}^{*}\|_2^2 = O_d(1)$ by Assumption \ref{ass:random_f_star}.(a), and $\opnorm{\bD_{<\ell}} = O_d(1)$ by the assumption that $\| \sigma \|_{L^2} < \infty$. Therefore, together with Eq.~\eqref{eq:T2_firstterm_varbd} we deduce that
    \begin{align}
    \label{eq:T2_firstterm_varbd_00}
        \E[\Var(\bf_{\geq \ell}^\sT \bZ \bR \bV_{<\ell} \mid \bX, \bW)] = o_d(1) \, .
    \end{align}

    For the second term in Eq.~\eqref{eq:T2_varbd}, we have
    \begin{align}
        \Var(\bf_{\geq \ell}^\sT \bZ \bR \bV_{<\ell} \mid \bX, \bW) &= \frac{1}{p} (\bbeta^*_{d})^{\sT} \bPhi_{<\ell}^{\sT} \bZ \bR  \left( \sum_{k\geq \ell} {F_k^2} \xi_k^2 \bQ_{k}^{\bW} \right) \bR \bZ^\sT \bPhi_{<\ell} \bbeta^*_{d}\, .
    \end{align}
    Similar as Eq.~\eqref{eq:T2_firstterm_varbd_1}, we can get
    \begin{align*}
        \frac{N_{\ell} }{p} \E \left\| \bPhi_{<\ell}^{\sT} \bZ \bR  \left( \sum_{k\geq \ell} {F_k^2} \xi_k^2 \bQ_{k}^{\bW} \right) \bR \bZ^\sT \bPhi_{<\ell} \right\|_{\op} = o_{d}(1)\, ,
    \end{align*}
    where we used that $\xi_k^2 = O_d(N_\ell^{-1})$ for $k \geq \ell$. We deduce that
    \begin{align}
        \label{eq:T2_secondterm_varbd_0}
        \E[\Var(\bf_{\geq \ell}^\sT \bZ \bR \bV_{<\ell} \mid \bX, \bW)] = o_d(1).
    \end{align}
    
    Finally, for the third term in Eq.~\eqref{eq:T2_varbd}, we can first apply Lemma C.8 in \cite{mei2022generalizationRF} to get
    \begin{equation}
    \label{eq:var_T2_highdegree_bd}
        \begin{aligned}
            \Var(\bf_{\geq \ell}^\sT \bZ \bR \bV_{\geq \ell} \mid \bX, \bW) 
            \leq&~ 
            C \sum_{k\geq\ell} \frac{\xi_k^2 F_k^4}{p} \left[ \Tr( \bQ_{k}^{\bX} \bZ \bR \bQ_{k}^{\bW} \bR \bZ^\sT )  + \Tr( \tfrac{\bPhi_{k} \bPsi_{k}^\sT}{N_k} \bZ \bR \tfrac{\bPhi_{k} \bPsi_{k}^\sT}{N_k} \bZ \bR )\right] \\
            &~ + \sum_{ \substack{k\neq k' \\ k,k'\geq\ell} } \frac{ \xi_{k'}^2 F_k^2 F_{k'}^2 }{{p} } \Tr( \bQ_{k}^{\bX} \bZ \bR \bQ_{k'}^{\bW} \bR \bZ^\sT   )  \, .
        \end{aligned}
    \end{equation}    
    Then based on Eq.~\eqref{eq:var_T2_highdegree_bd}, we show that 
    \begin{align}
    \label{eq:T2_highdegree_smallvariance}
        \E [ \Var(\bf_{\geq \ell}^\sT \bZ \bR \bV_{\geq \ell} \mid \bX, \bW) ] = \cO(d^{-\ell}) \, .
    \end{align}
    To see this, we discuss over three different cases:

    (i) $\kappa_1 = \kappa_2$. 
    From Eq.~\eqref{eq:var_T2_highdegree_bd} and using Lemma \ref{lem:inprodmtx_supopnorm_bd}, we can get
     \begin{align}
        &~\E [\Var(\bf_{\geq \ell}^\sT \bZ \bR \bV_{\geq \ell} \mid \bX, \bW) ] \nonumber\\
        \leq&~ 
         \sum_{k\geq\ell} \frac{2 C \xi_k^2 F_k^4}{\lambda} \E (\opnorm{\bQ_{k}^{\bX}} \opnorm{\bQ_{k}^{\bW}}) 
         + \sum_{k \geq \ell}  \sum_{k' \neq k}  \frac{\xi_{k'}^2  F_k^2 F_{k'}^2}{\lambda} \E (\opnorm{ \bQ_{k}^{\bX} } \opnorm{ \bQ_{k'}^{\bW} }) \nonumber \\
         \leq& C' \Big( \sum_{k\geq\ell} \xi_k^2 F_k^4 + \sum_{k\geq\ell} F_k^2 \cdot \sum_{k'\geq\ell} \xi_{k'}^2 F_{k'}^2 \Big) \nonumber \\
         =&~ \cO(d^{-\ell}) \, ,
    \end{align}  
    where in the last step, we use Assumption \ref{ass:random_f_star}, which states that $\sum_{k\geq\ell}F_{k}^2 < \infty$ and $\sum_{k\geq\ell}\xi_{k}^2 = \cO(d^{-\ell})$, which follows from the fact $|\sqrt{N_k}\xi_k - \mu_k| = o_d(1)$ (due to \eqref{eq:q_He_coeff_conv} in Appendix \ref{sec:Spherical-Harmonics}), the fact that $N_k \asymp d^{k}$ (c.f. \eqref{eq:usdim_def} in Appendix \ref{sec:Spherical-Harmonics}) and $\sum_{k\geq\ell} \mu_{k}^2 < \infty$ by Assumption \ref{ass:sigma} (a).   
    
    (ii) If $\kappa_1 < \kappa_2$, then by the same argument leading to \eqref{eq:under_chi2_concentrate} in Proposition \ref{prop:high_degree_concentration}, we can get
    \begin{align}
        \label{eq:T2_highdegree_approx_1}
        \E \Big| \frac{1 }{p} \Tr\big(  \bQ_{k}^{\bX} \bZ \bR  \bQ_{k'}^{\bW} \bR \bZ^\sT \big) - \frac{\mu_k^2}{N_k} \Tr ( \bQ_{k}^{\bW} \bG_{\geq \ell}^{\bW} \bQ_{k'}^{\bW} \bG_{\geq \ell}^{\bW} ) \Big| &= o_d(1) \, , \\
        \label{eq:T2_highdegree_approx_2}
        \E \Big| \frac{1 }{p N_k^2} \Tr( {\bPhi_{k} \bPsi_{k}^\sT} \bZ \bR {\bPhi_{k} \bPsi_{k}^\sT} \bZ \bR ) -  \frac{\mu_k^2}{N_k} \Tr ( \bQ_{k}^{\bW} \bG_{\geq \ell}^{\bW} \bQ_{k}^{\bW} \bG_{\geq \ell}^{\bW} ) \Big| &= o_d(1) \, ,
    \end{align}
    where $\bG_{\geq \ell}^{\bW} = ( \mu_\ell^2 \bQ_{\ell}^{\bW} + \mu_{>\ell}^2 \id_p )^{-1}$. Then substituting \eqref{eq:T2_highdegree_approx_1} and \eqref{eq:T2_highdegree_approx_2} into Eq.~\eqref{eq:var_T2_highdegree_bd}, we can get
    \begin{align}
        \E [\Var(\bf_{\geq \ell}^\sT \bZ \bR \bV_{\geq \ell} \mid \bX, \bW)] \leq&~ C \sum_{k\geq\ell} \frac{\xi_k^2 \mu_{k}^2 F_k^4 }{N_k} \E \big[ \Tr ( \bQ_{k}^{\bW} \bG_{\geq \ell}^{\bW} \bQ_{k'}^{\bW} \bG_{\geq \ell}^{\bW} ) \big] \nonumber \\
        &~+ \sum_{ \substack{k\neq k' \\ k,k'\geq\ell} } \frac{ \xi_{k'}^2 \mu_{k}^2 F_k^2 F_{k'}^2 }{N_k} \E \big[  \Tr ( \bQ_{k}^{\bW} \bG_{\geq \ell}^{\bW} \bQ_{k'}^{\bW} \bG_{\geq \ell}^{\bW} ) \big] \nonumber\\
        &~+ \Big( \sum_{k\geq\ell} {\xi_k^2 F_k^4 } + \sum_{k\geq\ell} F_k^2 \cdot \sum_{k' \geq \ell }  {\xi_{k'}^2  F_{k'}^2 } \Big) \cdot o_d(1) \nonumber \\
        =&~ O_d(d^{-\ell})\, ,
    \end{align}
    where in the last step we use $\opnorm{\bG_{\geq\ell}^{\bW} } < \infty$ and 
    $\sup_{k\geq\kappa_1} \{\E \opnorm{ \bQ_{k}^{\bW} }^r \} < \infty$, for any $r \geq 0$ (by Lemma \ref{lem:inprodmtx_supopnorm_bd}). 

    (iii) If $\kappa_1 > \kappa_2$, then we can follow the similar argument as in $\kappa_1 < \kappa_2$ case to obtain Eq.~\eqref{eq:T2_highdegree_smallvariance}. The details are omitted here.
    
    Finally, substituting the bounds \eqref{eq:T2_firstterm_varbd_00}, \eqref{eq:T2_secondterm_varbd_0} and \eqref{eq:T2_highdegree_smallvariance} into Eq.~\eqref{eq:T2_varbd}, we conclude that $\E[\Var(T_2\mid\bX,\bW)] = o_d(1) $.

\subsection{Proof of Proposition \ref{prop:tech_bounds_ell}}
\label{sec:non-RMT}

\subsubsection{Technical lemma on the SVD of the feature matrix}

Before proving Proposition \ref{prop:tech_bounds_ell}, we first show the following lemma which is adapted from Proposition 6 in \cite{mei2022generalization}.

\begin{proposition}[SVD of the feature matrix $\bZ$]\label{prop:SVD_Z}
Follow the assumptions and the notations in the proof of Theorem \ref{thm:main_theorem_RF}. Recall that $n = \Theta_d ( d^{\kappa_1})$, $p = \Theta_d ( d^{\kappa_2} )$ and $\ell = \lceil \min(\kappa_1,\kappa_2) \rceil$. Denote $r = \min(n,p)$ and $s = \max (n,p)$. In particular, there exists $\delta_0 >0$ such that $r \geq d^{\ell - 1 + \delta_0}$.  Consider the singular value decomposition of $\tbZ = (\Tilde Z_{ij})_{i \in [n],j\in[p]}$ with $\Tilde Z_{ij} = \sigma ( \< \bx_i , \bw_j \>) / \sqrt{s}$:
\[
\tbZ = \bP \bLambda \bQ^\sT = [ \bP_1 , \bP_2 ] \diag ( \bLambda_1 , \bLambda_2 ) [ \bQ_1, \bQ_2 ]^\sT \in \R^{n \times p} \, , 
\]
where $\bP \in \R^{n \times r}$ and $\bQ \in \R^{p \times r}$, and $\bP_1 \in \R^{n \times N_{< \ell} }$ and $\bQ_1 \in \R^{p \times N_{< \ell} }$ correspond to the left and right singular vectors associated to the largest $N_{< \ell}$ singular values $\bLambda_1$. Similarly, $\bP_2 \in \R^{n \times (r -N_{< \ell})}$ and $\bQ_2 \in \R^{p \times (r - N_{< \ell})}$ correspond to the left and right singular vectors associated to the last $(r - N_{< \ell})$ smallest singular values $\bLambda_2$.

Then the singular value decomposition has the following properties:
\begin{itemize}
    \item[(a)] 
    There exists $K>0$ such that for any $D>0$ and all large enough $d$,
    \begin{equation}
        \P \big( \sigma_{\min} (\bLambda_1) < K \sqrt{r d^{-(\ell-1)}} \big) \leq d^{-D}\, ,
    \end{equation}
    where $\sigma_{\min} (\bLambda_1) = \min_{i \in [N_{<\ell}]} \sigma_i (\Tilde \bZ)$.
    
    \item[(b)] The left and right singular vectors associated to the $(r - N_{<\ell})$ smallest signular values verify:
    \begin{equation}
    \frac{1}{\sqrt{n}} \| \bPsi_{< \ell}^\sT \bP_2 \|_{\op} = O_{d,\prec} \big(1/ \sqrt{r d^{-(\ell-1)}}\big) \, , \qquad \frac{1}{\sqrt{p}} \| \bPhi_{< \ell}^\sT \bQ_2 \|_\op = O_{d,\prec} \big(1/ \sqrt{r d^{-(\ell-1)}}\big) \, .
    \end{equation}
\end{itemize}
\end{proposition}

The feature matrix can be decomposed as $\tbZ = \tbZ_{< \ell} + \tbZ_{\geq \ell}$, where 
\[
\tbZ_{< \ell} = \bPsi_{< \ell} \bD_{< \ell} \bPhi_{< \ell}^\sT / \sqrt{s} \, , \qquad \tbZ_{\geq\ell} = \sum_{k \geq\ell} \xi_k \bPsi_k \bPhi_k^\sT / \sqrt{s} \, .
\]
For convenience, we introduce the following normalized matrices:
\[
\tbPsi_{<\ell} = \bPsi_{< \ell} / \sqrt{n} \, , \qquad \tbPhi_{< \ell} = \bPhi_{< \ell} / \sqrt{p} \, , \qquad \tbD_{< \ell} = \sqrt{r} \bD_{< \ell}\, .
\]

\begin{proof}[Proof of Proposition \ref{prop:SVD_Z}]
The proof follows similarly from the proof of \cite[Proposition 6]{mei2022generalization}. We will simply outline the differences. 

First, note that 
\[
\tbZ_{< \ell} \tbZ_{< \ell}^\sT = \tbPsi_{<\ell} \tbD_{< \ell} \tbPhi_{< \ell}^\sT \tbPhi_{< \ell}\tbD_{< \ell}\tbPsi_{<\ell}^\sT \succeq 
\sigma_{\min} (\tbD_{< \ell}^2) \cdot \sigma_{\min} (\tbPhi_{< \ell}^\sT \tbPhi_{< \ell}) \cdot \tbPsi_{<\ell}\tbPsi_{<\ell}^\sT \, .
\]
Hence, $\sigma_{\min} (\bZ_{< \ell}) \geq \sigma_{\min} (\tbD_{< \ell}) \sigma_{\min} (\tbPhi_{< \ell}) \sigma_{\min} (\tbPsi_{<\ell})$. Then by Lemma \ref{lem:lowdeg_concentration} and Assumption \ref{ass:sigma} (b), there exists $K>0$ such that for any $D>0$ and large enough $d$,
$$
\P( \sigma_{\min}(\tbZ_{< \ell}) < K \sqrt{ r d^{-(\ell-1)} } ) \leq d^{-D}.
$$
Furthermore, by Weyl's inequality, we have for $i \in [N_{<\ell}]$,
\[
| \sigma_i ( \tbZ ) - \sigma_i (\tbZ_{< \ell} ) | \leq \| \tbZ_{\geq \ell} \|_{\op}\, .
\]
and by Lemma \ref{lem:featuremtx_opnormbd}, we have $\| \tbZ_{\geq \ell} \|_{\op} = O_{\prec} (1)$. 
We deduce that there exists $K>0$ such that for any $D>0$ and large enough $d$,
$$
\P( \sigma_{\min} (\bLambda_1) < K \sqrt{ r d^{-(\ell-1)} } ) \leq d^{-D}.
$$
On the other hand, applying Weyl's inequality to the rest of the eigenvalues, we get 
\[
\max_{ i = r+1 , \ldots, s} \sigma_i (\tbZ) \leq \| \tbZ_{\geq \ell} \|_{\op} = O_{d,\prec} (1) \, .
\]

Let us now bound $\| \tbPsi_{<\ell}^\sT \bP_2\|_\op$. The proof for $\tbPhi_{< \ell}^\sT \bP_2 $ will follow from the same argument. Let us consider $\bu \in \R^{r - N_{< \ell}}$ the left leading right singular vector, i.e., $\| \bu \|_2 = 1$ and $\|  \tbPsi_{<\ell}^\sT \bP_2 \bu \|_2 = \| \tbPsi_{<\ell}^\sT \bP_2\|_\op$. For convenience, denote $\Tilde \bu =  \tbPsi_{<\ell}^\sT \bP_2 \bu$. We get
\begin{equation}
\begin{aligned}
\bu^\sT \bLambda_2^2 \bu =&~ \bu^\sT \bP_2^\sT \tbZ \tbZ^\sT \bP_2 \bu \\
=&~\bu^\sT \bP_2^\sT (\tbZ_{< \ell} \tbZ_{< \ell}^\sT + 2 \tbZ_{< \ell} \tbZ_{\geq \ell}^\sT + \tbZ_{\geq \ell} \tbZ_{\geq \ell}^\sT ) \bP_2 \bu \\
=&~ \tbu^\sT \tbD_{< \ell} (\id_{N_{< \ell}} + \bDelta_1) \tbD_{< \ell} \tbu + 2  \tbu^\sT \tbD_{< \ell} ( \tbPhi_{< \ell}^\sT \tbZ_{\geq \ell}^\sT \bP_2 \bu ) + \| \tbZ_{\geq \ell}^\sT \bP_2 \bu \|_2^2 \, , 
\end{aligned}
\end{equation}
where $\bDelta_1$ satisfies: for any $\varepsilon, D>0$ and all large $d$, $\P(\opnorm{\bDelta_1} \geq \varepsilon) \leq d^{-D}.$ From the above discussion, we have for any $D$ and large enough $d$,
\[
\begin{aligned}
&\bu^\sT \bLambda_2^2 \bu = O_{d,\prec} (1) \, ,\\
&\P(\tbu^\sT \tbD_{< \ell} (\id_{N_{< \ell}} + \bDelta_1) \tbD_{< \ell} \tbu \geq 2 \| \tbD_{< \ell} \tbu \|_2^2 ) \leq d^{-D}\, ,\\ 
 & \tbu^\sT \tbD_{< \ell} ( \tbPhi_{< \ell}^\sT \tbZ_{\geq \ell}^\sT \bP_2 \bu ) \geq  - \| \tbD_{< \ell} \tbu \|_2 \| \tbPhi_{< \ell} \|_\op \| \tbZ_{\geq \ell}^\sT \|_{\op} = - \| \tbD_{< \ell} \tbu \|_2 \cdot  O_{d,\prec} (1) \, , 
\end{aligned}
\]
Merging these bounds, we get
\[
\|\tbD_{< \ell} \tbu \|_2 = O_{d,\prec} (1)\, ,
\]
which implies $\| \tbu \|_2 = \| \tbPsi_{<\ell}^\sT \bP_2\|_\op = O_{d,\prec} (1/ \sqrt{r d^{-(\ell-1)}})$.

\end{proof}

\subsubsection{Proof of Proposition \ref{prop:tech_bounds_ell}}

The proof follows similarly from the proof of \cite[Proposition 7]{mei2022generalization}. For convenience, we consider the same notations as in the proof of Proposition \ref{prop:SVD_Z}. In particular, recall that $r = \min (n,p)$, $s = \max(n,p)$, $\tbZ = ( \sigma( \< \bx_i , \bw_j \>) / \sqrt{s} )_{i \in [n], j \in [p]}$, $\tbPsi_{<\ell} = \bPsi_{< \ell} / \sqrt{n}$, $\tbPhi_{< \ell} = \bPhi_{< \ell} / \sqrt{p}$ and $\tbD_{< \ell} = \sqrt{r} \bD_{< \ell}$. We further introduce $\tlambda = p \lambda / s$ and 
\[
\tbR = \frac{s}{p} \bR = ( \tbZ^\sT \tbZ + \tlambda \id_p )^{-1}\, .
\]

\noindent
\textbf{Step 1. Proof of Equation \eqref{eq:low_f_b1}.} 

With these notations, this is equivalent to showing that $\E \| \tbPsi_{<\ell}^\sT \tbZ \tbR \tbPhi_{< \ell} \tbD_{< \ell} - \id_{N_{< \ell}} \|_\op = o_{d}(1) $. We use the following identity $\tbPhi_{< \ell} \tbD_{< \ell} = \tbZ_{< \ell}^\sT (\tbPsi_{<\ell}^\sT)^\dagger = ( \tbZ - \tbZ_{\geq \ell} )^\sT (\tbPsi_{<\ell}^\sT)^\dagger$, 
so that
\begin{equation}\label{eq:PZRPD_d1}
    \tbPsi_{<\ell}^\sT \tbZ \tbR \tbPhi_{< \ell} \tbD_{< \ell} = \tbPsi_{<\ell}^\sT \tbZ \tbR \tbZ^\sT (\tbPsi_{<\ell}^\sT)^\dagger  - \tbPsi_{<\ell}^\sT \tbZ \tbR \tbZ_{\geq \ell}^\sT (\tbPsi_{<\ell}^\sT)^\dagger \, .
\end{equation}
We bound the first term using the singular value decomposition described in Proposition \ref{prop:SVD_Z}: 
\begin{align}
\tbPsi_{<\ell}^\sT \tbZ \tbR \tbZ^\sT (\tbPsi_{<\ell}^\sT)^\dagger =&~ \tbPsi_{<\ell}^\sT \bP_1 \frac{\bLambda_1^2}{\bLambda_1^2 + \tlambda}\bP_1^\sT (\tbPsi_{<\ell}^\sT)^\dagger  + \tbPsi_{<\ell}^\sT \bP_2 \frac{\bLambda_2^2}{\bLambda_2^2 + \tlambda}\bP_2^\sT (\tbPsi_{<\ell}^\sT)^\dagger  \\
=&~ \tbPsi_{<\ell}^\sT \bP_1\bP_1^\sT (\tbPsi_{<\ell}^\sT)^\dagger + \tbPsi_{<\ell}^\sT \bP_2\bP_2^\sT (\tbPsi_{<\ell}^\sT)^\dagger + \bDelta\\
=&~ \tbPsi_{<\ell}^\sT (\tbPsi_{<\ell}^\sT)^\dagger + \bDelta \\
=&~ \id_{N_{< \ell}} + \bDelta \, .\label{eq:PZRPD_b1}
\end{align}
where $\bDelta$ satisfies: $\opnorm{\bDelta} =  O_{d,\prec} ( \frac{1}{r d^{-(\ell-1)}} )$ and 
we denoted with a slight abuse of notation $\bLambda^2 / (\bLambda^2 + \tlambda) = \diag ( (\Lambda_i^2 / (\Lambda^2_i + \tlambda))_{i} )$. For the second term in Eq.~\eqref{eq:PZRPD_d1}, we similarly decompose 
\[
\tbPsi_{<\ell}^\sT \tbZ \tbR \tbZ_{\geq \ell}^\sT (\tbPsi_{<\ell}^\sT)^\dagger =  \tbPsi_{<\ell}^\sT \bP_1 \frac{\bLambda_1}{\bLambda_1^2 + \tlambda}\bQ_1^\sT \tbZ_{\geq \ell} (\tbPsi_{<\ell}^\sT)^\dagger + \tbPsi_{<\ell}^\sT \bP_2 \frac{\bLambda_2}{\bLambda_2^2 + \tlambda}\bQ_2^\sT \tbZ_{\geq \ell} (\tbPsi_{<\ell}^\sT)^\dagger \, ,
\]
where 
\[
\begin{aligned}
    \Big\Vert  \tbPsi_{<\ell}^\sT \bP_1 \frac{\bLambda_1}{\bLambda_1^2 + \tlambda}\bQ_1^\sT \tbZ_{\geq \ell} (\tbPsi_{<\ell}^\sT)^\dagger \Big\Vert_\op \leq &~ \| \tbPsi_{<\ell} \|_\op \Big\Vert \frac{\bLambda_1}{\bLambda_1^2 + \tlambda} \Big\Vert_\op \| \tbZ_{\geq \ell} \|_\op \| (\tbPsi_{<\ell}^\sT)^\dagger \|_\op,
\end{aligned}
\]
and
\[
\begin{aligned}
    \Big\Vert  \tbPsi_{<\ell}^\sT \bP_2 \frac{\bLambda_2}{\bLambda_2^2 + \tlambda}\bQ_2^\sT \tbZ_{\geq \ell} (\tbPsi_{<\ell}^\sT)^\dagger \Big\Vert_\op \leq &~ \| \tbPsi_{<\ell}^\sT \bP_2\|_\op \Big\Vert \frac{\bLambda_2}{\bLambda_2^2 + \tlambda} \Big\Vert_\op \| \tbZ_{\geq \ell} \|_\op \| (\tbPsi_{<\ell}^\sT)^\dagger \|_\op.
\end{aligned}
\]
We can control each term on the right-hand sides of two bounds above.
First, by Lemma \ref{lem:lowdeg_concentration}, we can get $\opnorm{\tbPsi_{<\ell}^\sT} = O_{d,\prec } (1)$ and $\opnorm{(\tbPsi_{<\ell}^\sT)^\dagger} = O_{d,\prec } (1)$.
Then we show that $\Big\Vert \frac{\bLambda_2}{\bLambda_2^2 + \tlambda} \Big\Vert_\op = O_{d,\prec } (1)$. On one hand, for $i > N_{<\ell}$,
\begin{align}
    |\sigma_{i}(\tbZ)| &= |\sigma_{i}(\tbZ) - \sigma_{i}(\tbZ_{<\ell})| \leq \opnorm{ \tbZ_{\geq \ell} }
    = O_{d,\prec } (1) \, ,
\end{align}
where the second step follows from  Weyl's inequality and the last step is due to Lemma \ref{lem:featuremtx_opnormbd}. This indicates $\opnorm{\bLambda_2} = O_{d,\prec } (1)$. Hence, when $\kappa_2 \leq \kappa_1$, $\Big\Vert \frac{\bLambda_2}{\bLambda_2^2 + \tlambda} \Big\Vert_\op \leq \tlambda^{-1} \opnorm{\bLambda_2} = O_{d,\prec } (1)$. On the other hand, when $\kappa_2 > \kappa_1$, by Lemma \ref{lem:smallest_eigval_RFM} there exists $C>0$ such that for any $D>0$ and all large enough $d$,
\begin{align}
    \P( \lambda_{\min}(\bLambda_2) < C  ) = \P( \lambda_{\min}(\tbZ) < C  ) \leq d^{-D} \, .
\end{align}
Hence, we still have $\Big\Vert \frac{\bLambda_2}{\bLambda_2^2 + \tlambda} \Big\Vert_\op = O_{d,\prec } (1)$. Now using Proposition \ref{prop:SVD_Z} and Lemma \ref{lem:featuremtx_opnormbd}, we can get 
\begin{align}
    \Big\Vert  \tbPsi_{<\ell}^\sT \bP_1 \frac{\bLambda_1}{\bLambda_1^2 + \tlambda}\bQ_1^\sT \tbZ_{\geq \ell} (\tbPsi_{<\ell}^\sT)^\dagger \Big\Vert_\op 
    = O_{d,\prec} \big((r d^{-(\ell-1)})^{-1/2}\big) \, ,
\end{align}
and
\begin{align}
    \Big\Vert  \tbPsi_{<\ell}^\sT \bP_2 \frac{\bLambda_2}{\bLambda_2^2 + \tlambda}\bQ_2^\sT \tbZ_{\geq \ell} (\tbPsi_{<\ell}^\sT)^\dagger \Big\Vert_\op 
    = O_{d,\prec} \big((r d^{-(\ell-1)})^{-1/2}\big) \, .
\end{align}
The above two displays indicate $\opnorm{\tbPsi_{<\ell}^\sT \tbZ \tbR \tbZ_{\geq \ell}^\sT (\tbPsi_{<\ell}^\sT)^\dagger} = O_{d,\prec} \big((r d^{-(\ell-1)})^{-1/2}\big)$.
Substituting this bound and Eq.~\eqref{eq:PZRPD_b1} in Eq.~\eqref{eq:PZRPD_d1} gives
\begin{align}
\label{eq:concentrate_2_id_1}
    \| \bPsi_{< \ell}^\sT \bZ \bR \bPhi_{<\ell} \bD_{< \ell} / \sqrt{p} - \id_{N_{< \ell}} \|_\op = O_{d,\prec} \big((r d^{-(\ell-1)})^{-1/2}\big) \, .
\end{align}
On the other hand, we can easily get the deterministic bound:
\begin{align}
    \opnorm{ \bPsi_{< \ell}^\sT \bZ \bR \bPhi_{<\ell} \bD_{< \ell} / \sqrt{p} } \leq K N_{<\ell} \sqrt{ n } \, .
\end{align}
where $K$ is some absolute constant. 
The above bound together with Eq.~\eqref{eq:concentrate_2_id_1} and the fact that $r \geq d^{\ell-1+\delta_0}$ for some $\delta_0>0$ directly implies the first bound of Proposition \ref{prop:tech_bounds_ell}.

\noindent
\textbf{Step 2. Proof of Equation \eqref{eq:low_f_b2}. }

Let $\tbU_{\geq\ell} = \frac{np}{s} \bU_{\geq\ell} $, where $\bU$ is defined in Eq.~\eqref{eq:def_U} and we have
\begin{align}
    \bPsi_{< \ell}^\sT \bZ \bR \bU_{\geq\ell} \bR \bZ^\sT \bPsi_{< \ell} = \tbPsi_{< \ell}^\sT \tbZ \tbR \tbU_{\geq\ell} \tbR \tbZ^\sT \tbPsi_{< \ell} \, .
\end{align}
Note that if $n \geq p$ 
\[
 \| \tbU_{\geq\ell} \|_{\op} = \Big\Vert \sum_{k \geq \ell} \xi_k^2 \bPhi_k \bPhi_k^\sT \Big\Vert_\op = O_{d,\prec} (1) \, ,
\]
and if $n < p$, 
\[
 \| \tbU_{\geq\ell} \|_{\op} = \frac{n}{p} \Big\Vert \sum_{k \geq \ell} \xi_k^2 \bPhi_k \bPhi_k^\sT \Big\Vert_\op = O_{d,\prec} (1) \, .
\]
We can now bound Eq.~\eqref{eq:low_f_b2} using the SVD decomposition of Proposition \ref{prop:SVD_Z} and Lemma \ref{lem:kernelmtx_opnormbd}:
\[
\begin{aligned}
    \| \tbPsi_{<\ell}^\sT \tbZ \tbR \tbU_{\geq\ell}^{1/2} \|_{\op} \leq &~ \Big \Vert \tbPsi_{<\ell}^\sT \bP_1 \frac{\bLambda_1}{\bLambda_1^2 +\tlambda} \bQ_1^\sT \tbU_{\geq\ell}^{1/2} \Big\Vert_{\op} + \Big \Vert \tbPsi_{<\ell}^\sT \bP_2 \frac{\bLambda_2}{\bLambda_2^2 +\tlambda} \bQ_2^\sT \tbU_{\geq\ell}^{1/2} \Big\Vert_{\op} \\
    =&~ O_{d,\prec} (1/ \sqrt{r d^{-(\ell-1)}}).\\
    =& O_{d,\prec} (d^{-\delta_0}) \, ,
\end{aligned}
\]
for some $\delta_0 > 0$.
Since $\opnorm{\tbU_{\geq\ell}} \leq \Trace(\tbU_{\geq\ell}) = p \sum_{k\geq\ell} \xi_k^2 N_k \leq Cp$, where $C$ is an absolute constant, $\opnorm{\tbZ \tbR} \leq \sqrt{ \tlambda^{-1} }$ and $\opnorm{\tbPsi_{<\ell}} \leq \sqrt{N_{<\ell}}$, we can get the deterministic bound: 
$$\opnorm{\bPsi_{< \ell}^\sT \bZ \bR \bU_{\geq\ell} \bR \bZ^\sT \bPsi_{< \ell}} \leq C \tlambda^{-1} N_{<\ell} p \, .$$
This together with $\opnorm{\bPsi_{< \ell}^\sT \bZ \bR \bU_{\geq\ell} \bR \bZ^\sT \bPsi_{< \ell}} = O_{d,\prec} (d^{-\delta_0})$ immediately implies Eq.~\eqref{eq:low_f_b2}.

\noindent
\textbf{Step 3. Proof of Equations \eqref{eq:low_f_b3} and \eqref{eq:low_f_b4}. }

First note that 
\[
\| \bH_F \|_\op = \Big\Vert \sum_{k \geq\ell} \frac{F_k^2}{N_k} \bPsi_k \bPsi_k^\sT \Big\Vert_\op = O_{d,\prec} (\frac{n}{N_\ell}) \, .
\]
Hence, 
\[
\begin{aligned}
   \Tr ( \bH_F \bZ \bR \bU_{< \ell } \bR \bZ^\sT )
   =&~ \frac{p}{s} \Tr ( \bH_F \tbZ  \tbR \bU_{< \ell} \tbR \tbZ^\sT ) \\
   =&~ O_{d,\prec} \left(\frac{np}{ N_\ell s } \right) \cdot  \Tr (    \bU_{< \ell} \tbR \tbZ^\sT\tbZ  \tbR ) \\
   =&~ O_{d,\prec} \left( \frac{np}{ N_\ell s} \right) \cdot  \Tr (  \bD_{< \ell}^2 \tbPhi_{< \ell}^\sT \tbR \tbPhi_{< \ell} ) \\
   =&~ O_{d,\prec} \left( \frac{np}{ N_\ell s } \right) \cdot  \Tr (  \bD_{< \ell}^2) \cdot \left( \frac{ 1 }{\tlambda + \sigma_{\min}(\bLambda_1^2)} +  \frac{ \| \tbPhi_{< \ell}^\sT \bQ_2 \|_\op^2 }{\tlambda + \sigma_{\min}(\bLambda_2^2)} \right) \\
   =&~ O_{d,\prec}(1) \cdot \left( \frac{ 1 }{\tlambda + \sigma_{\min}(\bLambda_1^2)} +  \frac{ \| \tbPhi_{< \ell}^\sT \bQ_2 \|_\op^2 }{\tlambda + \sigma_{\min}(\bLambda_2^2)} \right) \\   
   =&~ O_{d,\prec}\left(\frac{d^{\ell-1}}{r}\right) \, ,
\end{aligned}
\]
where the last step follows from Proposition \ref{prop:SVD_Z} and Lemma \ref{lem:smallest_eigval_RFM}.

On the other hand, we can show there exists some $\tau>0$ such that $|\Tr ( \bH_F \bZ \bR \bU_{< \ell } \bR \bZ^\sT )| \leq d^\tau$, for all large $d>0$. Combining the above two bounds and recall the fact that $r \asymp d^{\ell-1+\epsilon}$ for some $\epsilon>0$, we can reach at Eq.~\eqref{eq:low_f_b3}.

Equation \eqref{eq:low_f_b4} follows similarly and the details are omitted here.

\subsection{Proof of Proposition \ref{prop:high_degree_concentration}}
\label{sec:high_degree_concentration_proof}

We prove this proposition separately for each of the three regimes.

\subsubsection{Overparametrized regime $\kappa_1>\kappa_2$}

In the highly overparametrized regime, we can show:
\begin{lemma}
\label{lem:overparameterized}
When $\kappa_1 > \kappa_2$, it holds that
\begin{align}
    \Big|\chi_1 - \Tr ( \sum_{k \geq \ell} \frac{ F_k^2}{N_k}\xi_k^2 \bPsi_k \bPsi_k^\sT \bG_{\geq\ell} )\Big| &= o_{d,\P}(1) \, ,\label{Gamma_1_concentrate}\\
    \Big|\chi_2 - \Tr ( \bH_F \bG_{\geq\ell} \sum_{k\geq\ell}\xi_k^4 \bPsi_k \bPsi_k^\sT \bG_{\geq\ell} )\Big| &= o_{d,\P}\, ,(1)\label{Gamma_2_concentrate}\\
    \Big|\chi_3 - \Tr ( \bG_{\geq\ell} \sum_{k\geq\ell}\xi_k^4 \bPsi_k \bPsi_k^\sT \bG_{\geq\ell} )\Big| &= o_{d,\P}(1)\, ,
    \label{Gamma_3_concentrate}
\end{align}
where $\bG_{\geq\ell} = (\sum_{k\geq\ell}^{\infty} \xi_k^2 \bPsi_k \bPsi_k^\sT + \lambda \id_n)^{-1}.$
\end{lemma}
Based on Lemma \ref{lem:overparameterized} and Lemma \ref{lem:offdiagonal_kernelmtx_opnormbd} in Appendix \ref{sec:operator_norm}, we can now prove Eqs.~\eqref{eq:over_chi1_concentrate}-\eqref{eq:over_chi3_concentrate}. We just present the proof of Eq.~\eqref{eq:over_chi1_concentrate} and the other two are similar. 

From Lemma \ref{lem:kernelmtx_opnormbd} and the fact that $\kappa_2 \leq \ell$, we can get 
\begin{align}
\label{eq:over_chi1_concentrate_proof1}
    \left\| \sum_{k>\ell} \frac{ F_k^2}{N_k}\xi_k^2 \bPsi_k \bPsi_k^\sT \right\|_\op= O_{\prec}\Big( \frac{1}{N_{\ell} } \Big)\, .
\end{align}
Meanwhile, 
\begin{align}
    &~\left\| \frac{ F_\ell^2}{N_\ell}\xi_\ell^2 \bPsi_\ell \bPsi_\ell^\sT - \frac{ F_\ell^2}{N_\ell}\mu_\ell^2 \bQ_{\ell}^{\bX} \right\|_\op \nonumber \\
    =&~ (N_{\ell} \xi_{\ell}^2 - \mu_{\ell}^2 ) \left\| \frac{ F_\ell^2}{N_\ell} \bQ_{\ell}^{\bX} \right\|_\op \nonumber \\
    =&~ o_{d,\P} \Big( \frac{1}{N_{\ell} } \Big)\, ,
    \label{eq:over_chi1_concentrate_proof2}
\end{align}
where the last step follows from Lemma \ref{lem:offdiagonal_kernelmtx_opnormbd}, Lemma \ref{lem:inprodmtx_opnorm_prbd} and the fact that $N_{\ell} \xi_{\ell}^2 - \mu_{\ell}^2 = o_{d}(1) $ (due to Eq.~\eqref{eq:q_He_coeff_conv} in Appendix \ref{sec:Spherical-Harmonics}).
Similarly, we can get
\begin{align}
\label{eq:over_chi1_concentrate_proof3}
    \opnorm{ \bG_{\geq \ell} - {\bG}_{\geq \ell}^{\bX} } = o_{d,\P}(1)\, .
\end{align}
Substituting Eqs.~\eqref{eq:over_chi1_concentrate_proof1}, \eqref{eq:over_chi1_concentrate_proof2} and \eqref{eq:over_chi1_concentrate_proof3} into Eq.~\eqref{Gamma_1_concentrate}, we can obtain that
\begin{align}
\label{eq:chi1_prob_conv}
    \Big| \chi_1 - \frac{F_\ell^2 \mu_{\ell}^2 }{N_\ell} \Tr ( \bQ_\ell^{\bX} {\bG}_{\geq \ell}^{\bX} ) \Big| = o_{d,\P} (1) \, .
\end{align}
On the other hand, using Lemma \ref{lem:inprodmtx_opnorm_prbd}, it is not hard to show $\E|\chi_1|^2 < \infty$. Also since $ \opnorm{ \bQ_\ell^{\bX} {\bG}_{\geq \ell}^{\bX} } < \infty $, it holds that $  \frac{F_\ell^2 \mu_{\ell}^2 }{N_\ell} \Tr ( \bQ_\ell^{\bX} {\bG}_{\geq \ell}^{\bX} ) < \infty$. Therefore, $\E| \chi_1 - \frac{F_\ell^2 \mu_{\ell}^2 }{N_\ell} \Tr ( \bQ_\ell^{\bX} {\bG}_{\geq \ell}^{\bX} ) |^2 < \infty$ and together with Eq.~\eqref{eq:chi1_prob_conv}, we can get
\begin{align}
    \E \Big| \chi_1 - \frac{F_\ell^2 \mu_{\ell}^2 }{N_\ell} \Tr ( \bQ_\ell^{\bX} {\bG}_{\geq \ell}^{\bX} ) \Big| = o_{d} (1) \, .
\end{align}
This verifies Equation \eqref{eq:over_chi1_concentrate}.

\subsubsection{Proof of Lemma \ref{lem:overparameterized}}

\paragraph*{Proof for  $\chi_1$ (Eq.~\eqref{Gamma_1_concentrate}).} Since $\bZ \bR= \bPi \bZ$, we have $\chi_1 = \Tr( \bZ \bZ_F^\sT  \bPi)$, 
where 
\begin{align}
  \bPi := (  \bZ \bZ^\sT + \lambda \id_n )^{-1}.  
\end{align}
We can show $\|\bZ_{\geq\ell}\|_\op = O_{\prec}(1)$ (by Lemma \ref{lem:featuremtx_opnormbd}). Then 
\begin{align}
    \opnorm{\bPi \bZ_{< \ell}} \leq \opnorm{\bPi \bZ} + \opnorm{\bPi \bZ_{\geq\ell}} = O_{\prec}(1) \, ,
\end{align}
and thus
\begin{align}
    |\Tr ( \bZ_{< \ell} \bZ_F^\sT \bPi)| &\leq N_{< \ell} \opnorm{\bZ_F} \opnorm{\bPi \bZ_{<\ell} } = O_{\prec} \left(\frac{N_{<\ell}}{N_{\ell}} \right) = O_{\prec}\left(\frac{1}{d}\right)
    \label{eq:ZZ_F_G_small} \, ,
\end{align}
where we have used $\|\bPi \bZ\|_\op, \|\bPi \|_\op < \infty $ and $\|\bZ_F\|_\op = O_{\prec}(\frac{1}{N_{\ell}} )$  (by Lemma \ref{lem:featuremtx_opnormbd}, with $\varsigma_k = \frac{N_{\ell} F_k^2 \sqrt{N_k}\xi_k}{N_k}$). It follows that
 \begin{align}
    |\chi_1 - \Tr ( \bZ_{\geq \ell} \bZ_F^\sT \bPi )| =  |\Tr ( \bZ_{< \ell} \bZ_F^\sT \bPi)| = O_{\prec}\left(\frac{1}{d}\right) \, .
    \label{eq:Gamma1_apporx_1}
 \end{align}
Note that 
\begin{align*}
    \bPi - \bPi_{\geq\ell} &= \bPi_{\geq\ell}(\bZ_{\geq\ell} \bZ_{\geq\ell}^\sT - \bZ \bZ^\sT)\bPi \\
    &= -\bPi_{\geq\ell} \bZ_{< \ell} \bZ^\sT \bPi - \bPi_{\geq \ell} \bZ_{\geq \ell} \bZ_{< \ell}^\sT \bPi \, ,
\end{align*}
where $\bPi_{\geq\ell} = (  \bZ_{\geq\ell} \bZ_{\geq\ell}^\sT + \lambda \id_n )^{-1}$, so that
\begin{align}
    |\Tr [ \bZ_{\geq \ell} \bZ_F^\sT (\bPi - \bPi_{\geq \ell}) ]| 
    &\leq |\Tr (\bZ_{\geq \ell} \bZ_F^\sT \bPi_{\geq \ell} \bZ_{< \ell} \bZ^\sT \bPi)| + |\Tr(\bZ_{\geq \ell} \bZ_F^\sT \bPi_{\geq \ell} \bZ_{\geq \ell} \bZ_{< \ell}^\sT \bPi)| \, .
    \label{eq:Gamma1_apporx_2}
\end{align}
Then following the same steps in obtaining Eq.~\eqref{eq:ZZ_F_G_small}, one can show both terms on the right-hand side of Eq.~\eqref{eq:Gamma1_apporx_2} is of order $O_{\prec}(1/d)$, so
\begin{align}
|\Tr [ \bZ_{\geq \ell} \bZ_F^\sT (\bPi - \bPi_{\geq \ell}) ]| = O_{\prec}(d^{-1}) \, .
\label{eq:Gamma1_apporx_3}
\end{align}
Then Eqs.~\eqref{eq:Gamma1_apporx_1} and \eqref{eq:Gamma1_apporx_3} give us
\begin{align}
    \big| \chi_1 - \Tr[\bZ_{\geq \ell} \bZ_F^\sT \bPi_{\geq\ell}] \big| = O_{\prec}(d^{-1}).
    \label{eq:Gamma1_apporx_3p5}
\end{align}
Also 
\begin{align}
    & \left|\Tr \left( \bZ_{\geq \ell} \bZ_F^\sT \bPi_{\geq\ell}  - \sum_{k\geq\ell} \frac{ F_k^2}{N_k}\xi_k^2 \bPsi_k \bPsi_k^\sT \bG_{\geq\ell} \right)\right| \nonumber\\
    \leq& \big| \Tr[ \bZ_{\geq \ell} \bZ_F^\sT (\bPi_{\geq\ell} - \bG_{\geq\ell}) ] \big| 
    + \Big| \Tr\Big[\big(\bZ_{\geq \ell} \bZ_F^\sT - \sum_{k\geq\ell} \frac{ F_k^2}{N_k}\xi_k^2 \bPsi_k \bPsi_k^\sT \big) \bG_{\geq\ell} \Big] \Big| \nonumber \\
    \leq& n \opnorm{\bZ_{\geq \ell} \bZ_F^\sT } \cdot \opnorm{\bPi_{\geq \ell} - \bG_{\geq \ell}} + n \Big\| \bZ_{\geq \ell}\bZ_F^\sT - \sum_{k \geq \ell} \frac{ F_k^2}{N_k}\xi_k^2 \bPsi_k \bPsi_k^\sT \Big\|_\op \cdot \opnorm{\bG_{\geq \ell}} \nonumber \\
    =& o_{d,\P}(1)\, ,
    \label{eq:Gamma1_apporx_4}
\end{align}
where in the last step we use $\|\bZ_{\geq \ell}\|_\op= O_{\prec}(1)$,  $\|\bZ_F\|_\op = O_{\prec}(\frac{1}{N_{\ell}})$ and Eqs.~\eqref{eq:G_l_concentrate} and \eqref{eq:Z_ZF_concentrate} in Lemma \ref{lem:overpara_mtx_concentrate}. 
After combining  Eqs.~\eqref{eq:Gamma1_apporx_3p5} and \eqref{eq:Gamma1_apporx_4}, we reach at Eq.~\eqref{Gamma_1_concentrate}.

\paragraph*{Proof for  $\chi_2$ and $\chi_3$ (Eqs.~\eqref{Gamma_2_concentrate} and \eqref{Gamma_3_concentrate}).} Recall the definition of $\bU$ in Eq.~\eqref{eq:def_U} and it can be directly checked that $\bU_{\geq\ell} = p^{-1}\sum_{k\geq\ell}\xi_k^2\bPhi_k\bPhi_k^\sT$. We can show $\|\bH_F\|_\op = O_{\prec}(1)$ and $\|\bU_{\geq\ell}\|_\op = O_{\prec}(\frac{1}{N_\ell})$ [from Eqs.~\eqref{eq:sum_F_QkX} and \eqref{eq:sum_mu_QkW} in Lemma \ref{lem:kernelmtx_opnormbd}]. Then following similar steps leading to Eq.~\eqref{eq:Gamma1_apporx_3p5}, we can get 
\begin{align}
|\chi_2 - \Tr (\bH_F \bPi_{\geq \ell} \bZ_{\geq \ell} \bU_{\geq \ell} \bZ_{\geq \ell}^\sT \bPi_{\geq \ell})| = O_{\prec}(d^{-1})\, .
\label{eq:Gamma_3_approx_1}
\end{align}
Combining Eq.~\eqref{eq:Gamma_3_approx_1} with Eqs.~\eqref{eq:G_l_concentrate} and \eqref{eq:ZUZ_concentrate} in Lemma \ref{lem:overpara_mtx_concentrate}, we reach at Eq.~\eqref{Gamma_2_concentrate}. The proof of Eq.~\eqref{Gamma_3_concentrate} is completely analogous to Eq.~\eqref{Gamma_2_concentrate} and is omitted for brevity.

\subsubsection{Underparametrized regime $\kappa_1<\kappa_2$}
The proof is analogous to the $\kappa_1 > \kappa_2$ case and we just present a sketch of the proof. First, similar as Eqs.~\eqref{Gamma_1_concentrate}-\eqref{Gamma_3_concentrate}, we can show
\begin{align}
    \Big|\chi_1 - \Tr \Big[ (\sum_{k\geq\ell} F_k^2\xi_k^2 \bQ_k^\bW) \cdot \widetilde{\bG}_{\geq\ell} \Big] \Big| &= o_{d,\P}(1) \, ,\label{eq:under_Gamma_1_concentrate}\\
    \Big|\chi_2 - \Tr \Big[ \big(F_\ell^2\xi_\ell^2 \bQ_\ell^\bW + \frac{p}{n^2} \bZ_{\geq\ell}^\sT \cdot \sum_{k>\ell} F_k^2 \bQ_k^{\bW} \cdot  \bZ_{\geq\ell} \big) \cdot \widetilde{\bG}_{\geq\ell} \Big] \Big| &= o_{d,\P}(1) \, ,\\
    \Big|\chi_3 - \frac{1}{n} \Tr (\id_p) \Big| &= o_{d,\P}(1) \, ,
\end{align} 
where $\widetilde{\bG}_{\geq\ell} = (\sum_{k\geq\ell}\xi_k^2\bPhi_k\bPhi_k^\sT)^{-1}$. 
Then similar as Eqs.~\eqref{eq:over_chi1_concentrate_proof2} and \eqref{eq:over_chi1_concentrate_proof3} we can get
\begin{align}
    \opnorm{ \sum_{k\geq\ell} F_k^2\xi_k^2 \bQ_k^\bW - \frac{ F_\ell^2}{N_\ell}\mu_\ell^2 \bQ_{\ell}^{\bW} } = o_{d,\P} \Big( \frac{1}{N_{\ell} } \Big) \, ,
    \label{eq:under_chi1_concentrate_proof2}
\end{align}
and
\begin{align}
    \label{eq:under_chi1_concentrate_proof3}
    \opnorm{ \widetilde{\bG}_{\geq \ell} - {\bG}_{\geq \ell}^{\bW} } = o_{d,\P}(1)\, .
\end{align}
Then similar as Eq.~\eqref{eq:chi1_prob_conv}, after combining Eqs.~\eqref{eq:under_Gamma_1_concentrate}, \eqref{eq:under_chi1_concentrate_proof2} and \eqref{eq:under_chi1_concentrate_proof3}, we can obtain that
\begin{align}
    \Big| \chi_1 - \frac{F_\ell^2 \mu_\ell^2}{N_\ell} \Tr ( \bQ_\ell^{\bW} \bG_{\geq \ell}^{\bW} ) \Big| =&~  o_{d,\P}(1) \, .
    \label{eq:under_chi1_pconv}
\end{align}
On the other hand, by Assumption \ref{ass:sigma}.(c), we can get $\opnorm{\bQ_{\ell}^\bW\bG_{\geq\ell}^{\bW}} <\infty$ and similar as in case $\kappa_1 > \kappa_2$, we can also show $\E|\chi_1|^2 <\infty$. Together with Eq.~\eqref{eq:under_chi1_pconv}, we can show Eq.~\eqref{eq:under_chi1_concentrate}. 

The proof of Eqs.~\eqref{eq:under_chi2_concentrate} and \eqref{eq:under_chi3_concentrate} can be done in a similar way.

\subsubsection{Critical regime $\kappa_1=\kappa_2$}
We present the proof for Eq.~\eqref{eq:critical_chi1_concentrate}. The proof of Eqs.~\eqref{eq:critical_chi2_concentrate} and \eqref{eq:critical_chi3_concentrate} are similar and omitted.

Recall that $\chi_1 = \Tr (\bZ_{F}^\sT \bZ \bR)$ with
$$
\bZ_{F} = \frac{1}{\sqrt{p} } \sum_{k\geq\ell} F_k^2 \frac{\xi_k}{N_k} \bPsi_k \bPhi_k^\sT \, .
$$
Denote $\bZ_{F,>\ell} := \frac{1}{\sqrt{p} } \sum_{k > \ell} F_k^2 \frac{\xi_k}{N_k} \bPsi_k \bPhi_k^\sT$. By Lemma \ref{lem:featuremtx_opnormbd}, we can get: $\opnorm{\bZ_{F,>\ell}} = \mathcal{O}_{\prec}(\frac{1}{N_{\ell+1} })$. Then since $\opnorm{\bZ\bR} < \infty$ and $|\sqrt{N_{k}} \xi_k - \mu_k| = o_d(1)$ (due to Eq.~\eqref{eq:q_He_coeff_conv} in Appendix \ref{sec:Spherical-Harmonics}), we obtain 
\begin{align}
\label{eq:critical_chi1_pconv}
    \Big| \chi_1 - \frac{F_{\ell}^2 }{ N_{\ell} } \Tr\Big[\frac{\mu_{\ell} }{\sqrt{p} } \usp_{\ell}(\bW \bX^\sT) \bZ \bR\Big] \Big| = o_{d,\P}(1) \, .
\end{align}
Meanwhile, by Lemma \ref{lem:inprodmtx_supopnorm_bd} and Lemma \ref{lem:inprodmtx_opnorm_prbd}, it is easy to show that 
\[
\E|\chi_1|^2 < \infty\, , \qquad \E \big|\frac{F_{\ell}^2 }{ N_{\ell} } \Tr\big[\frac{\mu_{\ell} }{\sqrt{p} } \usp_{\ell}(\bW \bX^\sT) \bZ \bR\big] \big|^2 < \infty \, .
\]
Then in light of Eq.~\eqref{eq:critical_chi1_pconv}, we reach at Eq.~\eqref{eq:critical_chi1_concentrate}.

\subsubsection{Auxiliary lemmas}
\begin{lemma}
\label{lem:overpara_mtx_concentrate}
When $\kappa_1 > \kappa_2$, we have the following:
\begin{align}
    \big\| \bPi_{\geq \ell} - \bG_{\geq \ell} \big\|_\op &=  o_{d,\P}(1)\, , \label{eq:G_l_concentrate}\\
    \Big\|\bZ_{\geq \ell} \bU_{\geq \ell} \bZ_{\geq \ell}^\sT - \sum_{k \geq \ell} \xi_k^4 \bPsi_k \bPsi_k^\sT \Big\|_\op &=  o_{d,\P}\left(\frac{1}{n}\right)\, ,
    \label{eq:ZUZ_concentrate}\\ 
    \Big\| \bZ_{\geq \ell}\bZ_F^\sT - \sum_{k \geq \ell} \frac{ F_k^2}{N_k}\xi_k^2 \bPsi_k \bPsi_k^\sT \Big\|_\op &=  o_{d,\P}\left(\frac{1}{n}\right)\, . \label{eq:Z_ZF_concentrate}
\end{align}
\end{lemma}
\begin{proof}
In our proof, $\tau$ and $D$ are constants that may change from line to line.

\noindent
\textbf{Proof of Equation \eqref{eq:G_l_concentrate}.} Since 
$$
\bPi_{\geq \ell} - \bG_{\geq \ell} = \bG_{\geq \ell} \Big(   \sum_{k\geq\ell}^{\infty} \xi_k^2 \bPsi_k \bPsi_k^\sT - \bZ_{\geq\ell} \bZ_{\geq\ell}^\sT \Big) \bPi_{\geq\ell} \, ,
$$
and $\opnorm{\bPi_{\geq\ell}}, \opnorm{\bG_{\geq\ell}} \leq \frac{1}{\lambda}$, 
to prove Eq.~\eqref{eq:G_l_concentrate}, it suffices to show
\begin{align}
    \Big\|\bZ_{\geq\ell} \bZ_{\geq\ell}^\sT - \sum_{k\geq\ell}^{\infty} \xi_k^2 \bPsi_k \bPsi_k^\sT\Big\|_\op 
    &=  o_{d,\P}(1) \, . 
    \label{eq:ZgeqlZgeqlT_concentrate}
\end{align}
This directly follows from Eq.~\eqref{eq:fXWT_concentrate} in Lemma \ref{lem:fXWT_fWXT_concentrate} by letting $f(x)=\sigma_{\geq\ell}(x)$.

\noindent
\textbf{Proof of Equation \eqref{eq:ZUZ_concentrate}.}
By definition, we have $\kappa_1 > \ell-1$. First, let us consider the case when $\kappa_1 \in (\ell-1, \ell]$. Then $\kappa_2 < \kappa_1 \leq \ell$ and thus $n \ll N_{k}$ for any integer $k\geq\ell$. It follows that $\opnorm{\bU_{\geq\ell}} = O_{\prec}(\frac{1}{N_\ell}) = o_{d,\P}(\frac{1}{n})$ [by Eq.~\eqref{eq:sum_mu_QkW} in Lemma \ref{lem:kernelmtx_opnormbd}] and
\begin{align}
  \label{eq:case_I_smallopnorm_1}
  \opnorm{\sum_{k\geq\ell} \xi_k^4 \bPsi_k \bPsi_k^\sT}
  = O_{\prec}\left(\frac{1}{N_{\ell}}\right) 
  =  o_{d,\P}\left(\frac{1}{n}\right) \, ,
\end{align}
[by Eq.~\eqref{eq:sum_F_QkX} in Lemma \ref{lem:kernelmtx_opnormbd}]. Since $\opnorm{\bZ_{\geq\ell}} = O_{\prec}(1)$, we also get
\begin{align}
   \opnorm{\bZ_{\geq\ell} \bU_{\geq\ell} \bZ_{\geq\ell}} = O_{\prec} \left(\frac{1}{N_\ell} \right) = o_{d,\P} \left(\frac{1}{n} \right) \, .
   \label{eq:case_I_smallopnorm_2}
\end{align}
Therefore, Eq.~\eqref{eq:ZUZ_concentrate} holds, as the operator norms of both matrices on the left-hand side vanish as $o_{d,\P}(n^{-1})$.

Then we consider the case when $\kappa_1 > \ell$.
To this end, we have the following bound:
\begin{equation}
\label{eq:ZlUlZlT_decomp}
\begin{aligned}
    \Big\|\bZ_{\geq\ell} \bU_{\geq\ell} \bZ_{\geq\ell}^\sT - \sum_{k\geq\ell} \xi_k^4 \bPsi_k \bPsi_k^\sT \Big\|_\op &\leq \sum_{k\geq\ell}^{\lceil \kappa_1 \rceil - 1} \xi_k^2 \left\| \frac{\bZ_{>\ell} \bPhi_k \bPhi_k^\sT \bZ_{>\ell}^{\sT}}{p} - \xi_k^2 \bPsi_k \bPsi_k^\sT \right\|_\op  \\
    &~~+\frac{1}{p}\left\| \bZ_{\geq\ell} \sum_{k\geq \lceil \kappa_1 \rceil}^{\infty}  {\xi_k^2}\bPhi_k \bPhi_k^\sT \bZ_{\geq\ell}^{\sT} 
    \right\|_\op
    + \left\| \sum_{k\geq \lceil \kappa_1 \rceil}^{\infty}\xi_k^4 \bPsi_k \bPsi_k^\sT \right\|_\op \, .
\end{aligned}
\end{equation}
The second and third term on the right-hand side of Eq.~\eqref{eq:ZlUlZlT_decomp} can be bounded using the same arguments in obtaining Eqs.~\eqref{eq:case_I_smallopnorm_1} and \eqref{eq:case_I_smallopnorm_2}. In particular, we can show
\begin{align}
    \frac{1}{p}\left\| \bZ_{\geq\ell} \sum_{k\geq \lceil \kappa_1 \rceil}^{\infty}  {\xi_k^2}\bPhi_k \bPhi_k^\sT \bZ_{>\ell}^{\sT} 
    \right\|_\op = o_{d,\P}\left( \frac{1}{n} \right) \, m
    \text{~~and~~}
    \left\| \sum_{k\geq \lceil \kappa_1 \rceil}^{\infty}\xi_k^4 \bPsi_k \bPsi_k^\sT \right\|_\op 
    = o_{d,\P}\left(\frac{1}{n}\right) \, .
\end{align}
To control the first term, we can use the simple bound:
\begin{align}
    \left\|\frac{\bZ_{\geq\ell} \bPhi_k \bPhi_k^\sT \bZ_{\geq\ell}^{\sT}}{p} - \xi_k^2 \bPsi_k \bPsi_k^\sT\right\|_{\op} 
    &\leq \left( \left\|\frac{\bZ_{\geq\ell} \bPhi_k}{\sqrt{p}}\right\|_\op + \left\| \xi_k\bPsi_k \right\|_\op \right)
    \left\| \frac{\bZ_{\geq\ell} \bPhi_k}{\sqrt{p}} - \xi_k \bPsi_k \right\|_\op.
\end{align}
Besides, for any fixed $k\geq\ell$, $\opnorm{\xi_k \bPsi_k} = O_{\prec}(1)$ \cite[Proposition 8]{lu2022equivalence} and $\sup_{k\geq\ell}\xi_k^2 = O_d(\frac{1}{N_\ell})$,
so it suffices to show for every $k\in[\ell,\lceil\kappa_1\rceil-1]$,
\begin{align}
    \Big\| \frac{\bZ_{\geq\ell} \bPhi_k}{\sqrt{p}} - \xi_k \bPsi_k \Big\|_\op &=  o_{d,\P}(1).
    \label{eq:Z_l_Phi_k_concentrate}
\end{align}
which directly follows from Eq.~\eqref{eq:fXWT_Phik_concentrate} in Lemma \ref{lem:fXWT_fWXT_concentrate}, since $p \gg N_k$ when $k \leq \lceil\kappa_1\rceil-1$.

\noindent
\textbf{Proof of Equation \eqref{eq:Z_ZF_concentrate}.} The proof of Eq.~\eqref{eq:Z_ZF_concentrate} is completely analogous to Eq.~\eqref{eq:ZUZ_concentrate}. 
We omit it for brevity.
\end{proof}

\clearpage

\section{Stieltjes Transform}
\label{sec:Stieltjes_transform}

In this appendix, we prove the limit for the Stieltjes transform of the block matrix $\bA := \bA(\bq)$ define in Eq.~\eqref{eq:def_block_matrix_A}. Recall that the asymptotics for RFRR in the critical regime $\kappa_1 = \kappa_2$ are obtained as derivatives of the log-determinant of $\bA$.

\subsection{Proof of Proposition \ref{prop:resolvent_conv}}
\label{subsec:proof_stieltjes}
First, consider the case when $\sigma(x) = \sum_{k=0}^{L} \mu_k \usp_{k}(x)$, where $L\in\mathbb{Z}$, with $L\geq\ell$ and $\mu_{\ell} \neq 0$. Define $\hat{\sigma}(x)=\sum_{k=\ell}^{L} \mu_k \usp_{k}(x)$ and $\hbA$ as the matrix obtained after replacing $\sigma(x)$ in $\bA$ by $\hat{\sigma}(x)$.
It can be easily verified that $\text{rank}(\hbA - \bA) \leq 2 N_{<\ell}$. Then by Lemma 18 in \cite{lu2022equivalence}, we have 
\begin{align}
    \big| \hat{M}_{d}(z;\bq) - {M}_{d}(z;\bq) \big| \leq \frac{C N_{<\ell}}{m \eta} = o_d \left(\frac{1}{\eta} \right)\, ,
\end{align}
where $C$ is an absolute constant. Therefore, we can assume that $\mu_k=0$, when $k < \ell$. 

Now define the following partial Stieltjes transforms:
\begin{align}
    {M}_{1,d} ( z ; \bq) &= \frac{1}{m} \Tr_{[1:p]} \big[ ( \bA - z \id_m )^{-1} \big]\, , \hspace{4em} m_{1,d}( z ; \bq ) = \E[{M}_{1,d} ( z ; \bq)]\, ,\\
    {M}_{2,d} ( z ; \bq) &= \frac{1}{m} \Tr_{[p+1:p+n]} \big[ ( \bA - z \id_m )^{-1} \big]\, ,\qquad m_{2,d}( z ; \bq ) = \E[{M}_{2,d} ( z ; \bq)\, .
\end{align}
Note that $M_d(z,\bq) = {M}_{1,d} ( z ; \bq) + {M}_{2,d} ( z ; \bq)$.
The key step of proving Proposition \ref{prop:resolvent_conv} is to show as $d\to\infty$, the partial Stieltjes transforms converge to the fixed points of Eq.~\eqref{eq:fixedpoint_eq}. This is established in Lemma \ref{lem:fixedpoint_eq_conv}, whose proof is based on the leave-one-out approach for computing the Stieljes transform of inner-product kernel matrices \cite{cheng2013spectrum}. 
In particular, we can get
\begin{align}
    \label{eq:Md_concentrate_expectation}
    \E |M_d(z;\bq) - m(z;\bq)| = o_d(1)\, .
\end{align}

Finally, using the same approximation argument in \cite{lu2022equivalence} (see proof of Theorem 2), one can show that Eq.~\eqref{eq:Md_concentrate_expectation} holds for general $\sigma(x)$.
The uniform convergence can be established by the same procedure as in \cite{mei2022generalizationRF} (Step 3 in the proof of Proposition 8.3). We omit the details for the sake of brevity.

\subsubsection{Limits of partial Stieltjes transforms}
\begin{lemma}
\label{lem:fixedpoint_eq_conv}
Under the same setting as Proposition \ref{prop:resolvent_conv}, assume $\sigma(x) = \sum_{k=\ell}^{L} \mu_k \usp_{k}(x)$, where $L\in\mathbb{Z}$, with $L\geq\ell$ and $\mu_{\ell} \neq 0$. Let $z=E+i\eta$. When $(E, \eta, \bq,\psi_1,\psi_2)$ is in a bounded set, there exists an absolute $C>0$ such that for any $z$ with $\eta>0$,  $\varepsilon\in(0,1/2)$ and all large $d$,
\begin{align}
    \Big| {m}_{1,d}(z;\bq) - {\textsf{F}}_1\big( {m}_{1,d}(z;\bq), {m}_{2,d}(z;\bq) ; z, \bq \big) \Big| &\leq \frac{C}{\eta^8}  \max\left\{ \frac{1}{d^{\frac{1}{2}-\epsilon}}, \frac{1}{m^{\frac{1}{2}-\epsilon}}\right\}\, , \label{eq:m1_fixpt_concentrate}\\
    \Big|{m}_{2,d}(z;\bq) - {\textsf{F}}_2\big( {m}_{1,d}(z;\bq), {m}_{2,d}(z;\bq) ; z, \bq \big)\Big| &\leq \frac{C}{\eta^8}  \max\left\{ \frac{1}{d^{\frac{1}{2}-\epsilon}}, \frac{1}{m^{\frac{1}{2}-\epsilon}}\right\}\, . \label{eq:m2_fixpt_concentrate}
\end{align}
\end{lemma}
\begin{proof}
To simplify the notations, we will assume $t=0$ throughout the proof. The $t\neq 0$ case is completely the same after the replacement: $\mu_{\ell} \Rightarrow (1+t)\mu_{\ell}$.

Let $\bA^{[i]}$ be the minor of $\bA$ with the $i$th column and row removed.
Let $\bS(z) = (\bA - z\id_m)^{-1}$ and $\bS^{[i]}(z) = (\bA^{[i]} - z\id_{m-1})^{-1}$ be the resolvents of $\bA$ and $\bA^{[i]}$. Note that $M_{1,d}(z;\bq) = \frac{1}{m}\sum_{i=1}^{p} S_{ii}(z)$ and $M_{2,d}(z;\bq) = \frac{1}{m}\sum_{i=p+1}^{m} S_{ii}(z)$.

Let $\bA_{\cdot,i}$ be the $i$th column of $\bA$ with the $i$th entry removed.
By Schur's complement formula, for $1 \leq i \leq p$
\begin{align}
    S_{ii}(z) = ( - z + s_1 + s_2 - \bA_{\cdot,i}^\sT \bS^{[i]}(z) \bA_{\cdot,i})^{-1}\, ,
    \label{eq:schur_1}
\end{align}
and for $p+1 \leq i \leq p+n$,
\begin{align}
    S_{ii}(z) = ( - z + t_1 + t_2 - \bA_{\cdot,i}^\sT \bS^{[i]}(z) \bA_{\cdot,i})^{-1}\, .
    \label{eq:schur_2}
\end{align}
The next step is to compute the limit of $S_{ii}(z)$ based on Eqs.~\eqref{eq:schur_1} and \eqref{eq:schur_2}.
We will elaborate on $i=1$ case. The same analysis applies to $i\neq 1$ case as well, due to the symmetry of $\bA$.

The key to analyzing $S_{11}(z)$ is to handle the (weak) correlation between $\bA_{\cdot,1}$ and $\bS^{[1]}(z)$. 
To this end, we utilize the following representation of $\{\bx_a\}_{a\in[n]}$ and $\{\bw_{i}\}_{i\neq 1}$ \cite[Lemma 3]{lu2022equivalence}:
\begin{align}
    \bx_{a}^\sT &= \Big[\gamma_a, \frac{\sqrt{d-\gamma_a^2}}{\sqrt{d - 1}} \tbx_a^\sT \Big]
    \begin{bmatrix}
    \bw_1^\sT \\
    \bR_1^\sT
    \end{bmatrix} \label{eq:x_representation}\\
    \sqrt{d}\bw_{i}^\sT &= \Big[\theta_i, {\sqrt{d-\theta_i^2}} \tbw_i^\sT \Big]
    \begin{bmatrix}
    \bw_1^\sT \\
    \bR_1^\sT,
    \end{bmatrix}
    \label{eq:w_representation}
\end{align}
where $\bR_1^\sT \in\R^{d\times (d-1)}$ is an arbitrary matrix satisfying $\bR_1^\sT \bR_1 = \id_{d-1}$ and $\bR_1^\sT \bw_1 = \bzero$, 
$\gamma_a \simiid \taumeasure$,
$\tbx_a \simiid \Unif (\S^{d-2}(\sqrt{d-1}))$, 
$\theta_i \simiid \taumeasure$, 
$\tbw_{i} \simiid \Unif (\S^{d-2}(1))$,
and 
\[
\big\{
\bw_1, 
\{\gamma_a\}_{a\in[n]}, \{\tbx_a\}_{a\in[n]}, 
\{\theta_i\}_{i\neq 1},
\{\tbw_i\}_{i\neq 1}
\big\}
\]
are mutually independent.
Under this representation, we have
\begin{align}
    \bA_{\cdot,1} = 
    \begin{bmatrix}
        \frac{1}{\sqrt{N_{\ell}} } s_2 \usp_{\ell}(\btheta) \\
        \frac{1}{\sqrt{m} } \sigma(\bgamma)
    \end{bmatrix} \, ,
    \label{eq:bA_1_representation}
\end{align}
and
\begin{align}
    \bB := \bA^{[1]} = 
    \begin{bmatrix}
    \bB_{11} & \bB_{12}\\
    \bB_{21} & \bB_{22}
    \end{bmatrix},
    \label{eq:bA_not1_representation}
\end{align}
with
\begin{align}
    \bB_{11} &= s_1 \id_{p-1} + \frac{s_2}{\sqrt{N_{\ell}}} \usp_{\ell}\Big( \sqrt{d-1} \diag\{r(\theta_i)\} {\tbW\tbW^\sT} \diag\{r(\theta_i)\} + \frac{\btheta \btheta^\sT}{\sqrt{d}} \Big)\, , \\
    \bB_{12} &= \bB_{21}^\sT = \frac{1}{\sqrt{m}} \sigma\Big( \diag\{r(\theta_i)\} {\tbW \tbX^{\sT}} \diag\{r(\gamma_i)\} + \frac{\btheta \bgamma^{\sT}}{\sqrt{d}} \Big)\, , \\
    \bB_{22} &= t_1 \id_n + \frac{t_2}{\sqrt{N_{\ell}}} \usp_{\ell}\Big( \frac{1}{\sqrt{d-1}} \diag\{r(\gamma_i)\} {\tbX\tbX^\sT} \diag\{r(\gamma_i)\} + \frac{\bgamma \bgamma^\sT}{\sqrt{d}} \Big)\, ,
\end{align}
where 
\begin{align}
   r(x) := (1 - 1/d)^{-\frac{1}{4}} (1 - x^2/d)^{\frac{1}{2}}. 
\end{align}
Recall that $\bS^{[1]}(z) = (\bA^{[1]} - z\id_{m-1})^{-1}$, so from Eqs.~\eqref{eq:bA_1_representation} and \eqref{eq:bA_not1_representation} we can see under the new representation Eqs.~\eqref{eq:x_representation} and \eqref{eq:w_representation}, the correlation between $\bA_{\cdot,1}$ and $\bS^{[1]}(z)$ is fully captured by $\btheta$ and $\bgamma$. However, current form is still complicated, as $\btheta$ and $\bgamma$ are both hidden inside the non-linear functions $\usp_{\ell}(x)$ and $\sigma(x)$. To proceed, we use Proposition 1 in \cite{lu2022equivalence}, which states for $k=0,1,\cdots,L$ and $a\neq b$,
\begin{align}
    q_k\big( r(\varsigma_a)r(\varsigma_b)x + {\varsigma_a \varsigma_b}/\sqrt{d} \big) = \sum_{t=0}^k \tusp_{k-t}(x) r^{k-t}(\varsigma_a) r^{k-t}(\varsigma_b)\Big[ \sqrt{(k)_t} \frac{ \usp_t(\varsigma_a) \usp_t(\varsigma_b) }{d^{t/2}} + \cE_t(\varsigma_a, \varsigma_b) \Big],
    \label{eq:luyau_proposition1}
\end{align}
where 
$$
\varsigma_a = 
\begin{cases}
\theta_a, &1 \leq a \leq p-1\, ,\\
\gamma_{a-p+1}& p \leq a \leq m-1\, ,
\end{cases}
$$
$\tusp_k$ is the $k$th Gegenbauer polynomial in dimension $d-1$, $(k)_t = k!/(k-t)!$ and
\begin{align}
    \cE_t(\varsigma_a, \varsigma_b) := \left\{
    \sum_{0\leq \alpha, \beta \leq t} c_{\alpha\beta} \frac{\usp_{\alpha}(\varsigma_a) \usp_{\beta}(\varsigma_b)}{d^{\max\{ \alpha, \beta \}/2 + 1}} : c_{\alpha\beta}(d) = O_d(1), \, 0\leq \alpha,\beta \leq t
    \right\}\, . 
\end{align}
Now define
\begin{align}
    \tbB = 
    \begin{bmatrix}
    \tbB_{11} & \tbB_{12}\\
    \tbB_{21} & \tbB_{22}
    \end{bmatrix},
    \label{eq:tB_not1_representation}
\end{align}
with
\begin{align}
    \tbB_{11} &= s_1 \id_{p-1} + \frac{s_2}{\sqrt{N_{\ell}}} \tusp_{\ell}\Big( \sqrt{d-1} {\tbW\tbW^\sT} \Big)\, , \\
    \tbB_{12} &= \tbB_{21}^\sT = \frac{1}{\sqrt{m}} \sigma\big( {\tbW \tbX^{\sT}} \big) \, ,\\
    \tbB_{22} &= t_1 \id_n + \frac{t_2}{\sqrt{N_{\ell}}} \tusp_{\ell}\Big( \frac{1}{\sqrt{d-1}} {\tbX\tbX^\sT} \Big) \, .
\end{align}
Also for each $k$, define the following matrices:
\begin{align}
    [\tbC_k]_{a,b} &:= \frac{\tusp_{k}(\sqrt{d-1}\tbw_a^\sT \tbw_b)}{\sqrt{N_{\ell}}} \charfn_{a\neq b} \, , \\
    [\tbD_k]_{a,b} &:= \frac{\tusp_{k}(\tbx_a^\sT \tbw_b)}{\sqrt{m}} \, ,\\
    [\tbE_k]_{a,b} &:= \frac{\tusp_{k}(\tbx_a^\sT \tbx_b /\sqrt{d-1})}{\sqrt{N_{\ell}}} \charfn_{a\neq b}\, .
\end{align}
Using Eq.~\eqref{eq:luyau_proposition1}, we can approximate $\bB$ by $\tbB$ as:
\begin{align}
    \bB &= \tbB + \frac{1}{\sqrt{N_{\ell} m} } 
    \begin{bmatrix}
    \usp_{\ell}(\btheta) & \bzero \\
    \bzero & \usp_{\ell}(\bgamma)
    \end{bmatrix}
    \begin{bmatrix}
    \sqrt{\psi} s_2 & \mu_{\ell} \\
    \mu_{\ell} & \sqrt{\psi} t_2
    \end{bmatrix}
    \begin{bmatrix}
    \usp_{\ell}(\btheta) & \bzero \\
    \bzero & \usp_{\ell}(\bgamma)
    \end{bmatrix}^\sT +
    \begin{bmatrix}
    \bDelta_{11} & \bDelta_{12}\\
    \bDelta_{12}^\sT & \bDelta_{22}
    \end{bmatrix} \nonumber \\
    &= \tbB + \frac{1}{\sqrt{N_{\ell} m} } \bU_{\ell} \bT \bU_{\ell}^\sT + \bDelta \, ,
    \label{eq:B_approximation}
\end{align}   
where $\bDelta_{\mu \nu} = \sum_{i=1}^4 \bDelta_{i, \mu \nu}$ (with $\mu, \nu = 1\text{ or }2$) 
and $\bDelta_{i, \mu \nu}$ are defined as follows. 

\noindent(i) $\bDelta_{i, 11}$:
\begin{align}
    [\bDelta_{1, 11}]_{a,b} &= s_2 \sum_{0\leq t \leq \ell} \sqrt{(\ell)_{t}}  \big[ \tbC_{\ell-t} \big]_{a,b} \big( r^{\ell-t}(\theta_a) r^{\ell-t}(\theta_b) - 1 \big) \frac{ \usp_t(\theta_a) \usp_t(\theta_b) }{ d^{t/2} } \, ,\\
    [\bDelta_{2, 11}]_{a,b} &= s_2 \sum_{1\leq t < \ell} \sqrt{(\ell)_{t}} \big[ \tbC_{\ell-t} \big]_{a,b} \frac{ \usp_t(\theta_a) \usp_t(\theta_b) }{ d^{t/2} } \, ,\\
    [\bDelta_{3, 11}]_{a,b} &= 0 \, ,\\
    [\bDelta_{4, 11}]_{a,b} &= s_2 \sum_{0\leq t \leq \ell} \big[ \tbC_{\ell-t} \big]_{a,b} r^{\ell-t}(\theta_a) r^{\ell-t}(\theta_b) \cE_t(\theta_a, \theta_b) \, .
\end{align}
(ii) $\bDelta_{i, 12}$ ($=\bDelta_{i, 21}^\sT$):
\begin{align}
    [\bDelta_{1, 12}]_{a,b} &= \sum_{k\geq\ell}^{L} \mu_k \sum_{0\leq t \leq k} \sqrt{(k)_{t}} \big[  \tbD_{k-t} \big]_{a,b}\big( r^{k-t}(\theta_a) r^{k-t}(\gamma_b) - 1 \big) \frac{ \usp_t(\theta_a) \usp_t(\gamma_b) }{ d^{t/2} }\, , \\
    [\bDelta_{2, 12}]_{a,b} &= \mu_{\ell} \sum_{1\leq t < \ell} \sqrt{(\ell)_{t}} \big[  \tbD_{\ell-t} \big]_{a,b} \frac{ \usp_t(\theta_a) \usp_t(\gamma_b) }{ d^{t/2} } \, ,\\
    [\bDelta_{3, 12}]_{a,b} &= \sum_{k>\ell}^{L} \mu_{k} \sum_{1\leq t \leq k} \sqrt{(k)_{t}} \big[ \tbD_{k-t} \big]_{a,b} \frac{ \usp_t(\theta_a) \usp_t(\gamma_b) }{ d^{t/2} } \, ,\\
    [\bDelta_{4, 12}]_{a,b} &= \sum_{k\geq \ell}^{L} \mu_k \sum_{0\leq t \leq k} \big[ \tbD_{k-t} \big]_{a,b} r^{k-t}(\theta_a) r^{k-t}(\gamma_b) \cE_t(\theta_a, \gamma_b) \, .
\end{align}
(iii) $\bDelta_{i, 22}$:
\begin{align}
    [\bDelta_{1, 22}]_{a,b} &= t_2 \sum_{0\leq t \leq \ell} \sqrt{(\ell)_{t}} \big[  \tbE_{\ell-t} \big]_{a,b} \big( r^{\ell-t}(\gamma_a) r^{\ell-t}(\gamma_b) - 1 \big) \frac{ \usp_t(\gamma_a) \usp_t(\gamma_b) }{ d^{t/2} } \, ,\\
    [\bDelta_{2, 22}]_{a,b} &= t_2 \sum_{1\leq t < \ell} \sqrt{(\ell)_{t}} \big[  \tbE_{\ell-t} \big]_{a,b} \frac{ \usp_t(\gamma_a) \usp_t(\gamma_b) }{ d^{t/2} } \, ,\\
    [\bDelta_{3, 22}]_{a,b} &= 0 \, ,\\
    [\bDelta_{4, 22}]_{a,b} &= t_2 \sum_{0\leq t \leq \ell} \big[  \tbE_{\ell-t} \big]_{a,b} r^{\ell-t}(\gamma_a) r^{\ell-t}(\gamma_b) \cE_t(\gamma_a, \gamma_b) \, .
\end{align}
We will also write
\begin{align}
    \bDelta_{i} = 
    \begin{bmatrix}
    \bDelta_{i,11} & \bDelta_{i,12} \\
    \bDelta_{i,21} & \bDelta_{i,22}
    \end{bmatrix},
\end{align}
so $\bDelta = \sum_{i=1}^{4} \bDelta_{i}$.

From Eq.~\eqref{eq:B_approximation}, we have:
\begin{align}
    \bS^{[1]}(z) = \Big[ \tbS(z) - \tfrac{1}{\sqrt{N_{\ell} m} } \tbS(z) \bU_{\ell} \big( \bT^{-1} + \tfrac{\bU_{\ell}^\sT \tbS(z) \bU_{\ell}}{\sqrt{N_{\ell} m}} \big)^{-1} \bU_{\ell}^\sT \tbS(z) \Big]\big[\id - \bDelta \bS^{[1]}(z)\big]\, ,
    \label{eq:Snot1_approx}
\end{align}
where $\tbS(z) := (\tbB - z\id)^{-1}$. On the other hand, $\bA_{\cdot,1}$ can be decomposed as:
\begin{align}
    \bA_{\cdot,1} = \frac{1}{\sqrt{m}} \sum_{k\geq\ell}^{L} \bU_{k} \bbmu_{k} \, ,
    \label{eq:A1_decomp}
\end{align}
where 
\begin{align}
    \bU_k = 
    \begin{bmatrix}
    \usp_{k}(\btheta) & \bzero \\
    \bzero & \usp_{k}(\bgamma)
    \end{bmatrix},
    \hspace{2em}
    \bbmu_k = 
    \begin{bmatrix}
    \sqrt{\psi} s_2 \charfn_{k=\ell}\\
    \mu_k
    \end{bmatrix} \, .
\end{align}
Then combining Eqs.~\eqref{eq:Snot1_approx} and \eqref{eq:A1_decomp}, we can obtain
\begin{equation}
\label{eq:Anot1T_Snot1_Anot1}
\begin{aligned}
    \bA_{\cdot,1}^\sT \bS^{[1]}(z) \bA_{\cdot,1} &= \sum_{k,j\geq\ell}^{L} \bbmu_k^\sT \bchi_{kj}(z) \bbmu_j - \sqrt{\psi} \bvartheta(z)^\sT \sum_{k > \ell}^{L} \bchi_{\ell, k}(z) \bbmu_{k} - \bA_{\cdot,1}^{\sT} \tbS(z) \bDelta \bS^{[1]}(z) \bA_{\cdot,1}\\
    &~~~+ \frac{1}{\sqrt{N_{\ell}}}\bvartheta(z)^\sT \bU_{\ell}^\sT \tbS(z) \bDelta \bS^{[1]}(z) \bA_{\cdot,1} \, ,
\end{aligned} 
\end{equation}
where 
\begin{align}
\label{eq:def_bchi}
  \bchi_{kj}(z) = \frac{ 1 }{m} \bU_{k}^\sT \tbS(z) \bU_{j}   \, ,
\end{align}
and
\begin{align}
    \bvartheta(z) = \sum_{k\geq\ell}^{L}  \big( \bT^{-1} + \sqrt{\psi} \bchi_{\ell,\ell}(z) \big)^{-1} \bchi_{\ell,k}(z) \bbmu_{k} \, .
\end{align}
Now we compute the limit of $\bA_{\cdot,1}^\sT \bS^{[1]}(z) \bA_{\cdot,1}$. First, following the same proof of Lemma 10.12 in \cite{mei2022generalizationRF} (using Lemma \ref{lem:Mi_concentrate} and Lemma \ref{lem:bchi_concentrate}), we can get
\begin{align}
    \opnorm{ \big( \bT^{-1} + \sqrt{\psi} \bchi_{\ell,\ell}(z) \big)^{-1} } &\prec \frac{1}{\eta} \, , \label{eq:Tinv_opnormbd_1}\\
    \opnorm{ \big( \bT^{-1} + \sqrt{\psi} \bchi^{*}_{}(z) \big)^{-1} } &\prec \frac{1}{\eta}\, ,\label{eq:Tinv_opnormbd_2}
\end{align}
where 
\begin{align}
\label{eq:def_bchistar}
    \bchi^{*}(z) := \diag\{ {m}_{1,d}(z;\bq), {m}_{2,d}(z;\bq) \}\, .
\end{align}
Let us define
\begin{align}
    \bvartheta^{*}(z) := \big( \bT^{-1} + \sqrt{\psi} \bchi^{*}(z) \big)^{-1} \bchi^{*}(z) \bbmu_{\ell} \, .
\end{align}
Then we have
\begin{align}
    \| \bvartheta(z) - \bvartheta^{*}(z) \| &\leq \sqrt{\psi} \big\| \big( \bT^{-1} + \sqrt{\psi} \bchi_{\ell,\ell}(z) \big)^{-1} \big( \bchi_{\ell,\ell}(z) - \bchi^{*}(z) \big) \big( \bT^{-1} + \sqrt{\psi}\bchi^{*}(z) \big)^{-1} \bT^{-1} \bbmu_{\ell} \big\| \nonumber \\
    &~~+ \sum_{k>\ell}^{L}  \big\| \big( \bT^{-1} + \sqrt{\psi} \bchi_{\ell,\ell}(z) \big)^{-1} \bchi_{\ell,k}(z) \bbmu_{k} \big\| \nonumber \\
    &\prec \frac{1}{\eta^4 } \max\Big\{ \frac{1}{\sqrt{d}}, \frac{1}{\sqrt{m}}\Big\}\, ,
    \label{eq:vartheta_concentrate}
\end{align}
where the last step follows from Eqs.~\eqref{eq:Tinv_opnormbd_1} and \eqref{eq:Tinv_opnormbd_2}, Lemma \ref{lem:bchi_concentrate} and the fact that $\opnorm{\bT^{-1}} = O_d(1)$ (due to $\bq \in \cQ$ and $\mu_{\ell}>0$). After applying Eq.~\eqref{eq:vartheta_concentrate}, Lemma \ref{lem:Mi_concentrate}, Lemma \ref{lem:bchi_concentrate} and Lemma \ref{lem:Delta_quadratic_vanish} in Eq.~\eqref{eq:Anot1T_Snot1_Anot1}, we can get
\begin{equation}
\label{eq:Anot1T_Snot1_Anot1_concentrate}
\begin{aligned}
    \bA_{\cdot,1}^\sT \bS^{[1]}(z) \bA_{\cdot,1} &= \psi s_2^2 {m}_{1,d}(z;\bq) + \sum_{k\geq\ell}\mu_k^2 {m}_{2,d}(z;\bq) \\
    &~~~ -\sqrt{\psi}\bbmu_{\ell}^\sT \bchi^{*}(z) \big( \bT^{-1} + \sqrt{\psi} \bchi^{*}(z) \big)^{-1} \bchi^{*}(z) \bbmu_{\ell} + 
    \frac{1}{\eta^6} O_{\prec}\Big(  \max\Big\{ \frac{1}{\sqrt{d}}, \frac{1}{\sqrt{m}}\Big\}\Big)\,  .
\end{aligned}
\end{equation}
Substituting Eq.~\eqref{eq:Anot1T_Snot1_Anot1_concentrate} into Eq.~\eqref{eq:schur_1} and using the exchangeability of $\{S_{ii}(z)\}_{i=1}^{p}$, we can get (after some simplifications) for any $i\in \{1,2,\cdots,p\}$,
\begin{align}
    \frac{1}{S_{ii}(z)} =  \frac{\theta_1/\theta}{ {\textsf{F}}_1\big( {m}_{1,d}(z;\bq), {m}_{2,d}(z;\bq) ; z, \bq \big) } + \frac{1}{\eta^6} O_{\prec}\Big(  \max\Big\{ \frac{1}{\sqrt{d}}, \frac{1}{\sqrt{m}}\Big\}\Big).
    \label{eq:schur_concentrate_1}
\end{align}
Repeating the same steps leading to \eqref{eq:schur_concentrate_1}, we can also get for any $i\in \{p+1,p+2,\cdots,m\}$,
\begin{align}
    \frac{1}{S_{ii}(z)} =  \frac{\theta_2/\theta}{ {\textsf{F}}_2\big( {m}_{1,d}(z;\bq), {m}_{2,d}(z;\bq) ; z, \bq \big) } + \frac{1}{\eta^6} O_{\prec}\Big(  \max\Big\{ \frac{1}{\sqrt{d}}, \frac{1}{\sqrt{m}}\Big\}\Big).
    \label{eq:schur_concentrate_2}
\end{align}
It can be directly checked that $\Im[ {m}_{1,d}(z;\bq) ] \geq 0$, $\Im[ {m}_{2,d}(z;\bq) ] \geq 0$ and $|S_{ii}(z)| \leq \frac{1}{\eta}$, for all $i=1,2,\cdots,m$. Then together with Eq.~\eqref{eq:schur_concentrate_1}, Eq.~\eqref{eq:schur_concentrate_2} and Lemma \ref{lem:F1_F2_bd}, we can also get:
\begin{align}
    S_{ii}(z) = \begin{cases}
    \frac{\theta}{\theta_1} {\textsf{F}}_1\big( {m}_{1,d}(z;\bq), {m}_{2,d}(z;\bq) ; z, \bq \big) + \frac{1}{\eta^8} O_{\prec}\Big(  \max\Big\{ \frac{1}{\sqrt{d}}, \frac{1}{\sqrt{m}}\Big\}\Big), & 1\leq i \leq p\, , \\
    \frac{\theta}{\theta_2} {\textsf{F}}_2\big( {m}_{1,d}(z;\bq), {m}_{2,d}(z;\bq) ; z, \bq \big) + \frac{1}{\eta^8} O_{\prec}\Big(  \max\Big\{ \frac{1}{\sqrt{d}}, \frac{1}{\sqrt{m}}\Big\}\Big), & p+1 \leq i \leq m\, .
    \end{cases}
\end{align}
Since for all $i=1,2,\cdots,m$, $S_{ii}(z) \leq \frac{1}{\eta}$, we can get when $(E, \eta, \bq, \psi_1, \psi_2)$ is in a bounded set, for any $\epsilon\in(0,1/2)$,
\begin{align}
    \Big| \E S_{ii}(z) - \frac{\theta}{\theta_1} {\textsf{F}}_1\big( {m}_{1,d}(z;\bq), {m}_{2,d}(z;\bq) ; z, \bq \big) \Big| &\leq \frac{C}{\eta^8}  \max\Big\{ \frac{1}{d^{\frac{1}{2}-\epsilon}}, \frac{1}{m^{\frac{1}{2}-\epsilon}}\Big\}\, , \\
    \Big| \E S_{ii}(z) - \frac{\theta}{\theta_2} {\textsf{F}}_2\big( {m}_{1,d}(z;\bq), {m}_{2,d}(z;\bq) ; z, \bq \big)\Big| &\leq \frac{C}{\eta^8}  \max\Big\{ \frac{1}{d^{\frac{1}{2}-\epsilon}}, \frac{1}{m^{\frac{1}{2}-\epsilon}}\Big\}\, ,
\end{align}
where $C>0$ is an absolute constant.
Finally, note that $\E S_{ii}(z) = \frac{\theta}{\theta_1} {m}_{1,d}(z;\bq)$, for $1\leq i \leq p$ and $\E S_{ii}(z) = \frac{\theta}{\theta_2} {m}_{2,d}(z;\bq)$, for $p+1 \leq i \leq m$ and we reach at Eqs.~\eqref{eq:m1_fixpt_concentrate} and \eqref{eq:m2_fixpt_concentrate}.
\end{proof}

\subsubsection{Auxiliary lemmas}
\begin{lemma}
\label{lem:F1_F2_bd}
Suppose $\Im(m_1), \Im(m_2) \geq 0$ and $\eta = \Im(z) > 0$. Then 
\begin{align}
    | {\textsf{F}}_1(m_1, m_2 ; z, \bq) | &\in (0, {\eta}^{-1}] \label{eq:m1_bdimg} \, ,\\
    | {\textsf{F}}_2(m_1, m_2 ; z, \bq) | &\in (0, {\eta}^{-1}]\label{eq:m2_bdimg}\, .
\end{align}
\end{lemma}
\begin{proof}
Here, we only present the proof of Eq.~\eqref{eq:m1_bdimg}. The proof of Eq.~\eqref{eq:m2_bdimg} is the same.

Without loss of generality, assume $t=0$. From Eq.~\eqref{eq:F1_def}, we can get:

\begin{align}
    {\textsf{F}}_1(m_1, m_2 ; z, \bq) &=  \frac{\psi_1}{\psi} \big(-z + s_1 - \mu_{>\ell}^2 m_2 + Q\big)^{-1},
    \label{eq:F1_def2}
\end{align}    
where
\begin{align}
    Q = \frac{(1+\psi t_2 m_2 ) s_2 - \mu_{\ell}^2 m_2}{(1+\psi s_2m_1)(1+\psi t_2m_2)-\psi\mu_{\ell}^2 m_1m_2}\, .
\end{align}
One can check $\Im(Q)\leq 0$, when $\Im(m_1), \Im(m_2) \geq 0$. Therefore,  we have 
$$
\Im(-z+s_1-\mu_{>\ell}^2 m_2 + Q) \leq -\eta\, ,
$$
and substituting this bound into Eq.~\eqref{eq:F1_def2}, we can get Eq.~\eqref{eq:m1_bdimg}.
\end{proof}
\begin{lemma}
\label{lem:Mi_concentrate}
Suppose $\kappa_1 = \kappa_2$. There exists $c>0$ such that for $i=1, 2$ and any $t>0$,
\begin{align}
    \P( | {M}_{i,d}(z;\bq) - {m}_{i,d}(z;\bq) | \geq t ) \leq 2e^{-c m \eta^2 t^2}.
\end{align}
\begin{proof}
The proof is the same as Lemma 10.5 in \cite{mei2022generalizationRF} and is omitted.
\end{proof}
\end{lemma}
\begin{lemma}
\label{lem:bchi_concentrate}
For $1 \leq \alpha, \beta \leq L$, we have
\begin{align}
    \Big\| \bchi_{\alpha,\beta}(z) - 
    \begin{bmatrix}
    {M}_{1,d}(z;\bq)\charfn_{\alpha=\beta} & 0\\
    0 & {M}_{2,d}(z;\bq)\charfn_{\alpha=\beta}
    \end{bmatrix} \Big\|_\op
     {\prec} \frac{1}{\eta^2} \max\Big\{\frac{1}{\sqrt{d}},\frac{1}{\sqrt{m}}\Big\}\, .
\end{align}
\end{lemma}
\begin{proof}
The proof follows the same idea as Lemma 5 in \cite{lu2022equivalence}. For simplicity, we will omit some details here.
First, define 
\begin{align}
    {\tM}_{1,d} ( z ; \bq) &:= \frac{1}{m} \Tr_{[1:p-1]} \tbS(z)\, ,\\
    {\tM}_{2,d} ( z ; \bq) &:= \frac{1}{m} \Tr_{[p:m-1]} \tbS(z)\, ,
\end{align}
and
\begin{align}
    {M}_{1,d}^{[1]} ( z ; \bq) &:= \frac{1}{m} \Tr_{[1:p-1]} \bS^{[1]}(z)\, ,\\
    {M}_{2,d}^{[1]} ( z ; \bq) &:= \frac{1}{m} \Tr_{[p:m-1]} \bS^{[1]}(z)\, .
\end{align}
By the same argument leading to (4.22) in \cite{lu2022equivalence}, we can get 
\begin{align}
    \big| [\bchi_{\alpha,\beta}(z)]_{1,1} -  {\tM}_{1,d} ( z ; \bq) \charfn_{\alpha=\beta} \big| &\prec \frac{1}{\eta\sqrt{m}} \, , \label{eq:bchi_concentrate_step1_1}\\
    \big| [\bchi_{\alpha,\beta}(z)]_{2,2} -  {\tM}_{2,d} ( z ; \bq) \charfn_{\alpha=\beta} \big| &\prec \frac{1}{\eta\sqrt{m}}\, .
    \label{eq:bchi_concentrate_step1_2}
\end{align}
On the other hand, by the definition of $\bchi_{\alpha,\beta}$, we have for $r\in \N$ and some $C_r>0$ dependent on $r$, 
\begin{align}
    \big\| [\bchi_{\alpha,\beta}(z)]_{1,2} \big\|_{L^r} &\leq \frac{C_r}{m} \Big( \sum_{i=1}^{p-1} \sum_{j=p}^{m-1} |[\tbS(z)]_{i,j}|^2 \Big)^{1/2}  \nonumber\\
    &\leq \frac{C_r }{ \sqrt{m} \eta} \, ,
    \label{eq:bchi_12_Lr}
\end{align}
where the first inequality follows from Lemma 11 and Lemma 13 in \cite{lu2022equivalence} and in the second inequality, we apply:
$$
\sum_{i=1}^{p-1} \sum_{j=p}^{m-1} |[\tbS(z)]_{i,j}|^2 
\leq \sum_{i=1}^{m-1} \sum_{j=1}^{m-1} |[\tbS(z)]_{i,j}|^2 = \frac{\Im(\Tr \tbS(z))}{\eta} \leq \frac{m}{\eta^2}\, ,
$$
and the equality above follows from the Ward identity \cite{erdHos2017dynamical}.
Then we can apply Markov's inequality to \eqref{eq:bchi_12_Lr} to get:
\begin{align}
    \big| [\bchi_{\alpha,\beta}(z)]_{1,2} \big| \prec \frac{1}{\eta \sqrt{m}}\, 
    \label{eq:bchi_concentrate_step1_3}.
\end{align}
Combining Eqs.~\eqref{eq:bchi_concentrate_step1_1}, \eqref{eq:bchi_concentrate_step1_2} and \eqref{eq:bchi_concentrate_step1_3}, we get:
\begin{align}
    \Big\| \bchi_{\alpha,\beta}(z) - 
    \begin{bmatrix}
    {\tM}_{1,d}(z;\bq)\charfn_{\alpha=\beta} & 0\\
    0 & {\tM}_{2,d}(z;\bq)\charfn_{\alpha=\beta}
    \end{bmatrix} \Big\|_\op
     {\prec} \frac{1}{\eta \sqrt{m}}\, .
     \label{eq:bchi_concentrate_step1}
\end{align}
After that, we can follow the same argument leading to (4.26) in \cite{lu2022equivalence} to obtain:
\begin{align}
    \max\big\{ \big| {\tM}_{1,d} ( z ; \bq) - {M}^{[1]}_{1,d} ( z ; \bq) \big|, \big| {\tM}_{2,d} ( z ; \bq) - {M}^{[1]}_{2,d} ( z ; \bq) \big| \big\} \prec \frac{1}{\eta^2 \sqrt{d}}\, .
    \label{eq:bchi_concentrate_step2}
\end{align}
It remains to show ${M}^{[1]}_{i,d} ( z ; \bq) \approx {M}_{i,d} ( z ; \bq)$, for $i=1,2$. 
Note that we can write 
\begin{align}
    [\bS(z)]_{1:p,1:p} &= \big[\bA_{11} - \bA_{12} (\bB_{22} - z \id_{n})^{-1} \bA_{21} - z \id_{p} \big]^{-1} \nonumber \\
    &= \big[ \bOmega(z) - z\id_{p} \big]^{-1}\, ,
\end{align}
and
\begin{align}
    [\bS^{[1]}(z)]_{1:p-1,1:p-1} &= \big[\bB_{11} - \bB_{12} (\bB_{22} - z \id_{n})^{-1} \bB_{21} - z \id_{p-1} \big]^{-1} \nonumber \\
    &= \big[ \bOmega^{[1]}(z) - z\id_{p-1} \big]^{-1}\, .
\end{align}
It can be seen that $\bOmega^{[1]}(z)$ is a sub-matrix of $\bOmega(z)$,
so the eigenvalues of $\bOmega(z)$ and $\bOmega^{[1]}(z)$ are interlacing and we can get \cite[Lemma 7.5]{erdHos2017dynamical}:
\begin{align}
    \big| {M}^{[1]}_{1,d} ( z ; \bq) - {M}_{1,d} ( z ; \bq) \big| \leq \frac{C}{ m \eta}\, ,
    \label{eq:M_not1_approx_1}
\end{align}
where $C$ is some constant. Similarly, we have
\begin{align}
    \big| {M}^{[1]}_{d} ( z ; \bq) - {M}_{d} ( z ; \bq) \big| \leq \frac{C}{ m \eta}\, .
    \label{eq:M_not1_approx}
\end{align}
From Eqs.~\eqref{eq:M_not1_approx_1} and \eqref{eq:M_not1_approx}, we can get:
\begin{align}
    \big| {M}^{[1]}_{2,d} ( z ; \bq) - {M}_{2,d} ( z ; \bq) \big| \leq \frac{C}{m \eta}\, .
    \label{eq:M_not1_approx_2}
\end{align}

After combining Eqs.~\eqref{eq:bchi_concentrate_step1_1}, \eqref{eq:bchi_concentrate_step1_2}, \eqref{eq:bchi_concentrate_step1_3}, \eqref{eq:bchi_concentrate_step2}, \eqref{eq:M_not1_approx_1} and \eqref{eq:M_not1_approx_2}, we obtain the desired result.
\end{proof}
\begin{lemma}
\label{lem:diffdegree_small}
Let $\alpha, \beta, k \in \{0, 1,\cdots,L\}$, satisfying $\alpha \neq \beta$. Suppose $\kappa_1 = \kappa_2$. Then
\begin{align}
    \opnorm{ \bU_{\alpha}^\sT \tbS(z) \bLambda_\beta \tbH_k } = O_{\prec}\left(\frac{\sqrt{m}}{\eta}\right)\, ,
\end{align}
where 
\begin{align}
    \tbH_k = \begin{bmatrix}
    \tbC_k & \tbD_k \\
    \tbD_k^\sT & \tbE_k
    \end{bmatrix}\, ,
\end{align}
and
\begin{align}
    \bLambda_\beta := \diag
    \big\{
    \usp_\beta(\theta_1), \cdots, \usp_\beta(\theta_{p-1}), \usp_\beta(\gamma_1), \cdots, \usp_\beta(\gamma_n)  
    \big\} \, .
\end{align}
\end{lemma}
\begin{proof}
The proof is the same as Lemma 6 in \cite{lu2022equivalence} and is omitted.
\end{proof}

\begin{lemma}
\label{lem:Delta_quadratic_vanish}
Let $\alpha, \beta, \in \{\ell, \ell+1,\cdots,L\}$. Suppose $\kappa_1 = \kappa_2$. Then for any $c=1,\cdots,4$,
\begin{align}
    \frac{1}{m} \opnorm{\bU_{\alpha}^\sT \tbS(z) \bDelta_c \bS^{[1]}(z) \bU_{\beta}} 
    &= O_{\prec}\left(\frac{1}{\eta^2 \sqrt{d}}\right)\, .
    \label{eq:Delta_quadratic_vanish}
\end{align}
\end{lemma}
\begin{proof}
The proof follows the same idea of Lemma 7 in \cite{lu2022equivalence}. For simplicity, we will omit some details here.

First, we analyze $c=1$ and $3$. Since $\opnorm{\bU_\alpha} \prec \sqrt{m}$ (by \cite[Proposition  5]{lu2022equivalence}) and 
\[
\opnorm{\tbS(z)}, \opnorm{\bS^{[1]}(z)} \leq \frac{1}{\eta}\, ,
\]
we have
\begin{align}
    \frac{1}{m} \opnorm{ \bU_{\alpha}^\sT \tbS(z) \bDelta_c \bS^{[1]}(z) \bU_{\beta} } \prec \frac{1}{\eta^2} \opnorm{\bDelta_c}\, .
    \label{eq:quadratic_rawbd_124}
\end{align}
On the other hand, same as Lemma 7 in \cite{lu2022equivalence}, we can show
$
\max_{c = 1,2,4} \opnorm{\bDelta_c} \prec d^{-\frac{1}{2}}\, .
$
After substituting this bound into Eq.~\eqref{eq:quadratic_rawbd_124}, we obtain Eq.~\eqref{eq:Delta_quadratic_vanish} for $c=1,2$ and $4$.

Next we consider $c=2$. We have
\begin{align}
    \frac{1}{m} \opnorm{\bU_{\alpha}^\sT \tbS(z) \bDelta_3 \bS^{[1]}(z) \bU_{\beta}} &\leq \frac{C}{m} \sum_{1\leq t < \ell} \frac{1}{d^{t/2}} \opnorm{ \bU_{\alpha}^\sT \tbS(z) \bLambda_t \tbH_{\ell-t} \bLambda_t \bS^{[1]}(z) \bU_\beta } \nonumber \\
    &\leq \frac{C}{m} \sum_{1\leq t < \ell} \frac{1}{d^{t/2}} \opnorm{ \bU_{\alpha}^\sT \tbS(z) \bLambda_t \tbH_{\ell-t} } \cdot \| \bLambda_t \bS^{[1]}(z) \bU_\beta \| \nonumber \\
    &\prec \frac{1}{\eta^2 \sqrt{d}}\, ,
\end{align}
where in the last step we use Lemma \ref{lem:diffdegree_small}, $\opnorm{\bLambda_t} \prec 1$ , $\opnorm{\bU_\beta} \prec \sqrt{m}$ (by \cite[Proposition  5]{lu2022equivalence}) and the fact that $\opnorm{\bS^{[1]}(z)} \leq \frac{1}{\eta}$.

Finally, we analyze $c=4$. We know $\bDelta_5$ can be written as the linear combination of a finite family of matrices $\{ \bDelta_5^{(k,t)} \}$ defined as:
\begin{align}
    [\bDelta_5^{(k,t)}]_{a,b} &:= [\tbH_{k-t}]_{a,b} \cdot r^{k-t}(\varsigma_a) r^{k-t}(\varsigma_b) \cE_t(\varsigma_a, \varsigma_b) \, ,
\end{align}
where 
$$
\varsigma_a = 
\begin{cases}
\theta_a, &1 \leq a \leq p-1\, ,\\
\gamma_{a-p+1}& p \leq a \leq m-1\, .
\end{cases}
$$
Then following the same steps leading to (4.44) in \cite{lu2022equivalence}, we can get:
\begin{align}
    \frac{1}{m} \opnorm{ \bU_{\alpha}^\sT \tbS(z) \bDelta_5^{(k,t)} \bS^{[1]}(z) \bU_{\beta} } \prec \frac{1}{\eta^2 d}\, ,
\end{align}
which indicates Eq.~\eqref{eq:Delta_quadratic_vanish} by triangle inequality.
\end{proof}




\subsection{Proof of Proposition \ref{prop:derivative_converge}}
\label{sec:log_determinant}
\subsubsection{Auxiliary lemmas}
\begin{lemma}
\label{lem:Gdgrad_prbd}
    Under the same settings as Proposition \ref{prop:derivative_converge}, for any $\eta>0$, there exists $C>0$, such that for any $D>0$ and all large $d$, 
    \begin{align}
         \limsup_{d\to\infty} \sup_{\bq\in\R^5} \| \nabla_\bq G_d(i\eta;\bq) - \nabla_\bq g(i\eta;\bq) \|_2   &< \infty \, ,\\    \limsup_{d\to\infty} \sup_{\bq\in\R^5} \| \nabla_\bq^2 G_d(i\eta;\bq) - \nabla_\bq^2 g(i\eta;\bq) \|_2 &< \infty \, ,\\
        \limsup_{d\to\infty} \sup_{\bq\in\R^5} \| \nabla_\bq^3 G_d(i\eta;\bq) - \nabla_\bq^3 g(i\eta;\bq) \|_2 &< \infty \, .
    \end{align}
\end{lemma}
\begin{proof}
    From Lemma \ref{lem:inprodmtx_opnorm_prbd}, we can easily verify that for any $k\in\Z_{\geq 0}$, there exists $C>0$ such that
    \begin{align}
    \label{eq:QX_prbd}
        \sup_{d\geq 1} \E \opnorm{\bQ_{\ell}^{\bX}}^k &\leq C \, ,
    \end{align}
    and
    \begin{align}
    \label{eq:QW_prbd}
        \sup_{d\geq 1} \E \opnorm{\bQ_{\ell}^{\bW}}^k &\leq C\, .
    \end{align}
    We can also deduce for any $k\in\Z_{\geq 0}$, there exists $C>0$ such that
    \begin{align}
    \label{eq:qWX_prbd}
         \sup_{d\geq 1} \E \opnorm{ \xi_{\ell} \usp_{\ell}(\bW \bX^\sT) }^k \leq C \, ,
    \end{align}
    as we can embed $ m^{ -\frac{1}{2} }  \usp_{\ell}(\bW \bX^\sT)$ inside a $ (n+p) \times (n+p) $ Gegenbauer matrix of degree-$\ell$:
    \[
    \bQ_{\ell}^{\bW,\bX} = 
    \begin{pmatrix}
        \bQ_{\ell}^{\bW} & m^{ -\frac{1}{2} } \usp_{\ell}(\bW \bX^\sT) \\
        m^{ -\frac{1}{2} } \usp_{\ell}(\bX \bW^\sT) & \bQ_{\ell}^{\bX}
    \end{pmatrix},
    \]
    to which we can apply Lemma \ref{lem:inprodmtx_opnorm_prbd}, and use the bound: $\opnorm{ m^{ -\frac{1}{2} } \usp_{\ell}(\bW \bX^\sT)} \leq \opnorm{ \bQ_{\ell}^{\bW,\bX} }$.

    The rest of the proof is completely analogous to that of Lemma 11.3 in \cite{mei2022generalizationRF}. We omit the details here.
\end{proof}

\begin{lemma}
\label{lem:Gdgd_largeK}
    Under the same settings as Proposition \ref{prop:derivative_converge}, it holds that (i) 
    \begin{align}
    \label{eq:gd_diff_detreminbd}
        \lim_{K\to\infty} \big| g_d(i K;\bq) -  \log(-iK) \big| = 0
    \end{align}
    and (ii) 
    \begin{align}
        \lim_{K\to\infty} \sup_{d\geq 1} \E \big| G_d(i K;\bq) - \log(-iK) \big| = 0
    \end{align}
\end{lemma}
\begin{proof}
With the bounds \eqref{eq:QX_prbd}, \eqref{eq:QW_prbd} and \eqref{eq:qWX_prbd} on the moments of $\opnorm{\bA}$, the proof is completely analogous to that of Lemma 11.2 of \cite{mei2022generalizationRF}. The details are omitted.
\end{proof}

\clearpage

\section{Matrix Concentration and Spectral Bound}
\label{sec:matrix_concentration_gather}

In this appendix, we gather the proofs on the concentration and bounds on the operator norm of the different random matrices that appear in the proof of Theorem \ref{thm:main_theorem_RF}.

\subsection{Matrix concentration}
\label{sec:matrix_concentration}

\begin{lemma}
\label{lem:fXWT_fWXT_concentrate}
    Let $\{\varsigma_k\}_{k\geq\ell}^{\infty}$ be a sequence satisfying $\varsigma_k = 0$, for all $k<\ell$ and $\sum_{k\geq\ell}^{\infty} \varsigma_k^2 < \infty$. Define $f(x) := \sum_{k\geq\ell}^{\infty} \varsigma_k \usp_k(x)$. 
    Suppose $|f(x)| \leq C(1+|x|^K)$ for some $C$, $K>0$.    
    If $\kappa_1 > \kappa_2$, it holds that     
\begin{itemize}
    \item[(i)] 
    \begin{align}
    \label{eq:fXWT_concentrate}
        \Big\| \frac{ f(\bX \bW^\sT) f(\bW \bX^\sT) }{p} - \sum_{k\geq\ell}^{\infty} \varsigma_k^2 \bQ_k^{\bX} \Big\|_\op \prec \sqrt{\frac{n}{p}} \, .
    \end{align}
    
    
    \item[(ii)] 
    \begin{align}
    \label{eq:fXWT_Phik_concentrate}
       \Big\| \frac{ f(\bX \bW^\sT) \bPhi_k }{p} - \frac{ \varsigma_k }{ \sqrt{N_k} } \bPsi_k \Big\|_\op \prec \sqrt{\frac{n+N_k}{p}}\, .
    \end{align}  
    \end{itemize}
    
    If $\kappa_1 < \kappa_2$, the results of \eqref{eq:fXWT_concentrate} and \eqref{eq:fXWT_Phik_concentrate} hold with the switch: $p \leftrightarrow n$, $\kappa_1 \leftrightarrow \kappa_2$, $\bW \leftrightarrow \bX$ and $\bPhi \leftrightarrow \bPsi$.
\end{lemma}
\begin{proof}
    We present the proof for $\kappa_1 > \kappa_2$ case. The proof for $\kappa_1 < \kappa_2$ case is the same.

    (i) Define ${\bf}_{i} = \frac{1}{\sqrt{p}}{f}(\bX \bw_i)$ and $\hat{\bf}_{i} = {\bf}_{i} \charfn_{\| \bX \bw_i \|_{\infty} \leq B}$, where $B$ is some constant to be specified. 
    First, note $\bx_j^\sT \bw_i\sim \taumeasure$, for any fixed $\bx_j \in \S^{d-1}(\sqrt{d})$. 
    Define the set:
    \begin{align}
    \label{eq:event_B1}
    \mathcal{B} := \Big\{ \bX \in \R^{n\times d}: \sum_{j=1}^{d} X_{ij}^2 = d, \forall i \in [n] \Big\}\, .
\end{align}
    By the fact that $\taumeasure$ is a sub-Gaussian distribution with constant sub-Gaussian norm [which follows from (D.1) in \cite{lu2022equivalence} and Proposition 2.5.2 (b) in \cite{vershynin2018high}], we get: there exists some $C>0$ such that for any $\bX \in \cB$, 
    \begin{align}
    \label{eq:fXW_truncateprb}
        \P\big( f(\bX \bW^\sT) \neq \hat{f}(\bX \bW^\sT) \mid \bX \big) 
        \leq np \P(|\taumeasure| > B) 
        \leq 2np e^{-CB^2} \, .
    \end{align}
    Meanwhile, for any fixed $\bX\in\cB$, we can bound $\opnorm{ \E_{\bw_i} {\bf}_{i} {\bf}_{i}^\sT - \E_{\bw_i} \hat{\bf}_{i} \hat{\bf}_{i}^\sT }$ as follows ($\E_{\bw_i}$ denotes the expectation over $\bw_i$ conditioning on $\bX$): 
    \begin{align}
        \opnorm{ \E_{\bw_i} {\bf}_{i} {\bf}_{i}^\sT - \E_{\bw_i} \hat{\bf}_{i} \hat{\bf}_{i}^\sT } &=
        \max_{\|\bv\|=1} |\bv^\sT ( \E_{\bw_i} {\bf}_{i} {\bf}_{i}^\sT - \E_{\bw_i} \hat{\bf}_{i} \hat{\bf}_{i}^\sT ) \bv| \\
        &= \max_{\|\bv\|=1} |\bv^\sT \E_{\bw_i} ( {\bf}_{i} {\bf}_{i}^\sT \charfn_{ \|\bX\bw_i\|_\infty>B } ) \bv| \\
        &\leq \sqrt{ \E_{\bw_i} \|\bf_i\|^4 \P(\|\bX\bw_i\|_\infty>B \mid \bX) } \\
        &\leq \frac{C_0 n^{1.5}}{p} e^{-B^2/C_0} \, , \label{eq:Eff_diff}
    \end{align}
for some $C_0>0$, where in the last step, we use (recall that $|f(x)| \leq C(1+|x|^K)$ for some $C,K>0$)
\begin{align}
    \E_{\bw_i} \|\bf_i\|^4 &\leq \frac{n^2}{p^2} \E_{\bw_i}[ f(\bx_1^\sT \bw_i)^4 ]
    \leq \frac{C_1 n^2}{p^2} \E_{\bw_i}(1+|\bx_1^\sT \bw_i|^K)^4 \leq \frac{C_2 n^2}{p^2} \, ,
\end{align}
and $\P(\|\bX\bw_i\|_\infty>B \mid \bX) \leq 2n e^{-C_3 B^2}$ for some $C_1,C_2,C_3 >0$. 

Next, we control $ \opnorm{\sum_{i=1}^{p} \hat{\bf}_{i} \hat{\bf}_{i}^\sT - p \E_{\bw_i} (\hat{\bf}_{i} \hat{\bf}_{i}^\sT) } $. Define ${\bDelta}_i := \hat{\bf}_{i} \hat{\bf}_{i}^\sT - \E_{\bw_i} (\hat{\bf}_{i} \hat{\bf}_{i}^\sT)$ and set
\begin{align}
\label{eq:event_A1}
    \mathcal{A} := \Big\{ \bX \in \R^{n\times d}: \sum_{j=1}^{d} X_{ij}^2 = d, \forall i \in [n] ~\text{ and }~ \opnorm{ \sum_{k\geq\ell}^{\infty} {\varsigma}_k^2 \bQ_k^\bX } \leq M \Big\}\, ,
\end{align}
where $M>0$ is a constant to be specified.
We have
\begin{align}
    \Big\| \sum_{i=1}^{p} \E_{\bw_i}{\bDelta}_i^2 \Big\|_\op 
    =&~ p \big\| \E_{\bw_i} ( \|\hat{\bf}_{i}\|^2 \hat{\bf}_{i} \hat{\bf}_{i}^\sT ) - ( \E_{\bw_i} \hat{\bf}_{i} \hat{\bf}_{i}^\sT)^2 \big\|_\op \\
    \leq&~ p \Big[ \frac{n}{p} (1+B^K)^2 \opnorm{ \E_{\bw_i} ( \hat{\bf}_{i} \hat{\bf}_{i}^\sT ) } + \opnorm{ \E_{\bw_i} ( \hat{\bf}_{i} \hat{\bf}_{i}^\sT ) }^2 \Big] \, ,
\end{align}
and $\opnorm{ \E_{\bw_i} ( \hat{\bf}_{i} \hat{\bf}_{i}^\sT ) }$ can be bounded as:
\begin{align}
    \opnorm{ \E_{\bw_i} (\hat{\bf}_{i} \hat{\bf}_{i}^\sT) } \leq&~ \opnorm{ \E_{\bw_i} ({\bf}_{i} {\bf}_{i}^\sT) } + \opnorm{ \E_{\bw_i} ({\bf}_{i} {\bf}_{i}^\sT) - \E_{\bw_i} (\hat{\bf}_{i} \hat{\bf}_{i}^\sT) } \\
    \leq&~ \opnorm{ \frac{1}{p}\sum_{k\geq\ell}^{\infty} {\varsigma}_k^2 \bQ_k^\bX } + \frac{C n^{1.5}}{p} e^{-B^2/C}\, ,
\end{align}
where we use $\E_{\bw_i} ({\bf}_{i} {\bf}_{i}^\sT) = \frac{1}{p}\sum_{k\geq\ell}^{\infty} {\varsigma}_k^2 \bQ_k^\bX$ and Eq. \eqref{eq:Eff_diff}. Therefore, there exists $C>0$ such that for any $\bX\in \cA$, 
\begin{align}
    \Big\| \sum_{i=1}^{p} \E_{\bw_i}{\bDelta}_i^2 \Big\|_\op \leq \frac{Cn}{p}(1+B^K)^2 (M+Cn^{1.5}e^{-B^2/C})^2 \, .
\end{align}
On the other hand, there exists $C>0$ such that
\begin{align}
    \sup_{1 \leq i \leq p}\opnorm{{\bDelta}_i} &\leq 2 \|\hat{\bf}_{i}\|^2 
    \leq \frac{C n}{p}(1+B^K)^2 \, .
\end{align}
Then by matrix Bernstein's inequality \cite{tropp2012user}, there exists $c>0$ such that for any $t\geq 0$ and $\bX\in\mathcal{A}$
\begin{align}
    \P\big(\opnorm{ \sum_{i=1}^{p} \hat{\bf}_{i} \hat{\bf}_{i}^\sT - p \E_{\bw_i} \hat{\bf}_{i} \hat{\bf}_{i}^\sT } \geq t \mid \bX \big) 
    &= \P\Big(\opnorm{\sum_{i=1}^{p}{\bDelta}_i} \geq t \mid \bX \Big) \nonumber \\
    &\leq 2p \exp \Big( - \frac{cp}{ a n} \min\{t,t^2\} \Big)\, 
    \label{eq:ZhatZhatT_concentrate1}
\end{align}
where $a := (1+B^K)^2 (M+Cn^{1.5}e^{-B^2/C})^2$. 

Now we are ready to prove \eqref{eq:fXWT_concentrate}. For any $t>0$, we have
\begin{align}
    &\hspace{-3em}~\P\Big(\Big\| \frac{ f(\bX \bW^\sT) f(\bW \bX^\sT) }{p} - \sum_{k\geq\ell}^{\infty} \varsigma_k^2 \bQ_k^{\bX} \Big\|_\op \geq t \Big) \\
    =&~ \E_\bX \Big[ \P\big( \opnorm{ \sum_{i=1}^{p}\bf_i\bf_i^\sT -p\E_{\bw_i}(\bf_i\bf_i^\sT) } \geq t \mid \bX \big) \big( \charfn_{\bX\in\cA} + \charfn_{\bX\in\cA^C} \big) \Big] \\
    \leq&~ \E_\bX \Big[ \P\big( \opnorm{ \sum_{i=1}^{p}\bf_i\bf_i^\sT -p\E_{\bw_i}(\bf_i\bf_i^\sT) } \geq t \mid \bX \big)  \charfn_{\bX\in\cA} \Big] + \P(\cA^C)
    \label{eq:concen_fXW_00}
\end{align}
On the other hand, it holds that: for any $t>0$, 
\begin{align}
    &\hspace{-3em}~\P\Big( \opnorm{ \sum_{i=1}^{p}\bf_i\bf_i^\sT -p\E_{\bw_i}(\bf_i\bf_i^\sT) } \geq t \mid \bX \Big) \\
    \leq&~ \P\Big( \opnorm{ \sum_{i=1}^{p}\bf_i\bf_i^\sT -p\E_{\bw_i}\hat\bf_i\hat\bf_i^\sT } + p\opnorm{ \E_{\bw_i}\bf_i\bf_i^\sT -\E_{\bw_i}\hat\bf_i\hat\bf_i^\sT } \geq t \mid \bX \Big) \\
    \leq&~ \P\Big( \opnorm{ \sum_{i=1}^{p}\hat\bf_i\hat\bf_i^\sT - p\E_{\bw_i}\hat\bf_i\hat\bf_i^\sT } + p\opnorm{ \E_{\bw_i}\bf_i\bf_i^\sT -\E_{\bw_i}\hat\bf_i\hat\bf_i^\sT } \geq t \mid \bX \Big)  \\
    &~~~+ \P\big( f(\bX \bW^\sT) \neq \hat{f}(\bX \bW^\sT) \mid \bX \big)\\
    \leq&~ \P\Big( \opnorm{ \sum_{i=1}^{p}\hat\bf_i\hat\bf_i^\sT - p\E_{\bw_i}\hat\bf_i\hat\bf_i^\sT } \geq \frac{t}{2} \mid \bX \Big)  + 
    \P\Big( p\opnorm{ \E_{\bw_i}\hat\bf_i\hat\bf_i^\sT - \E_{\bw_i}\bf_i\bf_i^\sT } \geq \frac{t}{2} \mid \bX \Big) \\
     &~~~+ \P\big( f(\bX \bW^\sT) \neq \hat{f}(\bX \bW^\sT) \mid \bX \big).
     \label{eq:concen_fXW_0}
\end{align}
For any $\delta>0$, let $t=\sqrt{ \frac{n}{p} } d^\delta $ and we choose $B=d^{\epsilon_1}$ and $M=d^{\epsilon_2}$ for some small enough constants $\epsilon_1,\epsilon_2>0$. Then substituting Eqs.~\eqref{eq:fXW_truncateprb}, \eqref{eq:Eff_diff} and \eqref{eq:ZhatZhatT_concentrate1} into Eq.~\eqref{eq:concen_fXW_0}, we can get: for any $\bX\in\cA\subseteq\cB$, $D>0$ and all large $d$,
\begin{align}
    \P\Big( \opnorm{ \sum_{i=1}^{p}\bf_i\bf_i^\sT -p\E_{\bw_i}(\bf_i\bf_i^\sT) } \geq \sqrt{ \frac{n}{p} } d^\delta \mid \bX \Big) \leq \frac{d^{-D}}{2}
\end{align}
On the other hand, by Eq.~\eqref{eq:sum_F_QkX} in Lemma \ref{lem:kernelmtx_opnormbd}, $\P\big( \cA^{C} \big) \leq \frac{d^{-D}}{2}$ for any $D$ and all large $d$.
Substituting the previous two bounds into Eq.~ \eqref{eq:concen_fXW_00} leads to the desired result.

(ii) The proof is analogous to part (i) and is omitted.
\end{proof}

\begin{lemma}
    \label{lem:offdiagonal_kernelmtx_opnormbd}
    Let $\{\varsigma_k\}_{k\geq 0}^{\infty}$ be a non-negative sequence satisfying $\varsigma_k=0$ for all $k<\ell$ and $\sum_{k=0}^{\infty}\varsigma_k < \infty$. It holds that (i) for any integer $\kappa_0 \geq \kappa_1$,
    \begin{align}
        \opnorm{\sum_{k \geq \kappa_0} \varsigma_k \bQ_k^\bW - \sum_{k \geq \kappa_0} \varsigma_k \cdot \id_p} &= O_{\prec}\big( { d^{ \frac{\kappa_1 - \kappa_0 }{2} } } \big)\, . \label{eq:sum_F_QkW_offdiag}
    \end{align}
    and (ii) for any integer $\kappa_0 \geq \kappa_2$,
    \begin{align}
        \opnorm{ \sum_{k \geq \kappa_0} \varsigma_k \bQ_k^\bX - \sum_{k \geq \kappa_0} \varsigma_k \cdot \id_n } &= O_{\prec}\big( { d^{ \frac{\kappa_2 - \kappa_0 }{2} } } \big) \, .\label{eq:sum_mu_QkX_offdiag}
    \end{align}
\end{lemma}
\begin{proof}
    We show that the proofs of Eq.~\eqref{eq:sum_F_QkW_offdiag} and the proof of Eq.~\eqref{eq:sum_mu_QkX_offdiag} are the same. 
    
    By Proposition 8 in \cite{lu2022equivalence} (after some rescaling) and the fact that the diagonal elements of $\bQ_k^\bW$ are all 1, we can get for any $k\geq\kappa_1$, 
    \begin{align}
        \opnorm{\bQ_k^\bW - \id_p } = O_{\prec}\big( { d^{ \frac{\kappa_1 - k }{2} } } \big) \, ,
    \end{align}
    so for any finite integer $a$ and $b$ with $b \geq a \geq \kappa_1$, we have
    \begin{align}    
    \label{eq:ho_QW_concentrate_finite_1}
        \left\| \sum_{k \geq a}^{b} \varsigma_k \bQ_k^\bW - \sum_{k \geq a}^{b} \varsigma_k \cdot \id_p \right\|_\op = O_{\prec}\big( { d^{ \frac{\kappa_1 - a }{2} } } \big) \, .
    \end{align}
    On the other hand, same as \eqref{eq:ho_QX_concentrate_1}, we can get for any $D>0$, there exists $k_0 \geq \kappa_1$ such that for all large $d$,
    \begin{align}
        \left\| \sum_{k > k_0}^{\infty} \varsigma_k \bQ_k^\bW - \sum_{k > k_0}^{\infty} \varsigma_k \id_p\right\|_\op  \leq O_{\prec}( d^{-D} ) \, .
        \label{eq:ho_QW_concentrate_1}
    \end{align}
    Combining  Eqs.~\eqref{eq:ho_QW_concentrate_finite_1} and \eqref{eq:ho_QW_concentrate_1}, we get Eq.~\eqref{eq:sum_F_QkW_offdiag}.
\end{proof}


\begin{lemma}\label{lem:lowdeg_concentration}
    For any $k \geq 1$, there exist $c>0$ such that for any $t \geq 0$,
    \begin{align}
    \label{eq:QX_less_ell_concentrate}
        \P\Big( \opnorm{ \frac{1}{p} \bPhi_{<k}^\sT \bPhi_{<k} - \id_{ N_{<k} } } \geq t \Big) \leq 2 N_{<k} \exp \Big( - \frac{c p}{ N_{<k} } \min\{t,t^2\} \Big)\, .
    \end{align}
    Similarly, for any $k \geq 1$, there exist $c>0$ such that for any $t \geq 0$,
    \begin{align}
    \label{eq:QW_less_ell_concentrate}
        \P\Big( \opnorm{ \frac{1}{n} \bPsi_{<k}^\sT \bPsi_{<k} - \id_{ N_{<k} } } \geq t \Big) \leq 2 N_{<k} \exp \Big( - \frac{c n}{ N_{<k} } \min\{t,t^2\} \Big)\, .
    \end{align}
\end{lemma}
\begin{proof}
    The proof directly follows from Matrix Bernstein's inequality \cite{tropp2012user}.
\end{proof}


\subsection{Spectral bound of random feature matrix}

\begin{lemma}
\label{lem:featuremtx_opnormbd}
Let $\{\varsigma_k\}_{k\geq 0}^{\infty}$ be a sequence satisfying $\varsigma_k = 0$ for all $k<\ell$ and $\sum_{k\geq \ell}^{\infty} \varsigma_k^2 < \infty$ and define $f(x) := \sum_{k\geq \ell}^{\infty} \varsigma_k \usp_k(x)$. 
Suppose $|f(x)| \leq C(1+|x|^K)$ for some $C$, $K>0$.
It holds that
\begin{align}  
\label{eq:sigma_XW_opnormbd}
\opnorm{ f(\bX \bW^\sT) } = O_{\prec}\big(\max\{\sqrt{p},\sqrt{n}\}\big)\, . 
\end{align}
\end{lemma}
\begin{proof}
(i) $\kappa_1 > \kappa_2$. From Eq.~\eqref{eq:fXWT_concentrate}, for any $\varepsilon>0$, $D>0$ and all large $d$, 
\begin{align}
    \P\Big( \Big\|\frac{f(\bX \bW^\sT) f(\bW \bX^\sT)}{p} - \sum_{k\geq\ell}^{\infty} \varsigma_k^2 \bQ_k^\bX \Big\|_\op \geq \varepsilon \Big) \leq d^{-D}\, .
    \label{eq:feature_concentrate_1}
\end{align}
By Eq.~(55) in \cite{ghorbani2021linearized}, we know for any $D>0$, there exists $k_0>\ell$ such that for all large $d$,
\begin{align}
\label{eq:Esup_Q_bd}
    \E \sup_{k > k_0} \opnorm{\bQ_k^{\bX} - \id_n}^2 \leq d^{-D}\, .
\end{align}
In addition,
\begin{align}
    \left\| \sum_{k > k_0}^{\infty} \varsigma_k^2 \bQ_k^\bX - \sum_{k > k_0}^{\infty} \varsigma_k^2 \id_n \right\|_\op &\leq \sum_{k > k_0}^{\infty} \varsigma_k^2 \opnorm{\bQ_k^\bX - \id_n}\, .
\end{align}
Therefore, by Markov's inequality we get for any $\varepsilon>0$ and $D>0$, there exists $k_0>0$ such that for all large $d$,
\begin{align}
    \P\Big( \Big\| \sum_{k > k_0}^{\infty} \varsigma_k^2 \bQ_k^\bX - \sum_{k > k_0}^{\infty} \varsigma_k^2 \id_n \Big\|_\op \geq \varepsilon \Big) \leq d^{-D} \, .
    \label{eq:ho_QX_concentrate_1}
\end{align}
On the other hand, 
by Lemma \ref{lem:inprodmtx_opnorm_prbd}, there exists $C>0$ such that for any $D>0$ and all large $d$,
\begin{align}
    \P\Big(\Big\| \sum_{k\geq\ell}^{k_0} \varsigma_k^2 \bQ_k^\bX \Big\|_\op \geq C \Big) \leq d^{-D} \, .
    \label{eq:ho_QX_bdopnorm_1}
\end{align}
Combining Eqs.~\eqref{eq:feature_concentrate_1}, \eqref{eq:ho_QX_concentrate_1} and \eqref{eq:ho_QX_bdopnorm_1}, we conclude that there exists $C>0$, such that for any $D>0$ and all large $d$,
\begin{align}
  \label{eq:sigma_XW_opnormbd_01}
    \P\big( \opnorm{ f(\bX \bW^\sT) / \sqrt{p} } \geq C \big) \leq d^{-D}.   
\end{align}
This implies
\begin{align}  
\label{eq:sigma_XW_opnormbd_1}
\opnorm{ f(\bX \bW^\sT) } = O_{\prec}\big(\sqrt{p}\big)\, . 
\end{align}

(ii) $\kappa_1 = \kappa_2$.
Following the similar steps leading to \eqref{eq:fXWT_concentrate}, we can get for any $\varepsilon>0$, $D>0$ and all large $d$,
\begin{align}
    \P\Big( \Big\|\frac{f(\bX \bW^\sT) f(\bW \bX^\sT)}{p} - \sum_{k\geq\ell}^{\infty} \varsigma_k^2 \bQ_k^\bX \Big\|_\op \geq d^\varepsilon \Big) \leq d^{-D}\, .
    \label{eq:feature_concentrate_2}
\end{align}
Then combining Eqs.~\eqref{eq:feature_concentrate_2}, \eqref{eq:ho_QX_concentrate_1} and \eqref{eq:ho_QX_bdopnorm_1}, we again reach at Eq.~\eqref{eq:sigma_XW_opnormbd_1}.

(iii) $\kappa_1 < \kappa_2$. Following the same proof as (i), we can get
\begin{align}  
\label{eq:sigma_XW_opnormbd_2}
\opnorm{ f(\bX \bW^\sT) } = O_{\prec}\big(\sqrt{n}\big)\, . 
\end{align}

Finally, Eq.~\eqref{eq:sigma_XW_opnormbd} directly follows from Eqs.~\eqref{eq:sigma_XW_opnormbd_1} and \eqref{eq:sigma_XW_opnormbd_2}.
\end{proof}

\begin{lemma}
\label{lem:smallest_eigval_RFM}
    Suppose $\sum_{k>\ell} \mu_k^2 > 0$. When $\kappa_1 > \kappa_2$, there exists $C>0$ such that for any $D>0$ and all large $d$, 
    \begin{align}
    \label{eq:overparameterized_smallest_eig}
        \P\Big( \lambda_{\min} \Big( \frac{1}{p} \sigma(\bX \bW^\sT) \sigma(\bW \bX^\sT) \Big) < C \Big) \leq d^{-D}\, .
    \end{align}
    Similarly, when $\kappa_2 > \kappa_1$, there exists $C>0$ such that for any $D>0$ and all large $d$, 
    \begin{align}
    \label{eq:underparameterized_smallest_eig}
        \P\Big( \lambda_{\min} \Big( \frac{1}{n} \sigma(\bW \bX^\sT) \sigma(\bX \bW^\sT) \Big) < C \Big) \leq d^{-D}\, .
    \end{align}
\end{lemma}
\begin{proof}
    We present the proof of Eq.~\eqref{eq:overparameterized_smallest_eig} and the proof of Eq.~\eqref{eq:underparameterized_smallest_eig} is the same.

    Recall that $\bZ = \frac{1}{\sqrt{p} } \sigma (\bX \bW^\sT)$. Then we make the following decomposition:
    \begin{align}
        \bZ \bZ^\sT &= ( \bL + \bT ) ( \bL + \bT )^\sT + ( \bZ_{\geq\ell} \bZ_{\geq\ell}^\sT - \bT \bT^\sT ) \nonumber \\
        &  \succeq \bZ_{\geq\ell} \bZ_{\geq\ell}^\sT - \bT \bT^\sT \, ,
    \end{align}
    where
    \begin{align*}
        \bL &= \frac{\bPsi_{<\ell} }{\sqrt{n} } \cdot \sqrt{n} \bD_{<\ell} \cdot \Big(\frac{1}{p} \bPhi_{<\ell}^\sT \bPhi_{<\ell} \Big)^{\frac{1}{2} },\\
        \bT &= \bZ_{\geq \ell} \cdot \frac{\bPhi_{<\ell} }{\sqrt{p} } \cdot \Big(\frac{1}{p} \bPhi_{<\ell}^\sT \bPhi_{<\ell} \Big)^{ -\frac{1}{2} }.
    \end{align*}
    From Eq.~\eqref{eq:QX_less_ell_concentrate} in Lemma \ref{lem:lowdeg_concentration}, we can get for any $\varepsilon>0$, $D>0$ and all large $d$, 
    \begin{align}
    \label{eq:smallest_eig_bd_1}
        \P\Big( \lambda_{\min}(\tfrac{1}{p} \bPhi_{<\ell}^\sT \bPhi_{<\ell}) < 1 - \varepsilon \Big) \leq d^{-D}.
    \end{align}
    On the other hand, by Eq.~\eqref{eq:fXWT_Phik_concentrate} in Lemma \ref{lem:fXWT_fWXT_concentrate}, we have for any $\varepsilon>0$, $D>0$ and all large $d$, 
    \begin{align}
    \label{eq:crossterm_oper_bd_1}
        \P \Big( \opnorm{ \bZ_{\geq \ell} \cdot \tfrac{\bPhi_{<\ell} }{\sqrt{p} } } \geq \varepsilon \Big) \leq d^{-D}.
    \end{align}
    Combining Eqs.~\eqref{eq:smallest_eig_bd_1} and \eqref{eq:crossterm_oper_bd_1}, we get for any $\varepsilon>0$, $D>0$ and all large $d$,
    \begin{align}
    \label{eq:T_opernorm_bd_1}
        \P \Big( \opnorm{ \bT } \geq \varepsilon \Big) \leq d^{-D}.
    \end{align}
    On the other hand, by Eqs.~\eqref{eq:fXWT_concentrate} and \eqref{eq:sum_mu_QkX_offdiag}, and the condition that $\sum_{k>\ell} \mu_k^2 >0$, we can obtain that there exists $C>0$ such that for any $D>0$ and all large $d$,
    \begin{align}
    \label{eq:smallesteig_Zgeql_lbd}
        \P\big( \lambda_{\min}( \bZ_{\geq\ell} \bZ_{\geq\ell}^\sT) < 2C \big) \leq d^{-D}.
    \end{align}
    Finally, combining Eqs.~\eqref{eq:T_opernorm_bd_1} and \eqref{eq:smallesteig_Zgeql_lbd}, we reach at Eq.~\eqref{eq:overparameterized_smallest_eig}.
\end{proof}


\subsection{Spectral bound of kernel matrix}
\label{sec:operator_norm}
\begin{lemma}
\label{lem:kernelmtx_opnormbd}
Let $\{\varsigma_k\}_{k\geq\ell}^{\infty}$ be a non-negative sequence satisfying $\sum_{k=\ell}^{\infty}\varsigma_k < \infty$. It holds that
\begin{align}
\Big\| \sum_{k\geq\ell} \varsigma_k \bQ_k^\bX \Big\|_\op &= O_{\prec}\Big( \frac{\max\{n, N_{\ell}\}}{N_{\ell}} \Big)\, , \label{eq:sum_F_QkX}\\
\Big\| \sum_{k\geq\ell} \varsigma_k \bQ_k^\bW \Big\|_\op &= O_{\prec}\Big( \frac{\max\{p, N_{\ell}\}}{N_{\ell}} \Big)\, . \label{eq:sum_mu_QkW}
\end{align}
\end{lemma}
\begin{proof}
(i) For any $k < \kappa_2$, it holds that \cite[Lemma 14]{lu2022equivalence}
\begin{align}
    \opnorm{\bQ_k^\bX - {n/N_k}} = O_{\prec}(1) \, .
    \label{eq:sum_F_QkX_1}
\end{align}
On the other hand, following the same argument in Eqs.~\eqref{eq:ho_QX_concentrate_1} and \eqref{eq:ho_QX_bdopnorm_1}, we can get
\begin{align}
    \Big\| \sum_{k \geq \kappa_2} \varsigma_k \bQ_k^\bX \Big\|_\op = O_{\prec}(1).
    \label{eq:sum_F_QkX_2}
\end{align}
After combining Eqs.~\eqref{eq:sum_F_QkX_1} and \eqref{eq:sum_F_QkX_2}, we obtain  Eq.~\eqref{eq:sum_F_QkX}.

(ii) The proof of Eq.~\eqref{eq:sum_mu_QkW} is same as Eq.~\eqref{eq:sum_F_QkX}. We omit it for brevity.
\end{proof}

\begin{lemma}
\label{lem:inprodmtx_supopnorm_bd}
For any fixed $r\geq 0$, it holds that
\begin{align}
    \sup_{k\geq \kappa_1} \{\E \opnorm{ \bQ_{k}^{\bW} }^r \} < \infty ~\text{  and  }~ \sup_{k\geq \kappa_2} \{ \E \opnorm{ \bQ_{k}^{\bX} }^r \} < \infty\, .
\end{align}
\end{lemma}
\begin{proof}
For any $r \geq 0$, 
\begin{align}
    \E \opnorm{ \bQ_{k}^{\bX} }^r &\leq \E (1 + \opnorm{\bQ_{k}^{\bX} - \id_p})^r \leq C_r (1 + \E \opnorm{\bQ_{k}^{\bX} - \id_n}^r )\, ,
\end{align}
where $C_r$ is a constant that only depends on $r$. 
It holds that for any $s\in\N$ \cite[Eq.~(71)]{ghorbani2021linearized},
\begin{align}
    \E[ \Tr(\bQ_{k}^{\bX} - \id_n)^{2s}] \leq (Cs)^{3s} \frac{n^{s+1}}{d^{ks}} + C^s \Big( \frac{n}{d^k} \Big)^2 \, .
\end{align}
Therefore, for any $s>0$ and $k \geq \frac{\kappa_2(s+1)+1}{s}$, we have
\begin{align}
    \E \opnorm{\bQ_{k}^{\bX} - \id_n}^{2s} \leq \E[ \Tr(\bQ_{k}^{\bX} - \id_n)^{2s}] \leq \frac{(Cs)^{3s}}{d} \, ,
\end{align}
which indicates that for any fixed $r$, there exists $k_0 > \kappa_2$ such that 
\begin{align}
    \sup_{k\geq k_0} \E \opnorm{\bQ_{k}^{\bX} - \id_n}^r \leq \frac{C_r}{d} \, . 
\end{align}
On the other hand, by Lemma \ref{lem:inprodmtx_opnorm_prbd}, we can get
\begin{align}
    \sup_{\kappa_2 \leq k < k_0} \E \opnorm{\bQ_{k}^{\bX} - \id_n}^r \leq C_r\, . 
\end{align}
Therefore, $\sup_{k \geq \kappa_2} \E \opnorm{\bQ_{k}^{\bX} - \id_n}^r \leq C_r$ and thus $\sup_{k\geq \kappa_2} \E \opnorm{ \bQ_{k}^{\bX} }^r < \infty$. 

The proof for 
$\sup_{k\geq \kappa_1} \E \opnorm{ \bQ_{k}^{\bW} }^r < \infty$ 
is the same and is omitted.    
\end{proof}

\begin{lemma}
\label{lem:inprodmtx_opnorm_prbd}
 Suppose $\limsup_{d\to\infty} \frac{ \max \{p, n\} }{d^k} < \infty$. There exists $C>0$ such that
for any $t>0$, $D>0$ and all large $d$, 
\begin{align}
\label{eq:QkW_opnorm_concentrate}
    \P(\opnorm{\bQ_{k}^\bW} \geq C + t) \leq (d t)^{-D}\, ,
\end{align}
and  
\begin{align}
\label{eq:QkX_opnorm_concentrate}
    \P(\opnorm{\bQ_{k}^\bX} \geq C + t) \leq (d t)^{-D}.
\end{align}
Also for any $r>0$ there exists $K>0$ such that for all $d$,
\begin{align}
\label{eq:QkW_QkX_criticalpart_moment_bd}
    \max \{ \E \opnorm{ \bQ_k^{\bW} }^r,~ \E \opnorm{ \bQ_k^{\bX} }^r \} \leq K \, .
\end{align}
\end{lemma}
\subsection{Proof of Lemma \ref{lem:inprodmtx_opnorm_prbd}} 

We will prove Eq.~\eqref{eq:QkX_opnorm_concentrate}. The proof of Eq.~\eqref{eq:QkW_opnorm_concentrate} is the same as Eq.~\eqref{eq:QkX_opnorm_concentrate}. The moment bound Eq.~\eqref{eq:QkW_QkX_criticalpart_moment_bd} immediately follows from Eqs.~\eqref{eq:QkW_opnorm_concentrate} and \eqref{eq:QkX_opnorm_concentrate}, due to the identity $\E|X| = \int_{0}^{\infty} \P(|X| \geq t) dt $.

Recall that $\bQ_{k}^\bX = \frac{1}{N_k} \bPsi_k \bPsi_k^\sT $. 
The value of $\bQ_{k}^\bX$ is irrelevant to the choice of $\bpsi_k(\bx)$, but to facilitate the analysis of $\bA_k$, we will work with the following choice of $\bpsi_k(\bx)$:
\begin{align}
\label{eq:psik_decomp}
    \bpsi_k(\bx)^\sT = [\cbpsi_{k}(\bx)^\sT, \hbpsi_{k}(\bx)^\sT]\, , 
\end{align}
where $\cbpsi_{k}(\bx)\in \R^{d^{k}/k!}$ and $\hbpsi_{k}(\bx)\in \R^{N_k - d^{k}/k!}$. In particular, the entries of $\cbpsi_{k}(\bx)$ are of the form: $\mathcal{E}_k^{-1} {x_{i_1} x_{i_2} \cdots x_{i_k}}$, where 
$$
\mathcal{E}_k := \sqrt{\E x_{i_1}^2 x_{i_2}^2 \cdots x_{i_k}^2}\, , 
$$ 
with $\{i_1, i_2,\cdots,i_k\}$ being one of $k$-combinations of $[d]$, while $\hbpsi_{k}(\bx)$ can be composed of any degree-$k$ spherical harmonics such that the entries of $\bpsi_k(\bx)$ are orthonormal.
In particular, the orthogonality of $\cbpsi_{k}(\bx)$ can be verified as follows. 
Suppose $\{i_{1},i_{2},\cdots,i_{k}\}$ and $\{i_{1}',i_{2}',\cdots,i_{k}'\}$ are two different $k$-combinations. Then there exists an $i^* \in \{i_{1},i_{2},\cdots,i_{k}\}$, but $i^* \not\in \{i_{1}',i_{2}',\cdots,i_{k}'\}$. In other words, $i^*$ appears exactly once, so by Lemma \ref{lem:uniformdist_moments}, we can get
\begin{align}
    \E ( x_{i_1} \cdots x_{i_k} x_{i_1'} \cdots x_{i_k'} ) &= 0\, .
\end{align}
With $\bpsi_k(\bx)$ in Eq.~\eqref{eq:psik_decomp}, we further define:
\begin{align}
    \cbPsi_{k} &:= [ \cbpsi_{k} (\bx_1) , \cdots , \cbpsi_{k} (\bx_n) ]^\sT\, , \\
    \hbPsi_{k} &:= [ \hbpsi_{k} (\bx_1) , \cdots , \hbpsi_{k} (\bx_n) ]^\sT\, ,
\end{align}
and $\bQ_{k}^\bX$ can be decomposed as:
\begin{align}
    \bQ_{k}^\bX &= \frac{1}{N_k} (\cbPsi_{k} \cbPsi_{k}^\sT + \hbPsi_{k} \hbPsi_{k}^\sT)\, .
    \label{eq:Ak_decompose}
\end{align}
Let $\tbpsi_{k}(\bx)$ be a $\frac{d^k}{k!}$-dimensional vector whose entries are of the form $x_{i_1} x_{i_2} \cdots x_{i_k}$, with $\{i_1, i_2,\cdots,i_k\}$ being all $k$-combinations of $[d]$ and $\tbPsi_{k}$ is defined accordingly as $\cbPsi_{k}$ and $\hbPsi_{k}$.
Notice that we have the decomposition
\begin{align}
    \bQ_{k}^\bX 
    &= \frac{d^k/k!}{N_k \cE_k^2} (\bB + \bD) + \bC \, ,
    \label{eq:Ak_decompose_1}
\end{align}
where $\bC = \frac{1}{N_k}\hbPsi_{k} \hbPsi_{k}^\sT $ and $\bB$ and $\bD$ are the off-diagonal part and the diagonal part of $\frac{\tbPsi_{k} \tbPsi_{k}^\sT}{d^k/k!}$, respectively. It is not hard to show $\opnorm{\bD} \leq 1$. Indeed, by the property of spherical harmonics (c.f. Eq.~\eqref{eq:diagonal_Y_constant} in Appendix \ref{sec:Spherical-Harmonics}), we have $D_i \leq \Tilde{D}_i = 1$, $i=1,2,\cdots,d$, where $\Tilde{D}_i$ is the $i$th diagonal element of $\bQ_k^{\bX}$. Therefore, $\opnorm{\bD} \leq 1$ and we just need to show:

\noindent(I) there exists $C>0$ such that
for any $t>0$, $D>0$ and all large $d$, 
\begin{align}
\label{eq:QkX_opnorm_concentrate_2}
    \P(\opnorm{\bB} \geq C + t) \leq (dt)^{-D},
\end{align}
(II) there exists $C>0$ such that for any $t>0$, $D>0$ and all large $d$,
\begin{align}
\label{eq:bC_prob_bd}
    \P( \opnorm{\bC} \geq C + t) \leq (dt)^{-D},
\end{align}
Indeed, after combining Eqs.~\eqref{eq:QkX_opnorm_concentrate_2} and \eqref{eq:bC_prob_bd} with $\opnorm{\bD} \leq 1$, we get Eq.~\eqref{eq:QkX_opnorm_concentrate}.

(I) We use the moment method to prove Eq.~\eqref{eq:QkX_opnorm_concentrate_2}.
Our approach is primarily based on the proof of Proposition 5.2 in \cite{fan2019spectral}. 

For $h\in \N$, we have:
\begin{align}
    \E (\Tr \bB^h) &= \sum_{\substack{i_1,\cdots,i_h=1 \\ i_1\neq i_2,\cdots,i_h\neq i_1}}^{n} \E \big(\prod_{s=1}^{h} B_{i_s i_{s+1}}\big) \nonumber\\
    &= \frac{1}{(d^k/k!)^h} \sum_{\substack{i_1,\cdots,i_h=1 \\ i_1\neq i_2,\cdots,i_h\neq i_1}}^{n} \E \Big( \frac{1}{(k!)^h}\prod_{s=1}^{h}  \sum_{\substack{j_{s}^{1},\cdots,j_{s}^{k} = 1 \\ j_{s}^{1}\neq j_{s}^{2}\neq \cdots \neq j_{s}^{k}}}^{d} \prod_{a=1}^{k} x_{i_s j_s^a} x_{i_{s+1} j_s^a} \Big) \nonumber \\
    &= \frac{1}{d^{kh}} \sum_{\substack{i_1,\cdots,i_h=1 \\ i_1\neq i_2,\cdots,i_h\neq i_1}}^{n} \sum_{\substack{j_{1}^{1},\cdots,j_{1}^{k} = 1 \\ j_{1}^{1}\neq j_{1}^{2}\neq \cdots \neq j_{1}^{k}}}^{d} \cdots
    \sum_{\substack{j_{h}^{1},\cdots,j_{h}^{k} = 1 \\ j_{h}^{1}\neq j_{h}^{2}\neq \cdots \neq j_{h}^{k}}}^{d} \E \Big( \prod_{s=1}^{h} \prod_{a=1}^{k} x_{i_s j_s^a} x_{i_{s+1} j_s^a} \Big) \, 
    \label{eq:EtrBk}
\end{align}
where we let $i_{h+1} = i_1$. The product $\prod_{s=1}^{h} \prod_{a=1}^{k} x_{i_s j_s^a} x_{i_{s+1} j_s^a}$ in Eq.~\eqref{eq:EtrBk} can be identified as a single-cycle graph with $2h$ vertices:
\[
i_1 \to \{j_1^1,\cdots,j_1^k\} \to i_2 \to \{j_2^1,\cdots,j_2^k\}\to i_3 \to \cdots \to i_{h} \to \{j_h^1,\cdots,j_h^k\} \to i_1\, .
\]
More formally, we have the following definition of $h$-graph introduced in \cite{fan2019spectral}:
\begin{definition}[$h$-graph]
\label{def:h_graph}
For any integer $h\geq 2$, an $h$-graph is a single-cycle graph with $2h$ vertices and $2h$ edges. The vertices are categorized into two types: (1) $n$-vertices, indexed by $i_s \in [n]$; (2) $d$-vertices, indexed by  $\{j_1^s,j_2^s,\cdots,j_k^s\}$, which is a $k$-permutation of $[d]$ and these two types of vertices appear in the cycle alternatively.
\end{definition}
Based on the notion of $h$-graph, each \emph{non-zero} summand in Eq.~\eqref{eq:EtrBk} can be identified as a labeling of the vertices in an $h$-graph satisfying some rules. The formal definition is given as follows.
\begin{definition}
\label{def:multilabeling}
[$(n,d)$-multi-labeling of $h$-graph \cite{fan2019spectral}]
A $(n,d)$-multi-labeling of an $h$-graph is an assignment of an $n$-label in $[n]$ to each $i_s$ and a $d$-label in $[d]$ to each $j_s^a$, which satisfies:
\begin{enumerate}
    \item[(i)] $i_s\neq i_{s+1}$, for all $s\in [h-1]$ and $i_{h} \neq i_1$.
    \item[(ii)] $\{j_s^1, j_s^2,\cdots,j_s^k\}$ is a $k$-permutation of $[d]$.
    \item[(iii)] For each distinct pair $(i,j)\in [n]\times[d]$, there are even number (including 0) of edges whose endpoints is $i$ and some $\{j_s^1, j_s^2,\cdots,j_s^k\}$ that contains $j$.    
\end{enumerate}
\end{definition}
In the following, we will call any ordered tuple $\bj_s:=\{j_s^1, j_s^2,\cdots,j_s^k\}$ an ordered \emph{$d$-vertex label} and the set of all unordered tuples an \emph{unordered $d$-vertex label}. 

In a multi-labeling, all the $n$-vertices and $d$-vertices are ordered as follows. The $n$-vertex with label $i_s$ is called the $s$th $n$-vertex and the $d$-vertex with ordered label $\bj_s$ is called the $s$th vertex. The $n$/$d$-vertex label on a vertex is called the $u$th \emph{new} label, if it is different from all the preceding $n$/$d$-vertex labels, which include $(u-1)$ distinct $n$/$d$-vertex labels.

Since for any fix $i$, random variables $x_{i1},x_{i2},\cdots,x_{id}$ are exchangeable and for any fix $j$, random variables $x_{1j},x_{2j},\cdots,x_{nj}$ are also exchangeable, any two $(n,d)$-multi-labelings that are equal up to certain permutations of $[n]$ and $[d]$ lead to the same value of $\E \big( \prod_{s=1}^{h} \prod_{a=1}^{k} x_{i_s j_s^a} x_{i_{s+1} j_s^a} \big)$ in Eq.~\eqref{eq:EtrBk}. Correspondingly, we can  define the equivalence relation among all $(n,d)$-multi-labelings.
\begin{definition}[Equivalence of multi-labelings \cite{fan2019spectral}]
\label{def:equivalence_class}
Two multi-labelings are equivalent if one can mapped to the other by applying a permutation on $[n]$ and another permutation on $[d]$. 
\end{definition}
Based on the above definitions, we can now rewrite the summation in Eq.~\eqref{eq:EtrBk} as the summation over the equivalent class of $(n,d)$-multi-labelings:
\begin{align}
    \E (\Tr \bB^h) &= \frac{1}{d^{kh}} \sum_{\mathcal{L} \in \mathcal{C}_{h,k} } |\mathcal{L}| \cdot E_{\cL} \, ,
    \label{eq:EtrBk_1}
\end{align}
where $\mathcal{C}_{h,k}$ is the set of all $(n,d)$-multi-labelings equivalent classes in an $h$-graph with each $d$-vertex having $k$ $d$-labels, $\mathcal{L}$ is one of these equivalent classes, $|\mathcal{L}|$ is the total number of different $(n,d)$-multi-labelings that belongs to $\mathcal{L}$ and $E_{\cL}=\E \big( \prod_{s=1}^{h} \prod_{a=1}^{k} x_{i_s j_s^a} x_{i_{s+1} j_s^a} \big)$ for labelings in $\mathcal{L}$. 

Let us define $r_n(\cL)$ and $r_d(\cL)$ as the number of unique $n$-labels and $d$-labels in $\cL$, respectively. Also we denote $r(\cL):=r_n(\cL)+r_d(\cL)$. Notice that
\begin{align}
\label{eq:card_L_bd}
   |\cL| = \frac{n!}{\big(n-r_n(\cL)\big)!} \frac{d!}{\big(d-r_d(\cL)\big)!} \leq n^{r_n(\cL)} d^{r_d(\cL)}. 
\end{align}
Substituting Eq.~\eqref{eq:card_L_bd} into Eq.~\eqref{eq:EtrBk_1}, we get:
\begin{align}
    \E (\Tr \bB^h) &\leq \frac{1}{d^{kh}} \sum_{\mathcal{L} \in \mathcal{C}_{h,k} } n^{r_n(\cL)} d^{r_d(\cL)} \cdot E_{\cL} \nonumber\\
    &= d^k \sum_{\mathcal{L} \in \mathcal{C}_{h,k} } d^{-\Delta(\cL)} \cdot \big(\frac{n}{d^k}\big)^{r_n(\cL)} \cdot E_{\cL} \, ,
    \label{eq:EtrBk_2}
\end{align}
where 
\begin{align}
    \Delta(\cL) := k(h+1) - [k r_n(\cL)+r_d(\cL)]
\end{align}
We can show $\Delta(\cL)\geq 0$ for any $\cL\in\cC_{h,k}$ (Lemma \ref{lem:Delta_geq_0}). Therefore, from Eq.~\eqref{eq:EtrBk_2} we can verify that $\E (\Tr \bB^h) = \cO(d^k)$, when $h=\cO(1)$.

To control the operator norm, $h$ should grow with $d$. We need to further simplify Eq.~\eqref{eq:EtrBk_2}.
First, a bound can be obtained for $E_{\cL}$. Specifically, after substituting Eq.~\eqref{eq:EL_ubd_0} in Lemma \ref{lem:EL_ubd} into Eq.~\eqref{eq:EtrBk_2}, we can get for any $h=\cO( \log d )$ and all large $d$, 
\begin{align}
    \E (\Tr \bB^h) &\leq d^k \sum_{\mathcal{L} \in \mathcal{C}_{h,k}}  2\Big(\frac{(24 e \Delta(\cL))^{24 e}}{d}\Big)^{\Delta(\cL)} \cdot \big(\frac{n}{d^k}\big)^{r_n(\cL)} \, 
    \label{eq:EtrBk_2_tDelta}
\end{align}
where we let $0^0 := 1$ (such term appears when $\Delta(\cL)=0$).

Then we are going to simplify the right-hand side of Eq.~\eqref{eq:EtrBk_2_tDelta}. The general idea is to compare it with $\E (\Tr \bB_{G}^h)$, where $\bB_{G}$ is a shifted Wishart matrix: $\bB_{G} = \frac{1}{d_g} [\bG \bG^\sT - \diag(\bG \bG^\sT)]$, $\bG \in \R^{n_g\times d_g}$ and $G_{ij}\sim_{i.i.d.} \cN(0,1)$.
For $\bB_{G}$, we can utilize the existing results to control $\E (\Tr \bB_{G}^h)$.
For example, by Proposition 5.11 in \cite{fan2019spectral},
if $\gamma := n_g/d_g  \in(0,\infty)$ and $h \asymp \log d_g$, for any $\varepsilon>0$ and all large $d_g$, 
\begin{align}
\label{eq:EBG_h_bd}
    \E \opnorm{\bB_{G}}^h \leq (\lambda_{\gamma}^+ + \varepsilon )^h,
\end{align}
where $\lambda_{\gamma}^+ = \gamma + 2 \sqrt{\gamma}$ is the right boundary of the support of Marchenko-Pastur distribution (with ratio $\gamma$) shifted by $-1$.
Therefore, for any $\varepsilon>0$ and large enough $d$,
\begin{align}
\label{eq:wishart_tracemoment_bd_2}
    \E (\Tr \bB_{G}^h) \leq n_g (\lambda_{\gamma}^+ + \varepsilon )^h.
\end{align}
On the other hand, it can be shown that \cite[Lemma 5.16]{fan2019spectral}: 
for $h \asymp \log d$ there exists $C>0$ such that for all large $d_g$,
\begin{align}
\label{eq:wishart_tracemoment_bd_1}
    \E (\Tr \bB_{G}^h) \geq d_g C \sum_{\cL \in \mathcal{C}_{h,1}} \Big( \frac{1}{d_g} \Big)^{\Delta(\cL)} \Big( \frac{n_g}{d_g} \Big)^{r_n(\cL)}.
\end{align}

It turns out that we can bound the right-hand side of Eq.~\eqref{eq:EtrBk_2_tDelta} by the right-hand side of Eq.~\eqref{eq:wishart_tracemoment_bd_1}, which will eventually enable us to bound $\E(\Tr \bB^h)$ by $\E (\Tr \bB_{G}^h)$. 
In particular, by Lemma \ref{lem:ubd_compare}, there exists $C>0$ such that
\begin{align}
\label{eq:ubd_compare_1}
    &\sum_{\mathcal{L} \in \mathcal{C}_{h,k}}  \Big(\frac{(24 e \Delta(\cL))^{24 e}}{d}\Big)^{\Delta(\cL)} \cdot \Big(\frac{n}{d^k}\Big)^{r_n(\cL)} \nonumber \\ 
    &\hspace{3em}  \leq C h^2 \sum_{ {\cL} \in  \mathcal{C}_{h,1}}  \Big(\frac{C h \Delta(\cL)^C }{d}\Big)^{ \Delta(\cL)} \cdot \Big(\frac{n}{d^k/k!}\Big)^{r_n(\cL)} \nonumber \\
    &\hspace{3em} \leq C h^2 \sum_{ {\cL} \in  \mathcal{C}_{h,1}}  \Big(\frac{ 1 }{ d_g }\Big)^{ \Delta(\cL)} \cdot \Big(\frac{n_g}{d_g}\Big)^{r_n(\cL)},
\end{align}
where 
$
d_g = \frac{d}{C h (k h)^C },
$
$n_g = \frac{n}{d^k/k!} d_g $ and in the last step we use the fact that $\Delta(\cL) \leq kh $.

It can then be deduced from Eqs.~\eqref{eq:ubd_compare_1}, \eqref{eq:wishart_tracemoment_bd_1} and 
\eqref{eq:EtrBk_2_tDelta} that there exists $C>0$ such that for $h \asymp \log d$ and all large $d$,
\begin{align}
\label{eq:EtrBk_2_tDelta_compare}
    \E (\Tr \bB^h) &\leq d^k \cdot C h^2 \cdot \frac{\E (\Tr \bB_{G}^h)}{d_g}.
\end{align}
Substituting \eqref{eq:wishart_tracemoment_bd_2} into the comparison bound \eqref{eq:EtrBk_2_tDelta_compare}, we have for $h\asymp \log d$, there exists $C,c>0$ such that for any $\varepsilon>0$ and all large $d$,
\begin{align*}
    \E (\Tr \bB^h) &\leq n\cdot c h^2 (\lambda_{ \gamma }^+ + \varepsilon )^h \nonumber\\
    &\leq d^k C^h,
\end{align*}
where $\gamma = \frac{n}{d^k/k!}$ and $\lambda_{\gamma}^{+}$ is same as in Eq.~\eqref{eq:EBG_h_bd}.
As a result, there exists $C>1$ such that for any $t\geq 1$,
\begin{align}
    \P\big(\opnorm{\bB} \geq C+t \big) &\leq \P\big(\Tr{\bB^h} \geq (C+t)^h \big) \nonumber \\
    &\leq \frac{\E (\Tr \bB^h)}{(C+t)^h} \nonumber \\
    &\leq d^k \Big(\frac{C - 1}{ C - 1/2 + \sqrt{ {t} } }\Big)^h \cdot \Big(\frac{ C - 1/2 + \sqrt{ {t} }  }{C + t }\Big)^h\nonumber \\
    &\leq {d^k} \Big(\frac{C-1}{C- {1}/{2} }\Big)^h t^{-\frac{h}{2} } ,
    \label{eq:B_opnorm_prbd_1}
\end{align}
 After choosing $h= \frac{ (D+k) \log d }{\log(1+\frac{1}{2C} )} $ on the right-hand side of Eq.~\eqref{eq:QkX_opnorm_concentrate_2} and letting $C'=C+1$ and $t'=t-1$,
 we have for any $t'>0$ and large enough $d$,
\begin{align}
    \P\big(\opnorm{\bB} \geq C'+t' \big) \leq (d t')^{-D},
\end{align} 
 which is Eq.~\eqref{eq:QkX_opnorm_concentrate_2}.

(II) 
Denote $\bC_i := \frac{1}{N_k} \hbpsi_{k}(\bx_i) \hbpsi_{k}(\bx_i)^\sT$. We make the following decomposition:
\begin{align}
    \bC &= \sum_{i=1}^{n} \bC_i \charfn_{ \|\hbpsi_{k}(\bx_i)\|_2^2 < \hat{N}_k K_d } + \sum_{i=1}^{n} \bC_i \charfn_{ \|\hbpsi_{k}(\bx_i)\|_2^2 \geq \hat{N}_k K_d }  \nonumber \\
    &:= \bC_{<} + \bC_{\geq}\, ,
    \label{eq:bC_decomposition}
\end{align}
where $\hat{N}_k = N_k - d^k/k! \asymp d^{k-1}$ and $K_d>0$ is some truncation threshold to be chosen. For $\bC_{<}$ in Eq.~\eqref{eq:bC_decomposition}, we can first apply matrix Bernstein's inequality \cite{tropp2012user} to obtain that there exists $c>0$ such that for all $K_d\geq 1$ and $t>0$,
\begin{align}
\label{eq:centralized_Cleq_bd}
    \P( \opnorm{\bC_{<} - \E \bC_{<} } \geq t ) \leq 2\hat{N}_k \exp \big( -\tfrac{cd}{K_d} \min  \{ t, t^2 \} \big)\, .
\end{align}
The detailed steps are completely analogous to those leading to \eqref{eq:ZhatZhatT_concentrate1} and is omitted. Also we have
\begin{align}
\label{eq:ECleq_bd}
    \opnorm{ \E \bC_{<} } &\leq n \opnorm{ \E \bC_i } \nonumber \\
    &= \frac{n}{N_k}\, .
\end{align}
Combining Eqs.~\eqref{eq:centralized_Cleq_bd} and \eqref{eq:ECleq_bd}, and the condition $\kappa_1 \leq k$, we know there exists $C,c>0$, such that for all $K_d\geq 1$ and $t>0$,
\begin{align}
\label{eq:Cleq_bd_1}
    \P( \opnorm{\bC_{<}} \geq C + t ) \leq 2\hat{N}_k \exp \big( -\tfrac{cd}{K_d} \min  \{ t, t^2 \} \big)\, .
\end{align}

To analyze $\opnorm{\bC_{\geq}} $, we first recall a concentration result for $\psi_{ks}(\bx)$, with $\bx \sim \Unif(\S^{d-1} (\sqrt{d}))$. By Theorem 1 of \cite{beckner1992sobolev}, we have for any $k\in\mathbb{Z}_{>0}$, $s\in[N_k]$ and $q\geq 2$, 
\begin{align}
    \E |\psi_{ks}(\bx)|^q \leq (q-1)^{ \frac{kq}{2} } \, .
\end{align}
Then we have for any $t>0$
\begin{align}
\label{eq:hypercontrac_prob}
    \P(|\psi_{ks}(\bx)| \geq t) \leq \frac{ q^{ \frac{kq}{2} } }{t^q}.
\end{align}
Choosing optimal $q =  \frac{ t^{ \frac{2}{k} } }{e}$ in Eq.~\eqref{eq:hypercontrac_prob}, we have for any $t\geq (2e)^{ \frac{k}{2} }$, 
\begin{align}
\label{eq:hypercontrac_prob_2}
    \P(|\psi_{ks}(\bx)| \geq t) \leq e^{ - \frac{k}{2e} t^{ \frac{2}{k} } }.
\end{align}
Now we can apply Eq.~\eqref{eq:hypercontrac_prob_2} to get for any $K_d \geq (2e)^{k}$,
\begin{align}
    \P\big( \|\hbpsi_{k}(\bx_i)\|_2^2 \geq \hat{N}_k K_d \big) &\leq \sum_{s=1}^{ \hat{N}_k } \P \big( \hat{\psi}_{ks}(\bx_i)^2 \geq K_d \big) \nonumber \\
    &\leq \hat{N}_k e^{-\frac{k}{2e} K_d^{ \frac{2}{k} } }
\end{align}
Therefore, by union bound,
\begin{align}
\label{eq:Cge_bd_1}
    \P \Big( \opnorm{\bC_{\geq}} \geq \frac{\hat{N}_k K_d }{N_k} \Big) &\leq \P \big( \max_{ 1 \leq i \leq n} \|\hbpsi_{k}(\bx_i)\|_2^2 \geq \hat{N}_k K_d \big) \nonumber \\
    &\leq n \hat{N}_k e^{-\frac{k}{2e} K_d^{ \frac{2}{k} } }.
\end{align}
Now choose $K_d = (dt)^{\frac{k}{k+2} } $ in Eqs.~\eqref{eq:Cleq_bd_1} and \eqref{eq:Cge_bd_1}, with $\varepsilon \in (0,1)$. It can be directly checked there exists $C, c>0$ such that for any $t\geq 1$ and all large $d$, 
\begin{align}
\label{eq:Cleq_bd}
    \P( \opnorm{\bC_{<}} \geq C + t ) \leq c \exp \big\{ - c^{-1} (dt)^{\frac{2}{k+2} } \big\}
\end{align}
and
\begin{align}
\label{eq:Cge_bd}
    \P \Big( \opnorm{\bC_{\geq}} \geq t \Big) &\leq c \exp\big\{- c^{-1} (dt)^{\frac{2}{k+2} } \big\}.
\end{align}
Combining Eqs.~\eqref{eq:Cleq_bd} and \eqref{eq:Cge_bd}, we can get Eq.~\eqref{eq:bC_prob_bd}.


\subsubsection{Combinatorial results}
In this section, we collect the combinatorial results that are used in the proof of Eq.~\eqref{eq:QkX_opnorm_concentrate_2}. 
We first introduce two types of reduction of multi-labeling that are frequently used in the proof.
\begin{definition}[Singleton]
    A $n$-vertex is a \emph{singleton} if its label appears only once in the labeling. 
\end{definition}
\begin{definition}[Type-I and Type-II reduction]
\label{def:two_reductions}
    In a multi-labeling of an $h$-graph, $h\geq 3$:
    \begin{align}
    \label{eq:reduction_illustration}
       \cdots i_1 \to \{j_1^1,\cdots,j_1^k\} \to i_2 \to \{j_2^1,\cdots,j_2^k\}\to i_3 \to \{j_3^1,\cdots,j_3^k\} \cdots ,
    \end{align}
    the $n$-vertex $i_2$ is a singleton.
    The following are two types of reduction that removes $i_2$ and yields a new $(n,d)$-multi-labeling in an $(h-1)$ or $(h-2)$ graph:
    \begin{enumerate}
        \item Type-I reduction: if $i_1\neq i_3$, then remove $\{j_1^1,\cdots,j_1^k\}$,  $i_2$ and get:
        \[
        \cdots i_1 \to \{j_2^1,\cdots,j_2^k\}\to i_3 \to \{j_3^1,\cdots,j_3^k\} \cdots
        \]
        \item Type-II reduction: if $i_1 = i_3$, then remove $\{j_1^1,\cdots,j_1^k\}$, $i_2$, $\{j_3^1,\cdots,j_3^k\}$, $i_3$ and get:
        \[
        \cdots i_1 \to \{j_3^1,\cdots,j_3^k\} \cdots
        \]
    \end{enumerate}
\end{definition}
It can be directly checked after either reduction, the resultant labeling is still a valid labeling in a smaller $h$-graph. 

Now we are ready to state and prove our combinatorial results.
\begin{lemma}
\label{lem:Delta_geq_0}
   For any $\cL\in\mathcal{C}_{h,k}$,
\begin{align}
\label{eq:rL_bd}
    \Delta(\cL) \geq 0\, .
\end{align} 
\end{lemma}
\begin{proof}
    We prove this result by induction. In particular, we will use Eq.~\eqref{eq:reduction_illustration} for illustration.
    Based on Definition \ref{def:multilabeling}, one can check that Type-I reduction yields a multi-labeling $\cL_1$ in an $(h-1)$-graph, with 
    \begin{align}
    \label{eq:typeI-relation}
        r_n(\cL_1) = r_n(\cL) - 1 \text{~~and~~} r_d(\cL_1) = r_d(\cL)\, ,
    \end{align}
    and Type-II reduction yields a multi-labeling $\cL_2$ in an $(h-2)$-graph, with 
    \begin{align}
    \label{eq:typeII-relation}
        r_n(\cL_2) = r_n(\cL) - 1 \text{~~and~~} r_d(\cL_2) \geq r_d(\cL) - k.
    \end{align}
    If we repetitively apply the reductions , we will finally reach at a multi-labeling $\cL'$ of a $2$-graph or an $h'$-graph, $h'>2$, where each $n$-vertex label appears at least twice. 
    It can be directly verified that $\Delta(\cL') \geq 0$. Indeed, if $\cL'$ is a multi-labeling of a $2$-graph, then $r_n(\cL') = 2$, $r_d(\cL') = k$ and thus 
    $
    \Delta(\cL') = 3k-(2k+k) = 0;
    $
    and if $\cL'$ is a multi-labeling, each $n$-vertex label of which appears at least twice, then $r_n(\cL')\leq \frac{h}{2}$ and thus 
    $
    \Delta(\cL') \geq k(h'+1) - [\frac{kh'}{2}+r_d(\cL')] \geq k,
    $
    where the second inequality follows from $r_d(\cL')\leq\frac{kh'}{2}$, as each $d$-vertex label should appear at least twice.
    On the other hand, from Eqs.~\eqref{eq:typeI-relation} and \eqref{eq:typeII-relation}, one can check that 
    \begin{align}
    \label{eq:reduction_non_increasing}
        \Delta(\cL_i) \leq \Delta(\cL), ~i=1,2.
    \end{align}
    In other words, $\Delta(\cdot)$ is non-increasing after each reduction. This implies that $\Delta(\cL)\geq \Delta(\cL') \geq 0$, which gives Eq.~\eqref{eq:rL_bd}.
\end{proof}

\begin{lemma}
\label{lem:reductionmap_valid}
   For any $\cL \in \cC_{h,k}$ and any $n$-label $i$ in $\cL$, the number of edges between all $n$-vertices with label $i$ and all $d$-vertices with bad unordered $d$-vertex labels is even (can be zero). 
\end{lemma}
\begin{proof}
    For any given $n$-label $i$, the sum of degrees of all $n$-vertices with label $i$ is an even number, as the degree of each $n$-vertex is 2. On the other hand, by Definition \ref{def:goodbadlabel} we know the number of edges between all $n$-vertices with label $i$ and all $d$-vertices with good unordered $d$-vertex labels is also even. As a result, the number of edges between all $n$-vertices with label $i$ and all $d$-vertices with bad unordered $d$-vertex labels is even.
\end{proof}

\begin{lemma}
\label{lem:ubd_compare}
    For any $k\geq 1$, $D>0$ and $\gamma \geq 0$, there exists $C>0$ such that for  $h \asymp \log d$ and all large $d$,
    \begin{equation}
    \label{eq:ubd_compare}
        \begin{aligned}
            &\sum_{\mathcal{L} \in \mathcal{C}_{h,k}}  \Big(\frac{(D \Delta(\cL))^{D}}{d}\Big)^{\Delta(\cL)} \cdot \gamma^{r_n(\cL)} \\ &\hspace{3em}  \leq 
            C h^2 \sum_{ {\cL} \in  \mathcal{C}_{h,1}}  \Big(\frac{C h \Delta(\cL)^C }{d}\Big)^{ \Delta(\cL)} \cdot (\gamma k!)^{r_n(\cL)}. 
        \end{aligned}
    \end{equation}
\end{lemma}
\begin{proof}
    The left/right-hand side of Eq.~\eqref{eq:ubd_compare} can be understood as a weighted counting over the equivalence classes in $\cC_{h,k}$/ $\cC_{h,1}$.
    The proof idea is to construct a compression mapping $\varphi: \cC_{h,k} \to \cC_{h,1}$, $k\geq 1$ that maps each equivalence class in $\cC_{h,k}$ to another equivalence class in $\cC_{h,1}$. 
    Before formally introducing this mapping, let us 
    first define good/bad unordered $d$-vertex labels in a multi-labeling $L$.
    \begin{definition}
    \label{def:goodbadlabel}
        In a multi-labeling $L$, an unordered $d$-vertex label $\bj$ is a good unordered $d$-vertex label, if for any $n$-vertex label $i$, the number of edges between all $n$-vertices labeled with $i$ and all $d$-vertices labeled with $\bj$ is even (can be 0) in $L$.
        Otherwise, it is a bad unordered $d$-vertex label.
    \end{definition}    
    Now we are ready to define the map $\varphi$.
    \begin{definition}
    \label{def:reductionmap}
        A compression map $\varphi$ is a mapping from $\cC_{h,k}$ to $\cC_{h,1}$, where $2\leq h \leq d$ and $k\geq 1$. Given any $\cL\in \cC_{h,k}$, we first choose an arbitrary $L\in\cL$. Then we map $L$ to a $(n,d)$-labeling $\tilde{L}$, with $k=1$, via the following procedure:
    
        (i) We keep the same $n$-labelings in $\tilde{L}$ as in $L$.
        
        (ii) We map each distinct good unordered $d$-vertex label in $L$ to a distinct $d$-label in $\tilde{L}$ and the $s$-th new good unordered $d$-vertex label is mapped to $s$, where $s\in[h]$.
    
        (iii) Suppose there are $k_0$ distinct good unordered $d$-vertex labels in $L$. Then we map all the bad unordered $d$-vertex labels in $L$ to a single $d$-label $k_0+1$ in $\tilde{L}$.
    
        Let $\tilde{\cL}\in \cC_{h,1}$ be the equivalence class of $\tilde{L}$. Then $\varphi(\cL) := \tilde{\cL}$. 
    \end{definition}
    The fact that $\varphi(\cL)$ is a valid multi-labeling equivalence class in $\cC_{h,1}$ can be directly checked by verifying all three rules in Definition \ref{def:multilabeling}. In particular, (i) holds due to Definition \ref{def:reductionmap} (i); (ii) holds due to Definition \ref{def:reductionmap} (ii) and (iii); (iii) holds due to Definition \ref{def:goodbadlabel} and Lemma \ref{lem:reductionmap_valid}.
    
    Using $\varphi$, the right-hand side of Eq.~\eqref{eq:ubd_compare} becomes
    \begin{align}
        &\sum_{\cL  \in \mathcal{C}_{h,k} }  \Big(\frac{(D \Delta(\cL))^{D}}{d}\Big)^{\Delta(\cL)} \cdot \gamma^{r_n(\cL)} \nonumber\\
        =& \sum_{\tilde{\cL} \in \mathcal{C}_{h,1} } \sum_{\Delta=0}^{hk} \sum_{ \substack{\cL \in \varphi_{\Delta}^{-1}(\tilde{\cL})} } \Big(\frac{(D \Delta(\cL))^{D}}{d}\Big)^{\Delta(\cL)} \cdot \gamma^{r_n({\cL})} \nonumber \\
        =& \sum_{\tilde{\cL} \in \mathcal{C}_{h,1} } \gamma^{r_n(\tilde{\cL})} \cdot \sum_{\Delta=0}^{hk} 
        |\varphi_{\Delta}^{-1}(\tilde{\cL})| \cdot  \Big(\frac{(D \Delta)^{D}}{d}\Big)^{\Delta},
        \label{eq:comparison_bd_1}
    \end{align}
    where 
    \begin{align}
        \varphi_{\Delta}^{-1}(\tilde{\cL}) := \{ \cL \in \cC_{h,k}: \varphi(\cL) = \tilde{\cL}, \Delta(\cL) = \Delta \},
    \end{align}
    and in the last step we use $r_n(\cL) = r_n(\tilde{\cL})$, since $\varphi(\cL) = \tilde{\cL}$.

    To proceed, we need to obtain a bound for $|\varphi_{\Delta}^{-1}(\tilde{\cL})|$, given any $\tilde{\cL}$ and $\Delta$. First, we define conservative/liberal unordered $d$-vertex labels in a multi-labeling $L$.
    \begin{definition}
    \label{def:conserliberlabel}
        In a multi-labeling $L$, an unordered $d$-vertex label $\bj$ is a conservative unordered $d$-vertex label, if for any $d$-label $j_a\in \bj$, $a\in[k]$ and any unordered $d$-vertex label $\bj' \neq \bj$ in $L$, we have $j_a \not\in \bj'$.
        Otherwise, it is a liberal unordered $d$-vertex label.
    \end{definition}   
    The reason for introducing the notion of conservative/liberal labels will become clear after we describe our approach for bounding $|\varphi_{\Delta}^{-1}(\tilde{\cL})|$.
    To facilitate counting the equivalence class in $\varphi_{\Delta}^{-1}(\tilde{\cL})$, we need the following notion of canonical labeling.
    \begin{definition}
    \label{def:canonical}
        The canonical multi-labeling $L$ of an $\cL\in\cC_{h,k}$ is the one that satisfies:
    
        (i) The $i$th new $n$-label of $L$ is $i$,
    
        (ii) The $j$th new $d$-label of $L$ is $j$.

    \end{definition}
    It should be clear that in a canonical multi-labeling, any $n$-label $i\in[h]$ and any $d$-label $j\in[hk]$.
    Also it is not hard to show that each $\cL \in \cC_{h,k}$ has a unique representative canonical labeling, so it suffices to bound the number of canonical labelings. The approach is: given $\tilde{\cL}$ and $\Delta$, we enumerate a set of multi-labelings 
    such that the canonical labeling of all $\cL \in \varphi_{\Delta}^{-1}(\tilde{\cL})$ are included. Specifically,
    \begin{enumerate}
        \item
        Denote $\tilde{L}_0$ as the canonical labeling of $\tilde{\cL}$. Assign the same $n$-labels of $\tilde{L}_0$ to $L_0$.
        \item
        Choose one of or none of $d$-labels in $\tilde{L}_0$ as the image of bad unordered $d$-vertex labels in $L_0$.
        Then clearly, this choice determines which vertices in ${L}_0$ are assigned with good/bad unordered $d$-vertex labels. Denote ${\Omega}_{+}$/${\Omega}_{-}$ as the set of $d$-vertices in ${L}_0$ that have good/bad unordered $d$-vertex labels. 
        \item   
        If ${\Omega}_{-} = \emptyset$, then skip this step; otherwise we assign unordered $d$-vertex labels to the vertices in ${\Omega}_{-}$ as follows.
        Construct a graph $\cG$ with each of its vertex corresponding to each $d$-vertex in $L_0$ with bad unordered $d$-vertex labels. 
        In $\cG$, two vertices are connected, if and only if their corresponding $d$-vertices in ${L}_0$ can become neighboring vertices of a singleton after some Type-I and Type-II reductions. For all vertices in the same connected component of $\cG$, their corresponding vertices in $L_0$ are to be assigned with the same unordered $d$-vertex label. 
        \item
        If ${\Omega}_{+} = \emptyset$, then skip this step; otherwise, among those $d$-labels in $\tilde{L}_0$ identified as the image of good unordered $d$-vertex labels in $L_0$, choose a subset $\cS$ of them to be the image of liberal labels. This choice determines which vertices in $\Omega_{+}$ will be assigned with conservative or liberal unordered $d$-vertex labels. Denote ${\Omega}_{+,\text{C} }$/${\Omega}_{+,\text{L} }$ as the subset of $\Omega_{+}$ that have conservative/liberal unordered $d$-vertex labels.
        \item
        (1) For each connected component $T$ of $\cG$, choose an unordered tuple $\bj$ from $k$-combinations of $[hk]$ and the ordered $d$-vertex label of each corresponding $d$-vertex of $T$ in $\Omega_{-}$ is chosen to be an arbitrary permutation of $\bj$; 
        
        (2) For each distinct unordered $d$-vertex label in $\Omega_{+,\text{L}}$, choose an unordered tuple $\bj$ from $k$-combinations of $[hk]$ and the ordered $d$-vertex label of each vertex in $\Omega_{+,\text{L}}$ with $\bj$ is chosen to be an arbitrary permutation of $\bj$. 
        
        (3) For the $u$th $d$-vertex in $\Omega_{+,\text{C} }$ that has a new unordered $d$-vertex label, its ordered $d$-vertex label is the $[k(u-1)+1]$-th to $ku$-th smallest of the $d$-labels in $[hk]$ that are not used in the first two steps, arranged in an \emph{increasing} order. For any other vertices in $\Omega_{+,\text{C} }$, its ordered label is chosen to be an arbitrary permutation of its unordered label.
    \end{enumerate}
    It is not hard to show for any $\cL\in\varphi^{-1}_{\Delta}(\cL)$, the associated canonical labeling $L_0$ is one of the multi-labelings that can be generated by the above procedure. 
    Therefore it suffices to bound the number of possible multi-labelings that can be generated in the above procedure, given a fixed $\tilde{L}_0$ and $\Delta$.    

    A useful observation is that if there is no bad or liberal unordered $d$-vertex label in $L_0$, the assignment of unordered labels in $L_0$ is uniquely determined by $\tilde{L}_0$. Intuitively, this indicates that in order to control $|\varphi_{\Delta}(\tilde{L})|$, it suffices to control $|\Omega_{+,\text{L}}|$ and the number of connected components in $\cG$.    
    The details are as follows:
    \begin{enumerate}
        \item 
        There are at most $h+1$ ways of choosing which $d$-label (or no label) in $\tilde{L}_0$ is assigned as the image of bad unordered $d$-vertex labels in $L_0$. 
        \item 
        Since the vertices in the same connected component should be assigned with the same unordered $d$-vertex label and by Lemma \ref{eq:num_components_ubd}, there are at most  $48 \Delta$ connected components. Therefore, there are at most $48 \Delta$ distinct unordered $d$-vertex labels in $\Omega_{-}$ and there are at most $(hk)^{48\Delta k}$ different ways of choosing these unordered $d$-vertex labels.
        \item
        By Lemma \ref{lem:num_liberalvertex_bd}, we know $|\cS| \leq 48 \Delta$. Therefore, there are at most 
       \begin{align*}
            \sum_{\cS,|\cS|\leq 48\Delta} (hk)^{k|\cS|} &\leq \sum_{\cS,|\cS|\leq 48\Delta} (hk)^{48 \Delta k} \\
            &\leq (hk)^{48 \Delta k} \sum_{c=0}^{48\Delta} \binom{hk}{c} \\
            &\leq 2 (hk)^{96 \Delta k}
        \end{align*}
        ways of assigning unordered $d$-vertex labels to vertices in $\Omega_{+,\text{L} }$. 
        \item
        By construction, after the labels in $\Omega_{-}$ and $\Omega_{+,\text{L} }$ are fixed, the unordered $d$-vertex label of each $d$-vertex in $\Omega_{+,\text{C} }$ are also uniquely determined. 
        \item
        The last step is to count the number of different possible permutations, after the unordered $d$-vertex label of each $d$-vertex is fixed.
        Note that in the above procedure, the $d$-vertex labels of all vertices in 
        $\Omega_{+,\text{C} }$ with a \emph{new} unordered $d$-vertex label cannot be permuted, while the other $d$-vertex labels can be arbitrarily permuted.
        It is not hard to see there are at least $r_d(\tilde{\cL})-1-48\Delta(\tilde{\cL})$ vertices in 
        $\Omega_{+,\text{C} }$ with a {new} unordered $d$-vertex label, so there are at most 
        $$
        h - [r_d(\tilde{\cL})-1-48\Delta(\tilde{\cL})] = r_n(\tilde{\cL}) + 49 \Delta(\tilde{\cL})
        $$
        $d$-vertices that can be arbitrarily permuted.
        As a result, we have at most $(k!)^{r_n(\tilde{\cL})+49\Delta(\tilde{\cL})}$
        different permutations. 
    \end{enumerate}
    Therefore, we can get
    \begin{align}
    \label{eq:inverset_card_bd}
        |\varphi_{\Delta}^{-1}(\tilde{\cL})| \leq 2(h+1) (hk)^{144 \Delta k } \cdot (k!)^{r_n(\tilde{\cL})+49\Delta(\tilde{\cL})}.
    \end{align}
    On the other hand, by Lemma \ref{lem:DeltaL_bd_by_DeltacL}, for a given $\tilde{\cL}$, $\Delta(\cL)$ can be bounded as:
    \begin{align}
    \label{eq:DcL_bdby_DtcL_0}
        \Delta(\cL) \geq \frac{k}{1+48k} \Delta(\tilde{\cL}).
    \end{align}
    Substituting Eqs.~\eqref{eq:inverset_card_bd} and \eqref{eq:DcL_bdby_DtcL_0} into Eq.~\eqref{eq:comparison_bd_1}, we can get for $h \asymp \log d$ and all large $d$,
    \begin{align*}
        &\sum_{\cL  \in \mathcal{C}_{h,k} }  \Big(\frac{(D \Delta(\cL))^{D}}{d}\Big)^{\Delta(\cL)} \cdot \gamma^{r_n(\cL)} \nonumber\\
        \leq& {2(h+1)} \sum_{\tilde{\cL} \in \mathcal{C}_{h,1} } (\gamma k!)^{r_n(\tilde{\cL} )} (k!)^{49 \Delta(\tilde{\cL}) } \sum_{\Delta \geq \frac{k \Delta(\tilde{\cL})}{1+48k} }^{k h} 
        (hk)^{144 k \Delta} \cdot  \Big(\frac{(D \Delta)^{D} }{d}\Big)^{\Delta} \nonumber\\
        \leq& C h^2 \sum_{\tilde{\cL} \in \mathcal{C}_{h,1} } (\gamma k!)^{r_n(\tilde{\cL} )} \cdot \Big( \frac{C h \Delta(\tilde{\cL}) }{d} \Big)^{C \Delta(\tilde{\cL}) } 
    \end{align*}
    for some constant $C>0$.     
\end{proof}

\begin{remark}
   Based on the moment method and the comparison argument in the proof of Lemma \ref{lem:ubd_compare}, we can also see why the spectrum of $\bB$ should converge to Marchenko-Pastur distribution. For any fixed $h\geq 2$, from Eq.~\eqref{eq:EtrBk_1} we can get
    \begin{align}
         \frac{\E(\Tr \bB^h)}{n} &\to \sum_{ \substack{ \mathcal{L} \in \mathcal{C}_{h,k} \\ \Delta(\cL)=0 } } \big(\frac{n}{d^k}\big)^{r_n(\cL)-1} \cdot E_{\cL} \, .
    \end{align}
    By Lemma \ref{lem:num_badvertex_bd}, Lemma \ref{lem:num_liberalvertex_bd} and Lemma \ref{lem:conseq_npair_ubd} we know when $\Delta(\cL)=0$, $\cL$ is quite regular, in the sense that: (i) there is no bad or liberal unordered $d$-vertex and (ii) $b_{ij}=0\text{ or }2$. 
    From (i), we can get $|\varphi_{\Delta=0}^{-1}(\tilde{\cL})| = (k!)^{r_n(\tilde{\cL})-1}$ and
    from (ii), we have for large $d$,
    \begin{align}
        \E_{\cL} = \prod_{i=1}^{n} \E \Big( \prod_{ \substack{j=1\\b_{ij}=2} }^{d} x_{ij}^{b_{ij}} \Big)
        = \prod_{i=1}^{n} 2^{-B_i} \frac{ \Gamma(\frac{d}{2}) }{ \Gamma(\frac{B_i}{2} + \frac{d}{2}) } \approx 1,
    \end{align}
    where in the last step, we use the fact $\sum_{i=1}^{n} B_i = 2kh \ll d$.
    Therefore,
    \begin{align}
        \frac{\E(\Tr \bB^h)}{n} &\to \sum_{ \substack{ \mathcal{L} \in \mathcal{C}_{h,k} \\ \Delta(\cL)=0 } } \big(\frac{n}{d^k}\big)^{r_n(\cL)-1} \nonumber\\
        &= \sum_{ \substack{\tilde{\cL} \in \cC_{h,1} \\ \Delta(\tilde{\cL})=0} } \big(\frac{n}{d^k}\big)^{r_n(\tilde{\cL})-1} |\varphi_{\Delta=0}^{-1}(\tilde{\cL})| \nonumber\\
        &= \sum_{ \substack{\tilde{\cL} \in \cC_{h,1} \\ \Delta(\tilde{\cL})=0} } \big(\frac{n}{d^k/k!}\big)^{r_n(\tilde{\cL})-1} ,
    \end{align}
    and we identify that the right-hand side is actually the $h$-th moment of Marchenko-Pastur distribution with ratio $\frac{n}{d^k/k!}$, shifted by $-1$.
\end{remark}

\begin{lemma}
\label{eq:num_components_ubd}
    Suppose $\cL$ has $N_{-}$ vertices with bad unordered $d$-vertex labels and $N_{-}\geq 1$. Graph $\cG$ is constructed as follows. Each vertex of $\cG$ corresponds to each bad unordered $d$-vertex in $\cL$ and any two vertices of $\cG$ are connected if and only if their corresponding vertices in $\cL$ are the neighboring $d$-vertices of a singleton. Then we have 
    \begin{align}
    \label{eq:connected_ubd}
        C_{\cG} \leq 48 \Delta(\cL) \, ,
    \end{align}    
    where $C_{\cG}$ is the number of connected components in $\cG$.
\end{lemma}
\begin{proof}
    We will show $C_{\cG} \leq 48 \Delta(\cL)$ by induction.

    First, consider the special case when $h=2$ or there is no singleton in $\cL$.
    In the proof of Lemma \ref{lem:num_badvertex_bd}, we show that in this case, $N_{-} \leq 48\Delta(\cL)$. This implies Eq.~\eqref{eq:connected_ubd}, as $\cG$ has $N_{-}$ vertices and the number of connected components is bounded by the number of vertices.

    For the general $h\geq 3$ case, we can sequentially apply Type-I and Type-II reduction to reach one of the two special cases above. We denote $\cL'$ as the resultant equivalence class after one reduction and
    $\cG'$ as the corresponding graph of bad $d$-vertices. A convenient property is that after either type of reduction, good/bad unordered $d$-vertex labels remain good/bad (see the proof of Lemma \ref{lem:num_badvertex_bd}). This indicates that after every reduction, each $\cG$'s vertex that corresponds to the deleted bad $d$-vertices in $L$ will be removed, while the other $\cG$ vertices are preserved and still associated with the same $d$-vertices before the reduction.
    Now we are going to show for both types of reduction, it holds that
    \begin{align}
    \label{eq:cG_change_ubd}
        C_{\cG} - C_{\cG'} \leq \Delta(\cL) - \Delta(\cL') \, .
    \end{align}
    \begin{itemize}
        \item Type-I reduction. Every removed $d$-vertex in $\cL$ corresponds to a connected component of size $\geq 2$ in $\cG$. Therefore, we have $C_{\cG} = C_{\cG'}$.  Also, recall that $\Delta(\cL) = \Delta(\cL')$ for Type-I reduction, so Eq.~\eqref{eq:cG_change_ubd} holds.
        \item Type-II reduction. In this case, the deleted $d$-vertices share the same unordered $d$-label. Denote $k'$ as the number of $d$-labels in the deleted $d$-vertices that do not appear in other non-deleted $d$-vertices. 
        We know that $\Delta(\cL) - \Delta(\cL') = k - k'$.
        If $k'=k$, then both deleted $d$-vertices have good unordered $d$-vertex labels and thus $C_{\cG} - C_{\cG'} = 0 = \Delta(\cL) - \Delta(\cL')$; if $k' < k$, then both deleted $d$-vertices have bad unordered $d$-vertex labels and they corresponds to the same connected component of $\cG$. Therefore, $C_{\cG} - C_{\cG'} \leq 1 \leq k- k' = \Delta(\cL) - \Delta(\cL')$. As a result, Eq.~\eqref{eq:cG_change_ubd} still holds.
    \end{itemize}

    Based on Eq.~\eqref{eq:cG_change_ubd} and our proof for the special cases, we can get Eq.~\eqref{eq:connected_ubd} by induction.
\end{proof}

\begin{lemma}
\label{lem:num_badvertex_bd}
    For any $k\geq 1$, $2 \leq h \leq d$ and $\cL\in\cC_{h,k}$, the number of bad unordered $d$-vertex labels of $\cL$ is no more than $48\Delta(\cL)$.
\end{lemma}
\begin{proof}
Let  $\Delta := \Delta(\cL)$.
For $\cL$, denote $N_{-}$ as the number of vertices that have bad unordered $d$-vertex labels
and $\Omega_{-}$ as the set of all bad unordered $d$-vertex labels.
We are going to show by induction that
\begin{align}
\label{eq:COmega_ubd}
    |\Omega_{-}| \leq 48 \Delta.
\end{align}
We start by considering two special cases. 
\begin{itemize}
    \item $h=2$. In this case, there are two distinct $n$-labels, one distinct unordered $d$-vertex label and exactly two edges between each distinct $n$-label and the only unordered $d$-vertex label, which justifies Definition \ref{def:goodbadlabel}. Therefore, $|\Omega_{-}| = N_{-} = 0 = \Delta$ and Eq.~\eqref{eq:COmega_ubd} is satisfied.
    \item There is no singleton in $\cL$. 
    Suppose the $s$th $d$-vertex has a bad unordered $d$-vertex label. We discuss over two different scenarios. (a) There is one $d$-label $j$ in $\bj_s$ such that $N_j \geq 3$ [$N_j$ is defined in Eq.~\eqref{eq:Njdef}]. Then by Eq.~\eqref{eq:Njg2_sum_ubd}, the number of $d$-vertices with bad unordered $d$-vertex labels satisfying (a) is bounded $6\Delta$. (b) Any $d$-label $j$ in $\bj_s$ satisfy $N_j = 2$. If the two neighboring $n$-vertices of the $s$th $d$-vertex are labelled as $i$ and $i'$, then each of the $d$-labels of $\bj_s$ should appear on another $d$-vertex whose neighboring $n$-vertices are also labeled as $i$ and $i'$ (By Definition \ref{def:multilabeling} (iii)).  
    Besides, they shouldn't appear on the same $d$-vertex, otherwise this would contradict the fact that $\bj_s$ has a bad unordered $d$-vertex label. 
    This means
    \begin{align}
    \label{eq:unloyal_j}
        \exists t, u\neq s, (t\neq u) \text{ and } j_1, j_2 \in \bj_s, (j_1\neq j_2) \text{ such that } j_1 \in \bj_t \text{ and } j_2 \in \bj_u.
    \end{align}   
    Since the three pairs of consecutive $n$-labels surrounding $s$th, $t$th and $u$th $d$-vertex should all be $(i,i')$, we get $P_{i,i'} \geq 3$, where $P_{i,i'}$ is the number of times $(i,i')$ or $(i',i)$ appears as consecutive $n$-labels. This implies that the number of $d$-vertices with such bad unordered $d$-vertex label is bounded by $\sum_{1\leq i<i' \leq n} P_{i,i'} \charfn_{P_{i,i'}\geq 3}$.
    Then by Lemma \ref{lem:conseq_npair_ubd}, we know the number of $d$-vertices with bad unordered $d$-vertex labels satisfying (b) should be bounded by $42 \Delta$.
    
    Combining (a) and (b), we know if there is no singleton in $L_0$, $|\Omega_{-}| \leq N_{-} \leq 6\Delta + 42 \Delta = 48 \Delta$.
\end{itemize}

Now for general $h\geq 3$ case, we can sequentially apply Type-I or Type-II reduction to $\cL$ and reach at a new $\tilde{\cL}$ that is one of the two special cases discussed above. Denote that after one reduction,  
 $\cL$ becomes $\cL'$ and $\Omega_{-}$ becomes $\Omega_{-}'$. Next we show 
\begin{align}
\label{eq:card_Omega1_unchanged}
    |\Omega_{-}| = |\Omega_{-}'|.
\end{align}
To do this, we will use the example \eqref{eq:reduction_illustration} for illustration and here we are considering canonical labeling.
Let $i$ be the value of singleton $i_2$ to be removed and $\bj$ be the unordered $d$-vertex label of $\bj_1$ and $\bj_2$. If Type-II reduction is implemented, let $i'$ be the value of $i_1$ and $i_3$. We know $i\neq i'$. The case of Type-I and II reduction can be analyzed separately.
\begin{itemize}
    \item Type-I reduction. In this case, the number of edges between $\bj$ and $i$ is reduced by 2, while the numbers of edges between other pairs of distinct $d$-vertex labels and unordered $n$-vertex labels remain unchanged.
    \item Type-II reduction. In this case, the numbers of edges between $(\bj, i)$ and the numbers of edges between $(\bj, i')$ are both reduced by 2, while the numbers of edges between other distinct $n$ label and unordered $d$-vertex labels remain unchanged. 
\end{itemize}
Hence, after either reduction, the change of edges between each distinct pair of $n$ label and unordered $d$-vertex label is either 0 or 2, which are both even numbers. We can conclude that after either reduction, 
all good/bad unordered $d$-vertex labels will remain as good/bad (if they still exist after reduction). This implies that
the number of bad unordered $d$-vertex labels will not change, which proves Eq.~\eqref{eq:card_Omega1_unchanged}. Since $\Delta(\cL) \geq \Delta(\cL')$,
we can get Eq.~\eqref{eq:COmega_ubd} by induction.
\end{proof}

\begin{lemma}
\label{lem:num_liberalvertex_bd}
    For any $k\geq 1$, $2 \leq h \leq d$ and $\cL\in\cC_{h,k}$, the number of liberal unordered $d$-vertex labels of $\cL$ is no more than $48\Delta(\cL)$.
\end{lemma}
\begin{proof}
    The proof is similar as that of Lemma \ref{lem:num_badvertex_bd}. Let  $\Delta := \Delta(\cL)$. Denote the set of all liberal unordered $d$-vertex labels in $\cL$ as $\Omega_{\text{L} }$.
    We are going to show by induction that
    \begin{align}
    \label{eq:COmega_ubd_1}
        |\Omega_{\text{L} }| \leq 48 \Delta.
    \end{align}
    We first justify the following two special cases:
    \begin{itemize}
        \item $h=2$. There are two distinct $n$-labels and one distinct unordered $d$-vertex label. Clearly in this case, $\Omega_{\text{L} } = \emptyset$ and $\Delta= 3k - (2k + k ) = 0$. Therefore, Eq.~\eqref{eq:COmega_ubd} is satisfied.
        \item There is no singleton in $\cL$. For every unordered label $\bj\in \Omega_{\text{L} }$, we discuss two different scenarios. (a) There is one $d$-label $j$ in $\bj$ such that $N_j \geq 3$. Then by Eq.~\eqref{eq:Njg2_sum_ubd}, the number of such $\bj\in\Omega_{\text{L} }$ is bounded by $6\Delta$. (b) Any $d$-label $j$ in $\bj$ satisfy $N_j = 2$. 
        In this case, Definition \ref{def:conserliberlabel} implies that
        \begin{align}
        \label{eq:unloyal_j_1}
            \exists j_1, j_2 \in \bj~ (j_1\neq j_2) \text{ such that } j_1 \in \bj_1 \text{ and } j_2 \in \bj_2, \text{ where }\bj, \bj_1 \text{ and } \bj_2 \text{ are different.}
        \end{align}   
        Then by the same argument following Eq.~\eqref{eq:unloyal_j}, we conclude that the number of $\bj\in\Omega_{\text{L} }$ satisfying (b) is bounded by $42\Delta$.
        
        Combining (a) and (b), we know when there is no singleton in $\cL$, $|\Omega_{\text{L} }| \leq 48 \Delta$.
    \end{itemize}

    Now for general $h\geq 3$ case, we can sequentially apply Type-I or Type-II reduction and reach at one of the above two cases. Denote that after one reduction,  $\cL$ becomes $\cL'$ and $\Omega_{\text{L} }$ becomes $\Omega_{\text{L} }'$. Next we show 
    \begin{align}
    \label{eq:card_Omega3_unchanged}
        |\Omega_{\text{L} }| - |\Omega_{\text{L} }'| \leq 2 \big[ \Delta(\cL) - \Delta(\cL') \big].
    \end{align}
    We use the example \eqref{eq:reduction_illustration} for illustration.
    Let $i$ be the $n$-label of the second $n$-vertex  and $\bj$ be the unordered $d$-vertex label of the first and second $d$-vertex. If Type-II reduction is implemented, let $i'$ be the value of the first and third $n$-vertex.
    We discuss over two different reductions.
    \begin{itemize}
        \item Type-I reduction. In this case, $\cL$ and $\cL'$ share the same set of unordered $d$-vertex labels, $|\Omega_{\text{L} }| = |\Omega_{\text{L} }'|$. Also we know $\Delta(\cL) = \Delta(\cL')$, so Eq.~\eqref{eq:card_Omega3_unchanged} holds.
        \item Type-II reduction. We denote by $k'$ the number of $d$-labels in $\bj$ that are not present in any other unordered $d$-vertex label in $\cL$. It can be directly verified by definition that $\Delta(\cL) - \Delta(\cL') = k - k'$ and
        $$
        \begin{cases}
        |\Omega_{\text{L} }|=|\Omega_{\text{L} }'|, &k=k',\\
        |\Omega_{\text{L} }|\leq |\Omega_{\text{L} }'|+ k - k' + 1, &k'<k.
        \end{cases}
        $$
        These together imply that 
        \begin{align}
        \label{eq:c_cprime_diff}
            |\Omega_{\text{L} }| - |\Omega_{\text{L} }'| \leq 2 \big[\Delta(\cL) - \Delta(\cL')\big]
        \end{align}
        after each reduction.
    \end{itemize}
    Based on Eq.~\eqref{eq:card_Omega3_unchanged} and the justifications on the two special cases, we get Eq.~\eqref{eq:COmega_ubd_1} by induction.
\end{proof}

\begin{lemma}
\label{lem:DeltaL_bd_by_DeltacL}
    For any $\tilde{\cL}\in\cC_{h,1}$ and any $\cL \in \{ \cL \in \cC_{h,k}: \varphi(\cL) = \tilde{\cL}\}$, it holds that
    \begin{align}
    \label{eq:DcL_bdby_DtcL}
        \Delta(\cL) \geq \frac{k}{1+48k} \Delta(\tilde{\cL}).
    \end{align}
\end{lemma}
\begin{proof}
    Let $N_{\text{bad} }$ be the number of distinct bad unordered $d$-vertex labels in $\cL$. By construction of $\varphi$ in Definition \ref{def:reductionmap}, we have
    \begin{align}
        r_d(\cL) \leq k [r_d(\tilde{\cL}) +  N_{\text{bad} } ].
    \end{align}
    Recall that $\Delta(\cL) = k(h+1) - [k r_n(\cL) + r_d(\cL)] $. Hence, we can get
    \begin{align}
    \label{eq:DcL_bdby_DtcL0}
        \Delta(\cL) &\geq k(h+1) - k[ r_n(\cL) + r_d(\tilde{\cL}) +  N_{\text{bad} }] \nonumber\\
        &= k [\Delta( \tilde{\cL} ) - N_{\text{bad} }].
    \end{align}
    Meanwhile, by Lemma \ref{lem:num_badvertex_bd} we have $N_{\text{bad} } \leq 48 \Delta(\cL)$. Substituting it into Eq.~\eqref{eq:DcL_bdby_DtcL0}, we get Eq.~\eqref{eq:DcL_bdby_DtcL}.
\end{proof}

\begin{lemma}
\label{lem:conseq_npair_ubd}
    Consider an equivalence class $\cL\in \cC_{h,k}$. For $i, i' \in [n]$, with $i<i'$, define $P_{i,i'}$ as the number of times $(i,i')$ or $(i',i)$ appears as consecutive $n$-labels in $\cL$. Then we have:
    \begin{align}
    \label{eq:Piiprime_bd}
        \sum_{1 \leq i<i' \leq n} P_{i,i'} \charfn_{P_{i,i'} > 2} \leq 6 \Big( 6 + \frac{1}{k} \Big) \Delta(\cL)\, .
    \end{align}
\end{lemma}
\begin{proof}
    We prove by induction. First, consider the following two special cases. 
    \begin{itemize}
        \item $h=2$. In this case, there are two distinct $n$-labels and one distinct unordered $d$-vertex label, so $\Delta(\cL)=0$ and
        $
        \sum_{1 \leq i<i' \leq n} P_{i,i'} \charfn_{P_{i,i'} > 2} = 0.
        $
        which trivially satisfy Eq.~\eqref{eq:Piiprime_bd}.
        \item There is no singleton in $\cL$. 
        Recall that
        \begin{align}
            r_n = \frac{k(h+1) - r_d - \Delta(\cL)}{k}.
        \end{align}
        By Definition \ref{def:multilabeling}, we must have  $r_d \leq \frac{kh}{2}$, so
        \begin{align}
        \label{eq:rn_lbd_1}
            r_n \geq \frac{h}{2} + 1 - \frac{ \Delta(\cL) }{k}.
        \end{align}
        Then we can get
        \begin{align}
        \label{eq:distinct_conseq_lbd}
            \sum_{1 \leq i < i' \leq n} \charfn_{P_{i,i'} \geq 1 } &\geq r_n - 1 \geq \frac{h}{2} - \frac{ \Delta(\cL) }{k} , 
        \end{align}
        where the last step is due to Eq.~\eqref{eq:rn_lbd_1}. In other words, Eq.~\eqref{eq:distinct_conseq_lbd} means there are at least $\frac{h}{2} - \frac{ \Delta(\cL) }{k}$ distinct pair of consecutive $n$-labels $(i,i')$, with $i<i'$.
        
        Next we bound the total number of distinct pairs $(i,i')$ satisfying $P_{i,i'}=1$. If a $d$-label $j$ in $\cL$ satisfy $N_j\leq 2$, then by Definition \ref{def:multilabeling} (iii), the two $d$-vertices with label $j$ should be surrounded by the same pair of $n$-labels $(i,i')$, which indicates that $P_{i,i'}\geq 2$. Therefore, if $P_{i,i'} = 1$ for some $(i,i')$ in $\cL$, the sandwiched $d$-vertex should only have $d$-label $j$, satisfying $N_j>2$.  By Eq.~\eqref{eq:Njg2_sum_ubd}, the number of such $d$-vertices is bounded by $6(\Delta(\cL)-k)$ and thus 
        \begin{align}
        \label{eq:distinct_conseq_ubd_1}
            \sum_{1\leq i<i' \leq n} \charfn_{P_{i,i'}=1} \leq 6(\Delta(\cL)-k).
        \end{align}

        Now denote 
        $
        c := \sum_{1\leq i<i' \leq n} \charfn_{P_{i,i'}=2} .
        $
        From Eqs.~\eqref{eq:distinct_conseq_lbd} and \eqref{eq:distinct_conseq_ubd_1}, we have
        \begin{align}
            2c + 3 \Big[ \frac{h}{2} - \frac{ \Delta(\cL) }{k} - 6(\Delta(\cL)-k) - c\Big] \leq h \, ,
        \end{align}
        which gives
        \begin{align}
            c \geq \frac{h}{2} - 3 \Big( \frac{1}{k} + 6 \Big) \Delta(\cL) + 18k\, .
        \end{align}
        Therefore, we have
        \begin{align}
            \sum_{1\leq i<i' \leq n} P_{i,i'} \charfn_{P_{i,i'}=2} =2c \geq h - 6 \Big( \frac{1}{k} + 6 \Big) \Delta(\cL) + 36k\, ,
        \end{align} 
        and thus
        \begin{align}
            \sum_{1\leq i<i' \leq n} P_{i,i'} \charfn_{P_{i,i'} > 2} &\leq h - \sum_{1\leq i<i' \leq n} P_{i,i'} \charfn_{P_{i,i'}=2}  \nonumber \\
            &\leq 6 \Big( \frac{1}{k} + 6 \Big) \Delta(\cL) - 36k.
        \end{align}
    \end{itemize}

    For $h\geq 3$ case, we can apply Type-I and Type-II reduction sequentially and arrive at a new equivalence class that must be one of the above special cases. Besides, it can be directly checked that after each step of reduction, $\sum_{1\leq i<i'\leq n} P_{i,i'} \charfn_{P_{i,i'} > 2}$ remains unchanged. From Eq.~\eqref{eq:reduction_non_increasing} we know after each reduction, $\Delta(\cL)$ is non-increasing. Therefore, by induction we get the desired result.
\end{proof}

\begin{lemma}
    For each $j\in[d]$, define 
    \begin{align}
    \label{eq:Njdef}
     N_j := \text{number of times }j\text{ appears as a }d\text{-label in } \cL  .
    \end{align}
    Then it holds that (i) if $N_j>0$, then $N_j\geq 2$; (ii) if $\cL$ has no singleton, then
    \begin{align}
    \label{eq:Njg2_sum_ubd}
        \sum_{j=1}^{d} N_{j} \charfn_{N_{j}>2} \leq 6(\Delta(\cL)-k).
    \end{align}
\end{lemma}
\begin{proof}
    (i) This can be directly checked from Definition \ref{def:multilabeling} (iii).
    
    (ii) Let $c = |\{j:N_j=2\}|$. Then $2c+3(r_d(\cL)-c) \leq kh $ and thus $c\geq 3 r_d(\cL) - kh$. Therefore, the set $\{j:N_j=2\}$ contributes at least $6 r_d(\cL) - 2kh$ of all $d$-labels or equivalently
    \begin{align}
    \label{eq:Njeq2_sum_bd}
        \sum_{j=1}^{d} N_{j} \charfn_{N_{j}=2} \geq  6 r_d(\cL) - 2kh.
    \end{align}
    As a result,
    \begin{align}
        \sum_{j=1}^{d} N_{j} \charfn_{N_{j}>2} &\teq{\text{(a)}} kh - \sum_{j=1}^{d} N_{j} \charfn_{N_{j}=2} \nonumber \\
        & \tleq{\text{(b)}} 3kh - 6 r_d(\cL) ,
        \label{eq:Njgeq2_sum_bd}
    \end{align}
    where (a) follows from the fact that for any $j\in[d]$, $N_j\geq 2$, if $N_j>0$ and $ \sum_{j=1}^{d} N_{j} = kh$, and (b) follows from Eq.~\eqref{eq:Njeq2_sum_bd}. On the other hand, when there is no singleton in the $h$-graph, $r_n(\cL)\leq\frac{h}{2}$, so
    \begin{align}
        \Delta(\cL) &\geq k(h+1) - \frac{kh}{2} - r_d(\cL)\, .
        \label{eq:DeltacL_bd_nosingleton}
    \end{align}
    Combining Eqs.~\eqref{eq:Njgeq2_sum_bd} and \eqref{eq:DeltacL_bd_nosingleton}, we can get Eq.~\eqref{eq:Njg2_sum_ubd}. 
\end{proof}

\begin{lemma}
\label{lem:S_L_ubd}
    For an equivalence class $\cL$, let $b_{ij}$ be the total number of edges between all $d$-vertices with $d$-label $j$ and all $n$-vertices with $n$-label $i$. 
    It holds that  
    \begin{align}
    \label{eq:ScL_ubd_0}
        \sum_{i=1}^{n} \sum_{j=1}^{d} b_{ij} \charfn_{b_{ij}>2} \leq 12 \Delta(\cL).
    \end{align}
\end{lemma}
\begin{proof}
    We prove Eq.~\eqref{eq:ScL_ubd_0} by induction. The proof is adapted from that of Lemma 5.7 in \cite{fan2019spectral}. In the following, Denote $S_\cL := \sum_{i=1}^{n} \sum_{j=1}^{d} b_{ij} \charfn_{b_{ij}>2}.$ 
    
         
    When $h=2$ or $h=3$, according to Definition \ref{def:multilabeling} (iii), we have $b_{ij}=0\text{ or }2 $ for all $(i,j)\in[n]\times[d]$. Then $S_\cL = \sum_{i=1}^{n} \sum_{j=1}^{d} b_{ij} \charfn_{b_{ij}>2} = 0 \leq 12\Delta(\cL)$.
    Now for $h\geq 4$, suppose Eq.~\eqref{eq:ScL_ubd} holds for $h-1$ and $h-2$. If there is no singleton in the $h$-graph, then $r_n(\cL)\leq \frac{n}{2}$. For $j$ satisfying $N_j\leq 2$, we must have $b_{ij}=0\text{ or }2$ for all $i\in[n]$. Therefore,
    \begin{align}
        S_\cL &= \sum_{j=1}^{d} \charfn_{N_j>2} \sum_{i=1}^{n} b_{ij} \charfn_{b_{ij}>2} \nonumber\\
        &\tleq{\text{(a)} } \sum_{j=1}^{d} \charfn_{N_j>2} \cdot 2 N_j \nonumber\\
        & \tleq{\text{(b)} } 12 (\Delta(\cL) - k)\, ,
    \end{align}
    where (a) follows from $\sum_{i=1}^{n} b_{ij} \charfn_{b_{ij}>2} \leq \sum_{i=1}^{n} b_{ij} = 2N_j$ and (b) is due to Eq.~\eqref{eq:Njg2_sum_ubd}. If there exists a singleton, then we can consider the two types of reduction. Denote the equivalence class and number of appearance of $(i,j)$ after reduction as $\cL'$ and $b_{ij}'$. By Type-I reduction, we get an $h-1$-graph and it can be directly checked that $S_\cL = S_{\cL'}$. Therefore, by induction
    \begin{align}
        S_\cL &= S_{\cL'} \nonumber\\
        &\leq 12 \Delta(\cL') \nonumber\\
        &\leq 12 \Delta(\cL)\, ,
        \label{eq:ScL_ubd_1}
    \end{align}
    where the last step follows from the property \eqref{eq:reduction_non_increasing}. Alternatively, by Type-II reduction we get an $(h-2)$-graph. Recall that two deleted $d$-vertices have the same $d$-labeling. Suppose $k - k'$ of them also appear on some $d$-vertices that are not deleted.
    For such deleted index $j$, $b_{ij}-b_{ij}' \leq 2$, for any $n$-label $i$. Besides, other deleted indexes $j$ do not contribute to $S_{\cL}$, since they should satisfy $b_{ij} = 2$. Therefore,
    \begin{align}
        S_{\cL} - S_{\cL'} \leq 2 (k - k').
    \end{align}
    By induction, $S_{\cL'}\leq 12 \Delta(\cL')$ and we have
    \begin{align}
        S_{\cL} 
        &\leq 12 \Delta(\cL')  + 2 ( k - k' ) \nonumber\\
        &= 12 \Delta(\cL) - 10 ( k - k' ) \nonumber\\
        &\leq 12 \Delta(\cL)
        \label{eq:ScL_ubd_2}
    \end{align}
    where in the second to last step, we use the fact that $r_n(\cL') = r_n(\cL) - 1$ and $r_d(\cL') = r_d(\cL) - k'$.
    Combining Eqs.~\eqref{eq:ScL_ubd_1} and \eqref{eq:ScL_ubd_2}, we get Eq.~\eqref{eq:ScL_ubd_0}.   
\end{proof}

\subsubsection{Moment bounds}
\begin{lemma}
\label{lem:EL_ubd}
   For any $h=\cO(d^{\frac{1}{6}-\varepsilon })$, $\varepsilon > 0$ and all large $d$,
\begin{align}
    E_{\cL} &\leq 2 \big[ (24 e \Delta(\cL))^{24 e} \big]^{\Delta(\cL)}.
    \label{eq:EL_ubd_0}
\end{align}
In particular, when $\Delta(\cL)=0$, the term $(24 e \Delta(\cL))^{24 e \Delta(\cL)}$ on the right-hand side of Eq. \eqref{eq:EL_ubd_0} equals to $1$.
\end{lemma}
\begin{proof}
    Based on Lemma \ref{lem:uniformdist_moments} (I), $E_{\cL}=0$ if there exists an odd $b_{ij}$ and then Eq.~\eqref{eq:EL_ubd_0} trivially holds. Therefore, in what follows, we assume $b_{ij}$ are all even numbers.
        
    First, if $kh\leq d - 1$, we show there exists $C>0$ such that
    \begin{align}
        E_\cL &\tleq{ \text{} } \Big( \frac{d}{d- kh} \Big)^{ kh } \cdot \Big( 1 + \frac{ 4 C k^2 h^2 }{ \sqrt{d} } \Big)^{ \frac{2kh}{3} } (2e S_\cL)^{2e S_\cL}\, , \label{eq:EL_ubd}
    \end{align}
    where
    \begin{align}
    \label{eq:ScL_def}
    S_\cL := \sum_{i=1}^{n} \sum_{j=1}^{d} b_{ij} \charfn_{b_{ij}>2}.
    \end{align}
    Denote $b_{ij}$ as the number of appearance of $x_{ij}$ in the product $\prod_{s=1}^{h} \prod_{a=1}^{k} x_{i_s j_s^a} x_{i_{s+1} j_s^a}$.
    Using the representation $\bx_i = \frac{\tbx_i }{\|\tbx_i\|_2 / \sqrt{d} }$, with $\tbx_i \sim_{i.i.d.} \cN(\bzero,\id_d)$, we have when $kh \leq {d-1}$,
    \begin{align}
        E_\cL &= \prod_{i=1}^{n} \E \Big( \prod_{j=1}^{d} x_{ij}^{b_{ij}} \Big) \nonumber \\
        &\teq{\text{ } } \prod_{\substack{i=1,B_i>0} }^{n}  \E \Big({\frac{ {d} }{ \sum_{j=1}^{d} \tx_{ij}^2 }}\Big)^{ \frac{B_i}{2} } \prod_{ \substack{j=1} }^{d} \tx_{ij}^{b_{ij}} \nonumber\\
        &\tleq{} \prod_{\substack{i=1,B_i>0} }^{n} \E \Big({\frac{ {d} }{ \sum_{j=1}^{d} \tx_{ij}^2 \charfn_{b_{ij} \neq 2} }}\Big)^{\frac{B_i}{2} } \prod_{ \substack{j=1} }^{d} \tx_{ij}^{b_{ij}} \nonumber \\
        &\teq{\text{} } \prod_{\substack{i=1,B_i>0} }^{n} \Big( \frac{d}{d_i} \Big)^{ \frac{{B}_i}{2} } \E \Big({\frac{ d_i }{ \sum_{j=1,b_{ij}\neq 2}^{d} \tilde{x}_{ij}^2 }}\Big)^{\frac{ {B}_i}{2} }  \prod_{ \substack{j=1, b_{ij} \neq 2} }^{d} \tx_{ij}^{b_{ij}}  \, ,
        \label{eq:EL_ubd_1}
    \end{align}
    where $B_i = \sum_{j=1}^{d} b_{ij}$, $d_i = \sum_{j=1}^{d} \charfn_{b_{ij}\neq 2}$ and
    the second to last step is valid since 
    \begin{align}
    \label{eq:numoftwo_bd}
        \sum_{j=1}^{d} \charfn_{b_{ij} = 2} \leq \frac{1}{2} \sum_{j=1}^{d} b_{ij} = \frac{B_i}{2} \leq kh,
    \end{align}
    which implies $d_i \geq d- kh $ and thus $d_i\geq 1$, when $kh \leq {d-1}$.
    We can further bound Eq.~\eqref{eq:EL_ubd_1} as follows:
    \begin{align}
        E_\cL 
        &\tleq{\text{(a)} } \Big( \frac{d}{d- kh} \Big)^{ kh } \prod_{\substack{i=1,B_i>0  } }^{n} \E \Big({\frac{ d_i }{ \sum_{j=1, {b}_{ij} > 2}^{d} {\tx}_{ij}^2 }}\Big)^{\frac{ {B}_i}{2} } \prod_{ \substack{j=1, {b}_{ij} > 2} }^{d} {\tx}_{ij}^{{b}_{ij}} \nonumber \\
        &\tleq{\text{(b)} } \Big( \frac{d}{d- kh} \Big)^{ kh } \prod_{i=1,B_i>0}^{n} \Big( 1 + \frac{ C B_i^2}{ \sqrt{d} } \Big) \prod_{j=1, {b}_{ij} > 2}^{d} (2e {b}_{ij})^{2e {b}_{ij}} \nonumber \\
        &\tleq{ \text{(c)} } \Big( \frac{d}{d- kh} \Big)^{ kh } \cdot \Big( 1 + \frac{ 4 C  k^2 h^2 }{ \sqrt{d} } \Big)^{ \frac{2kh}{3} } \prod_{j=1, {b}_{ij} > 2}^{d} (2e {b}_{ij})^{2e {b}_{ij}} \nonumber \\
        &\tleq{ \text{} } \Big( \frac{d}{d- kh} \Big)^{ kh } \cdot \Big( 1 + \frac{ 4 C k^2 h^2 }{ \sqrt{d} } \Big)^{ \frac{2kh}{3} } (2e S_\cL)^{2e S_\cL} \, , \label{eq:EL_ubd0} 
    \end{align}
    where (a) is due to $d_i\geq d - kh$ and $\sum_{i=1}^{n}B_i = 2kh$, (b) follows from Eq.~\eqref{eq:unif_moments_bd} in Lemma \ref{lem:uniformdist_moments} and (c) follows from $B_i\leq 2kh$ and the bound: $\sum_{i=1}^{n} 3 \cdot \charfn_{B_i > 2} \leq 2kh$.  

    On the other hand, in Lemma \ref{lem:S_L_ubd} we obtain a bound for $S_\cL$ in terms of $\Delta(\cL)$:
    \begin{align}
    \label{eq:ScL_ubd}
        S_\cL \leq 12 \Delta(\cL).
    \end{align}

    Finally, substituting  Eq.~\eqref{eq:ScL_ubd} into Eq.~\eqref{eq:EL_ubd}, we can get Eq.~\eqref{eq:EL_ubd_0}.
\end{proof}

\begin{lemma}
\label{lem:uniformdist_moments}
    Let $\bx \sim \Unif(\S^{d-1} (\sqrt{d}))$ and $ \{ b_{j} \}_{j\in [d]} $ be a sequence of non-negative integers. Then we have: (I) 
    \begin{align}
    \label{eq:Exij_odd_zero}
        \E \Big( \prod_{j=1}^{d} x_{j}^{b_{j}} \Big) \geq 0,
    \end{align}
    with inequality achieved only if $\{b_{ij}\}$ are all even numbers,
    and (II) there exists $C>0$ such that for any $\{b_{j}\}_{j\in[d]}$,
    \begin{align}
    \label{eq:unif_moments_bd}
        \E \Big( \prod_{j=1}^{d} x_{j}^{b_{j}} \Big) \leq \Big( 1 + \frac{ C B^2}{ \sqrt{d} } \Big) \prod_{j=1}^{d} (2e b_j)^{2e b_j},
    \end{align}
    where $B= \sum_{j=1}^{d}b_j$ and we let $0^{0} = 1$.
\end{lemma}
\begin{proof}
    First, we show Eq.~\eqref{eq:Exij_odd_zero}. Suppose $b_{1}$ is an odd number. Since
    \begin{align*}
        (x_{1},x_{2},\cdots,x_{d}) \teq{\text{d} } (-x_{1},x_{2},\cdots,x_{d}),
    \end{align*}
    where $\teq{\text{d} } $ means equal in distribution, we have
    \begin{align}
        \E \Big( \prod_{j=1}^{d} x_{j}^{b_{j}} \Big) &= \E \Big[ (-x_{1})^{b_{1}} \prod_{j\neq 1} x_{j}^{b_{j}} \Big] \nonumber \\
        &= - \E \Big( \prod_{j=1}^{d} x_{j}^{b_{j}} \Big),
    \end{align}
    implying $\E \Big( \prod_{j=1}^{d} x_{j}^{b_{j}} \Big) = 0$. Therefore, $\E \Big( \prod_{j=1}^{d} x_{j}^{b_{j}} \Big) \neq 0$, only when $\{b_j\}_{j=1}^{d}$ are all even. When this is the case, $\E \Big( \prod_{j=1}^{d} x_{j}^{b_{j}} \Big) \geq 0$. 

    Next we show Eq.~\eqref{eq:unif_moments_bd}. This is obvious when $B=0$. Also when there exists an odd $b_j$, from part (I) we know $\E \Big( \prod_{j=1}^{d} x_{j}^{b_{j}} \Big) = 0$, so Eq.~\eqref{eq:unif_moments_bd} still holds. Hence, it remains to consider the case when $B\geq 2$ and $b_j$ are all even numbers.
    
    We know $\bx \sim \Unif(\S^{d-1} (\sqrt{d}))$ can be represented as: $\bx = \sqrt{d}\frac{\tbx}{\|\tbx\|_2}$, with $\tbx\sim\cN(\bzero,\id_d)$. Let us denote 
    $r = \frac{\sum_{j=1}^{d} \tx_{j}^{2}}{d}$ and $\Delta_r = r - 1$. Let $X\sim\cN(0,1)$. When $B\geq 2$ and $b_j$ are all even,
    \begin{align}
        \bigg| \E \Big( \prod_{j=1}^{d} x_{j}^{b_{j}} \Big) - \prod_{j=1}^{d} \E ( X^{b_{j}} ) \bigg|^2 =& \bigg| \E \big( { r^{ -\frac{B}{2} } } - 1 \big) \prod_{ \substack{j=1} }^{d} \tx_{j}^{b_{j}} \bigg|^2 \nonumber \\
        \tleq{\text{(a)} }& \bigg[ \E \frac{ \big| r^B - 1 \big| }{ r^B } \prod_{ \substack{j=1} }^{d} \tx_{j}^{b_{j}} \bigg]^2 \nonumber \\
        \tleq{ \text{(b)} }& \E \Big[ \frac{  B \Delta_r + \frac{B(B-1)}{2} \Tilde{r}^{B-2} \Delta_r^2 }{r^B} \Big]^2 \E \Big( \prod_{ \substack{j=1} }^{d} \tx_{j}^{2 b_{j}} \Big) \nonumber \\
        \tleq{ \text{(c)} }& 2 \E \Big[ \frac{ (B \Delta_r)^2 + B^4 {r}^{2 ( B-2 )} \Delta_r^4 }{r^{2B}} \Big] \E \Big( \prod_{ \substack{j=1} }^{d} \tx_{j}^{2 b_{j}} \Big) \nonumber \\
        \tleq{ \text{(d)} }& 2 \Big[ B^2 \Big(
 \E \Delta_r^4 \cdot \E \frac{1}{r^{4B}} \Big)^{\frac{1}{2} } + B^4 \Big(  \E \Delta_r^8 \cdot \E \frac{1}{r^{8}} \Big)^{\frac{1}{2} } \Big] \E \Big( \prod_{ \substack{j=1} }^{d} \tx_{j}^{2 b_{j}} \Big)\, ,
\label{eq:xij_moments_diff_1}
    \end{align}
    where in (a) we use $|\frac{1}{\sqrt{x} } - 1| \leq \frac{ |x-1| }{x} $, for $x > 0$, in (b) we use Cauchy-Schwartz inequality and Taylor's theorem, with $\tilde{r}\in[0,r]$, in step (c), we use $(a+b)^2\leq 2(a^2+b^2)$ and in step (d) we use Cauchy-Schwartz inequality again. We need to compute the moments appearing in Eq.~\eqref{eq:xij_moments_diff_1}. Since $\Delta_r$ is a sub-exponential random variable with $\text{Var}(\Delta_r)\asymp \frac{1}{d}$, there exists $C>0$ such that for any $k\in \mathbb{Z}_{>0}$, 
    \begin{align}
        \E(|\Delta_r|^k) \leq \Big( \frac{C k}{\sqrt{d} } \Big)^k.
        \label{eq:subsexp_moment_bd}
    \end{align}
    On the other hand, we know $\sum_{j=1}^{d}\tx_j^2 \sim \chi_d^2 \teq{ \text{d} } 2\text{Gamma}(d/2) $. Therefore, for $\frac{d}{2}\geq k+1$,
    \begin{align}
        \E \frac{1}{r^k} &= \Big(\frac{d}{2}\Big)^k \E_{X\sim \text{Gamma}(d/2)} X^{-1} \nonumber \\
        &= \Big(\frac{d}{2}\Big)^k \frac{ \Gamma(\frac{d}{2}-k) }{ \Gamma(\frac{d}{2}) } \nonumber\\
        &\leq 2 e^{ k }
        \label{eq:inverse_chisq_moments_bd}
    \end{align}
    Substituting Eqs.~\eqref{eq:subsexp_moment_bd} and \eqref{eq:inverse_chisq_moments_bd} into Eq.~\eqref{eq:xij_moments_diff_1}, we know there exists $C>0$ such that for any $\{b_{ij}\}$,
    \begin{align}
        \bigg| \E \Big( \prod_{j=1}^{d} x_{j}^{b_{j}} \Big) - \prod_{j=1}^{d} \E ( X^{b_{j}} ) \bigg|^2 &\leq \frac{CB^4e^{2B}}{d} \prod_{j=1}^{d} \E ( X^{2b_{j}} )
        = \frac{CB^4}{d} \prod_{j=1}^{d} (2e b_j)^{2b_j},
        \label{eq:unif_moments_bd_1}
    \end{align}
    where we use $\E X^{k} \leq k^k$ for $X\sim\cN(0,1)$ and $k\in \Z_{>0}$. Clearly, Eq.~\eqref{eq:unif_moments_bd_1} implies Eq.~\eqref{eq:unif_moments_bd}. 
\end{proof}

\clearpage

\section{Proof of Training Error and $\ell_2$ Norm}\label{app:train_err}
In this section, we prove the asymptotic formula for the training error and the normalized squared $\ell_2$ norm of $\hba_{\lambda}$ in Theorem \ref{thm:main_theorem_RF}. Many of the steps are very similar to the proof for the test error and we omit them for the sake of brevity.

\subsection{Training error}
Recall that the training error is defined as:
\begin{align}
    R_{\train} (f_* ; \bX , \bW , \beps, \lambda) =&~ \frac{1}{n} \sum_{i \in [n]} (y_i - h_{\RF} ( \bx_i ; \hat \ba_{\lambda}  ) )^2\, ,
\end{align}
where
\begin{align*}
    h_{\RF} ( \bx ; \ba  ) = \frac{1}{\sqrt{p}} \sum_{j \in [p]} a_j \sigma ( \< \bx , \bw_j \> ) = \bZ \ba \, ,
\end{align*}
and
\begin{align*}
    \hat \ba_{ \lambda} =& \argmin_{\ba \in \R^p} \Big\{ \sum_{i \in [n]} \big( y_i - h_{\RF} ( \bx_i ; \ba  ) \big)^2 + \lambda \| \ba \|_2^2 \Big\} \\
    =& \bR \bZ^\sT \by\, .
\end{align*}
Plugging this formula in the training error yields
\begin{align}
    R_{\train} (f_* ; \bX , \bW , \beps, \lambda) =&~ \frac{1}{n} ( \|\by\|^2 - \by^\sT \bZ \bR \bZ^\sT \by - \lambda \by^\sT \bZ \bR^2 \bZ^\sT \by) \, .
\end{align}
Similar as Proposition \ref{prop:concentration_on_expectation}, we can show
\begin{equation}
    \E_{\bX,\bW,\beps,f_*} \Big[ \Big\vert R_{\train} ( f_{*,d} ; \bX , \bW , \beps, \lambda) - \E_{\beps, f_*} [ R_{\train} ( f_{*,d} ; \bX , \bW , \beps, \lambda)] \Big\vert \Big] = o_d (1) \, .
\end{equation}
Denote $\overline{R}_{\train} := \E_{\beps, f_*} [ R_{\train} ( f_{*,d} ; \bX , \bW , \beps, \lambda)]$ and recall that 
\begin{align*}
\label{eq:y_decomp_train}
    \by 
    =&  \underbrace{ \sum_{k<\ell} \bPsi_{k}(\bX) \bbeta_{d,k}^{*} }_{ \bf_{<\ell} }
    + \underbrace{ \sum_{k\geq\ell} \bPsi_{k}(\bX) \bar{\bbeta}_{d,k} }_{ \bf_{\geq \ell} } 
    + \beps \, .
\end{align*}
Then we can get
\begin{equation}
    \overline{R}_{\train} = \cB_{\train,<\ell} + \cB_{\train,\geq \ell} + \cV_{\train}\, ,
    \label{eq:R_train_expectation} 
\end{equation}
where we defined
\begin{align}
    \cB_{\train,<\ell} &= \frac{1}{n} ( \|\bf_{<\ell}\|^2 -  \bf_{<\ell}^{\sT} \bZ \bR \bZ^\sT \bf_{<\ell} - \lambda \bf_{<\ell}^\sT \bZ \bR^2 \bZ^\sT \bf_{<\ell} )\, , \label{eq:B_train_lowdegree} \\
    \cB_{\train,\geq \ell} &=  \sum_{k\geq\ell} F_k^2 - \frac{1}{n}  \Tr( \bH_{\bF} \bZ \bR \bZ^\sT ) - \frac{\lambda}{n} \Tr( \bH_{\bF} \bZ \bR^2 \bZ^\sT ) \, , \label{eq:B_train_highdegree} \\
    \cV_{\train} &= \rho_{\varepsilon}^2 \big[ 1 - \frac{1}{n}\Tr(\bZ \bR \bZ^\sT) - \frac{\lambda}{n} \Tr(\bZ \bR^2 \bZ^\sT) \big] \, .\label{eq:V_train}
\end{align}
and $\bH_{\bF} = \sum_{k\geq\ell} F_k^2 \bQ_k^{\bX}$.

We can show that 
\begin{align*}
    \E \left\| \frac{1}{n} \bPsi_{<\ell}^{\sT} \bZ \bR \bZ^\sT \bPsi_{<\ell} - \id_{ N_{<\ell} }\right\|_\op = o_d(1)\, , \\
    \E \left\| \frac{1}{n} \bPsi_{<\ell}^{\sT} \bZ \bR^2 \bZ^\sT \bPsi_{<\ell} \right\|_\op = o_d(1)\, ,
\end{align*}
so that from Eq.~\eqref{eq:B_train_lowdegree} we get
\begin{align}
\label{eq:B_train_lowdegree_vanish}
    \E | \cB_{\train,<\ell} | = o_d(1)\, .
\end{align}
Hence, from Eqs.~\eqref{eq:B_train_lowdegree}, \eqref{eq:B_train_highdegree} and \eqref{eq:V_train} we know it remains to compute the limits of the following quantities:
\begin{align}
    \chi_1 :=& \frac{1}{n} \Tr( \bH_{\bF} \bZ \bR \bZ^\sT ) \, ,\\
    \chi_2 :=& \frac{1}{n} \Tr( \bH_{\bF} \bZ \bR^2 \bZ^\sT )\, , \\
    \chi_3 :=& \frac{1}{n} \Tr(\bZ \bR \bZ^\sT)\, , \\
    \chi_4 :=& \frac{1}{n} \Tr(\bZ \bR^2 \bZ^\sT)\, .
\end{align}
Similarly to Proposition \ref{prop:high_degree_concentration}, we can prove the following proposition:
\begin{proposition}
    Under the assumptions of Theorem \ref{thm:main_theorem_RF}, we have:
\begin{itemize}
    \item[\emph{(I)}] If $\kappa_1 > \kappa_2$, then
    \begin{align}
        \label{eq:over_chi1_concentrate_train}
        \E \Big| \chi_1 - \big(\sum_{k\geq\ell}F_k^2 - \frac{\lambda}{n} \Tr (\bH_{\geq\ell}^{\bX} \bG_{\geq \ell}^{\bX}) \big) \Big| =&~  o_d(1) \, ,\\ \label{eq:over_chi2_concentrate_train}
        \E \Big| \chi_2 - \frac{1}{n} \Tr  \big(\bH_{\geq\ell}^{\bX} \bG_{\geq\ell}^{\bX} - \lambda \bH_{\geq\ell}^{\bX} (\bG_{\geq\ell}^{\bX})^2 \big) \Big| = &~ o_d(1)\, , \\
        \E \Big| \chi_3 - \big[ 1 - \frac{\lambda}{n} \Tr (\bG_{\geq\ell}^{\bX}) \big] \Big| = &~ o_d(1)\, , \\
        \label{eq:over_chi4_concentrate_train}
        \E \Big| \chi_4 - \frac{1}{n}\Tr\big( \bG_{\geq\ell}^{\bX} - \lambda ( \bG_{\geq\ell}^{\bX} )^2 \big) \Big| = &~ o_d(1) \, ,
    \end{align}
    with $\bG_{\geq \ell}^{\bX} = \big( \mu_\ell^2 \bQ_{\ell}^{\bX} + ( \mu_{>\ell}^2 + \lambda) \id_n \big)^{-1}$ and $\bH_{\geq \ell}^{\bX} = F_{\ell}^2 \bQ_{\ell}^{\bX} + F_{>\ell}^2 \id_n$.
    
    \item[\emph{(II)}] If $\kappa_1 < \kappa_2$, then
    \begin{align}
        \label{eq:under_chi1_concentrate_train}
        \E \Big| \chi_1 - \frac{F_\ell^2}{ N_\ell } \Tr( \mu_{\ell}^2  \bQ_{\ell}^\bW \cdot \bG_{\geq\ell}^{\bW} ) \Big| =&~  o_d(1),\\
        \E | \chi_2 | =&~ o_d(1)\, ,\\
        \E | \chi_3 | =&~ o_d(1)\, ,\\
        \label{eq:under_chi4_concentrate_train}
        \E | \chi_4 | =&~ o_d(1)\, ,
    \end{align}
    with $\bG_{\geq \ell}^{\bW} = ( \mu_\ell^2 \bQ_{\ell}^{\bW} + \mu_{>\ell}^2 \id_p )^{-1}$.

    \item[\emph{(III)}] If $\kappa_1 = \kappa_2$, then
    \begin{align}
        \E \Big| \chi_1 - \frac{1}{n} \Tr (\bH_{\geq\ell}^{\bX} \bZ \bR \bZ^\sT) \Big| ~=& o_d(1)\, , \label{eq:critical_chi1_concentrate_train} \\
        \E \Big| \chi_2 - \frac{1}{n} \Tr (\bH_{\geq\ell}^{\bX} \bZ \bR^2 \bZ^\sT) \Big| ~=& o_d(1)\, . \label{eq:critical_chi2_concentrate_train}
    \end{align}
\end{itemize}
\end{proposition}

Using this proposition, we analyze the three regimes separately.

\noindent\textbf{(I) Overparametrized regime $\kappa_1 > \kappa_2$.}

From Eqs.~\eqref{eq:R_train_expectation}, \eqref{eq:B_train_lowdegree_vanish}, \eqref{eq:over_chi1_concentrate_train} and \eqref{eq:over_chi4_concentrate_train}, we get
\begin{align}
    \E \Big| \overline{R}_{\train} - \Big[ \frac{\lambda^2}{n} \Tr \big( \bH_{\geq\ell}^{\bX} (\bG_{\geq\ell}^{\bX})^2 \big) + \frac{\lambda^2\rho_{\varepsilon}^{2}}{n} \Tr(\bG_{\geq\ell}^{\bX})^2 \Big] \Big| = o_d(1)\, .
\end{align}
Then we can follow the same steps in \cite{misiakiewicz2022spectrum} to obtain the desired results (see B.4 in \cite{misiakiewicz2022spectrum}).

\noindent\textbf{(II) Underparametrized regime $\kappa_1 < \kappa_2$.}
In the underparametrized case, the training error are asymptotically the same as the test error, up to a difference $\rho_{\epsilon}^2$.
The proof is also the same as Eq.~\eqref{eq:asymp_testerr} and we omit the details here.

\noindent\textbf{(III) Critical regime $\kappa_1 = \kappa_2$.}
Substituting Eqs.~\eqref{eq:B_train_lowdegree_vanish}, \eqref{eq:critical_chi1_concentrate_train} and \eqref{eq:critical_chi2_concentrate_train} into Eq.~\eqref{eq:R_train_expectation}, we get
\begin{align}
    \overline{R}_{\train} =& (F_{>\ell}^2 + \rho_{\varepsilon}^2) \big[ 1 - \frac{1}{n}\Tr(\bZ \bR \bZ^\sT) \big] + F_{\ell}^2 \big[ 1 - \frac{1}{n}\Tr(\bQ_{\ell}^{\bX} \bZ \bR \bZ^\sT) \big] \nonumber \\
    & -\lambda \big[ \frac{F_{\ell}^2}{n} \Tr (\bQ_{\ell}^{\bX} \bZ \bR^2 \bZ^\sT) + \frac{F_{>\ell}^2 + \rho_{\varepsilon}^2}{n} \Tr(\bZ \bR^2 \bZ^\sT)  \big] + \Delta \nonumber\\
    =& \lambda F_{\ell}^2 \cdot \big[\frac{1}{n} \Tr(\bQ_{\ell}^{\bX} \bPi) - \frac{1}{n} \Tr (\bQ_{\ell}^{\bX} \bZ \bR^2 \bZ^\sT) \big] \nonumber \\
    &+ \lambda (F_{>\ell}^2 + \rho_{\varepsilon}^2) \cdot \big[ \frac{1}{n} \Tr (\bPi) - \frac{1}{n} \Tr(\bZ \bR^2 \bZ^\sT) \big] + \Delta \, ,
    \label{eq:Rtrain_bar_approx}
\end{align}
where $\bPi = (\lambda \id_n + \bZ \bZ^\sT)^{-1}$ and $\Delta$ satisfies $\E|\Delta| = o_d(1)$. It can be checked from Eq.~\eqref{eq:Gd_def} that:
\begin{align}   
    \label{eq:partial_G_6}
    \partial_{t_1} G_d( i\sqrt{\theta_1 \theta^{-1} \lambda}; \bzero ) &= i  \sqrt{\theta_1 \theta^{-1} \lambda} \cdot \frac{1}{m} \Tr ( \Tilde{\bPi} )\, , \\ 
    \label{eq:partial_G_7}
    \partial_{t_2} G_d( i\sqrt{\theta_1 \theta^{-1} \lambda}; \bzero ) &= i \sqrt{\theta_1 \theta^{-1} \lambda} \cdot \frac{1}{m}\Tr ( \bQ_{\ell}^{\bX} \Tilde{ \bPi } ) \, ,
\end{align}    
where $\Tilde{ \bPi } = \big( \theta_1 \theta^{-1} \lambda \id_n + \tbZ \tbZ^\sT \big)^{-1}$. Substituting Eqs.~\eqref{eq:partial_G_2}, \eqref{eq:partial_G_3}, \eqref{eq:partial_G_6} and \eqref{eq:partial_G_7} into Eq.~\eqref{eq:Rtrain_bar_approx}, we obtain
\begin{align}
\label{eq:E_Rtrain_concentrate}
    \E \big| \overline{R}_{\train} - [F_\ell^2 \Gamma_4  + (F_{>\ell}^2 + \rho_{\varepsilon}^2) \Gamma_5] \big| = o_d(1)\, ,
\end{align}
where
\begin{align}   
    \label{eq:Gamma_4_def}
    \Gamma_4 &= -i \sqrt{ \frac{ \theta_1 \theta \lambda}{ \theta_2^2 } } \partial_{t_2} G_d( i\sqrt{\theta_1 \theta^{-1} \lambda}; \bzero ) + \frac{ \theta_1 \lambda }{ \theta_2 } \partial_{s_1,t_2} G_d( i\sqrt{\theta_1 \theta^{-1} \lambda}; \bzero ) \, ,\\
    \label{eq:Gamma_5_def}
    \Gamma_5 &= -i \sqrt{ \frac{ \theta_1 \theta \lambda}{ \theta_2^2 } } \partial_{t_1} G_d( i\sqrt{\theta_1 \theta^{-1} \lambda}; \bzero ) + \frac{ \theta_1 \lambda }{ \theta_2 } \partial_{s_1,t_1} G_d( i\sqrt{\theta_1 \theta^{-1} \lambda}; \bzero )\, .
\end{align}
Taking derivative on both sides of \eqref{eq:calG_z_def}, we can obtain:
\begin{align}
    \label{eq:partial_g_t1}
    \partial_{t_1} g( i\sqrt{\theta_1 \theta^{-1} \lambda}; \bzero ) &= m_2(i\sqrt{\theta_1 \theta^{-1} \lambda}; \bzero) \, ,\\
    \label{eq:partial_g_t2}
    \partial_{t_2} g( i\sqrt{\theta_1 \theta^{-1} \lambda}; \bzero ) &= \frac{ m_2(i\sqrt{\theta_1 \theta^{-1} \lambda}; \bzero) }{1 - \psi \mu_{\ell}^2 m_1(i\sqrt{\theta_1 \theta^{-1} \lambda}; \bzero) \cdot m_2(i\sqrt{\theta_1 \theta^{-1} \lambda}; \bzero)}\, .
\end{align}
Combining Eqs.~\eqref{eq:partial_g_t1}, \eqref{eq:partial_g_t2} and \eqref{eq:g_partial_deri} with Eqs.~\eqref{eq:G_1stDeri_conv} and \eqref{eq:G_2ndDeri_conv} in Proposition \ref{prop:derivative_converge}, we can compute the limit value of the partial derivatives on the right-hand side of Eqs.~\eqref{eq:Gamma_4_def} and \eqref{eq:Gamma_5_def} and get:
\begin{align}
    \label{eq:E_Rtrain_concentrate_1}
    \E \Big| \Gamma_4 - \underbrace{ \Big[ \cL_1 (\zeta , \theta_1 ,\theta_2 , \psi , \barlambda) - \frac{\theta_1\barlambda}{\theta_2} \cdot \frac{\cA_1 (\zeta , \theta_1 ,\theta_2 , \psi , \barlambda)}{\cA_0 (\zeta , \theta_1 ,\theta_2 , \psi , \barlambda)} \Big] }_{ := \cB_{\train} (\zeta , \theta_1 ,\theta_2 , \psi , \barlambda )} \Big| &= o_d(1)\, , \\
    \label{eq:E_Rtrain_concentrate_2}
    \E \Big| \Gamma_5 -  \underbrace{  \Big[ \cL_2 (\zeta , \theta_1 ,\theta_2 , \psi , \barlambda) - \frac{\theta_1\barlambda}{\theta_2} \cdot \frac{\cA_2 (\zeta , \theta_1 ,\theta_2 , \psi , \barlambda)}{\cA_0 (\zeta , \theta_1 ,\theta_2 , \psi , \barlambda)} \Big] }_{ := \cV_{\train} (\zeta , \theta_1 ,\theta_2 , \psi , \barlambda ) } \Big| &= o_d(1)\, ,
\end{align}
where
\begin{equation}
    \begin{aligned}
          \cL_1 (\zeta , \theta_1 ,\theta_2 , \psi , \barlambda) :=&~ -i \sqrt{ \frac{ \theta_1 \theta \bar\lambda}{ \theta_2^2 } } \cdot \frac{ \nu_2\big( i (\theta_1 \theta^{-1} \bar\lambda)^{1/2} \big) }{ 1 - \psi \chi \zeta^2} \, , \\
           \cL_2 (\zeta , \theta_1 ,\theta_2 , \psi , \barlambda) :=&~ -i \sqrt{ \frac{ \theta_1 \theta \bar\lambda}{ \theta_2^2 } } \cdot \nu_2\big( i (\theta_1 \theta^{-1} \bar\lambda)^{1/2} \big) \, ,\\
           \cA_1 (\zeta , \theta_1 ,\theta_2 , \psi , \barlambda) :=&~ - \chi^2( \psi \chi \zeta^4 - \psi \chi \zeta^2 + \psi_2 \zeta^2 + \zeta^2 - \psi_2 \psi \chi  \zeta^4 +1 ) \, ,\\
           \cA_2 (\zeta , \theta_1 ,\theta_2 , \psi , \barlambda) := &~ \chi^2( \psi \chi \zeta^2 - 1) ( \psi^2 \chi^2 \zeta^4 - 2 \psi \chi \zeta^2 + \zeta^2 + 1 ) \, ,\\
           \cA_0 (\zeta , \theta_1 ,\theta_2 , \psi , \barlambda) := &~ - \psi^3 \chi^5 \zeta^6 + 3 \psi^2 \chi^4 \zeta^4 + (\psi_1 \psi_2 - \psi_2 - \psi_1 +1) \psi \chi^3 \zeta^6 - 2 \psi \chi^3 \zeta^4 - 3 \psi \chi^3 \zeta^2 \\
           &~ + (\psi_1 + \psi_2 - 3 \psi_1 \psi_2 +1) \chi^2 \zeta^4 + 2 \chi^2 \zeta^2 + \chi^2 + 3 \frac{\theta_1\theta_2}{\theta^2} \psi \chi \zeta^2 - \frac{\theta_1\theta_2}{\theta^2} \, .
     \end{aligned}
 \end{equation}
 Combining Eqs.~\eqref{eq:E_Rtrain_concentrate}, \eqref{eq:E_Rtrain_concentrate_1} and \eqref{eq:E_Rtrain_concentrate_2}, we deduce that
 \begin{align}
     \E \big| \overline{R}_{\train} - [F_\ell^2 \cdot \cB_{\train} (\zeta , \theta_1 ,\theta_2 , \psi , \barlambda )  + (F_{>\ell}^2 + \rho_{\varepsilon}^2) \cdot \cV_{\train} (\zeta , \theta_1 ,\theta_2 , \psi , \barlambda )] \big| = o_d(1) \, .
\end{align}
Finally, applying Lemma \ref{lem:equi_fixedpoints} in Appendix \ref{sec:explicit_formula}, we get Eq.~\eqref{eq:asymp_trainerr}.

\subsection{Squared $\ell_2$ norm of $\hba_\lambda$}
In what follows, we derive the asymptotic formulas [c.f. Eq. \eqref{eq:asymp_anorm} in Theorem \ref{thm:main_theorem_RF}] for the (normalized) squared $\ell_2$ norm of the optimizer in Eq. \eqref{eq:def_RFRR}. The proof is similar to that of the training/test errors in the previous sections. 

Note that
$\|\hat \ba_{ \lambda}\|^2 = \by^\sT \bZ \bR^2 \bZ^\sT \by.$
Similar to Proposition \ref{prop:concentration_on_expectation},
\begin{align}
    \E \Big|
    \frac{\| \hat \ba_{ \lambda} \|^2 }{p} - \underbrace{ \frac{1}{p} \Big[ \bf_{<\ell}^\sT \bZ \bR^2 \bZ^\sT \bf_{<\ell} + \Tr( \bH_{\bF} \bZ \bR^2 \bZ^\sT ) + \rho_{\varepsilon}^2 \Tr(\bZ \bR^2 \bZ^\sT) \Big] }_{ \E_{\beps, f_*}(\frac{1}{p}\| \hat \ba_{ \lambda} \|^2) } \Big| = o_d(1)\, .
\end{align}
Recall that $\E\Big|\frac{1}{p}\bf_{<\ell}^\sT \bZ \bR^2 \bZ^\sT \bf_{<\ell}\Big| = o_d(1)$, so
\begin{align}
\label{eq:asqnorm_approx1}
    \E \Big|
    \frac{\|\hat \ba_{ \lambda}\|^2 }{p} -  \frac{1}{p} \Big[\Tr( \bH_{\bF} \bZ \bR^2 \bZ^\sT ) + \rho_{\varepsilon}^2 \Tr(\bZ \bR^2 \bZ^\sT) \Big]  \Big| = o_d(1)\, .
\end{align}
Analogously, we can also show that
\begin{align}
\label{eq:asqnorm_approx11}
    \E \Big|
    \frac{\| \hat \ba_{ \lambda} \|^2 }{n} -  \frac{1}{n} \Big[\Tr( \bH_{\bF} \bZ \bR^2 \bZ^\sT ) + \rho_{\varepsilon}^2 \Tr(\bZ \bR^2 \bZ^\sT) \Big]  \Big| = o_d(1)\, .
\end{align}
Then we discuss over the different regimes separately.

\noindent\textbf{(I) Overparametrized regime $\kappa_1 > \kappa_2$.}

In this case, we can use Eqs.~\eqref{eq:over_chi2_concentrate_train} and \eqref{eq:over_chi4_concentrate_train} and substitute them into Eq.~\eqref{eq:asqnorm_approx11} to get:
\begin{align}
\label{eq:asqnorm_approx2}
    \E \Big|
    \frac{\| \hat \ba_{ \lambda} \|^2 }{n} -  \frac{1}{n} \Tr  \big[\bH_{\geq\ell}^{\bX} \bG_{\geq\ell}^{\bX} - \lambda \bH_{\geq\ell}^{\bX} (\bG_{\geq\ell}^{\bX})^2 \big]  - \frac{ \rho_{\varepsilon}^2 }{n}\Tr\big[ \bG_{\geq\ell}^{\bX} - \lambda ( \bG_{\geq\ell}^{\bX} )^2 \big] \Big| = o_d(1)\, .
\end{align}
Then similar as the proof of Theorem 3 in \cite{misiakiewicz2022spectrum}, one can show:
\begin{align}
    \E \Big| \frac{1}{n} \Tr  \big[\bH_{\geq\ell}^{\bX} \bG_{\geq\ell}^{\bX} - \lambda \bH_{\geq\ell}^{\bX} (\bG_{\geq\ell}^{\bX})^2 \big] - \frac{ F_{\ell}^2 \cB_1 + F_{>\ell}^2 \cB_2 }{\mu_{>\ell}^2}
     \Big| = o_d(1).
\end{align}
and
\begin{align}
    \E \Big| \frac{\rho_{\varepsilon}^2}{n} \Tr  \big[ \bG_{\geq\ell}^{\bX} - \lambda  (\bG_{\geq\ell}^{\bX})^2 \big] - \frac{ \rho_\varepsilon^2 \cB_2 }{\mu_{>\ell}^2}
     \Big| = o_d(1).
\end{align}
where
\begin{align}
    \cB_1 &= \frac{1 - \eta\vartheta}{\zeta^2} - \frac{\barlambda \cB_{\test} \vartheta^2}{\zeta^4} \\
    \cB_2 &= \frac{\vartheta}{\zeta^2} - \frac{\barlambda (\cV_{\test} + 1)\vartheta^2}{\zeta^4}
\end{align}
and $\eta$, $\vartheta$, $\cB_{\test}$ and $\cV_{\test}$ are the same as defined in Eqs. \eqref{eq:defchi} and \eqref{eq:B_V_test_2}. 
As a result,
\begin{align}
    \E \Big|
    \frac{\| \hat \ba_{ \lambda} \|^2 }{n} - \frac{ F_{\ell}^2 \cB_1 + (F_{>\ell}^2 + \rho_\varepsilon^2) \cB_2 }{\mu_{>\ell}^2}  \Big| = o_d(1)\, .
\end{align}

\noindent\textbf{(II) Underparametrized regime $\kappa_1 < \kappa_2$.}

Similar as the proof of Eq. \eqref{eq:under_chi2_concentrate}, we can get:
\begin{align}
\label{eq:asqnorm_approx1_0}
    \E \Big| \frac{1}{p} \Tr( \bH_{\bF} \bZ \bR^2 \bZ^\sT ) - \frac{\psi_1 F_\ell^2}{\mu_{>\ell}^2} \Big[ \frac{1}{p} \Tr( \zeta^2 \bQ_\ell^\bW + \id_p  )^{-1} - \frac{1}{p} \Tr(\zeta^2 \bQ_\ell^\bW + \id_p)^{-2} \Big] \Big| = o_d(1).
\end{align}
Denote the Stieltjes transform of $\zeta^2 \bQ_\ell^\bW$ as:
\begin{align}
\label{eq:Rdz}
    R_d(z) := \frac{1}{p} \Tr( \zeta^2 \bQ_\ell^\bW - z \id_p )^{-1}
\end{align}
and its first-order derivative is:
\begin{align}
\label{eq:Rdz_1stderi}
    R_d'(z) := \frac{1}{p} \Tr( \zeta^2 \bQ_\ell^\bW - z \id_p )^{-2}
\end{align}
which equals to the third trace term in Eq. \eqref{eq:asqnorm_approx1_0}, when $z=-1$. By Theorem 1 in \cite{lu2022equivalence}, we can show for any $z\in\cA := \{z \mid \Re(z)<0 \text{ or } \Im(z)>0\}$, it holds that $\Delta_d(z) := R_d(z) - R(z) = O_{\prec}(d^{-1/2})$, where 
\begin{align}
    R(z) = \frac{ -(-z+\zeta^2-\zeta^2\psi_1) + \sqrt{ ( -z+\zeta^2-\zeta^2\psi_1 )^2 - 4z\zeta^2\psi_1 } }{-2z\zeta^2\psi_1}
\end{align}
Besides, since $\Delta_d(z)$ is Lipschitz continuous in any compact subset $\cB$ of $\cA$, we also have the uniform concentration: $\sup_{z\in\cB}\Delta_d(z) = O_{\prec}(d^{-1/2})$. Notice that $\Delta_d(z)$ is analytic in $\cA$, so by Cauchy's integral formula, $ \Delta_d(z) = \frac{1}{2\pi i} \oint_{\cC} \frac{ \Delta_d(s) }{s-z} ds,$
where $\cC$ is a closed contour in $\cA$, surrounding $z$. Hence,
$\Delta_d'(z) = \frac{1}{2\pi i} \oint_{\cC} \frac{ \Delta_d(s) }{ (s-z)^2 } ds.$
and after applying $\sup_{z\in\cB}\Delta_d(z) = O_{\prec}(d^{-1/2})$, we can obtain that: for any $z\in\cA$, $\Delta_d'(z) = R_d'(z) - R'(z) = O_{\prec}(d^{-1/2})$. As a result, letting $z=-1$ in Eq. \eqref{eq:Rdz} and \eqref{eq:Rdz_1stderi} and substituing them, together with the concentration of $\Delta_d(-1)$ and $\Delta_d'(-1)$ into Eq. \eqref{eq:asqnorm_approx1_0}, we get
\begin{align}
    \E \Big| \frac{1}{p} \Tr( \bH_{\bF} \bZ \bR^2 \bZ^\sT ) - \frac{\psi_1 F_\ell^2}{\mu_{>\ell}^2} [R(-1)-R'(-1)] \Big| = o_d(1).
\end{align}
After some simplifications, we have:
\begin{align}
\label{eq:asqnorm_approx1_1}
    \E \Big| \frac{1}{p} \Tr( \bH_{\bF} \bZ \bR^2 \bZ^\sT ) - \frac{\psi_1 F_\ell^2}{\mu_{>\ell}^2} \cdot \frac{ \vartheta \zeta^2 (1-\psi_1) + \vartheta^2 \zeta^2 \psi_1 }{ 1 + \zeta^2(1-\psi_1+2\psi_1\vartheta) } \Big| = o_d(1).
\end{align}
where
\begin{align}
    \vartheta = R(-1) = \frac{ -(1+\zeta^2-\zeta^2\psi_1) + \sqrt{ ( 1+\zeta^2-\zeta^2\psi_1 )^2 + 4\zeta^2\psi_1 } }{2\zeta^2\psi_1}.
\end{align}
On the other hand, since $\| \bZ \bR^2 \bZ^\sT \|_\op = \cO_d(1)$, we have $\frac{1}{p} \Tr(\bZ \bR^2 \bZ) = o_d(1)$. 
Substituting Eq. \eqref{eq:asqnorm_approx1_1} and $\frac{1}{p} \Tr(\bZ \bR^2 \bZ) = o_d(1)$ into Eq. \eqref{eq:asqnorm_approx1}, we reach at:
\begin{align}
    \E \Big|
    \frac{\| \hat \ba_{ \lambda} \|^2 }{p} -  \frac{\psi_1 F_\ell^2}{\mu_{>\ell}^2} \cdot \frac{ \vartheta \zeta^2 (1-\psi_1) + \vartheta^2 \zeta^2 \psi_1 }{ 1 + \zeta^2(1-\psi_1+2\psi_1\vartheta) }  \Big| = o_d(1)\, .
\end{align}

\noindent\textbf{(III) Critical regime $\kappa_1 = \kappa_2$.}

From Eqs.~\eqref{eq:partial_G_2} and \eqref{eq:partial_G_3}, we have
\begin{align}
    \frac{1}{p}\Tr(\bZ \bR^2 \bZ^\sT ) = - \partial_{s_1,t_1} G_d( i\sqrt{\gratio_1 \gratio^{-1} \lambda}; \bzero )\, ,\,\,\,\,\, 
    \frac{1}{p}\Tr( \bH_{\bF} \bZ \bR^2 \bZ^\sT ) = -\partial_{s_1,t_2} G_d( i\sqrt{\theta_1 \theta^{-1} \lambda}; \bzero )\, .
\end{align}
Applying Proposition \ref{prop:derivative_converge} in Eq.~\eqref{eq:asqnorm_approx11}, we can get
\begin{align}
    \E \Big|
    \frac{\| \hat \ba_{ \lambda} \|^2 }{n} - \frac{\theta_1}{\mu_{>\ell}^2 \theta_2} \frac{F_{\ell}^2 \cA_1 (\zeta , \theta_1 ,\theta_2 , \psi , \barlambda) + (F_{>\ell}^2 + \rho_{\varepsilon}^2) \cA_2 (\zeta , \theta_1 ,\theta_2 , \psi , \barlambda)}{\cA_0 (\zeta , \theta_1 ,\theta_2 , \psi , \barlambda)}  \Big| = o_d(1) \, .
\end{align}
Then using the correspondence between $(\nu_1,\nu_2)$ and $(\tau_1,\tau_2)$ given in Lemma \ref{lem:equi_fixedpoints}, we can show that 
\begin{align}
    \E \Big|
    \frac{\| \hat \ba_{ \lambda} \|^2 }{n} - \frac{1}{\mu_{>\ell}^2} [F_{\ell}^2 ( \tau_2(\barlambda) + \barlambda \tau_2'(\barlambda)  ) + (F_{>\ell}^2 + \rho_{\varepsilon}^2)( \tau_1(\barlambda) + \barlambda \tau_1'(\barlambda) ) ]  \Big| = o_d(1)\, .
\end{align}
In the special case when $\kappa_1 = \kappa_2 < \ell$, we have $\tau_1(z) = \tau_2(z)$ and thus
\begin{align}
    \E \Big|
    \frac{\| \hat \ba_{ \lambda} \|^2 }{n} - \frac{ (F_{\geq\ell}^2 + \rho_{\varepsilon}^2)( \tau_1(\barlambda) + \barlambda \tau_1'(\barlambda) }{\mu_{>\ell}^2} \Big| = o_d(1)\, .
\end{align}
where $\tau_1(z)$ is the same as the one given in Eq. \eqref{eq:FixedPoints_1}.

\clearpage

\section{The Kernel and Approximation Limits of RFRR}
\label{sec:kernel_approx}

In this section, we recall the asymptotic train and test errors in the kernel ridge regression (KRR) limit $p = \infty$, obtained in \cite{misiakiewicz2022spectrum,hu2022sharp,xiao2022precise}. We obtain analogously the asymptotics for the approximation limit $n = \infty$.

\subsection{The kernel ridge regression limit $p = \infty$}

As originally introduced in \cite{rahimi2008random}, the random feature model can be seen as a random approximation to a kernel which satisfies
\begin{equation}\label{eq:lim_Kernel}
\lim_{p \to \infty} \frac{1}{p} \sum_{j \in [p]} \sigma ( \< \bx_1 , \bw_j \>) \sigma ( \< \bx_2 , \bw_j \>) = K  ( \< \bx_1 , \bx_2 \>) := \E_{\bw} [ \sigma ( \< \bx_1 , \bw \> ) \sigma ( \< \bx_2 , \bw \>)] \, ,
\end{equation}
by law of large number. For fixed $n,d$ and taking $p \to \infty$, the RFRR solution $h ( \bx ; \hat \ba_{\lambda})$ converges to a function $\hat h_{\KRR,\lambda} (\bx)$ solution of the following \textit{kernel ridge regression problem}:
\[
\hat h_{\KRR , \lambda} = \argmin_h \Big\{ \sum_{i\in [n]} (y_i - h (\bx_i) )^2 + \lambda \| h \|_{\cH}^2 \Big\} \, ,
\]
where $\| \cdot \|_{\cH}$ is the RKHS norm associated to the kernel \eqref{eq:lim_Kernel}. By the representer theorem, this function can be represented explicitly by
\[
\begin{aligned}
\hat h_{\KRR , \lambda}  (\bx) = &~ \sum_{i \in [n]} \hat u_j K ( \<\bx , \bx_j \>) \, , \\
\hat \bu :=&~ \argmin_{\bu \in \R^n} \Big\{ \sum_{i \in [n]} \big(y_i -  \< \bu , \bk (\bx_i) \> \big)^2 + \lambda \bu^\sT \bK \bu \Big\} \, , 
\end{aligned}
\]
where $\bk (\bx) = ( K( \< \bx , \bx_1 \>) , \ldots , K ( \<\bx , \bx_n \>) ) \in \R^n$ and $\bK = ( K ( \< \bx_i , \bx_j \>) )_{ij \in [n]} \in \R^{n \times n}$. 

We denote the test/training errors of KRR by
\begin{align}
    R_{\KRR} (f_*; \bX, \beps, \lambda) =&~ \E_{\bx} \Big[ \big(f_* (\bx) -  \hat{h}_{\KRR, \lambda} ( \bx   ) \big)^2 \Big]\, \\
    R_{\KRR,\train} (f_*; \bX, \beps, \lambda) =&~ \frac{1}{n} \sum_{i \in [n]} (y_i - \hat{h}_{\KRR, \lambda} ( \bx_i ) )^2
\end{align}
and the normalized squared RKHS norm by
\begin{align}
    L_{\KRR,\lnorm} (f_*; \bX, \beps, \lambda) =&~ \frac{1}{n} \| h \|_{\cH}^2 = \frac{1}{n} {\hat \bu}^\sT \bK {\hat \bu}
\end{align}
The limiting learning behavior of kernel ridge regression in the polynomial scaling regime was studied in \cite{ghorbani2021linearized,misiakiewicz2022spectrum,hu2022sharp,xiao2022precise} and we copy below their results.

\begin{theorem}[KRR in the polynomial scaling \cite{misiakiewicz2022spectrum,hu2022sharp,xiao2022precise}] \label{thm:KRR_asymptotics} 
Assume ${n}/{d^{\kappa_2}} \to \gratio_2$ for some $\kappa_2,\gratio_2 >0$ and let $\psi_2 := \lim_{d\to\infty} \frac{n}{d^\ell/\ell!}$, where $\ell = \lceil \kappa_2 \rceil$.
 Let $\{ f_{*,d} \in L^2 (\S^{d-1} (\sqrt{d}) )\}_{d \geq 1}$ be a sequence of functions that satisfies Assumption \ref{ass:random_f_star} at level $\ell$ and let $\sigma : \R \to \R$ be an activation function satisfying Assumption \ref{ass:sigma} at level $\ell$. 
     Then the asymptotic test/training errors and the normalized squared RKHS norm of KRR satisfy: 
 \begin{equation}
     \begin{aligned}
         \E_{\bX,\beps,f_*} \Big\vert R_{\KRR} ( f_{*,d} ; \bX , \beps, \lambda) - \Big[ F_\ell^2 \cdot \cB_{\test} + (F_{>\ell}^2 + \rho_\eps^2 ) \cdot \cV_{\test}  + F_{>\ell}^2 \Big]  \Big\vert =&~ o_d(1) \, ,\\
          \E_{\bX,\beps,f_*} \Big\vert R_{\KRR,\train} ( f_{*,d} ; \bX , \beps, \lambda) - \alpha_c^2 \Big[ F_\ell^2 \cdot \cB_{\test}  + (F_{>\ell}^2 + \rho_\eps^2 ) \cdot \big( \cV_{\test} + 1 \big) \Big]  \Big\vert =&~ o_d(1) \, ,
     \end{aligned}
 \end{equation}
 and
 \begin{equation}
     \begin{aligned}
     \E_{\bX,\beps,f_*} \Big\vert  L_{\KRR, \lnorm} (f_* ; \bX, \beps, \lambda) - \Big[ F_\ell^2 \cdot \cB_{\lnorm} + (F_{>\ell}^2 + \rho_\varepsilon^2) \cdot \cV_{\lnorm} \Big] \Big\vert = o_d(1),
     \end{aligned}
 \end{equation}
 where $(\cB_{\test}, \cV_{\test}, \alpha_c)$ and $(\cB_{\lnorm}, \cV_{\lnorm})$ are the same as in Eqs. \eqref{eq:B_V_test_2} and \eqref{eq:B_V_norm_2}.
\end{theorem}
We can find that the expressions in Theorem \ref{thm:KRR_asymptotics} coincides with those of Theorem \ref{thm:main_theorem_RF} in the overparametrized case $\kappa_1 > \kappa_2$.

\subsection{The approximation limit $n = \infty$}

In this section, we consider the approximation error of the random feature function class. We will show that it corresponds to the infinite data limit $n = \infty$ of random feature ridge regression.

Recall that we define the random feature function class as
\[
\cF_{\RF} ( \bW) = \Big\{ h_{\RF} (\bx ; \ba ) = \frac{1}{\sqrt{p}} \sum_{j\in [p]} a_j \sigma ( \< \bx , \bw_j \> ) : \,\,\, \ba = (a_1 , \ldots, a_p) \in \R^p \Big\} \, ,
\]
where $\bW = [ \bw_1 , \ldots , \bw_p ]^\sT \in \R^{p \times d}$ with $\bw_j \sim_{iid} \Unif (\S^{d-1})$. The approximation error is given by
\[
R_{\App} ( f_* ; \bW ) = \inf_{h \in \cF_{\RF} ( \bW) } \E_{\bx} \big[ (f_*(\bx) - h (\bx) )^2 \big] = \inf_{\ba \in \R^p} \E_{\bx} \Big[ \Big(f_*(\bx) - \frac{1}{ \sqrt{p} }\sum_{j \in [p]} a_j \sigma (\< \bx , \bw_j \>) \Big)^2 \Big] \, .
\]
Denote the solution of the second optimization problem in the above display as $\hat{\ba}_{\App}$ and its (normalized) squared $\ell_2$ norm as:
\begin{align}
    L_{\App,\lnorm} (f_*; \bW) =&~ \frac{1}{p} \| \hat{\ba}_{\App} \|_2^2.
\end{align}
The following theorem depicts the limiting behavior of $R_{\App} ( f_* ; \bW )$ and $L_{\App,\lnorm} (f_*; \bW)$ in the polynomial scaling regime: $p/d^{\kappa_1} \to \gratio_1$.
\begin{theorem}[Approximation error of RFRR in the polynomial scaling ] \label{thm:App_asymptotics} 
Assume ${p}/{d^{\kappa_1}} \to \gratio_1$ for some $\kappa_1,\gratio_1 >0$ and let $\psi_1 := \lim_{d\to\infty} \frac{p}{d^\ell/\ell!}$, where $\ell = \lceil \kappa_1 \rceil$.
 Let $\{ f_{*,d} \in L^2 (\S^{d-1} (\sqrt{d}) )\}_{d \geq 1}$ be a sequence of functions that satisfies Assumption \ref{ass:random_f_star} at level $\ell$ and let $\sigma : \R \to \R$ be an activation function satisfying Assumption \ref{ass:sigma} at level $\ell$. 
     Then the asymptotic approximation error and the normalized squared $\ell_2$ norm of $\hat{\ba}_{\App}$ satisfy: 
 \begin{equation}
     \begin{aligned}
         \E_{\bW,f_*} \Big\vert R_{\App} ( f_{*,d} ; \bW ) - \Big[ F_\ell^2 \cdot \cB_{\test} + (F_{>\ell}^2 + \rho_\eps^2 ) \cdot \cV_{\test}  + F_{>\ell}^2 \Big]  \Big\vert =&~ o_d(1) \, 
     \end{aligned}
 \end{equation}
 and
 \begin{equation}
     \begin{aligned}
     \E_{\bX,\beps,f_*} \Big\vert  L_{\KRR, \lnorm} (f_* ; \bX, \beps, \lambda) - \Big[ F_\ell^2 \cdot \cB_{\lnorm} + (F_{>\ell}^2 + \rho_\varepsilon^2) \cdot \cV_{\lnorm} \Big] \Big\vert = o_d(1),
     \end{aligned}
 \end{equation}
 where $(\cB_{\test}, \cV_{\test}, \alpha_c)$ and $(\cB_{\lnorm}, \cV_{\lnorm})$ are the same as in Eqs. \eqref{eq:B_V_test_3} and \eqref{eq:B_V_norm_3}.
\end{theorem}
\begin{proof}
We show how to prove the limiting formula for $R_{\App} ( f_{*,d} ; \bW )$. The proof for $L_{\KRR, \lnorm} (f_* ; \bX, \beps, \lambda)$ follows analogously and the details are omitted.

Taking the expectation over $\bx$, we get
\begin{align}
    R_{\App} ( f_* ; \bW ) &= \inf_{\ba \in \R^p} \E_{\bx}[f_{*}(\bx)^2] - \frac{2}{\sqrt{p}} \ba^\sT \sum_{k\geq 0} \xi_k \bPhi_k \bbeta_k + \frac{1}{p} \ba^\sT \sum_{k\geq 0} \mu_k^2 \bQ_k^{\bW} \ba \\
    &= \sum_{k\geq 0} F_k^2 - \big( \sum_{k\geq 0} \xi_k \bPhi_k \bbeta_k \big)^\sT \big( \sum_{k\geq 0} \mu_k^2 \bQ_k^{\bW} \big)^{-1} \big( \sum_{k\geq 0} \xi_k \bPhi_k \bbeta_k \big) \, ,
\end{align}
where $F_k^2 = \|\bbeta_k\|^2$ for $k<\ell$
and the optimal ${\ba}$ of the above optimization problem is:
\begin{align}
    \hat{\ba}_{\App} = \sqrt{p} \big( \sum_{k\geq 0} \mu_k^2 \bQ_k^{\bW} \big)^{-1} \big( \sum_{k\geq 0} \xi_k \bPhi_k \bbeta_k \big) \, .
\end{align}
Similar as Proposition \ref{prop:concentration_on_expectation}, one can show
\begin{equation}
    \E_{\bW,f_*} \Big[ \Big\vert R_{\App} ( f_* ; \bW ) - \E_{f_*} [ R_{\App} ( f_* ; \bW )] \Big\vert \Big] = o_d (1) \, . 
\end{equation}
Direct computation gives:
\begin{align}
    \E_{f_*} [ R_{\App} ( f_* ; \bW )] &= \sum_{k\geq 0} F_k^2 - \sum_{k < \ell} \xi_k^2 \big( \bPhi_k \bbeta_k \big)^\sT \big( \sum_{k\geq 0} N_k \xi_k^2 \bQ_k^{\bW} \big)^{-1} \big(\bPhi_k \bbeta_k \big) \\
    &\hspace{4.5em} - \sum_{k \geq \ell} F_k^2 \xi_k^2 \Tr\big[ \bQ_k^{\bW} \big( \sum_{k\geq 0} N_k \xi_k^2 \bQ_k^{\bW} \big)^{-1} \big] \, .
\end{align}
Similar as Proposition \ref{prop:high_degree_concentration} (II), one can verify: for $k<\ell$
\begin{align}
    \E \Big| \xi_k^2 \big( \bPhi_k \bbeta_k \big)^\sT \big( \sum_{k\geq 0} N_k \xi_k^2 \bQ_k^{\bW} \big)^{-1} \big(\bPhi_k \bbeta_k \big) - F_k^2 \Big| = o_d(1) \, ,
\end{align}
and
\begin{align}
    \E \Big| \sum_{k \geq \ell} F_k^2 \xi_k^2 \Tr\big[ \bQ_k^{\bW} \big( \sum_{k\geq 0} N_k \xi_k^2 \bQ_k^{\bW} \big)^{-1} \big] - \frac{F_\ell^2 }{N_\ell} \Tr ( \mu_\ell^2 \bQ_\ell^{\bW} \cdot \bG_{\geq \ell}^{\bW} )  \Big| = o_d(1)\, .
\end{align}
As a result, we can obtain
\begin{align}
\label{eq:approxerr_concen}
    \E \left\vert R_{\App} ( f_* ; \bW ) - \left[ F_{\ell}^2\left( 1 - \frac{1}{ N_\ell } \Tr( \mu_{\ell}^2  \bQ_{\ell}^\bW \cdot \bG_{\geq\ell}^{\bW} ) \right) + F_{>\ell}^2 \right] \right\vert = o_d(1) \, .
\end{align}
which is the same as Eq.~\eqref{eq:under_Rtrain_concentrate}. Then we can follow the proof therein and get:
\begin{align}
    \E \left\vert R_{\App} ( f_* ; \bW ) - ( F_\ell^2 \cdot \cB_{\test} + F_{>\ell}^2) \right\vert = o_d(1)\, ,
\end{align}
where
\begin{align}
\cB_{\test} =  
\begin{cases}
\frac{1}{2} \Big[ 1 - \psi_1 - \zeta^{-2} + \sqrt{(1 + \psi_1 + \zeta^{-2})^2 - 4\psi_1} \Big] & \text{if }\;\kappa_1 = \ell \, , \\
1& \text{if }\;\kappa_1 < \ell\, .
\end{cases}
\end{align}
\end{proof}
Theorem \ref{thm:App_asymptotics} reveals that the asymptotic limits of RFRR approximation error and $\frac{1}{p}\|\hat{\ba}_{\App}\|_2^2$ coincide with the limits of RFRR test error and $\frac{1}{p}\| \hba_{\lambda} \|_2^2$ given by Theorem \ref{thm:main_theorem_RF} in the underparametrized case $\kappa_1 < \kappa_2$.
In particular, the approximation error in the polynomial scaling regime $p\asymp d^{\kappa} $ was first considered in \cite{ghorbani2021linearized} for $\kappa \not\in \naturals$. Our theorem completes this picture by computing the approximation error for $p \asymp d^\ell, \ell \in \naturals$.

\clearpage

\section{Properties of Fixed-point Equations}
\label{sec:explicit_formula}

\subsection{Two equivalent characterizations}

In this section, we compare two equivalent characterizations of the asymptotic training and test errors obtained in \cite{mei2022generalizationRF,adlam2020neural} in the linear scaling. 

\subsubsection{Asymptotic predictions from \cite{mei2022generalizationRF}}

For convenience, we copy here the results for the asymptotic test and training errors obtained in \cite{mei2022generalizationRF}. Note that the formulas differ slightly from \cite{mei2022generalizationRF}, as we chose a different normalization for the regularization parameter $\lambda$. 

 \begin{definition}[Fixed points formula]\label{def:fixedpoint_nu} We define $\nu_1,\nu_2 : \C_+ \to \C_+$ to be the unique functions that satisfy the following conditions:
 \begin{itemize}
     \item[(i)] $\nu_1,\nu_2$ are analytic functions on $\C_+$;
     \item[(ii)] For $\Im (z) > 0$, $\nu_1(z) , \nu_2 (z)$ are fixed points of 
     \begin{equation}\label{eq:fixed_point_nu}
         \begin{aligned}
              \nu_1 =&~ \frac{\psi_1}{\psi} \Big( -z - \nu_2 - \frac{\zeta^2 \nu_2}{1 - \zeta^2 \psi \nu_1 \nu_2} \Big)^{-1} \, ,\\
              \nu_2 =&~ \frac{\psi_2}{\psi} \Big( -z - \nu_1 - \frac{\zeta^2 \nu_1}{1 - \zeta^2 \psi \nu_1 \nu_2} \Big)^{-1} \, ;
         \end{aligned}
     \end{equation}
     \item[(iii)] $(\nu_1 (z) , \nu_2 (z) )$ is the unique fixed point of Eq.~\eqref{eq:fixed_point_nu} with $| \nu_1 ( z) | \leq \psi_1 \psi^{-1} / \Im (z)$, $| \nu_2 ( z) | \leq \psi_2 \psi^{-1} / \Im (z)$ for $\Im (z)> C$ and $C$ sufficiently large.
 \end{itemize}
 \end{definition}
Here $\zeta = \mu_1 / \mu_{>1}$, $\psi_1 = p/d$, $\psi_2 = n/d$ and $\psi = \psi_1 + \psi_2$.
Let us comment on the interpretation of these fixed points. Recall the notation $\bZ = \sigma ( \bX \bW^\sT) / \sqrt{p} $, $\mu_k = \E_G [ \sigma(G) \He_k (G)] $ and $\mu_{>1}^2 = \E_{G} [ \sigma(G)^2 ] - \mu_0^2 - \mu_1^2$, where $G\sim\cN(0,1)$. Then 
\begin{equation}\label{eq:interpretation_nus}
\begin{aligned}
    \nu_1 \left(i \sqrt{ \frac{\psi_1 \barlambda}{\psi } }  \right) = &~ i \sqrt{\frac{\psi \barlambda}{\psi_1}} \cdot \lim_{d \to \infty, p/d \to \psi_1, n/d \to \psi_2 } \frac{1}{m} \E \big[ \Tr \big( ( \bZ^\sT \bZ/ \mu_{>1}^2 + \barlambda \id_p )^{-1} \big) \big] \, , \\
     \nu_2 \left(i \sqrt{ \frac{\psi_1 \barlambda}{\psi} }  \right) = &~ i \sqrt{\frac{\psi \barlambda}{\psi_1}} \cdot \lim_{d \to \infty, p/d \to \psi_1, n/d \to \psi_2 } \frac{1}{m} \E \big[ \Tr \big( ( \bZ \bZ^\sT/ \mu_{>1}^2 + \barlambda \id_n )^{-1} \big) \big] \, ,
     \end{aligned}
\end{equation}
where $\barlambda = \lambda/ \mu_{>1}^2$, as we are considering the $\ell = 1$ case. 
In particular, $ \nu_2 (i u ) - \nu_1 (i u )= i \cdot \frac{ \mu_{>1} (\psi_2 - \psi_1) }{u\psi}$ and are pure imaginary when $u>0$.

Denote 
  \begin{equation}
     \begin{aligned}
          \chi :=&~ \nu_1 \Big(i \sqrt{ \frac{\psi_1 \barlambda}{\psi} }  \Big) \cdot \nu_2 \Big(i \sqrt{ \frac{\psi_1 \barlambda}{\psi} }  \Big)\, ,
     \end{aligned}
\end{equation}
  and define the following quantities for the test error:

\begin{equation}
     \begin{aligned}
          \cE_0(\zeta , \psi_1 , \psi_2 , \barlambda) := &~ -\psi^5 \chi^5 \zeta^6 + 3 \psi^4 \chi^4 \zeta^4 + ( \psi_1 \psi_2 - \psi_2 - \psi_1 +1) \psi^3 \chi^3 \zeta^6 - 2 \psi^3 \chi^3 \zeta^4 - 3 \psi^3 \chi^3 \zeta^2 \\
          &~ + (\psi_1 + \psi_2 - 3\psi_1 \psi_2 +1)\psi^2 \chi^2\zeta^4 + 2\psi^2\chi^2\zeta^2 + \psi^2 \chi^2 + 3 \psi_1 \psi_2 \psi \chi \zeta^2  - \psi_1 \psi_2 \, , \\
          \cE_1(\zeta , \psi_1 , \psi_2 , \barlambda)  := &~ \psi_2 \psi^3 \chi^3 \zeta^4 - \psi_2 \psi^2 \chi^2 \zeta^2 + \psi_1 \psi_2 \psi \chi \zeta^2 - \psi_1 \psi_2 \, ,\\
          \cE_2(\zeta , \psi_1 , \psi_2 , \barlambda)  := &~ \psi^5 \chi^5 \zeta^6 - 3 \psi^4 \chi^4 \zeta^4   \\
          &~ + (\psi_1 - 1 ) \psi^3 \chi^3 \zeta^6 + 2 \psi^3 \chi^3 \zeta^4 + 3 \psi^3 \chi^3 \zeta^2 + ( - \psi_1 - 1) \psi^2 \chi^2 \zeta^4 - 2 \psi^2 \chi^2 \zeta^2  - \psi^2 \chi^2 \, .
    \end{aligned}
\end{equation}
 We further define the following quantities for the training error:
\begin{equation}
    \begin{aligned}
          \cL_1 (\zeta , \psi_1 ,\psi_2 , \barlambda) :=&~ -i \sqrt{ \frac{ \psi_1 \psi \bar\lambda}{ \psi_2^2 } } \cdot \frac{ \psi \cdot \nu_2\big( i (\psi_1 \psi^{-1} \bar\lambda)^{1/2} \big) }{ 1 - \psi \chi \zeta^2} \, , \\
           \cL_2 (\zeta , \psi_1 ,\psi_2 , \barlambda) :=&~ -i \sqrt{ \frac{ \psi_1 \psi \bar\lambda}{ \psi_2^2 } } \cdot \psi \cdot \nu_2\big( i (\psi_1 \psi^{-1} \bar\lambda)^{1/2} \big) \, ,\\
           \cA_1 (\zeta , \psi_1 ,\psi_2 , \barlambda) :=&~ - \psi^2\chi^2 ( \psi \chi \zeta^4 - \psi \chi \zeta^2 + \psi_2 \zeta^2 + \zeta^2 - \psi_2 \psi \chi  \zeta^4 +1 ) \, ,\\
           \cA_2 (\zeta , \psi_1 ,\psi_2 , \barlambda) := &~ \psi^2 \chi^2 ( \psi \chi \zeta^2 - 1) ( \psi^2 \chi^2 \zeta^4 - 2 \psi \chi \zeta^2 + \zeta^2 + 1 ) \, ,\\
           \cA_0 (\zeta , \psi_1 ,\psi_2 , \barlambda) := &~ - \psi^5 \chi^5 \zeta^6 + 3 \psi^4 \chi^4 \zeta^4 + (\psi_1 \psi_2 - \psi_2 - \psi_1 +1) \psi^3 \chi^3 \zeta^6 - 2 \psi^3 \chi^3 \zeta^4 - 3 \psi^3 \chi^3 \zeta^2 \\
           &~ + (\psi_1 + \psi_2 - 3 \psi_1 \psi_2 +1) \psi^2 \chi^2 \zeta^4 + 2 \psi^2 \chi^2 \zeta^2 + \psi^2 \chi^2 + 3 \psi_1 \psi_2 \psi \chi \zeta^2 - \psi_1 \psi_2 \, .
     \end{aligned}
 \end{equation} 
 Based on the above definitions, we can define the bias and variance terms for the test error:
 \begin{align}
     \cB_{\test} (\zeta , \psi_1 , \psi_2, \barlambda ) := &~ \frac{\cE_1 (\zeta , \psi_1 , \psi_2, \barlambda )}{\cE_0 (\zeta , \psi_1 , \psi_2, \barlambda) }\, , \\
     \cV_{\test} (\zeta , \psi_1 , \psi_2, \barlambda ) := &~ \frac{\cE_2 (\zeta , \psi_1 , \psi_2, \barlambda )}{\cE_0 (\zeta , \psi_1 , \psi_2, \barlambda) }\, ,
 \end{align}
 and for the training error:
 \begin{align}
           \cB_{\train} (\zeta , \psi_1 , \psi_2, \barlambda ) := &~ \cL_1 (\zeta , \psi_1 ,\psi_2 , \barlambda) - \barlambda \frac{\cA_1 (\zeta , \psi_1 , \psi_2, \barlambda )}{\cA_0 (\zeta , \psi_1 , \psi_2, \barlambda) }\, , \\
     \cV_{\train} (\zeta , \psi_1 , \psi_2, \barlambda ) := &~ \cL_2 (\zeta , \psi_1 ,\psi_2 , \barlambda) - \barlambda \frac{\cA_2 (\zeta , \psi_1 , \psi_2, \barlambda )}{\cA_0 (\zeta , \psi_1 , \psi_2, \barlambda) }\, ,
 \end{align}
 The asymptotics of the training and test errors are given by:
 \begin{equation}
     \begin{aligned}
     \cR_{\test} ( F_* , \zeta , \psi_1 , \psi_2, \barlambda ) := &~ F_1^2 \cdot \cB_{\test} (\zeta , \psi_1 , \psi_2, \barlambda )  + (F_{>1}^2 + \sigma_\eps^2) \cdot \cV_{\test} (\zeta , \psi_1 , \psi_2, \barlambda ) +  F_{>1}^2 \, , \\
     \cR_{\train} ( F_* , \zeta , \psi_1 , \psi_2, \barlambda ) := &~ F_1^2 \cdot \cB_{\train} (\zeta , \psi_1 , \psi_2, \barlambda )  + (F_{>1}^2 + \sigma_\eps^2) \cdot \cV_{\train} (\zeta , \psi_1 , \psi_2, \barlambda ) \, .
     \end{aligned}
 \end{equation}
 
 \begin{theorem}[RFRR in linear scaling {\cite[Theorem 2]{mei2022generalizationRF}}]\label{thm:asympSong}
 Let $p,n,d \to \infty$ with $p/d \to \psi_1$ and $n/d \to \psi_2$. Assume $\sigma$ satisfies Assumption \ref{ass:sigma} at level $1$ and let $\{ f_{*,d} \in L^2 (\S^{d-1} (\sqrt{d}))\}_{d\geq 1}$ be a sequence a function satisfying Assumption \ref{ass:random_f_star} at level $1$. 
 
 Then for any value of the regularization parameter $\lambda >0$, the asymptotic training and test errors of random feature ridge regression satisfy
 \begin{equation}
     \begin{aligned}
     \E_{\bX,\bW,\beps,f_*} \Big\vert R_{\train} ( f_{*,d} ; \bX , \bW , \beps, \lambda) - \cR_{\train} ( F_* , \zeta , \psi_1 , \psi_2, \barlambda ) \Big\vert = o_d(1) \, , \\
     \E_{\bX,\bW,\beps,f_*}  \Big\vert R_{\test} ( f_{*,d} ; \bX , \bW , \beps, \lambda) - \cR_{\test} ( F_* , \zeta , \psi_1 , \psi_2, \barlambda ) \Big\vert = o_d(1) \, . 
     \end{aligned}
 \end{equation}
 \end{theorem}
The expression of test and training errors in Theorem \ref{thm:asympSong} coincides with Eq.~\eqref{eq:R_test_form1} and  \eqref{eq:E_Rtrain_concentrate}, when $\kappa_1 = \kappa_2 =1$. In particular, the fixed points in Definition \ref{def:fixedpoint_nu} correspond to those in Definition \ref{def:fixedpoint_nu_0}, as the special case $\ell=1$. 
 
 \subsubsection{Alternative characterization from \cite{adlam2020neural}}
 \label{sec:alternate}
 
Let us introduce a simplified `linearized' model. Consider a sequence $(d_0,n_0,p_0) \in \naturals^3$ with $d_0, n_0 , p_0 \to \infty$ with $p_0 / d_0 \to \psi_1$ and $n_0 / d_0 \to \psi_2$ for some $(\psi_1, \psi_2) \in \R_{>0}^2$. We consider a sequence of matrices $\bX_0 \in \R^{n_0 \times d_0}$, $\bW_0 \in \R^{p_0 \times d_0} $ and $\bTheta_0 \in \R^{n_0 \times p_0}$ with i.i.d.~standard Gaussian entries. We define 
\begin{equation}
\label{eq:linear_Gauss_model_2}
\bZ_0 := \frac{1}{\sqrt{p}}\Big(\frac{\zeta}{ \sqrt{d_0} } \bX_0 \bW_0^\sT + \bTheta_0\Big) \in \R^{n_0 \times p_0} \, .
\end{equation}
where $\zeta = \mu_1 / \mu_{>1}$.
The resolvent of $\bZ_0 \bZ_0^\sT$ is $\bG (z) = ( \bZ_0 \bZ_0^\sT + z \id_{n_0} )^{-1}$ and we define two traces:
\begin{equation}\label{eq:def_taus}
\tau_{1,d_0} (z) := \frac{1}{n_0} \E \big[ \Tr (\bG (z)) \big] \, , \qquad \tau_{2,d_0} (z) := \frac{1}{n_0 } \E \big[ \Tr ( (\bX_0 \bX_0^\sT /d_0) \bG (z)) \big] \, .
\end{equation}

 The following asymptotic expressions for $(\tau_{1,d_0},\tau_{2,d_0})$ were obtained in \cite{adlam2020neural} using a linear pencil method.
 
 \begin{proposition}[Fixed point equations for $(\tau_1, \tau_2)$ {\cite[Proposition 1]{adlam2020neural}}]\label{prop:PenningtonFixedPoints}
 As $d_0,n_0,p_0 \to \infty$ with $p_0 / d_0 \to \psi_1 \in (0,\infty)$ and $n_0 / d_0 \to \psi_2 \in (0,\infty)$. Then $(\tau_{1,d_0} (z),\tau_{2,d_0} (z)) \to (\tau_1 (z),\tau_2 (z))$ given by the unique solution of the coupled polynomial equations 
\begin{equation}\label{eq:PenningtonFixedPoints}
 \begin{aligned}
 \zeta^2 \tau_1 \tau_2 (z \tau_1 - 1) + \frac{\psi_1}{\psi_2} ( \zeta^2 \tau_1 \tau_2 + (\tau_2 - \tau_1)/ \psi_2 ) =&~ 0 \, , \\
  \zeta^2 \tau_1 \tau_2 ( z \tau_1 - 1)  + (\tau_1 - \tau_2 ) ( \tau_1 + \zeta^2 \tau_2) / \psi_2 = &~ 0 \, ,
 \end{aligned}
 \end{equation}
 such that $\tau_1(z),\tau_2(z) \in \C_+$ when $z \in \C_+$.
 \end{proposition}
 Note that the fixed-point equation \eqref{eq:PenningtonFixedPoints} coincides with Eq.~\eqref{eq:FixedPoints} in the main theorem. It is the case when $\ell=1$ and correspondingly, we have $\zeta = \mu_1 / \mu_{>1}$ and $\barlambda = \lambda/\mu_{>1}^2$ here.
 In \cite{adlam2020neural,adlam2022random}, it was shown that $(\tau_{1,d_0}, \tau_{2,d_0})$ are asymptotically the same, if  $\bZ_0$ is replaced by $\Tilde \bZ_0 = \sigma (\bX_0 \bW_0^\sT) / \sqrt{p_0}$.
In other words, the linear Gaussian model $\bZ_0$ is equivalent to the non-linear model $\tilde{\bZ}_0$.
This corresponds to Gaussian equivalence principle (see Section \ref{sec:GaussianEquivalence}) in the special case of $\ell=1$.
 Based on these asymptotic equivalence, \cite{adlam2020neural} proved the following characterization for the asymptotic training and test errors:
 
 \begin{theorem}[RFRR in linear scaling {\cite[Theorem 1]{adlam2020neural}}]\label{thm:asympPennington}
  Let $p,n,d \to \infty$ with $p/d \to \psi_1$ and $n/d \to \psi_2$. Assume $\sigma$ satisfies Assumption \ref{ass:sigma} at level $1$ and let $\{ \bbeta_d \in \R^d \}_{d\geq 1}$ be a sequence of vectors such that $f_{*,d} (\bx) = \< \bbeta_d , \bx \>$ and $\| \bbeta_d \|_2 \to F_1$. 
  Consider $(\tau_1,\tau_2)$ as defined in Proposition \ref{prop:PenningtonFixedPoints} evaluated at ${\rm Re} (z) = \barlambda$ with $\Im (z) \to 0^+$. 
 
 Then for any value of the regularization parameter $\lambda >0$, the asymptotic training and test errors of random feature ridge regression satisfy
 \begin{equation}
     \begin{aligned}
     \E_{\bX,\bW,\beps} \big[ R_{\train} ( f_{*,d} ; \bX , \bW , \beps, \lambda) \big] =&~  F_1^2 \cdot \big( - \barlambda^2 \tau_2 '(\barlambda) \big) + \rho_\eps^2 \cdot \big( - \barlambda^2 \tau_1 '(\barlambda) \big)\, , \\
     \E_{\bX,\bW,\beps}  \big[ R_{\test} ( f_{*,d} ; \bX , \bW , \beps, \lambda)\big] = &~ F_1^2 \cdot \Big( - \frac{\tau_2'(\barlambda)}{\tau_1^2(\barlambda)}  \Big) + \rho_\eps^2 \cdot \Big( - \frac{\tau_1'(\barlambda)}{\tau_1^2(\barlambda)} -1 \Big) \, . 
     \end{aligned}
 \end{equation}
 \end{theorem}
We see that Theorem \ref{thm:asympPennington} indeed corresponds to Theorem \ref{thm:main_theorem_RF} in the linear scaling regime $\kappa_1 = \kappa_2 = 1$. Our results encompass the more general polynomial scaling regime. In this regime, we have:
\begin{equation}
\tau_{1} (z) = \lim_{d\to\infty} \frac{\mu_{>\ell}^2}{n} \E \big[ \Tr (\bPi (\mu_{>\ell}^2 z)) \big] \, , \qquad \tau_{2} (z) = \lim_{d\to\infty} \frac{\mu_{>\ell}^2}{n} \E \big[ \Tr ( \bQ_{\ell}^{\bX} \bPi (\mu_{>\ell}^2 z)) \big] \, .
\end{equation}
where $\bPi(\lambda) = (  \bZ \bZ^\sT + \lambda \id_n )^{-1}$. We can also obtain the approximation for the GCV coefficient $\alpha_c$: 
\begin{align}
  \alpha_c = \barlambda^2 \tau_1^2 (\barlambda) \approx \left[\frac{\lambda}{n} \Tr (\bZ \bZ^\sT + \lambda \id_n)^{-1} \right]^2\, .  
\end{align}

  
 \subsubsection{Correspondence between two fixed-point equations}
The following lemma computes the derivatives of the fixed points $(\tau_1,\tau_2)$ in Definition \ref{def:fixedPointsTau} and relates $(\tau_1,\tau_2)$ to the fixed points $(\nu_1,\nu_2)$ in Definition \ref{def:fixedpoint_nu_0} obtained in the proof of the main theorem.

\begin{lemma}\label{lem:equi_fixedpoints}
When $\kappa_1 = \kappa_2$, for any $u \in \reals_{>0}$
we have the following correspondence between $(\nu_1,\nu_2)$ and $(\tau_1,\tau_2)$:
\begin{equation}
\label{eq:tau_u_nu_u_corrs}
\tau_1 (u) = -i \sqrt{\frac{\gratio_1 \gratio}{\gratio_2^2 u}} \nu_2 \Big( i \sqrt{\gratio_1 u / \gratio} \Big) \, , \qquad \tau_2 (u) = -i \sqrt{\frac{\gratio_1 \gratio}{\gratio_2^2 u}} \cdot \frac{ \nu_2 (i \sqrt{\gratio_1 u / \gratio})}{1 - \zeta^2 \psi \nu_1 (i \sqrt{\gratio_1 u / \gratio}) \nu_2 (i \sqrt{\gratio_1 u / \gratio})} \, .
\end{equation}
\end{lemma}

\begin{proof}[Proof of Lemma \ref{lem:equi_fixedpoints}]
We first consider the case when $\kappa_1 = \kappa_2 = \ell$. Recall that in our proof of Proposition \ref{prop:resolvent_conv} (see Appendix \ref{subsec:proof_stieltjes}), we show that
\begin{align}
    m_1(z;\bzero) &= \lim_{d\to\infty} m_{1,d}(z,\bzero)\\
    &= \lim_{d\to\infty} \frac{1}{m} \E \big[ \Tr_{[1:p]} \big( ( \bA(\bzero) - z \id_m )^{-1} \big) \big] \\
    &= \lim_{d\to\infty} \frac{z}{m} \E \Tr (-z^2\id_p + \frac{\theta_1}{\theta} \bZ^\sT \bZ)^{-1}
\end{align}
and
\begin{align}
    m_2(z;\bzero) &= \lim_{d\to\infty} \frac{z}{m} \E \Tr (-z^2\id_n + \frac{\theta_1}{\theta} \bZ \bZ^\sT )^{-1}.
\end{align}
By a similar proof, we can also get
\begin{align}
    m_1(z;\bzero)
    &= \lim_{\substack{d_0\to\infty\\ \frac{p_0}{d_0} \to \theta_1, \frac{n_0}{d_0} \to \theta_2}} \frac{z}{m_0} \E \Tr (-z^2\id_{p_0} + \frac{\theta_1 \mu_{>\ell}^2}{\theta}  \bZ_0^\sT \bZ_0)^{-1}
\end{align}
and similarly,
\begin{align}
    m_2(z;\bzero) &= \lim_{\substack{d_0\to\infty\\ \frac{p_0}{d_0} \to \theta_1, \frac{n_0}{d_0} \to \theta_2}} \frac{z}{m_0} \E \Tr (-z^2\id_{n_0} + \frac{\theta_1 \mu_{>\ell}^2}{\theta} \bZ_0 \bZ_0^\sT)^{-1}
\end{align}
where $\bZ_0 := (\frac{\zeta}{ \sqrt{d_0} } \bX_0 \bW_0^\sT + \bTheta_0) / \sqrt{p}$ and $\bX_0 \in \R^{n_0 \times d_0}$, $\bW_0 \in \R^{p_0 \times d_0} $ and $\bTheta_0 \in \R^{n_0 \times p_0}$ are independent and all have i.i.d.~standard Gaussian entries and $m_0 := n_0+p_0$.
Then using the relationship in Eq. \eqref{eq:v_m_relation}: $\nu_i(z) = m_i(\mu_{>\ell} z; \bzero)\cdot\mu_{>\ell}$, $i=1,2$, we can get
\begin{align}
\nu_1\Big(i\sqrt{\frac{\theta_1 u}{\theta}}\Big) 
&= i \sqrt{\frac{\gratio u}{\gratio_1}} \cdot \lim_{\substack{d_0\to\infty\\ \frac{p_0}{d_0} \to \theta_1, \frac{n_0}{d_0} \to \theta_2}} \frac{1}{m_0} \E \big[ \Tr \big( ( \bZ_0^\sT \bZ_0 + u \id_{p_0} )^{-1} \big) \big]
\end{align}
and
\begin{align}    
  \nu_2\Big(i\sqrt{\frac{\theta_1 u}{\theta}}\Big) &= i \sqrt{\frac{\gratio u}{\gratio_1}} \cdot \lim_{\substack{d_0\to\infty\\ \frac{p_0}{d_0} \to \theta_1, \frac{n_0}{d_0} \to \theta_2}} \frac{1}{m_0} \E \big[ \Tr \big( ( \bZ_0 \bZ_0^\sT + u \id_{n_0} )^{-1} \big) \big]
\end{align}
for any $u>0$.
On the other hand, we have
\begin{align}
\tau_{1} (z) =&  \lim_{\substack{d_0\to\infty\\ \frac{p_0}{d_0} \to \theta_1, \frac{n_0}{d_0} \to \theta_2}} \frac{1}{n_0 } \E \big[ \Tr (\bG (z)) \big] \, ,\\
\qquad \tau_{2} (z) =&  \lim_{\substack{d_0\to\infty\\ \frac{p_0}{d_0} \to \theta_1, \frac{n_0}{d_0} \to \theta_2}} \frac{1}{n_0 } \E \big[ \Tr ( (\bX_0 \bX_0^\sT /d_0) \bG (z)) \big] \, 
\end{align}
where $\bG (z) = ( \bZ_0 \bZ_0^\sT + z \id_{n_0} )^{-1}$.
Then we can deduce that
\[
\nu_2 ( i \sqrt{\gratio_1 u / \gratio} ) = i  \cdot \sqrt{\frac{\gratio_2^2 u}{\gratio_1 \gratio}} \tau_1 (u) \, 
\]
which is the first equation in \eqref{eq:tau_u_nu_u_corrs}.

Then we turn to the second equation in \eqref{eq:tau_u_nu_u_corrs}.  From Eqs. \eqref{eq:G_1stDeri_conv} and \eqref{eq:partial_G_7}
we have: when $\kappa_1 = \kappa_2 = \ell$,
\begin{align}
    \partial_{t_2} g (i \sqrt{\gratio_1 \lambda / \gratio} ; \bzero ) = i \sqrt{ \frac{\theta_2^2 \lambda}{\theta_1\theta} } \cdot \lim_{\substack{d\to\infty\\ \frac{p}{d^\ell} \to \theta_1, \frac{n}{d^\ell} \to \theta_2}} \frac{1}{n} \E \Tr ( \bR \bQ_{\ell}^{\bX}  ) 
\end{align}
and one can verify the following Gaussian equivalence by a similar proof:
\begin{align}
\partial_{t_2} g (i \sqrt{\gratio_1 \lambda / \gratio} ; \bzero )
= i \sqrt{ \frac{\theta_2^2 \lambda}{\theta_1\theta} }
\lim_{\substack{d_0\to\infty\\ \frac{p_0}{d_0} \to \theta_1, \frac{n_0}{d_0} \to \theta_2}} \frac{1}{n_0} \E \Tr \Big[ \big( \lambda \id_{n_0} + \mu_{>\ell}^2 \bZ_0 \bZ_0^\sT \big)^{-1} ( \bX_0 \bX_0^\sT / d_0) \Big]
\end{align}
Therefore,
\begin{align}
    \partial_{t_2} g (i \sqrt{\gratio_1 \lambda / \gratio} ; \bzero ) &= 
    \frac{i}{\mu_{>\ell}} \sqrt{ \frac{\theta_2^2 \lambda}{\theta_1\theta\mu_{>\ell}^2} }\lim_{\substack{d_0\to\infty\\ \frac{p_0}{d_0} \to \theta_1, \frac{n_0}{d_0} \to \theta_2}} \frac{1}{n_0} \E \Tr \Big[ \big( \tfrac{\lambda}{\mu_{>\ell}^2} \id_{n_0} +  \bZ_0 \bZ_0^\sT \big)^{-1} ( \bX_0 \bX_0^\sT / d_0) \Big] \\
    &= \frac{i}{\mu_{>\ell}} \sqrt{ \frac{\theta_2^2 \lambda}{\theta_1\theta\mu_{>\ell}^2} } \tau_2(\lambda/\mu_{>\ell}^2)
\end{align}
On the other hand, from Eqs.   \eqref{eq:v_m_relation} and \eqref{eq:partial_g_t2}, we have
\begin{align}
    \partial_{t_2} g (i \sqrt{\gratio_1 \lambda / \gratio} ; \bzero ) = \frac{ \nu_2 ( i \sqrt{\frac{\gratio_1 \lambda}{\gratio \mu_{>\ell}^2}} ) / \mu_{>\ell} }{1 - \zeta^2 \psi \nu_1 ( i \sqrt{  \frac{\gratio_1 \lambda}{\gratio \mu_{>\ell}^2}})\nu_2 ( i \frac{\gratio_1 \lambda}{\gratio \mu_{>\ell}^2} )}
\end{align}
Combining the above two equations together and set $u = \lambda / \mu_{>\ell}^2$, we can get
\[
\frac{ \nu_2 ( i \sqrt{\gratio_1 u / \gratio}) }{1 - \zeta^2 \psi \nu_1 ( i \sqrt{\gratio_1 u / \gratio})\nu_2 ( i \sqrt{\gratio_1 u / \gratio})} = i \cdot \sqrt{\frac{\gratio_2^2 u}{\gratio_1 \gratio}} \tau_2 (u) \, ,
\]
which is exactly the second equation in Eq.  \eqref{eq:tau_u_nu_u_corrs}.

When $\kappa_1 = \kappa_2 = \kappa \in (\ell-1,\ell)$, we have $\psi = 0$. In this case, $\nu_1(z)$ and $\nu_2(z)$ can be computed analytically:
     \begin{align}
         \nu_1(z) =& \frac{ - \big( \frac{\gratio_1 - \gratio_2}{\gratio} + \frac{z^2}{1+\zeta^2} \big) - \sqrt{ \big( \frac{\gratio_1 - \gratio_2}{\gratio} + \frac{z^2}{1+\zeta^2} \big)^2 - \frac{4 \gratio_1}{\gratio} \frac{z^2}{1+\zeta^2} } }{2z} \\
         \nu_2(z) =& \frac{ - \big( \frac{\gratio_2 - \gratio_1}{\gratio} + \frac{z^2}{1+\zeta^2} \big) - \sqrt{ \big( \frac{\gratio_2 - \gratio_1}{\gratio} + \frac{z^2}{1+\zeta^2} \big)^2 - \frac{4 \gratio_2}{\gratio} \frac{z^2}{1+\zeta^2} } }{2z} \label{eq:nu_2_le_l}
     \end{align}
     Then together with the explicit formulas of $\tau_1(u)$ and $\tau_2(u)$ in Eq. \eqref{eq:FixedPoints_1}, we can directly verify Eq. \eqref{eq:tau_u_nu_u_corrs} in this special case.
\end{proof}

\subsection{Some special regimes} 
\label{sec:specialregime}
In this section, we show that the limiting formulas in the overparametrized and underparametrized regime (c.f. Eq. \eqref{eq:B_V_test_2} and \eqref{eq:B_V_test_3}) as well as the $\kappa_1 = \kappa_2 < \ell$ regime (c.f. Eq. \eqref{eq:FixedPoints_1}) can all be derived as certain limits of the fixed point equation \eqref{eq:FixedPoints}.
\subsubsection{ Overparametrized regime }
The overparametrized regime corresponds to the case when 
$\psi_1 \to \infty$ and $\psi_2 \in (0, \infty)$ is fixed. In this case, from \eqref{eq:FixedPoints} we can get:
\begin{equation}
    \begin{aligned}
        \zeta^2 \tau_1 \tau_2 +  \frac{ \tau_2 - \tau_1 }{\psi_2} &= \varepsilon_1 \\
        (z+1)\tau_1 - 1 + \zeta^2 \tau_2 &= \varepsilon_2
    \end{aligned}
\end{equation}
where $\varepsilon_1,\varepsilon_2\to 0$ as $\psi_1 \to \infty$. By solving the above equation, we can get: as $\psi_1 \to \infty$,
\begin{equation}
\label{eq:tau_over_limit}
\begin{aligned}
    \tau_1(z) &\to \frac{ - [( \frac{z+1}{\zeta^2} + 1 ) \frac{1}{\psi_2} - 1] + \sqrt{ [( \frac{z+1}{\zeta^2} + 1 ) \frac{1}{\psi_2} - 1]^2 + \frac{4(z+1)}{\psi_2 \zeta^2} } }{2(z+1)} \\
    \tau_2(z) &\to \frac{1 - (z+1)\tau_1(z)}{\zeta^2}
\end{aligned}    
\end{equation}
Then one can check: as $\psi_1 \to \infty$,
\begin{align}
\label{eq:tau1zeta_over_limit}
    \tau_1(\barlambda) \zeta^2 \to \vartheta
\end{align}
where $\vartheta$ is defined in Eq. \eqref{eq:defchi}.
Combining Eqs. 
\eqref{eq:tau_over_limit}, \eqref{eq:tau1zeta_over_limit}, \eqref{eq:B_V_test_1} and \eqref{eq:B_V_norm_1}, we can verify $(\cB_{\test},\cV_{\test},\alpha_c)$ and $(\cB_{\lnorm}, \cV_{\lnorm})$ in Eqs. \eqref{eq:B_V_test_2} and \eqref{eq:B_V_norm_2} are the $\psi_1 \to \infty$ limit of the corresponding quantities in Eqs. \eqref{eq:B_V_test_1} and \eqref{eq:B_V_norm_1}.
 
\subsubsection{ Underparametrized regime } 
The underparametrized limit corresponds to the case when 
$\psi_2 \to \infty$ and $\psi_1 \in (0, \infty)$ is fixed. In this case, from Eq. \eqref{eq:FixedPoints} we can get:
\begin{align}
    z \tau_1 - 1 &= \varepsilon_1 \\
    \psi_1 \zeta^2 \tau_1 \tau_2 + (\tau_2 - \tau_1 ) ( \tau_1 + \zeta^2 \tau_2 ) &= \varepsilon_2
\end{align}
where $\varepsilon_1,\varepsilon_2\to 0$ as $\psi_2 \to \infty$. Solving the above equation, we can get: as $\psi_2 \to \infty$,
\begin{equation}
    \begin{aligned}
        \tau_1(z) &\to \frac{1}{z} \\
        \tau_2(z) &\to \frac{ - ( 1 - \zeta^2 + \psi_1 \zeta^2 ) + \sqrt{ ( 1 - \zeta^2 + \psi_1 \zeta^2 )^2 + 4\zeta^2 } }{2 \zeta^2 z}
    \end{aligned}
\end{equation}
Combing the above two limits with Eqs. \eqref{eq:B_V_test_1} and \eqref{eq:B_V_norm_1}, we can verify that Eqs. \eqref{eq:B_V_test_3} and \eqref{eq:B_V_norm_3} are the $\psi_2 \to \infty$ limit of Eqs. \eqref{eq:B_V_test_1} and \eqref{eq:B_V_norm_1}.

\subsubsection{ The $\kappa_1 = \kappa_2 < \ell$ regime } 
The $\kappa_1 = \kappa_2 < \ell$ regime corresponds to the limiting case when $\psi_1 \to 0$ and $\psi_2\to 0$, while $\frac{\psi_1}{\psi_2}\in[c,1/c]$ for some $c>0$.
In this case, from the first equation of \eqref{eq:FixedPoints}, we can get
\begin{align}
    \tau_1 - \tau_2 = \varepsilon_1
\end{align}
where $\varepsilon_1\to 0$. On the other hand, from fixed point equation \eqref{eq:FixedPoints} we can obtain that
\begin{align}
    (z \tau_1 - 1) + ( \tau_1 + \zeta^2 \tau_2 ) \Big[ (z \tau_1 - 1) \frac{\psi_2}{\psi_1} + 1 \Big] = 0.
\end{align}
Combining the above two equations, we can get:
\begin{align}
    \tau_1 (z), \tau_2 (z) \to \frac{1}{2 z} \left\{ \left( 1 - \gamma - \frac{\gamma z}{1+\zeta^2} \right) + \sqrt{ \left( 1 - \gamma - \frac{\gamma z}{1+\zeta^2}  \right)^2 + \frac{4\gamma z}{1+\zeta^2} } \right\}.
\end{align}
which are the expressions of $\tau_1(z)$ and $\tau_2(z)$ in Eq. \eqref{eq:FixedPoints_1}.

\end{document}